%% file: main.tex
\documentclass{article}

\usepackage{arxiv}

\input{MyMacros}
\input{MyAbbreviavtions}

\title{Spectral flow on the manifold of SPD matrices for multimodal data processing}

\author{
  Ori Katz \\
  Viterbi Faculty of Electrical Engineering\\
  Technion -- Israel Institute of Technology, Israel\\
  \texttt{orikats@campus.technion.ac.il} \\
  \And
   Roy R. Lederman  \\
  Department of Statistics and Data Science\\
   Yale University\\
  \texttt{roy.lederman@yale.edu} \\
   \And
  Ronen Talmon  \\
  Viterbi Faculty of Electrical Engineering\\
  Technion -- Israel Institute of Technology, Israel\\
  \texttt{ronen@ee.technion.ac.il } \\
}

\begin{document}

\etocdepthtag.toc{mtchapter}
\etocsettagdepth{mtchapter}{subsection}
\etocsettagdepth{mtappendix}{none}

\maketitle

\input{body}

\section*{Acknowledgements}
This work was funded by the European Unions Horizon 2020 research grant agreement 802735.

\clearpage
\bibliography{references}
\bibliographystyle{ieeetr}

\newpage
\appendix
\setcounter{page}{1}
\etocdepthtag.toc{mtappendix}
\etocsettagdepth{mtchapter}{none}
\etocsettagdepth{mtappendix}{subsection}
\begin{center}
    \Huge
    \textbf{Appendix}
\end{center}
\tableofcontents
\input{supplement}

\end{document}

%% file: MyMacros.tex
\usepackage[utf8]{inputenc} 
\usepackage[T1]{fontenc}    
\usepackage{hyperref}       
\usepackage{url}            
\usepackage{booktabs}       
\usepackage{amsfonts}       
\usepackage{nicefrac}       
\usepackage{microtype}      
\usepackage{lipsum}
\usepackage{microtype}
\usepackage{graphicx}
\usepackage{booktabs}
\usepackage{hyperref}

\usepackage{amsmath}
\usepackage{amssymb}
\usepackage{mathtools}
\usepackage{amsthm}
\usepackage[capitalize,noabbrev]{cleveref}
\usepackage[textsize=tiny]{todonotes}
\usepackage[numbers]{natbib}
\usepackage{float}
\usepackage{xr}
\usepackage{lipsum}
\usepackage{amsfonts}
\usepackage{graphicx}
\usepackage{epstopdf}
\usepackage{enumitem}
\usepackage{amssymb}
\usepackage{amsmath}
\usepackage{thmtools,thm-restate}
\usepackage{algorithm}
\usepackage[noend]{algpseudocode}
\usepackage{lineno,hyperref}
\usepackage{cleveref}
\usepackage{listings}
\usepackage{color}
\usepackage[toc,page,titletoc]{appendix}
\usepackage{etoc}
\usepackage[normalem]{ulem}
\usepackage{float}
\usepackage{makecell}
\usepackage{graphicx}
\usepackage{graphicx}
\usepackage{subfig}
\usepackage{multirow}
\usepackage{array}
\usepackage{multirow}
\usepackage{comment}
\usepackage{etoc}

\ifpdf
  \DeclareGraphicsExtensions{.feps,.pdf,.png,.jpg}
\else
  \DeclareGraphicsExtensions{.eps}
\fi

\setlist[enumerate]{leftmargin=.5in}
\setlist[itemize]{leftmargin=.5in}

\newcommand\numberthis{\addtocounter{equation}{1}\tag{\theequation}}

\newcounter{algsubstate}

\newenvironment{algsubstates}
{\setcounter{algsubstate}{0}%
	\renewcommand{\State}{%
		\refstepcounter{algsubstate}
		\Statex {\footnotesize\alph{algsubstate}:}\space}}
{}

\crefname{paragraph}{paragraph}{paragraphs}
\crefname{supp}{}{Supplements}
\newcommand{\centered}[1]{\begin{tabular}{l} #1 \end{tabular}}
\newcommand{\norm}[1]{\left\lVert#1\right\rVert}

\newcolumntype{T}{>{\tiny}l} 
\newcolumntype{H}{>{\Huge}l} 

\makeatletter
\ifcase \@ptsize \relax
  \newcommand{\miniscule}{\@setfontsize\miniscule{5}{6}}
\or
  \newcommand{\miniscule}{\@setfontsize\miniscule{5}{6}}
\fi
\makeatother

\theoremstyle{plain}
\newtheorem{theorem}{Theorem}[section]
\newtheorem{proposition}[theorem]{Proposition}
\newtheorem{lemma}[theorem]{Lemma}
\newtheorem{corollary}[theorem]{Corollary}
\theoremstyle{definition}

\theoremstyle{remark}
\newtheorem{remark}[theorem]{Remark}

%% file: MyAbbreviavtions.tex
\usepackage[acronym]{glossaries}
\newacronym{CMR}{CMR}{Common to Measurement-specific Ratio}
\newacronym{EVFD}{EVFD}{Eigenvalue Flow Diagram}
\newacronym{EVD}{EVD}{Eigenvalue Decomposition}
\newacronym{SPD}{SPD}{Symmetric and Positive-Definite}
\newacronym{DM}{DM}{Diffusion Map}
\newacronym{SM}{Appendix}{Appendix}

%% file: body.tex
\begin{abstract}
In this paper, we present a new spectral analysis and a low-dimensional embedding of two multimodal aligned datasets.
Our approach combines manifold learning with the Riemannian geometry of symmetric and positive-definite (SPD) matrices. 
Manifold learning typically includes the spectral analysis of a single kernel matrix corresponding to a single dataset or a concatenation of several datasets.
Here, we use the Riemannian geometry of SPD matrices to devise an interpolation scheme for combining two kernel matrices corresponding to two, possibly multimodal, datasets. 
We study the way the spectra of the kernels change along geodesic paths on the manifold of SPD matrices.
We demonstrate that this change enables us, in a purely unsupervised manner, to derive an informative spectral representation of the relations between the two datasets. 
Based on this representation, we propose a new multimodal manifold learning method.  
We demonstrate our spectral representation and method on simulations and real-measured data.
\end{abstract}

\section{Introduction}
\label{sec:intro}
Multimodal data collections have recently become ubiquitous mainly because modern data acquisition often makes use of multiple types of sensors that can provide complementary aspects of complex phenomena and robustness to interference and noise.
The widespread availability of multimodal data calls for the development of analysis and processing tools that particularly address the challenges arising from the heterogeneity of the data (see  \citep{khaleghi2013multisensor,lahat2015multimodal,gravina2017multi} and references therein).

The focal point of the majority of multimodal data analysis methods is the fusion of the information from all modalities.
Here, we focus on a different aspect and attempt to analyze the commonality and difference between the modalities. Specifically, we consider two aligned datasets of multimodal measurements and assume that they share both mutual components and measurement-specific components. Typically, the mutual components are associated with the `signal', which we want to extract, and the measurement-specific components are associated with interference or noise, which we want to discard. Given the two datasets, our goal is to discover the underlying driving components and to characterize whether they are mutual or measurement-specific.

To this end, we take a geometric approach combining manifold learning and Riemannian geometry.
Manifold learning is a class of dimensionality reduction methods that operate under the so-called manifold assumption, i.e., assuming that high-dimensional data lie on or close to a low-dimensional manifold. Notable examples include ISOMAP \citep{tenenbaum2000global}, LLE \cite{roweis2000nonlinear}, HLLE 
\cite{donoho2003hessian}, Laplacian Eigenmap \cite{belkin2003laplacian} and Diffusion Map \cite{coifman2006diffusion}.
For multimodal data analysis, manifold learning entails important advantages, since manifold learning methods are unsupervised, data-driven, and they require only minimal prior knowledge on the data, circumventing the need to hand-craft specially-designed features and solutions for each specific modality.
Indeed, a large body of evidence implies that manifold learning (and related kernel) methods facilitate efficient and informative analysis of multimodal data (see \Cref{sec:related}).

Common practice in classical manifold learning is to build an operator based on pairwise affinities between the data.
This operator is typically a symmetric and positive kernel matrix, or a non-symmetric matrix that is similar to a symmetric and positive kernel matrix, facilitating the use of spectral analysis to embed the data into a low-dimensional space. 
When two (or more) datasets are given, this practice is limiting, and the question  how to extend the spectral analysis of one matrix (corresponding to one dataset) to several matrices (corresponding to several datasets) arises. 

We posit that one natural extension is by using the Riemannian geometry of symmetric and positive-definite (SPD) matrices \cite{pennec2006intrinsic,bhatia2009positive}.
Specifically, we consider an interpolation scheme of SPD matrices along the geodesic path between two SPD kernel matrices, where each SPD kernel matrix is constructed based on one of the two given datasets.
The advantage of such a scheme is that it allows us to preserve the SPD structure of the kernel matrices used in manifold learning, and specifically, spectral analysis still applies to each matrix on the geodesic path.

From an algorithmic viewpoint, we propose to consider a sequence of SPD matrices along the geodesic path.
Based on this sequence, we construct a representation, which we term \emph{eigenvalues flow diagram} (\acrshort*{EVFD}), depicting the variation of the spectra of the matrices along the geodesic path.
Our theoretical analysis involves the investigation of the way the spectrum of the matrices changes along geodesic paths. To the best of our knowledge, such a problem has not been investigated in the past both in this specific context, nor in the broad context of Riemannian geometry of SPD matrices. 
We show that the variation of the spectrum exhibits prototypical characteristics, which are highly informative for the problem at hand. Specifically, we show how the mutual and measurement-specific spectral components can be identified, and how an embedding based only on the mutual components can be devised from the \acrshort*{EVFD}. 

We test our approach on simulations, a toy problem, and two applications to real-world multimodal data: electronic noise and industrial condition monitoring. We demonstrate that the \acrshort*{EVFD} provides an informative representation. Specifically in these considered tasks, we show that the embedding derived from the \acrshort*{EVFD} is advantageous compared to alternative embedding methods.

Our main contributions are as follows.
(i) We propose a framework for multimodal spectral component analysis combining manifold learning with the Riemannian geometry of SPD matrices. The proposed framework is unsupervised and data-driven, and does not assume any prior knowledge regarding the modalities of the data.
(ii) We study the variation of the spectra of SPD matrices along geodesic paths. To the best of our knowledge, it is the first time such a variation is studied.
We present theoretical analysis of special cases, devise concrete algorithms, and provide empirical evidence for the usefulness of this approach.
(iii) We present an interpolation scheme of two kernel matrices corresponding to two aligned datasets that gives rise to a new spectral representation. We demonstrate both theoretically and empirically that this representation is informative, enabling us to identify the sources of variability influencing the data and extracting the mutual-relations between these sources.
(iv) We introduce a new multi-manifold learning method that extracts the latent intrinsic common manifold.

\section{Related work}
\label{sec:related}
A central element in the problem we consider is finding common components in data, which is a long standing problem in applied science. Perhaps the most notable method for this purpose is canonical correlation analysis (CCA) introduced by \cite{hotelling1936relations}. In CCA, given two sets of aligned measurements, the goal is to find two linear projections of the sets, such that the correlation between the projections is maximized. This simplistic description highlights two of CCA's shortcomings: (i) the projections are global and linear, and (ii) the (correlation) criterion is linear. Indeed, over the years, many nonlinear extensions have been proposed using kernels \cite{lai2000kernel,akaho2006kernel}, nonparametric methods \cite{michaeli2016nonparametric}, and artificial neural networks \cite{andrew2013deep}, to name but a few. 

In the context of manifold learning, fusing information from several datasets has been addressed in \cite{keller2009audio,davenport2010joint,boots2012,eynard2015multimodal,salhov2020multi,lindenbaum2020multi}.
While these methods take into account the data from all the datasets, recently, in \cite{lederman2018learning,talmon2018latent}, a manifold based approach was applied in order to extract the latent common components that were modeled by the existence of a common manifold. Specifically, it was shown that alternating applications of diffusion map operators extract the latent common components from two aligned datasets, while attenuating the measurement-specific components. For more information on this method, termed alternating diffusion (AD), see Appendix \ref{sec:sup_related}.

The Riemannian geometry of SPD matrices plays an important role in this work. The Riemannian geometry of SPD matrices is well-studied and has been shown to be useful in many tasks and applications. For example, it was used for diffusion tensor processing in medical imaging \cite{pennec2006riemannian,pennec2006intrinsic}, video processing \cite{tuzel2008pedestrian}, video retrieval \cite{li2015face,qiao2019deep}, image set classification \cite{wang2012covariance}, brain–computer interface \cite{barachant2011multiclass,barachant2013classification,barachant2014plug} and domain adaptation \cite{zanini2017transfer,rodrigues2018riemannian,yair2019parallel}. 

While most studies on SPD matrices focus on covariance matrices, the present work considers the Riemannian geometry of SPD kernel matrices.
This facilitates a significant extension of the scope and increases the number of potential applications. For example, in contrast to covariance matrices, SPD kernel matrices can encode nonlinear and local associations and characterize temporal behavior.

In the literature, there exist several geometries associate with the SPD manifold, e.g., the log-Euclidean metric \cite{arsigny2007geometric}, Thompson's metric \cite{thompson1963certain}, and the Bures-Wasserstein metric \cite{bhatia2019bures}. Here, we make use of the affine-invariant geometry \cite{pennec2006intrinsic,bhatia2009positive}. The motivation for this particular geometry is twofold. First, it has an extension to symmetric and positive semi-definite (SPSD) matrices \cite{bonnabel2009riemannian}, which is important in practice when the kernel matrices are relatively large and tend to be low rank (see details in Appendix \ref{sec:Implementations}). Second, the affine-invariant metric is tightly related to statistical metrics such as the Fisher-Rao metric, the symmetric Kullback-Leibler divergence \cite{costa2015fisher,said2017riemannian}, the S-divergence \cite{sra2016positive} and the Hellinger distance \cite{bhatia2019matrix}. Specifically, the Fisher-Rao metric between two Gaussian distributions with equal means coincides with the affine-invariant metric between their covariance matrices. We note that the kernel matrices we consider can be viewed as covariance matrices of Gaussian distributions in an appropriate reproducing kernel Hilbert space (RKHS) \cite{sindhwani2007relational}.

\section{Preliminaries}
\label{sec:prelim}
\subsection{On the Riemannian geometry of SPD matrices}
\label{sec:RiemannianIntro}
Let $\mathcal{S}(n)$ denote the set of symmetric matrices in $\mathbb{R}^{n\times n}$. 
A matrix $\mathbf{K} \in \mathcal{S}(n)$ is called SPD if all its eigenvalues are strictly positive. Let $\mathcal{P}(n)$ denote the subset of SPD matrices, which is an open set in $\mathcal{S}(n)$ and forms a differentiable Riemannian manifold with the following affine-invariant inner product \cite{pennec2006intrinsic}:
\begin{equation}
\label{eq:metric}
    \langle \mathbf{S}_1, \mathbf{S}_2 \rangle _{\mathcal{T}_K\mathcal{P}(n)} = \langle \mathbf{K}^{-\frac{1}{2}} \mathbf{S}_1 \mathbf{K}^{-\frac{1}{2}}, \mathbf{K}^{-\frac{1}{2}} \mathbf{S}_2 \mathbf{K}^{-\frac{1}{2}} \rangle,
\end{equation}
where $\mathcal{T}_{\mathbf{K}}\mathcal{P}(n) = \{\mathbf{K}\} \times \mathcal{P}(n)$ is the tangent space to $\mathcal{P}(n)$ at the point $\mathbf{K} \in \mathcal{P}(n)$, and $\langle \mathbf{A},\mathbf{B}\rangle = \text{tr}(\mathbf{A}^T\mathbf{B})$ is the standard inner product.

The Riemannian geometry of SPD matrices encompasses many useful properties, making it a powerful framework for a wide range of data analysis paradigms from kernel-based \cite{jayasumana2013kernel} and statistical \cite{sra2015conic} approaches to neural networks \cite{huang2017riemannian}.
For example, the logarithmic and exponential maps have closed-form expressions, and there are many efficient algorithms for calculating the Riemannian mean (defined using the Fr\'{e}chet mean) of a set of SPD matrices \cite{barachant2013classification,moakher2005differential}.
Perhaps the most important property in the context of this paper is that the Riemannian manifold $\mathcal{P}(n)$ has a unique geodesic path between any two points $\mathbf{K}_1,\mathbf{K}_2 \in \mathcal{P}(n)$. Moreover, this geodesic path has a closed-form expression given by \cite{bhatia2009positive}:
\begin{equation}
\label{eq:geodesic}
	\gamma(t)= \mathbf{K}_1^{1/2}\big( \mathbf{K}_1^{-1/2}\mathbf{K}_2\mathbf{K}_1^{-1/2} )^t \mathbf{K}_1^{1/2},
\end{equation}
where $0\leq t \leq1$ is a parametrization of the arc-length of the path from the initial point $\mathbf{K}_1$ ($t=0$) to the end point $\mathbf{K}_2$ ($t=1$).
The arc-length of the geodesic path has a closed-form expression as well, inducing a Riemannian distance. For more details on $\mathcal{P}(n)$, see \cite{pennec2006intrinsic,bhatia2009positive}.

\subsection{Diffusion maps}
\label{subsec:dm_doubly}

Diffusion maps is a manifold learning method introduced by \cite{coifman2006diffusion}. Compared to other manifold learning methods that facilitate low-dimensional data embedding based on spectral analysis, diffusion maps presents two additional components. First, it presents \emph{diffusion geometry}, formulating a notion of diffusion on data. Second, it defines a useful distance, termed \emph{diffusion distance}, which is an important derivative of the diffusion geometry.

Consider a measure space $(\mathcal{M},d\mu)$, where $\mathcal{M}$ is a compact smooth Riemannian manifold and $d\mu(x) = p(x)dx$ is a measure with density $p(x) \in C^3(\mathcal{M})$, which is a positive definite function with respect to the volume measure $dx$ on $\mathcal{M}$. Assume that $\mathcal{M}$ is isometrically embedded in $\mathbb{R}^d$.
Suppose that a set $\{ x_i \in \mathcal{M} \}_{i=1}^n \subset \mathbb{R}^d$ of $n$ points from the manifold is given, sampled from a density $p(x)$.
The construction of diffusion maps is as follows \cite{coifman2006diffusion}.
First, an $n \times n$ affinity matrix $\mathbf{W}$ with entries
\begin{equation}\label{eq:dm_W}
	W_{i,j}=\exp\left(-\|x_i-x_j\|^2/\epsilon\right), \ i,j=1,\hdots,n,
\end{equation}
is built, where $\epsilon>0$ is a tunable scale. Second, a two-step normalization is applied.
The first normalization is given by $\overline{\mathbf{K}} = \overline{\mathbf{D}}^{-1} \mathbf{W} \overline{\mathbf{D}}^{-1}$, where $\overline{\mathbf{D}} = \text{diag}(\mathbf{W}\mathbf{1})$ is a diagonal matrix and $\mathbf{1}$ is a column vector of all ones. This normalization is designed to handle non-uniform sampling ($p(x) \neq \text{const})$. The second normalization is given by 
\begin{equation}\label{eq:dm_A}
    \mathbf{A} = \mathbf{D}^{-1} \overline{\mathbf{K}},
\end{equation}
where $\mathbf{D} = \text{diag}(\overline{\mathbf{K}}\mathbf{1})$ is another diagonal matrix.
The normalized kernel $\mathbf{A}$ is a row-stochastic matrix that is similar to the following symmetric matrix
\begin{equation}\label{eq:dm_K}
    \mathbf{K} = \mathbf{D}^{1/2} \mathbf{A} \mathbf{D}^{-1/2} = \mathbf{D}^{-1/2} \overline{\mathbf{K}} \mathbf{D}^{-1/2}.
\end{equation}
As a result, $\mathbf{K}$ and $\mathbf{A}$ share the same real eigenvalues, and if $v$ is a right eigenvector of $\mathbf{A}$, then $\tilde{v} = \mathbf{D}^{1/2} v$ is an eigenvector of $\mathbf{K}$. In addition, because the affinity matrix $\mathbf{W}$ relies on a Gaussian kernel, $\mathbf{K}$ is SPD.

Denote the eigenvalues of $\mathbf{A}$ by $1=\mu_1 \ge \mu_2 \ge \ldots \ge \mu_n$, the left eigenvectors by $\{ u_i \in \mathbb{R}^n \}_{i=1}^n$ and the right eigenvectors by $\{ v_i \in \mathbb{R}^n \}_{i=1}^n$. This eigenvalue decomposition (\acrshort*{EVD}) gives rise to a family of nonlinear embeddings, termed \emph{diffusion maps}:
\begin{equation}\label{eq:dm}
    \Psi _t (x_i) = \left( \mu_2^t v_2 (x_i), \mu_3^t v_3 (x_i),\ldots, \mu_{\ell}^t v_\ell (x_i) \right)^T,
\end{equation}
for any $x_i, i=1,\ldots,n$, where $t>0$ and $\ell \in \{2,\ldots,n\}$ are adjustable parameters.

Perhaps the most useful property of diffusion maps is that it defines a Euclidean space, where the Euclidean distance between the embedded points, namely $\| \Psi _t (x_i) - \Psi _t (x_j) \|_2$, is the best approximation in $\ell < n$ dimensions of the following diffusion distance
\begin{equation}\label{eq:diff_dist}
    \sum _{l=1}^n \left( p_t(x_i,x_l) - p_t(x_j,x_l) \right)^2 \frac{1}{u_1(l)},
\end{equation}
where $p_t(x_i,x_j)$ is the $(i,j)$-th entry of $\mathbf{A}^t$, and equality is obtained for $\ell = n$.
We remark that the notion of diffusion stems from the matrix $\mathbf{A}$ functioning as a transition probability matrix of a Markov chain defined on a graph, whose nodes are $\{x_i\}_{i=1}^n$. Accordingly, $p_t(x_i,x_j)$ is the probability to transition (diffuse) from $x_i$ to $x_j$ in $t$ steps.

It was shown in \cite{coifman2006diffusion} that in the limit $n\rightarrow \infty$ and $\epsilon \rightarrow 0$, the row-stochastic affinity matrix $\mathbf{A}$ has a tight connection to the Laplace-Beltrami operator $\Delta$ on $\mathcal{M}$, and therefore, approximations of eigenvalues and eigenfunctions of $\Delta$ can be derived from the eigenvalues and eigenvectors of $\mathbf{A}$. This, combined with results from \cite{berard1994embedding,jones2008manifold} showing that the eigenvalues and eigenfunctions of $\Delta$ embody all the geometric information on $\mathcal{M}$, gives the theoretical justification to the embedding in \eqref{eq:dm} as a representation of the manifold $\mathcal{M}$. For more details, see Appendix \ref{sec:sup_related}.
We directly relate to these results in our simulations presented in Appendix \ref{sec:sup_SimuResults}.

\section{Problem formulation}
\label{sec:ProblemFormulation}
Consider three Riemannian manifolds: $\mathcal{M}_x,\mathcal{M}_y$ and  $\mathcal{M}_z$. These Riemannian manifolds are hidden and are accessed by two measurement functions $g$ and $h$:
\begin{eqnarray*}
\label{eq:ProblemSettings}
g:\mathcal{M}_x \times \mathcal{M}_y \rightarrow  \mathcal{O}_{1}\\
h:\mathcal{M}_x \times \mathcal{M}_z \rightarrow  \mathcal{O}_{2},
\end{eqnarray*}
which are smooth \emph{isometric} embedding of the respective product manifolds $\mathcal{M}_x \times \mathcal{M}_y$ and $\mathcal{M}_x \times \mathcal{M}_z$ into the observable spaces $\mathcal{O}_1$ and $\mathcal{O}_2$.
Importantly, $g$ ignores $\mathcal{M}_z$, and $h$ ignores $\mathcal{M}_y$. 
Moreover, the two measurements capture a common structure, the manifold $\mathcal{M}_x$, and each measurement is affected by an additional measurement-specific structure, the manifolds $\mathcal{M}_y$ or $\mathcal{M}_z$. 

Let $(x_i,y_i,z_i) \in \mathcal{M}_x \times \mathcal{M}_y \times \mathcal{M}_z$ be a latent realization from some joint distribution defined on the product manifold. This realization is not accessible directly, but gives rise to a pair of \emph{aligned} measurements $(s^{(1)}_i,s^{(2)}_i)$, such that $s^{(1)}_i = g(x_i,y_i)$ and $s^{(2)}_i = h(x_i,z_i)$.
Considering $n$ realizations of the latent triplets $\{(x_i,y_i,z_i)\}_{i=1}^n$, which give rise to two aligned and accessible measurement sets $\{s^{(1)}_i\}_{i=1}^n$ and $\{s^{(2)}_i\}_{i=1}^n$, our main goal is to build an informative representation that reveals the underlying manifold structure of the two sets of measurements in an unsupervised fashion, namely, the common manifold $\mathcal{M}_x$ and the measurement-specific manifolds $\mathcal{M}_y$ and $\mathcal{M}_z$. 
In addition, we seek an embedding of the realizations that represents only the latent samples $\{x_i\}_{i=1}^n$ from the common manifold $\mathcal{M}_x$ in the following sense.

We present here a method that yields an SPD matrix $\mathbf{K}$ analogous to the kernel matrix in \eqref{eq:dm_K}. We show that the principal eigenvectors of $\mathbf{K}$ are associated with the common manifold $\mathcal{M}_x$. Based on these principal eigenvectors, we propose an embedding of the data as in diffusion maps \eqref{eq:dm}, such that the induced diffusion distances \eqref{eq:diff_dist} correspond to the Euclidean distances between the samples $x_i$ from $\mathcal{M}_x$ (and are insensitive to the values of $y_i$ and $z_i$).

We note that a similar problem formulation was considered previously in \cite{lederman2018learning,talmon2018latent}.

\section{The eigenvalue flow diagram}
\label{sec:method}

\subsection{Building the diagram}
\label{subsec:Algorithm}

The construction of the diagram consists of two main steps.
At the first step, given the two sets of $n$ aligned measurements $\{s^{(1)}_i\}_{i=1}^n, \{s^{(2)}_i\}_{i=1}^n$, we build two kernels, which are constructed separately for the two sets as described in \Cref{subsec:dm_doubly}.
Concretely, we build two affinity matrices $\mathbf{W}^{(1)}$ and $\mathbf{W}^{(2)}$ using Gaussian kernels:
\begin{gather}
\label{eq:AffinityMatrices}
W_{i,j}^{(1)}=\exp\left(-\|s^{(1)}_i-s^{(1)}_j\|_{M_1}^2/\varepsilon^{(1)}\right),\\ W_{i,j}^{(2)}=\exp\left(-\|s^{(2)}_i-s^{(2)}_j\|_{M_2}^2/\varepsilon^{(2)}\right),
\end{gather}
for all $i,j=1,\hdots,n$, where $\varepsilon^{(1)}$ and $\varepsilon^{(2)}$ are tunable kernel scales, and $\|\cdot\|_{M_1}$ and $\|\cdot\|_{M_2}$ are two norms induced by two metrics corresponding to the two observable spaces $\mathcal{O}_1$ and $\mathcal{O}_2$.
Next, we apply a two-step normalization to $\mathbf{W}^{(1)}$ and $\mathbf{W}^{(2)}$, by computing the row-stochastic matrices $\mathbf{A}_1$ and $\mathbf{A}_2$ by \eqref{eq:dm_A}, and then obtaining the SPD kernels $\mathbf{K}_1$ and $\mathbf{K}_2$ by \eqref{eq:dm_K}.

At the second step, the Riemannian geometry of the SPD matrices $\mathbf{K}_1$ and $\mathbf{K}_2$ is utilized. We consider a discrete uniform grid of $N_t$ points from the interval $[0,1]$, denoted by $\{t_i\}^{N_t}_{i=1}$, and compute a sequence of matrices $\gamma(t_i)$ along the geodesic path connecting $\mathbf{K}_1$ and $\mathbf{K}_2$ on this grid:
\begin{equation}\label{eq:geodesic_grid}
    \gamma(t_i)= \mathbf{K}_1^{1/2}\left( \mathbf{K}_1^{-1/2}\mathbf{K}_2\mathbf{K}_1^{-1/2} \right)^{t_i} \mathbf{K}_1^{1/2}.
\end{equation}
Then, we apply \acrshort*{EVD} to $\gamma(t_i)$ and obtain the largest $K+1$ eigenvalues of $\gamma(t_i)$, denoted by $\{\mu_{t_i}^{k}\}_{k=1}^{K+1}$.
Finally, the resulting \acrshort*{EVFD} is obtained by scatter plotting the logarithm of the largest $K$ eigenvalues, $\{ \log(\mu_{t_i}^{k})\}_{k=2}^{K+1}$, ignoring the trivial $\mu_{t_i}^1$, as a function of $t_i$.
The entire algorithm and additional implementation notes using the geometry of symmetric positive semi-definite (SPSD) kernels appear in \Cref{sec:Implementations}.

\subsection{Illustrative toy example}
\label{subsec:ToyExample}
In order to illustrate the \acrshort*{EVFD}, we revisit the toy problem from \cite{lederman2018learning}.
The problem consists of three puppets: Yoda, Bulldog and Bunny that were placed on three rotating displays. These puppets were captured simultaneously by snapshots from two cameras, as depicted in Fig. \ref{fig:Puppets_EVDiagrams_Geodesic}(A). With respect to the considered problem setting, the rotation angles of Bulldog, Yoda and Bunny are realizations from a joint distribution on the product manifold: $(x,y,z) \in \mathcal{M}_x \times \mathcal{M}_y \times \mathcal{M}_z$, where $\mathcal{M}_x$, $\mathcal{M}_y$, and $\mathcal{M}_z$ equal the 1-sphere, and the snapshots from the two cameras are the measurements $s^{(1)} = g(x,y)$ and $s^{(2)} = h(x,z)$.

On the left-hand side of Fig. \ref{fig:Puppets_EVDiagrams_Geodesic}(B), we depict the \acrshort*{EVFD} obtained based on pairs of simultaneous snapshots from the two cameras. The vertical axis corresponds to the position $t_i$ along the geodesic path, and the horizontal axis corresponds to values $\log(\mu_{t_i}^{k})$. Namely, a point $(\log(\mu_{t_i}^{k}),t_i)$ in the diagram represents the logarithm of the $k$th eigenvalue of $\gamma(t_i)$. A further illustration is given in the following video \href{https://youtu.be/EXPIm0eAuQw}{link}. The code for implementing this illustrative toy example is publicly available \href{https://github.com/OriKatz/SpectralFlowAnalysis}{here}.

\begin{figure}[t]
	\centering
	\includegraphics[width=0.85\textwidth]{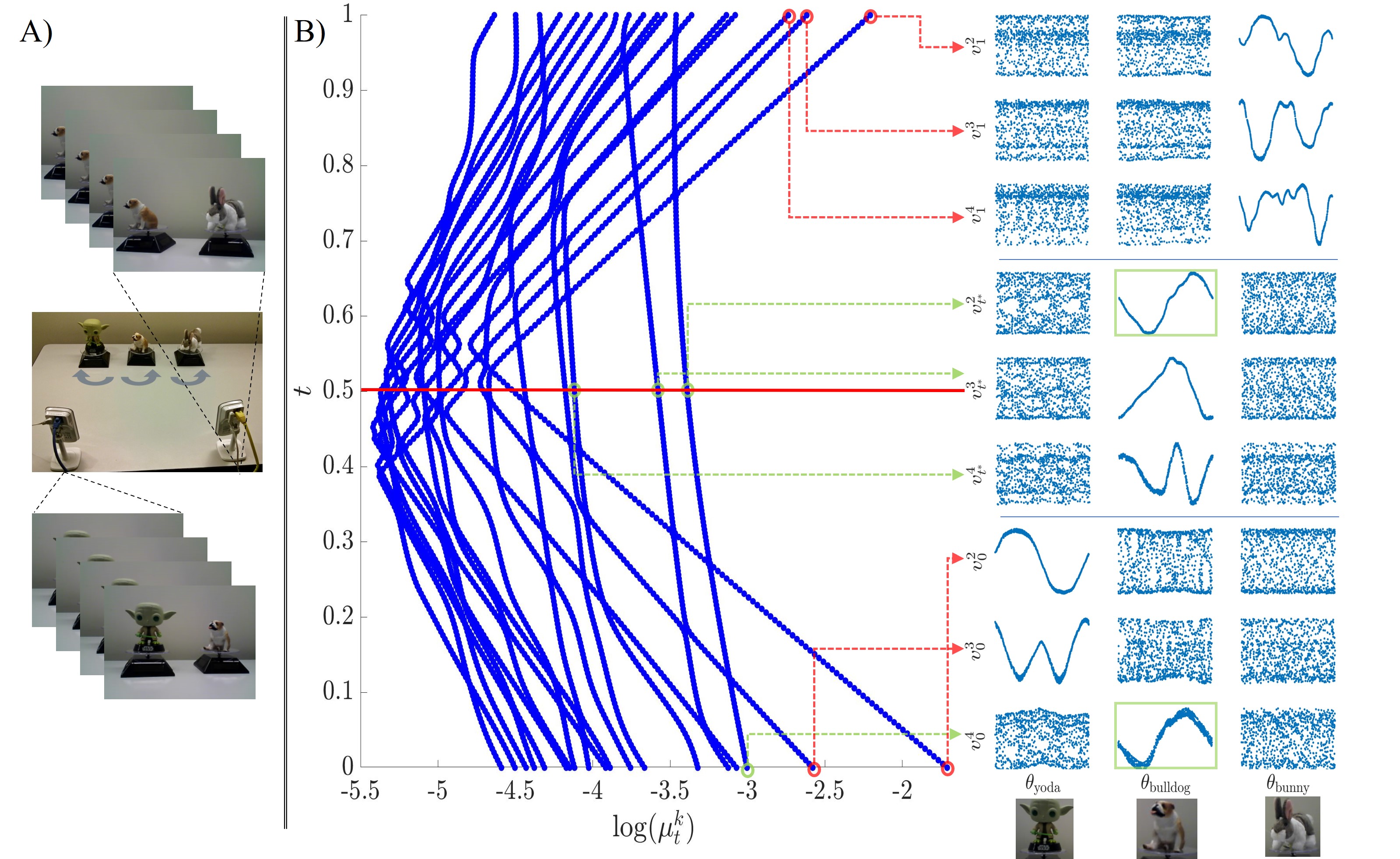} 	
	\caption{Results on the toy problem. (A) The experimental setup. 
		(B) The \acrshort*{EVFD} with insets presenting the scatter plots between the eigenvectors corresponding to the marked eigenvalues and the rotation angles of the figures.}
	\label{fig:Puppets_EVDiagrams_Geodesic}
\end{figure}

We posit that simply by looking at the \acrshort*{EVFD}, important information about the tasks we aim to accomplish is revealed, considering only the measurements without any knowledge on the hidden angles of the rotating puppets.
First, we visually identify ``continuous'' curves along the vertical axis $t$, despite having no such explicit connection. Namely, at two consecutive vertical coordinates $t_i$ and $t_{i+1}$, we plot the logarithm of the eigenvalues of $\gamma(t_i)$ and $\gamma(t_{i+1})$, respectively, so that the eigenvalues at different $t_i$ coordinates are only implicitly connected through the matrices $\gamma(t_i)$ and $\gamma(t_{i+1})$. 
Second, these curves along $t$ assume two forms: (i) approximately straight vertical lines, connecting a point on the spectrum of $\mathbf{K}_1$ (at $t_1=0$) and a point on the spectrum of $\mathbf{K}_2$ (at $t_{N_t}=1$), and (ii) curves with steep slopes.
We demonstrate that the eigenvalues lying on the straight vertical lines are associated with the common manifold. 
We pick $3$ points along the geodesic path at $t=0$, $t^*=0.5$ and $t=1$. At each of these points, we take the eigenvectors corresponding to the leading $3$ eigenvalues of $\gamma(t)$ (marked by red and green circles) and present on the right-hand side of Fig. \ref{fig:Puppets_EVDiagrams_Geodesic}(B) the scatter plots of these eigenvectors as a function of the (hidden) angles of the rotating Yoda, Bulldog and Bunny.
We see in the insets marked by green arrows that the eigenvectors corresponding to eigenvalues that lie on straight vertical lines are highly correlated with the rotation angle of the common Bulldog. Moreover, by observing their correspondence with the angle of Bulldog (marked by the green frames), we can see that these eigenvectors do not significantly change along the geodesic path.
In contrast, in the insets marked by red arrows, we see that the eigenvectors corresponding to eigenvalues that lie on curves with steep slopes correlate with the rotation angles of Yoda and Bunny, and therefore, are associated with the measurement-specific manifolds. 

The above characterization of the eigenvectors as common or measurement-specific was done merely by a visual inspection of the \acrshort*{EVFD}, which is computed in a purely unsupervised manner. More specifically, it was obtained by inspecting the vertical direction of the diagram, namely, the `flow' of the spectra of the matrices induced by marching along the geodesic path, rather than the horizontal direction, which is the more typical spectral analysis.
We note that there might exist eigenvectors that are associated both with the common manifold and the measurement-specific manifolds. In this work, we consider these `mixed' eigenvectors together with the measurement-specific eigenvectors, and simply view them as `non-common' eigenvectors.

\subsection{Analysis}
\label{subsec:TheorericFoundations}
We make formal some of the empirical characterizations of the \acrshort*{EVFD} made in \Cref{subsec:ToyExample}. Here we only state the results, while the proofs appear in \Cref{sec:sup_TheorericFoundations}.

\begin{proposition}
\label{prop:MutualEigenValues}
	If $v \in \mathbb{R}^{n}$ is an eigenvector of $\mathbf{K}_1$ associated with the eigenvalue $\mu_1$ and an eigenvector of $\mathbf{K}_2$ associated with the eigenvalue $\mu_2$ (i.e. $v$ is a common eigenvector of $\mathbf{K}_1$ and $\mathbf{K}_2$), then $v$ is also an eigenvector of $\gamma(t)$ for any $t \in (0,1)$ with the corresponding eigenvalue:
	\begin{equation}
	\mu_t = \mu_1^{1-t}  \mu_2^{t}.
	\label{eq:CommonEigenVals}
	\end{equation}
\end{proposition}
It follows from \Cref{prop:MutualEigenValues} that the logarithm of $\mu_t$ is linear in $t$.
\begin{corollary}
	\label{cor:loglinear_behaviour}
	The flow of the eigenvalues $\mu_t$ corresponding to common eigenvectors is log-linear with respect to $t$, i.e.
	\begin{equation}
	\log(\mu_t) = (1-t)\log(\mu_1) + t \log(\mu_2).
	\end{equation}
\end{corollary}
This result was demonstrated in the puppets toy example in Fig. \ref{fig:Puppets_EVDiagrams_Geodesic}(B), where the eigenvectors associated with eigenvalues lying on log-linear straight lines are correlated with the angle of the common Bulldog.

Next, we show that the strict requirement of identical shared eigenvectors can be relaxed.
Denote the EVD of the SPD kernels by $\mathbf{K}_1 = \mathbf{V}_1 \mathbf{M}_1 \mathbf{V}^{\top}_1$ and $\mathbf{K}_2 = \mathbf{V}_2 \mathbf{M}_2 \mathbf{V}^{\top}_2$, where the eigenvalues on the diagonal of $\mathbf{M}_1$ are sorted in descending order.
We consider similar but not identical eigenvectors by assuming that the eigenvectors of $\mathbf{K}_2$ are perturbations of the eigenvectors of $\mathbf{K}_1$, i.e., $\mathbf{V}_2=\mathbf{V}_1+\epsilon\mathbf{E}$, where $\norm{\mathbf{E}}=1$ and $\epsilon >0$. Note that this means that all the eigenvectors are nearly common.
With a slight abuse of notation, let $\mu_1^k$ and $\mu_2^k$ denote the $k$th eigenvalues of $\mathbf{K}_1$ and $\mathbf{K}_2$, respectively.
The following proposition shows that in this case, the \acrshort*{EVFD} presents only near log-linear trajectories. 
\begin{proposition}
\label{prop:WeaklyCommon}
    If $\frac{1}{c} \leq \ell^k:=\frac{\mu_2^k}{\mu_1^k}\leq c$ for some constant $c\geq1$ for all $k=1,\ldots,n$, then for any $t \in (0,1)$:
    \[
    \norm{\left(\gamma(t) - \left(\mu^k_1\right)^{1-t}\left(\mu^k_2\right)^t \mathbf{I}\right)v_1^k} = \mathcal{O}(\epsilon),
    \]
    where $v_1^k$ is the $k$th eigenvector of $\mathbf{K}_1$, and the implied constant depends on $c^t$ and $1/\min_i(\mu_1^i \gamma_i)$, where $\gamma_i:= \text{min}_{j\neq i} |\ell^i - \ell^j|$.
\end{proposition}

\begin{remark}
\label{rem:weakly}
In practice, $\mathbf{K}_1$ and $\mathbf{K}_2$ tend to be close to low-rank matrices, and the constant $\min_i(\mu_1^i \gamma_i)$ might be small. However, we empirically observe that it is sufficient to consider $\min_{i\sim k}(\mu_1^i \gamma_i)$, i.e., eigenvalues $\mu_1^i$ close to $\mu_1^k$. Since we are usually interested in principal eigenvectors $v_1^k$ with large eigenvalues $\mu_1^k$, the value of $\min_{i\sim k}(\mu_1^i \gamma_i)$ is typically sufficiently large. See details in \Cref{sec:sup_TheorericFoundations}.
\end{remark}

\Cref{prop:WeaklyCommon} states that (nearly) common eigenvectors are preserved along the geodesic path. We remark that other schemes, e.g., a linear interpolation or an interpolation along the geodesic path induced by the log-Euclidean metric, exhibit the same property. Our empirical results show that the main advantage of the interpolation along this particular geodesic path is the ability to suppress the non-common eigenvectors. The following proposition provides a partial explanation to the correspondence between the \acrshort*{EVFD} and the non-common eigenvectors.
\begin{proposition}
\label{thm:OnlyMutuals}
	If $v\in \mathbb{R}^{n}$ is an eigenvector of $\mathbf{K}_1$ but not an eigenvector of $\mathbf{K}_2$, then $v$ is not an eigenvector of $\gamma(t)$.
\end{proposition}

In \Cref{sec:CaseStudy}, we further investigate the \acrshort*{EVFD} under a different, yet related \emph{discrete} graph model using spectral graph theory, allowing us to enhance the results on the non-common eigenvectors, 

\section{Common manifold learning}
\label{sec:embedding}
To make the characterization of the eigenvectors presented in \Cref{sec:method} systematic, we present in \Cref{sec:Implementations} algorithms to resolve the trajectories in the \acrshort*{EVFD} and to identify the common and non-common eigenvectors based on the shape of the trajectories according to the analysis from \Cref{subsec:TheorericFoundations}.
These algorithms enable us to divide the eigenvalues (and associate eigenvectors) of $\gamma(t)$, excluding the largest trivial eigenvalue, into two disjoint sets: the set of indices of (nearly) common eigenvectors $\mathcal{S}_{\text{c}}$ and the remaining indices $\mathcal{S}_{\text{nc}}$.
Forming these two sets, we define the Common to Measurement-specific Ratio (\acrshort*{CMR}) by:
\begin{equation}
\label{eq:COR}
\text{CMR}(t)=\left( \sum\limits_{k \in \mathcal{S}_{\text{c}}}\mu_{t}^{k} \right)/\left(\sum\limits_{k \in \mathcal{S}_{\text{nc}}}{\mu_{t}^{k}}\right),
\end{equation}
and $t^*$ by
\begin{equation}\label{eq:max_CMR}
    t^*=\underset{t\in[0,1]}{\operatorname{argmax}} \ {\text{CMR}(t)}.
\end{equation}
Then, applying \acrshort*{EVD} to $\gamma(t^*)$ yields a set of eigenvectors $\{ v _{t^*}^k\}$ corresponding to the eigenvalues $\{ \mu _{t^*}^k\}$ sorted in descending order. These eigenvectors are used for building an embedding:
\begin{equation}\label{eq:prop_embed}
    (s_i^{(1)}, s_i^{(2)}) \mapsto \left( v_{t^*}^k(i)\right)_{\{k | 2 \le k \le \ell+1 \}} \in \mathbb{R}^{\ell}, 
\end{equation}
where $\ell$ is a tunable parameter.
By definition, at $t^*$, the eigenvalues corresponding to eigenvectors associated with the common manifold $\mathcal{M}_x$ are more pronounced, and therefore, they appear higher in the spectrum.
As a result, the embedding in \eqref{eq:prop_embed} is likely to give rise to a representation in which the common manifold is enhanced compared to diffusion maps based on $\mathbf{K}_1$ or $\mathbf{K}_2$.
Alternatively, given the set of indices $\mathcal{S}_{\text{c}}$, an embedding based only on the identified common eigenvectors can be defined as follows:
\begin{equation}\label{eq:prop_embed2}
     (s_i^{(1)}, s_i^{(2)}) \mapsto \left( v_{t^*}^k(i)\right)_{\{k | k \in  \mathcal{S}_{\text{c}} \}}. 
\end{equation}

\begin{figure}[t]\centering
    \includegraphics[width=0.85\textwidth]{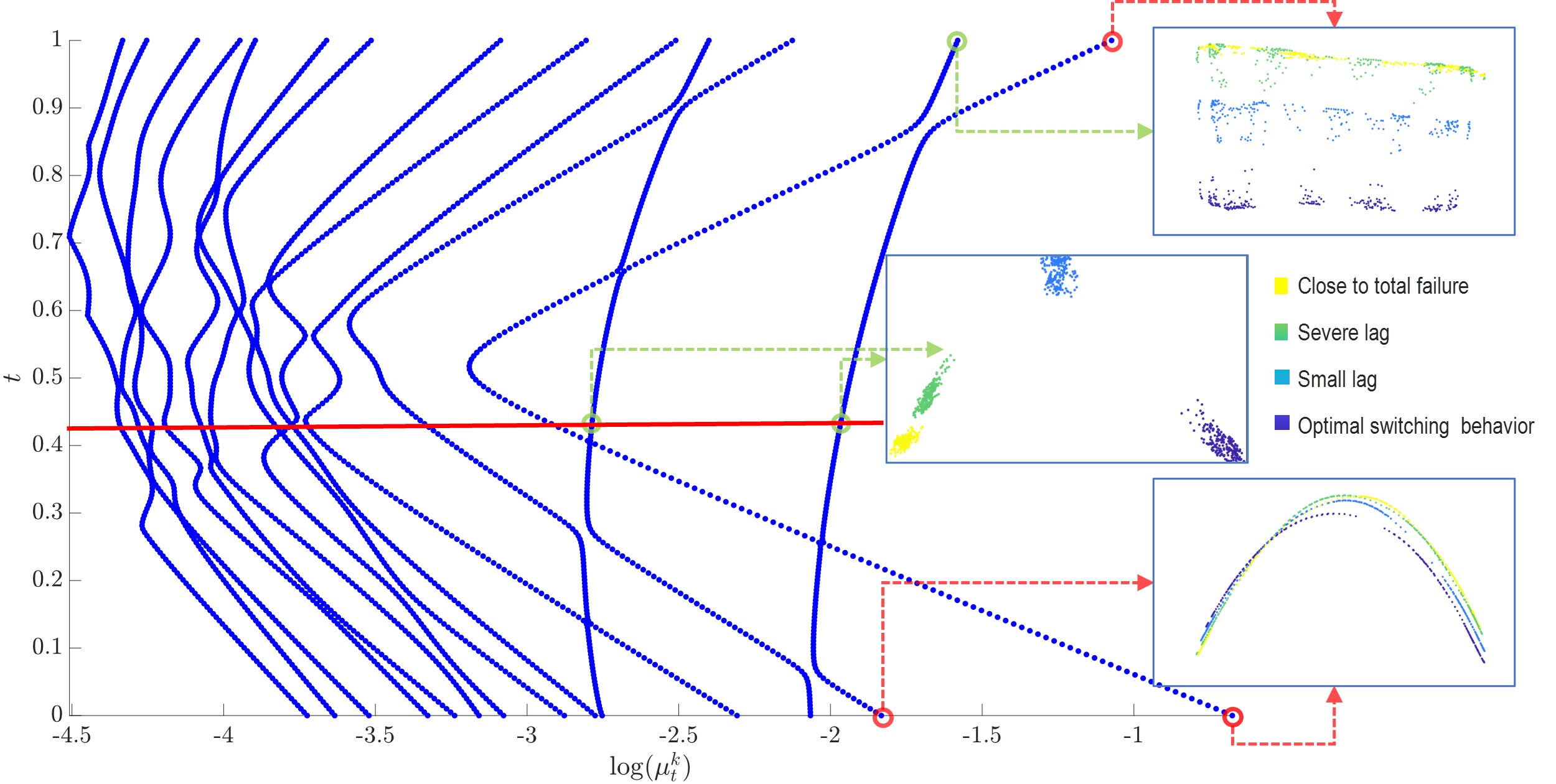} 
	\caption{The \acrshort*{EVFD} obtained based on condition monitoring measurements acquired by the pressure meter ({\tt PS2}) and by the volume flow meter ({\tt FS1}) with faults.}
	\label{fig:CM_ValveWithSineNoise_EVDiagram_Geodesic}
\end{figure}

\begin{table}[t]
        \caption{The smoothness $S^{10}_{\mathbf{K}}(x)$ averaged over all pairs of sensors (See \cref{sec:sup_ExpResults} for the full table with the standard deviations).}
        \label{tab:cm_result}
        \vskip -0.2in
        \setlength{\tabcolsep}{2pt}
        \begin{center}
        \begin{small}
        \begin{sc}
        \begin{tabular}{lccccc}
        \toprule
        Faults Type & Ours & Linear & AD  & NCCA & KCCA \\
        \midrule
        \makecell{Harmonic}    & $\mathbf{0.92}$   &   $0.78$   &   $0.82$   &   $0.89 $ &   $0.89$    \\
        \makecell{Sawtooth}    & $\mathbf{0.94}$   &   $0.92$   &   $0.91$   &   $0.82 $ &   $0.92$    \\
        \bottomrule
        \end{tabular}
        \end{sc}
        \end{small}
        \end{center}
        \vskip -0.2in
\end{table}

\section{Experimental results}
\label{sec:RealData}

We test the proposed framework on two real-world datasets.

In such datasets, $\mathbf{K}_1$ and $\mathbf{K}_2$ do not share identical eigenvectors. For simplicity, we refer to nearly-common eigenvectors as common eigenvectors. We show that the \acrshort*{EVFD}s indeed consist of nearly-linear trajectories (demonstrating \Cref{prop:WeaklyCommon}) that can be separated from other trajectories (both visually and automatically), and that these nearly-linear trajectories are associated with informative eigenvectors.
In \Cref{sec:sup_SimuResults}, we present additional simulations. The code for generating all the figures and tables presented in the paper is  publicly available \href{https://github.com/OriKatz/SpectralFlowAnalysis}{here}.

We evaluate the proposed embedding extracted from the \acrshort*{EVFD} and compare it to four baseline alternatives that address a similar task.
The first baseline is the linear interpolation scheme:
\begin{equation}
\label{eq:LinearInterpolation}
    \mathbf{L}(t) = (1-t)\mathbf{K}_1 + t\mathbf{K}_2, \ t \in (0,1).
\end{equation}

The second and third baselines are nonlinear variants of canonical correlation analysis (CCA): Kernel CCA (KCCA) \cite{akaho2006kernel} and Nonparametric CCA (NCCA) \cite{michaeli2016nonparametric}. The fourth baseline is alternating diffusion (AD) \cite{lederman2018learning}. Similarly to our approach, AD is a manifold learning method based on geometric considerations, whereas KCCA and NCCA are primarly driven by statistical considerations.

For quantitative evaluation, we consider the \emph{$\ell$-truncated smoothness score} proposed in \cite{yair2020spectral}. 
Let $\lambda_1\geq \lambda_2 \ldots \lambda_n$ denote the eigenvalues of some SPD kernel matrix $\mathbf{K}$, and let $\{ v_i \}_{i=1}^n$ be their corresponding eigenvectors. 
The \emph{$\ell$-truncated smoothness score} of a vector $x\in \mathbb{R}^n$ with respect to $\mathbf{K}$ is given by:
\begin{equation}\label{eq:smoothness}
    S^{\ell}_{\mathbf{K}}(x)= \|\mathbf{V}_{\ell} ^{\top} x \|_2^2 / \|x \|_2^2,
\end{equation}
where $\mathbf{V}_{\ell} = [v_1, \ldots, v_{\ell}] \in \mathbb{R}^{n\times \ell}$. 
By definition, $S^{n}_{\mathbf{K}}(x)=1$, and high smoothness scores for $\ell \ll n$ mean that $x$ can be accurately expressed by the leading $\ell$ eigenvectors of $\mathbf{K}$. 
This smoothness score is tightly related to the standard Laplacian score \cite{he2005laplacian}. The Laplacian score is typically used for evaluating features given a Laplacian, whereas here we evaluate a kernel matrix $\mathbf{K}$ (analogous to the Laplacian) given ``features'' $x$. In our experiments, we measure the smoothness score of the latent common samples $\{x_i \in \mathcal{M}_x \}_{i=1}^n$ when they are available in simulations and toy experiments. Otherwise, we measure the smoothness score of some hidden labels. We note that when the latent common samples are high-dimensional, we embed them isometrically in $\mathbb{R}^d$ and compute $S^{\ell}_{\mathbf{K}}(\mathbf{X})= \|\mathbf{V}_{\ell} ^{\top} \mathbf{X}\|_F^2 / \ \| \mathbf{X}\|_F^2$,
where $\mathbf{X} = [x_1,\ldots, x_n] \in \mathbb{R}^{d \times n}$ and $\|\cdot\|_{F}$ is the Frobenius norm
In \Cref{sec:sup_ExpResults} we present additional evaluation metrics.

\subsection{Condition monitoring}
\label{subsec:ConditionMonitoring}
Condition monitoring is a process that aims to assess the condition of a certain machinery for the purpose of early identification of faults and malfunctions \cite{higgs2004survey}. 
Often in condition monitoring, incorporating data from two or more sensors is beneficial, because it enhances the robustness to failures of the condition monitoring system itself; since simultaneous failures in more than one sensor are rare, utilizing two or more sensors can help to distinguish between the state of the monitored system and the state of the monitoring sensors.

We examine monitoring of a hydraulic test rig \cite{helwig2015condition}. 
The monitoring system, consisting of pressure, flow, temperature, and electrical power sensors, is designed to monitor the condition of four components of the hydraulic system: valve's lag, cooling efficiency, internal pump leakage, and the hydraulic accumulator pressure.
Here, we focus on the valve condition using multimodal measurements acquired by the pressure, flow, and electrical power sensors that are mounted on the circuit of the main pump and are relevant for the valve condition (six sensors overall).
In addition, we consider a challenging scenario and introduce sensor faults \cite{helwig2015d8} by adding harmonic and sawtooth waves to the sensors.
See \Cref{sec:sup_Code} for more details on the data and setting.

The obtained \acrshort*{EVFD} is presented in Fig. \ref{fig:CM_ValveWithSineNoise_EVDiagram_Geodesic} with $K=10$ and $N_t=200$. The point $t^*$ at which the CMR is maximal is marked by a horizontal red line. 
We observe two visually distinct trajectories of eigenvalues corresponding to common eigenvectors. At the boundaries $t=0$ and $t=1$, these common eigenvectors are not associated with the leading eigenvalues, and as $t$ approaches $t^*$ the associated eigenvalues become the dominant two. Same as in Fig. \ref{fig:Puppets_EVDiagrams_Geodesic}(B), we present in insets the $2$D embedding \eqref{eq:prop_embed} based on the eigenvectors associated with the two leading eigenvalues marked by arrows, where the color indicates the valve condition. This confirms that the common eigenvectors, identified by the \acrshort*{EVFD}, are informative and correspond to the valve condition, whereas the non-common eigenvectors are not.
We further demonstrate this claim using the $\ell$-truncated smoothness score, where $x$ in \eqref{eq:smoothness} is the vector indicating the true valve's lag.
We compute its average over all possible pair combinations of the examined relevant six sensors.  The results are summarized in \cref{tab:cm_result}. We can see that our method obtains superior smoothness score compared with the other four baseline methods. 
In \Cref{sec:sup_ExpResults}, we present additional experimental results, including the \acrshort*{EVFD} of other sensor combinations and the \acrshort*{EVFD} obtained based on the linear interpolation.

\subsection{Artificial olfaction for gas identification}
\label{subsec:ArtificialOlfaction}

\begin{figure}[t]\centering
    \includegraphics[width=0.85\textwidth]{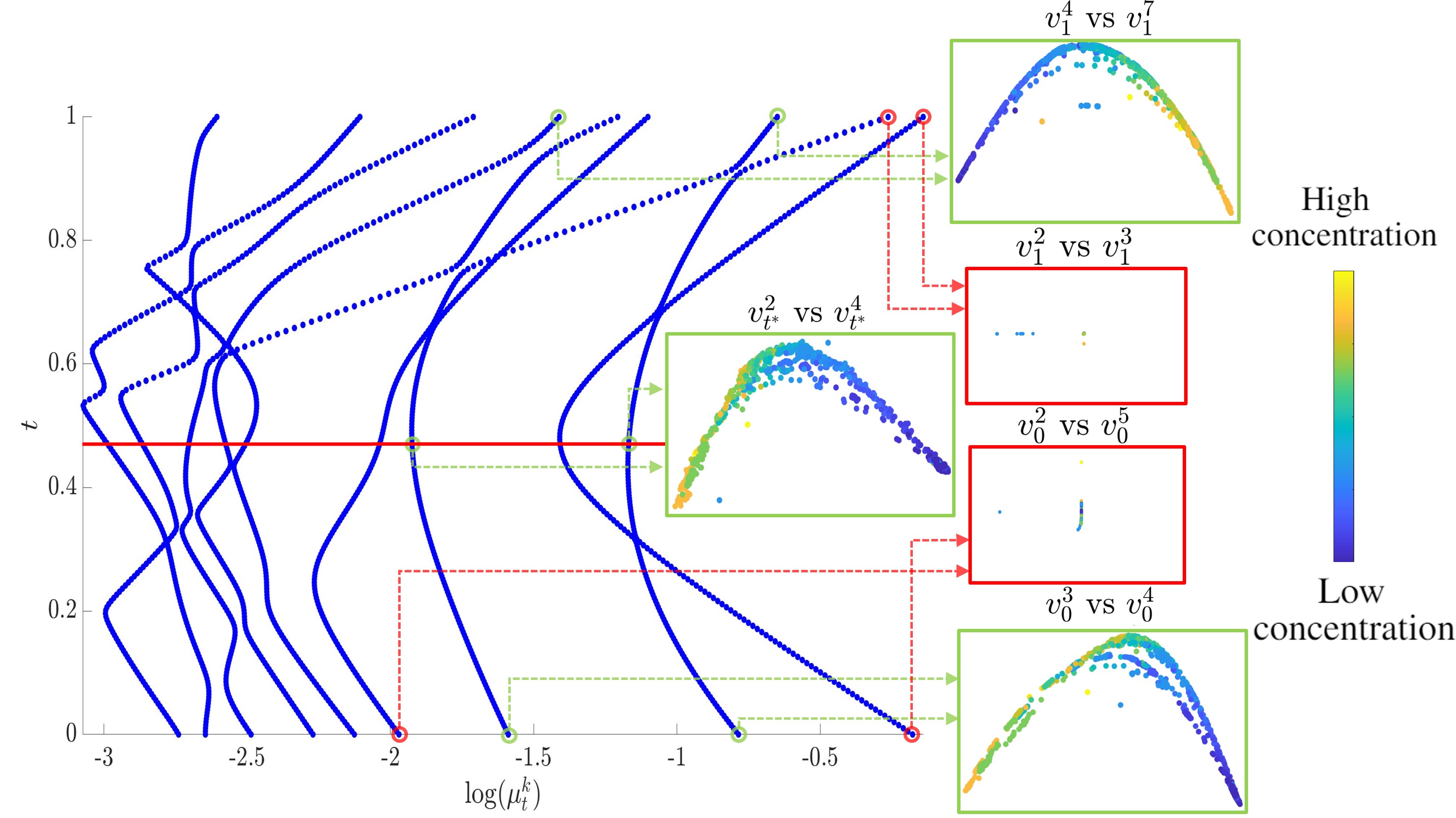} \\
	\caption{The \acrshort*{EVFD} based on measurements acquired by the gas sensors indexed $5$ and $11$. }
	\label{fig:Gas_Acetald_EVDiagram_Geodesic}
\end{figure}

Artificial olfaction, also known as electronic noise (E-Nose), is a relatively new technology that aims to detect, classify, and quantify a chemical analyte or odor. 
We consider a dataset comprising recordings of six distinct pure gaseous substances: Ammonia, Acetaldehyde, Acetone, Ethylene, Ethanol, and Toluene, each dosed at various concentrations \cite{vergara2012chemical}. 
We consider recordings from four sensors mounted on four different devices. 
See \Cref{sec:sup_Code} for more details on the setting.

In Fig. \ref{fig:Gas_Acetald_EVDiagram_Geodesic}, we present the \acrshort*{EVFD} based on two sensors with $K=10$ and $N_t=200$. 
Same as in Fig. \ref{fig:CM_ValveWithSineNoise_EVDiagram_Geodesic}, we display insets with the embedding based on the eigenvectors associated with the eigenvalues marked by arrows at $t=0$, $t=t^*$, and $t=1$. The embedding is colored according to the Acetaldehyde concentration. 
We observe that the nearly-common eigenvectors that are associated with less curved eigenvalue trajectories (marked by green arrows) are indeed informative and correspond to the Acetaldehyde concentration. In addition, we observe that these eigenvalues are enhanced as $t$ approaches $t^*$ and that their corresponding eigenvectors do not vary significantly. In contrast, observing the red insets, we can see that the non-common eigenvectors, i.e., the eigenvectors associated with eigenvalue trajectories with steep slopes, do not exhibit any correspondence with the Acetaldehyde concentration and seem to be related to outliers and interference.

In \Cref{tab:gas_results}, we present the $\ell$-truncated smoothness score, where $x$ in \eqref{eq:smoothness} is the vector of the true (hidden) concentrations of the respective injected gas. We see that our method achieves the highest smoothness, expect for Ethanol, where the performance of all methods is relatively low.
See \Cref{sec:sup_ExpResults} for additional results.

\begin{table}[t]
        \caption{The smoothness $S^5_{\mathbf{K}}(x)$ averaged over all pairs of sensors (See \cref{sec:sup_ExpResults} for the full table with the standard deviations).}
        \label{tab:gas_results}
            \begin{center}
            \setlength{\tabcolsep}{1.5pt}
            \begin{small}
            \begin{sc}
            \begin{tabular}{lccccc}
            \toprule

            Analyte & Ours & Linear & AD  & NCCA & KCCA \\
            \midrule
            Ammonia         &   $\mathbf{0.78 }$ &   $0.76$   &   $0.71$   &   $0.28$           &   ${0.77}$  \\
            Acetaldehyde    &   $\mathbf{0.78 }$ &   $0.68$   &   $0.55$   &   $0.54$           &   ${0.74}$  \\
            Acetone         &   $\mathbf{0.83}$  &   $0.81$   &   $0.64$   &   $0.56$           &   $0.8$  \\
            Ethanol         &   $0.23 $          &   $0.17$   &   $0.13$   &   $\mathbf{0.34}$  &   $0.17$  \\
            Ethylene        &   $\mathbf{0.75}$  &   $0.72$   &   $0.62$   &   $0.41$           &   $0.73$  \\
            Toluene         &   $\mathbf{0.54}$  &   $0.51$   &   $0.53$   &   $0.48$           &   $0.53$  \\
            \midrule 
            Mean &   $\mathbf{0.65 }$  &   $0.61$   &   $0.53$   &   $0.43$   &   $0.62$  \\
            \bottomrule
            \end{tabular}
            \end{sc}
            \end{small}
            \end{center}
            \vskip -0.2in
\end{table}

\section{Conclusions}
\label{sec:Conclusions}
In this paper, we studied the evolution of the spectra of SPD matrices along geodesic paths and discovered some intriguing properties. Specifically, we showed that the evolution of the spectra gives rise to an informative and useful representation of the relationships between the spectral components of two aligned, possibly multimodal, sets of measurements.
We presented theoretical analysis of special cases and devised a new multi-manifold learning method that is based on kernel interpolation along geodesic paths on the manifold of SPD matrices.
We demonstrated our approach in simulations, a toy problem, and two real-world datasets.

In future work, we plan to address the current limitations of our work. 
First, we will extend the theoretical analysis and investigate the properties of the evolution of the spectrum in broader contexts as the experiments indicate that the results could be more general (see Fig. \ref{fig:NonFixedDispersion} in \Cref{sec:sup_SimuResults}). Special focus will be put on the analysis of the non-common components, comparing their evolution along geodesic paths to other paths as well as to other geometries.
Second, we will extend the proposed method to support more than two sets of measurements (see preliminary results in \Cref{sec:sup_SimuResults}). 
Third, one limitation on the scalability of our method concerns kernel sparsity. Typically, the kernels $\mathbf{K}_1$ and $\mathbf{K}_2$ are built as sparse matrices, e.g., by considering only $k$-nearest neighbors in the affinity matrix, facilitating the analysis of large datasets. Here, even if $\mathbf{K}_1$ and $\mathbf{K}_2$ are sparse, $\gamma(t)$ might not be sparse. Imposing sparsity on $\gamma(t)$ is therefore important for increasing practicality and will allow us to extend the applicability of our method to large datasets.

%% file: supplement.tex
\newpage

\section{Additional experimental results}
\label{sec:sup_ExpResults}
In \Cref{subsec:sup_CM} and in \Cref{subsec:sup_ArtificialOlfaction} we present additional experimental results. 
In \Cref{subsec:polyfit} we present additional evaluation metric for quantifying the correspondence between the diffusion distances $d_{\mathbf{K}}(i,j)$ w.r.t. a kernel $\mathbf{K}$ and the spectrum of the Euclidean distances between the latent samples $\|x_i-x_j\|_2$.

\subsection{Condition monitoring}
\label{subsec:sup_CM}

In this section, we extend the experimental study presented in \cref{subsec:ConditionMonitoring}. We demonstrate our results on different pairs of sensors and evaluate them with respect to different labels. In addition, we show the comparison to $\mathbf{L}(t)$ in more detail. Finally, we present the $\ell$-truncated smoothness score as a function of $\ell$ for a fixed point along the geodesic path as well as as a function of the point along the geodesic path for a fixed $\ell$.

\cref{tab:supp_cm_result} is the same as \cref{tab:cm_result} but with the standard deviation in addition to the mean of the smoothness $S^{10}_{\mathbf{K}}(x)$. Note that NCCA and KCCA yield two sets of left and right eigenvectors, and therefore the scores presented in the table are the average scores obtained by using these two sets.

In Fig. \ref{fig:CM_ValveWithSineNoise_EVDiagram_Geodesic} we presented the \acrshort*{EVFD} obtained based on the pressure meter ({\tt PS2}) and the flow meter ({\tt FS1}). 
In Fig. \ref{fig:Valve_LinearComparisions}, we show the diagrams obtained based on four other pairs of sensors.
For each pair, we depict the \acrshort*{EVFD} obtained by \Cref{alg:EvfdCalculation} along the geodesic path on the left. We compare this diagram to the diagram obtained by marching along the linear path $\mathbf{L}(t)$, which is depicted on the right. This linear diagram is computed based on a variant of \Cref{alg:EvfdCalculation}, where we replace the geodesic path $\gamma(t_j)$ in \cref{eq:geodesic_grid} with the linear interpolation $\mathbf{L}(t_j)$.
Each point in the diagrams, representing an eigenvalue, is colored according to the correlation between the corresponding eigenvector the vector consisting of the (labels of the) valve condition.

We see in this figure that the diagram presented for a specific pair of sensors in Fig. \ref{fig:CM_ValveWithSineNoise_EVDiagram_Geodesic} consisting of visually distinct common and non-common eigenvalues (corresponding to log-linear and fast decaying trajectories, respectively), is indeed prototypical, and similar diagrams are obtained based on other pairs of sensors as well. In addition, we see that log-linear trajectories of eigenvalues, corresponding to common eigenvectors, are consistently informative.
As evident from the figure, these intriguing results (especially, the fast decay of the non-common eigenvalues) are much more apparent in the diagrams based on the geodesic path (on the left) compared to the diagrams based on the linear interpolation (on the right).

\begin{table}[th]
        \caption{The smoothness $S^{10}_{\mathbf{K}}(x)$ averaged over all pairs of sensors.}
        \label{tab:supp_cm_result}
        \setlength{\tabcolsep}{2pt}
        \begin{center}
        \begin{small}
        \begin{sc}
        \begin{tabular}{cccccc}
            \toprule
            Faults Type & Ours & Linear & AD  & NCCA & KCCA \\
            \midrule
           \makecell{Harmonic}    & $\mathbf{0.92 \pm0.04}$   &   $0.78 \pm0.23$   &   $0.82 \pm0.17$   &   $0.89 \pm0.04$ &   $0.89 \pm0.05$    \\
           \makecell{Sawtooth}    & $\mathbf{0.94 \pm0.06}$   &   $0.92 \pm0.06$   &   $0.91 \pm0.05$   &   $0.82 \pm0.07$ &   $0.92 \pm0.06$    \\
            \bottomrule
        \end{tabular}
        \end{sc}
        \end{small}
        \end{center}
\end{table}

\begin{figure*}[ht]\centering
	\begin{tabular}{ccc}
        \subfloat{\includegraphics[width=0.4\textwidth]{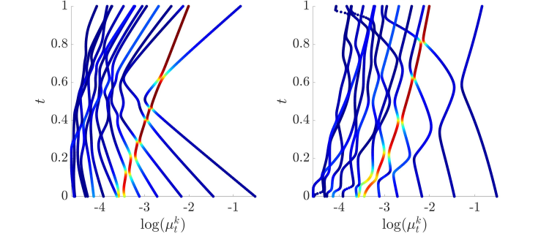}} &
        \subfloat{\includegraphics[width=0.4\textwidth]{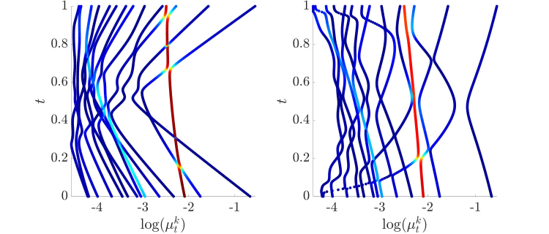}} &
        \multirow{2}{*}[3cm]{\subfloat{\includegraphics[height=0.28\textheight]{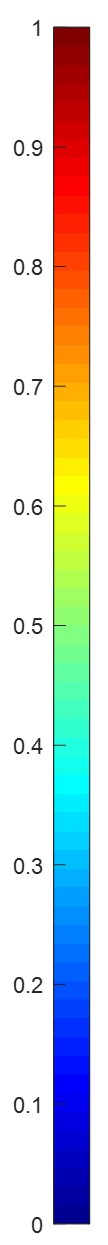}}} \\
        (a) & (b) &\\
        \subfloat{\includegraphics[width=0.4\textwidth]{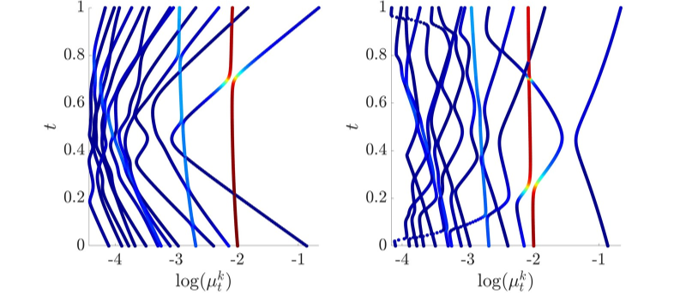}} &
        \subfloat{\includegraphics[width=0.4\textwidth]{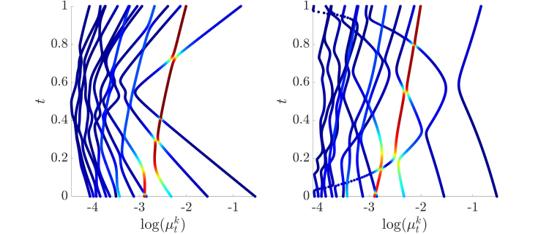}}  \\
        (c) & (d) &\\
	\end{tabular}
\caption{The proposed \acrshort*{EVFD}s (depicted on the left) compared to the diagrams based on the linear interpolation $\mathbf{L}(t)$ (depicted on the right) for different pairs of sensors. The color indicates the correlation between the valve condition and the corresponding eigenvector. (a) The pressure meter {\tt PS1} and the pressure meter {\tt PS2}. (b) The pressure meter {\tt PS3} and the flow meter {\tt FS1 }. (c) The flow meter {\tt FS1} and the system efficiency {\tt SE}. (d) The pressure meter {\tt PS2} and the pressure meter {\tt PS3}.}
	\label{fig:Valve_LinearComparisions}
\end{figure*}

The analysis so far was focused on the valve condition. Next, we demonstrate results with respect to the cooler.

In Fig. \ref{fig:CM_Cooler_EVDiagram}, we present the \acrshort*{EVFD} based on the pressure meter ({\tt PS5}) and the electrical power meter ({\tt EPS1}). 
This pressure meter ({\tt PS5}) is located on the circuit with the cooler. The electrical power meter ({\tt EPS1}) is located remotely, but monitors the power supply of the circuit with the cooler.
The insets show scatter plots of selected eigenvectors colored by the cooling efficiency.
Same as in Fig. \ref{fig:CM_ValveWithSineNoise_EVDiagram_Geodesic}, green arrows correspond to common eigenvectors and red arrows correspond to non-common eigenvectors. 
In the \acrshort*{EVFD}, we identify two near log-linear trajectories that correspond to two common eigenvectors. In the insets, we observe that these common eigenvectors are informative and strongly associated with the cooling efficiency. 
The \acrshort*{EVFD} also shows that the two common eigenvectors are more dominant in the electrical power meter (at $t=1$) than in the pressure meter (at $t=0$), arguably implying that the electrical power meter bears more information on the cooling efficiency.
In Fig. \ref{fig:CM_Cooler_Smoothness}, we present the $\ell$-truncated smoothness of the cooling efficiency w.r.t. $\gamma(t)$ and $\mathbf{L}(t)$. 
In Fig. \ref{fig:CM_Cooler_Smoothness}(A) we present it as function of $t$ for $\ell=10$, and in Fig. \ref{fig:CM_Cooler_Smoothness}(B) we present it as function of $\ell$ at $t^*$.
Observing Fig. \ref{fig:CM_Cooler_Smoothness}(A), we can see that the majority of the kernels along the geodesic path attains higher smoothness scores with respect to the corresponding kernels along the linear path. This behavior is even more significant at the middle of the paths. In addition, we can see that the point $t$ at which the CMR is maximal (marked by a horizontal red line) provides a good approximation, in an unsupervised manner, for the point that maximizes the smoothness score.
Observing the variations of the smoothness scores as function of $\ell$ in Fig. \ref{fig:CM_Cooler_Smoothness}(B), we see that the principal eigenvector of $\gamma(t^*)$ captures the cooling efficiency in contrast to the principal eigenvector of $\mathbf{L}(t^*)$.

\begin{figure}[t]\centering
		\includegraphics[width=0.65\textwidth]{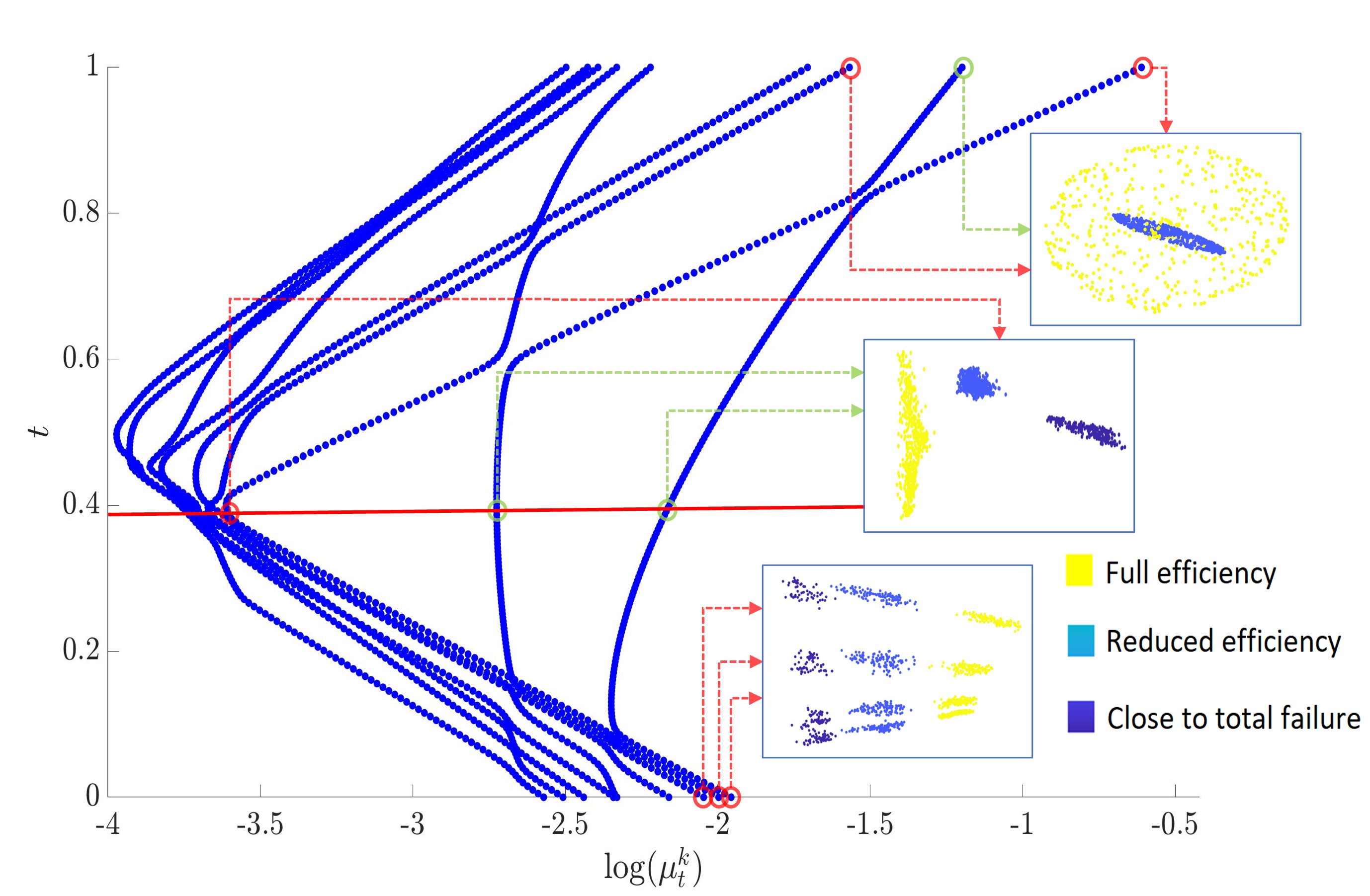}
	\caption{
	The \acrshort*{EVFD} obtained based on measurements acquired by the pressure meter ({\tt PS5}) and the electrical power meter ({\tt EPS1}).
	The insets display scatter plots of eigenvectors corresponding to the circled eigenvalues. 
	The points in the scatter plots are colored by the cooling efficiency.
	}
	\label{fig:CM_Cooler_EVDiagram}
\end{figure}

\begin{figure}[t]\centering
	\begin{tabular}{cc}
		\includegraphics[width=0.35\textwidth]{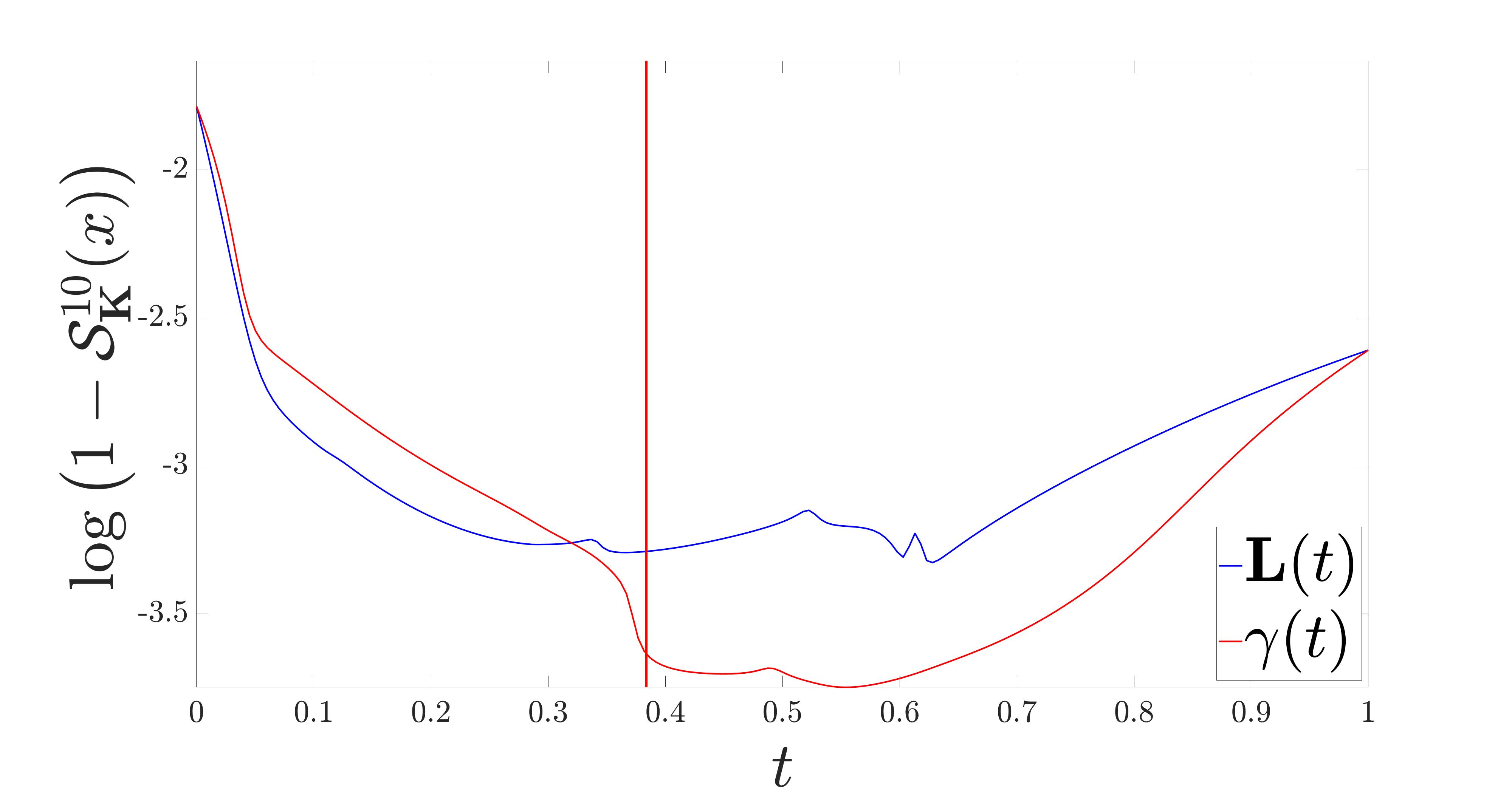}&
		\includegraphics[width=0.35\textwidth]{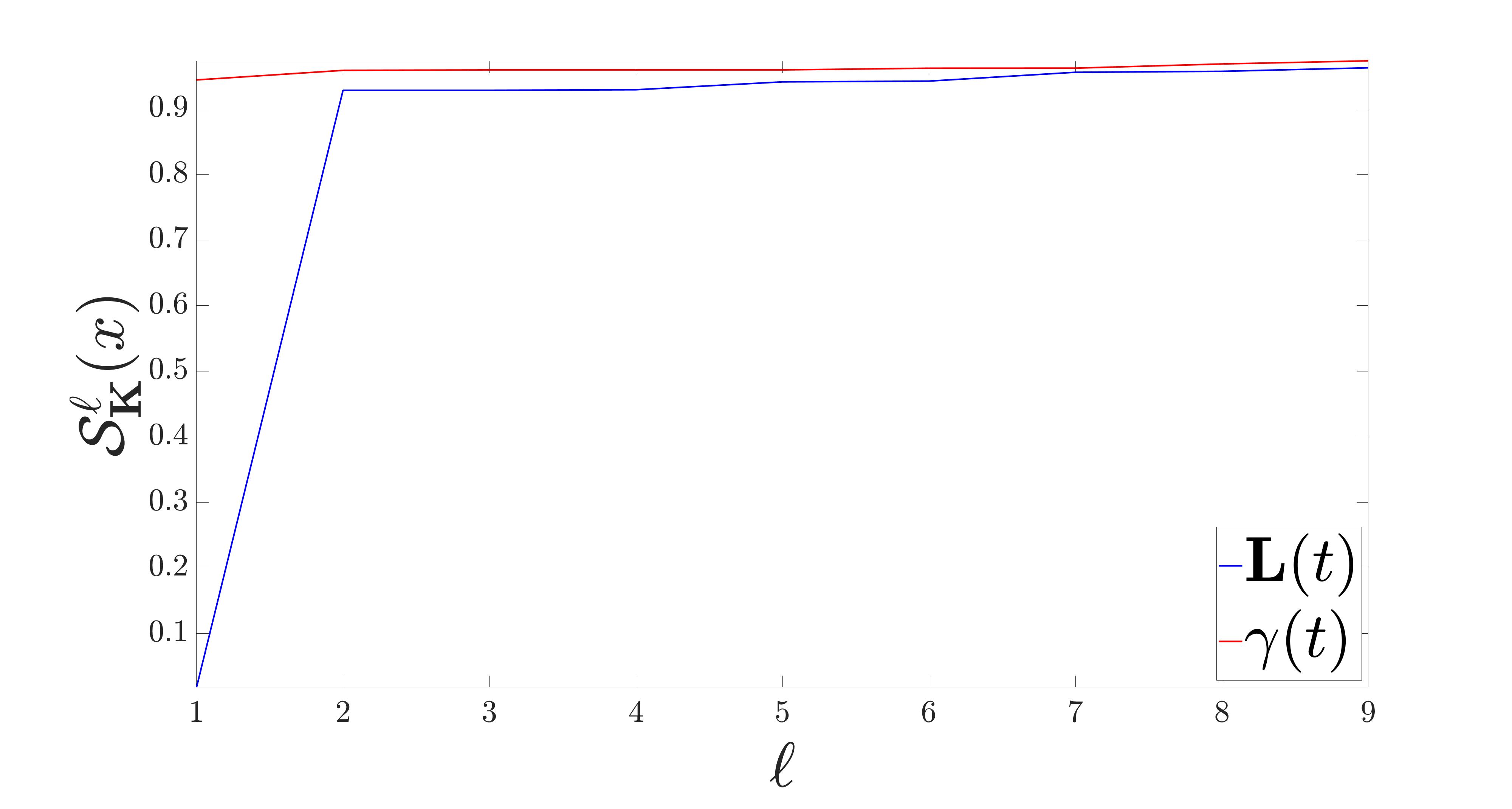}\\
		(a) & (b)
	\end{tabular}
	\caption{ (a) The $\ell$-truncated smoothness of the cooling efficiency w.r.t. $\gamma(t)$  and $\mathbf{L}(t)$ as function of $t$ for $\ell=10$.
	(b) The $\ell$-truncated smoothness of the cooling efficiency as function of $\ell$ at $t^*$ w.r.t. $\gamma(t)$ and $\mathbf{L}(t)$.
	}
	\label{fig:CM_Cooler_Smoothness}
\end{figure}

\subsection{Artificial olfaction for gas identification}
\label{subsec:sup_ArtificialOlfaction}

\begin{table}[t]
        \caption{The smoothness $S^5_{\mathbf{K}}(x)$ averaged over all pairs of sensors.}
        \label{tab:supp_gas_results}
            \begin{center}
            \setlength{\tabcolsep}{1.5pt}
            \begin{small}
            \begin{sc}
            \begin{tabular}{lccccc}
                \toprule
                Analyte & Ours & Linear & AD  & NCCA & KCCA \\
                \midrule
                Ammonia         &   $\mathbf{0.78 \pm0.08}$  &   $0.76\pm0.07$   &   $0.71\pm0.06$   &   $0.28\pm0.2$   &   ${0.77\pm0.08}$  \\
                Acetaldehyde    &   $\mathbf{0.78 \pm0.07}$  &   $0.68\pm0.22$   &   $0.55\pm0.33$   &   $0.54\pm0.24$   &   ${0.74\pm0.12}$  \\
                Acetone         &   $\mathbf{0.83 \pm0.03}$  &   $0.81\pm0.06$   &   $0.64\pm0.22$   &   $0.56\pm0.11$   &   $0.8\pm0.04$  \\
                Ethanol         &   $0.23 \pm0.17$           &   $0.17\pm0.15$   &   $0.13\pm0.15$   &   $\mathbf{0.34\pm0.12}$   &   $0.17\pm0.15$  \\
                Ethylene        &   $\mathbf{0.75 \pm0.08}$  &   $0.72\pm0.07$   &   $0.62\pm0.27$   &   $0.41\pm0.15$   &   $0.73\pm0.07$  \\
                Toluene         &   $\mathbf{0.54 \pm0.01}$  &   $0.51\pm0.03$   &   $0.53\pm0.01$   &   $0.48\pm0.04$   &   $0.53\pm0.01$  \\
                \midrule 
                Mean &   $\mathbf{0.65 \pm0.08}$  &   $0.61\pm0.1$   &   $0.53\pm0.17$   &   $0.43\pm0.14$   &   $0.62\pm0.08$  \\
                \bottomrule
            \end{tabular}
            \end{sc}
            \end{small}
            \end{center}
\end{table}

In this section, we extend the experimental study presented in \Cref{subsec:ArtificialOlfaction}.
\cref{tab:supp_gas_results} is the same as \cref{tab:gas_results} but with the standard deviation in addition to the mean of the smoothness $S^5_{\mathbf{K}}(x)$.

In Fig. \ref{fig:Gas_LinearComparisions}, similarly to Fig. \ref{fig:Valve_LinearComparisions}, we show the proposed \acrshort*{EVFD}s obtained by four other pairs of sensors and compare them to the diagram obtained by the linear interpolation $\mathbf{L}(t)$. We note that we only combine sensors of different types (there are 16 sensors in the array: four types of sensors, and four sensors of each type). 
Same as in Fig. \ref{fig:Valve_LinearComparisions}, we see that in the proposed \acrshort*{EVFD}s, the common eigenvectors are more dominant for $t \approx 0.5$ and that the eigenvalues corresponding to the non-common eigenvectors decay faster compared to the diagrams based on the linear interpolation.
In addition, we see that the identified common eigenvectors are typically informative, as they are highly correlated with the (hidden) has concentration (indicated by the color). 

\begin{figure*}[ht]\centering
	\begin{tabular}{ccc}
        \subfloat{\includegraphics[width=0.4\textwidth]{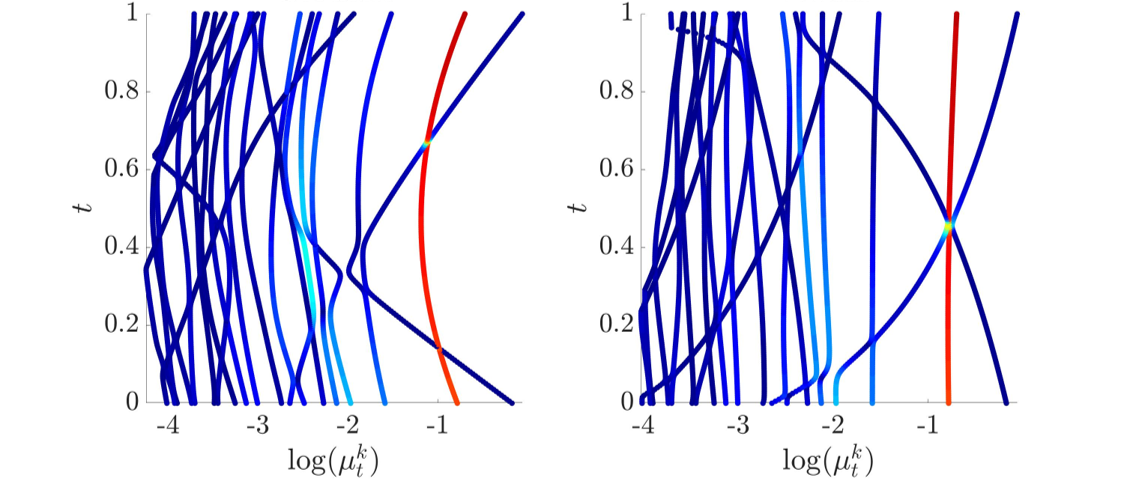}} &
        \subfloat{\includegraphics[width=0.4\textwidth]{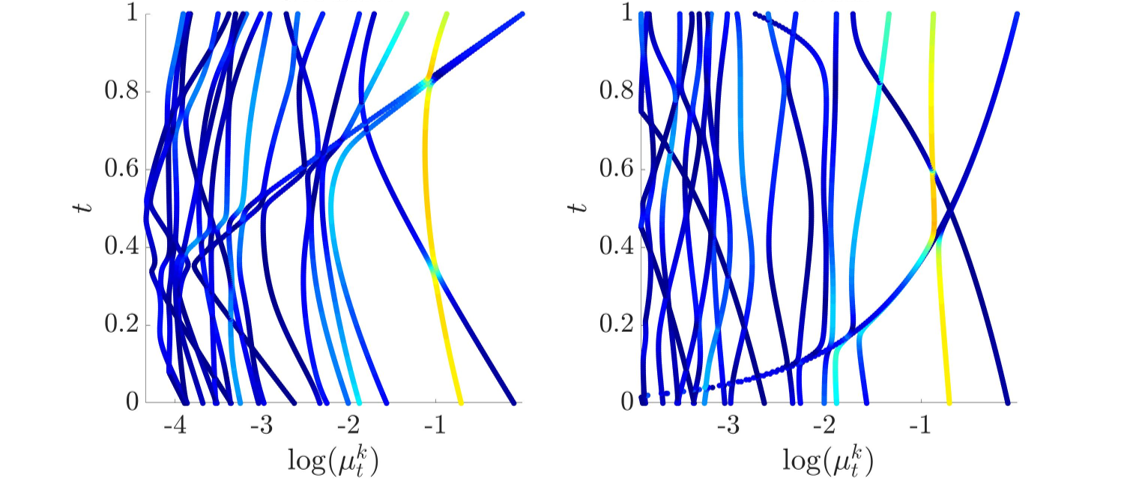}} &
        \multirow{2}{*}[3cm]{\subfloat{\includegraphics[height=0.28\textheight]{FiguresLinearComparisons/Colorbar01.jpg}}} \\
        (a) & (b) &\\ 
        \subfloat{\includegraphics[width=0.4\textwidth]{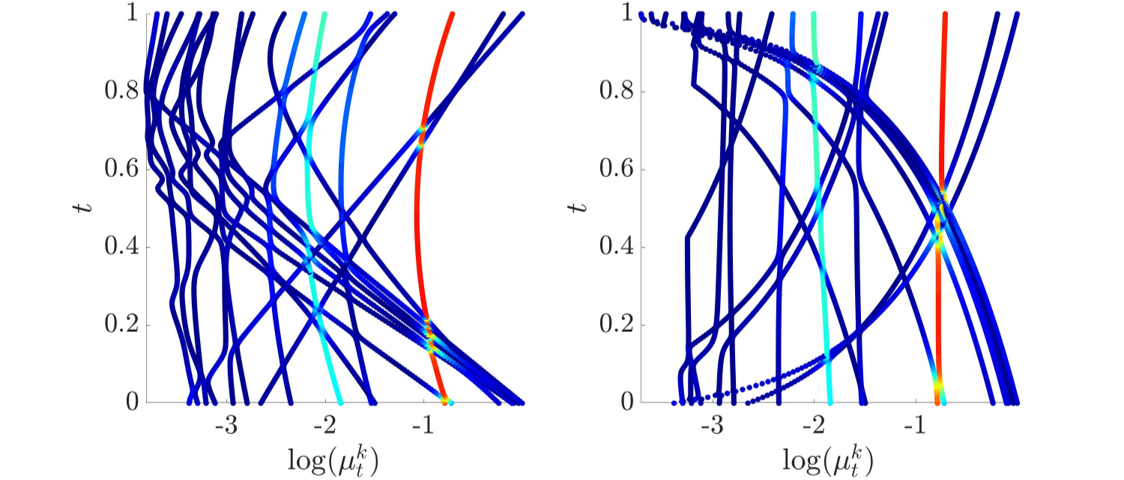}} &
        \subfloat{\includegraphics[width=0.4\textwidth]{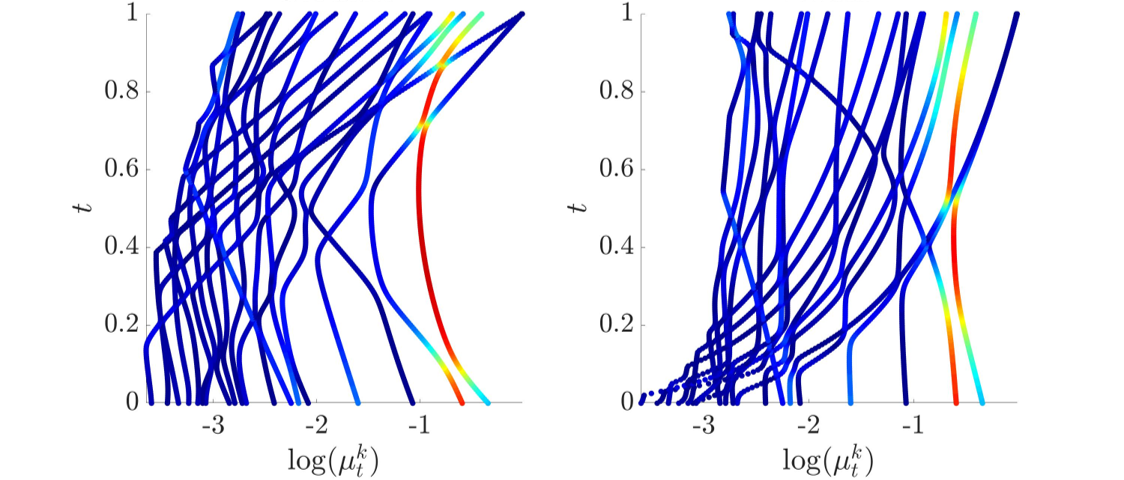}}  \\
        (c) & (d) 
	\end{tabular}\vspace{-0.2cm}
\caption{The proposed \acrshort*{EVFD}s (depicted on the left) compared to the diagrams based on the linear interpolation $\mathbf{L}(t)$ (depicted on the right) for different pairs of sensors. The color indicates the correlation between the gas concentration and the corresponding eigenvector. (a) sensor 5 and sensor 7 colored according to the Acetaldehyde concentration. (b) sensor 4 and sensor 14 colored according to the Ethylene concentration. 	(c) sensor 9 and sensor 13 colored according Ethanol concentration. (d) sensor 14 and sensor 16 colored according to the Ammonia concentration.} 
 	\label{fig:Gas_LinearComparisions}
\end{figure*}

\subsection{Quantification using polynomial fit}
\label{subsec:polyfit} 

In this section, we present an additional quantitative evaluation metric based on an unnormalized variant of the diffusion distance \eqref{eq:diff_dist} proposed in \cite{lederman2018learning}. Specifically, the unnormalized diffusion distance between the $i$th and $j$th samples with respect to the kernel matrix $\mathbf{K}$ is given by:
\begin{equation}\label{eq:unnorm_d_dist}
    d_{\mathbf{K}}(i,j) = \| \boldsymbol{\delta}_i^T \mathbf{A}^t - \boldsymbol{\delta}_j^T \mathbf{A}^t \|_2,
\end{equation}
where $\mathbf{A}=\mathbf{D}^{-1}\mathbf{K}$, $\mathbf{D}=\text{diag}(\mathbf{K} \mathbf{1})$, and $\boldsymbol{\delta}_i$ and $\boldsymbol{\delta}_j$ are $n$-dimensional one-hot vectors, whose $i$th and $j$th elements equal $1$, respectively, and all other elements equal $0$.
The distance in \eqref{eq:unnorm_d_dist} could be viewed as the Euclidean distance between two `masses' after $t$ steps of a Markov chain propagation initialized at $x_i$ and $x_j$. 

The evaluation metric is a quantification of the correspondence between $d_{\mathbf{K}}(i,j)$ and Euclidean distance between the latent samples $\|x_i - x_j\|_2$ using a polynomial fit.
First, we compute all the pairwise distances $d_{\mathbf{K}}(i,j)$ (according to \cref{eq:diff_dist}) and $\|x_i - x_j\|_2$ for $i,j=1,\ldots,n$.
Then, we apply a polynomial fit:
\begin{equation}\label{eq:polyfit}
    \underset{p}{\operatorname{min}}{\sum_{i,j=1}^n \bigg(  \|x_i-x_j\|_2-p(d_{\mathbf{K}}(i,j))  \bigg)^2}
\end{equation}
where $p(\cdot)$ denotes a polynomial of order $P$. 

We demonstrate this metric on the puppets example from \cref{subsec:ToyExample}. 
We compute the polynomial fit using $P=3$ between the diffusion distances \eqref{eq:unnorm_d_dist} and the distances between the angles of Bulldog. We compare the fit obtained by the diffusion distances induced by the kernels $\gamma(t), \mathbf{L}(t)$, and the alternating diffusion kernel $\mathbf{K}_{\text{AD}}$ \cite{lederman2018learning}.
The results are presented in Fig. \ref{fig:Puppets_Polyfit}, where we can see that $\gamma(t)$ achieves a better polynomial fit compared to $\mathbf{L}(t)$ and $\mathbf{K}_{\text{AD}}$ for $0.3<t<0.7$. Indeed, these values of $t$ coincide with the values of $t$ in the diagram in Fig. \ref{fig:Puppets_EVDiagrams_Geodesic}, where the eigenvalues on straight lines become dominant.
In order to further illustrate the obtained scores, in Fig. \ref{fig:PuppetsPolyFittWithADSingle}, we present a visualization of the correspondence between the distances by displaying scatter plots of $\|x_i-x_j\|_2$ versus $d_{\mathbf{K}}(i,j)$) and superimpose the polynomial fit using $P=3$ (marked by a red curve).
We compare the correspondence obtained w.r.t. the kernels $\gamma(t^*)$, $\mathbf{L}(t^*)$, and $\mathbf{K}_{\text{AD}}$. 
We see that the correspondence obtained by the proposed method with the kernel $\gamma(t^*)$ is (visually) superior compared to the correspondence obtained by the linear interpolation with the kernel $\mathbf{L}(t^*)$ and by AD with the kernel $\mathbf{K}_{\text{AD}}$.
\vspace{-0.1cm}
\begin{figure}[H]\centering \vspace{-0.2cm}
    	 \includegraphics[width=0.3\textwidth]{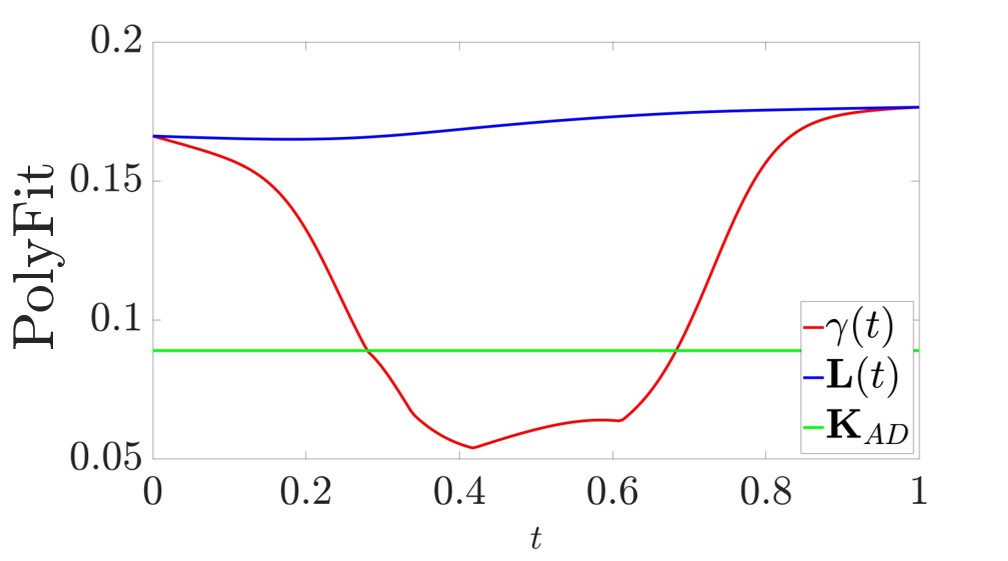}\vspace{-0.2cm}
    \caption{The polynomial fit between the diffusion distances obtained by $\gamma(t)$, $\mathbf{L}(t)$, and $\mathbf{K}_{\text{AD}}$, and the (Euclidean) distances between the angles of Bulldog. }
    \label{fig:Puppets_Polyfit}
\end{figure}
\begin{figure}[H]\centering \vspace{-0.2cm}
    \begin{tabular}{ccc}
	 \includegraphics[width=0.2\textwidth]{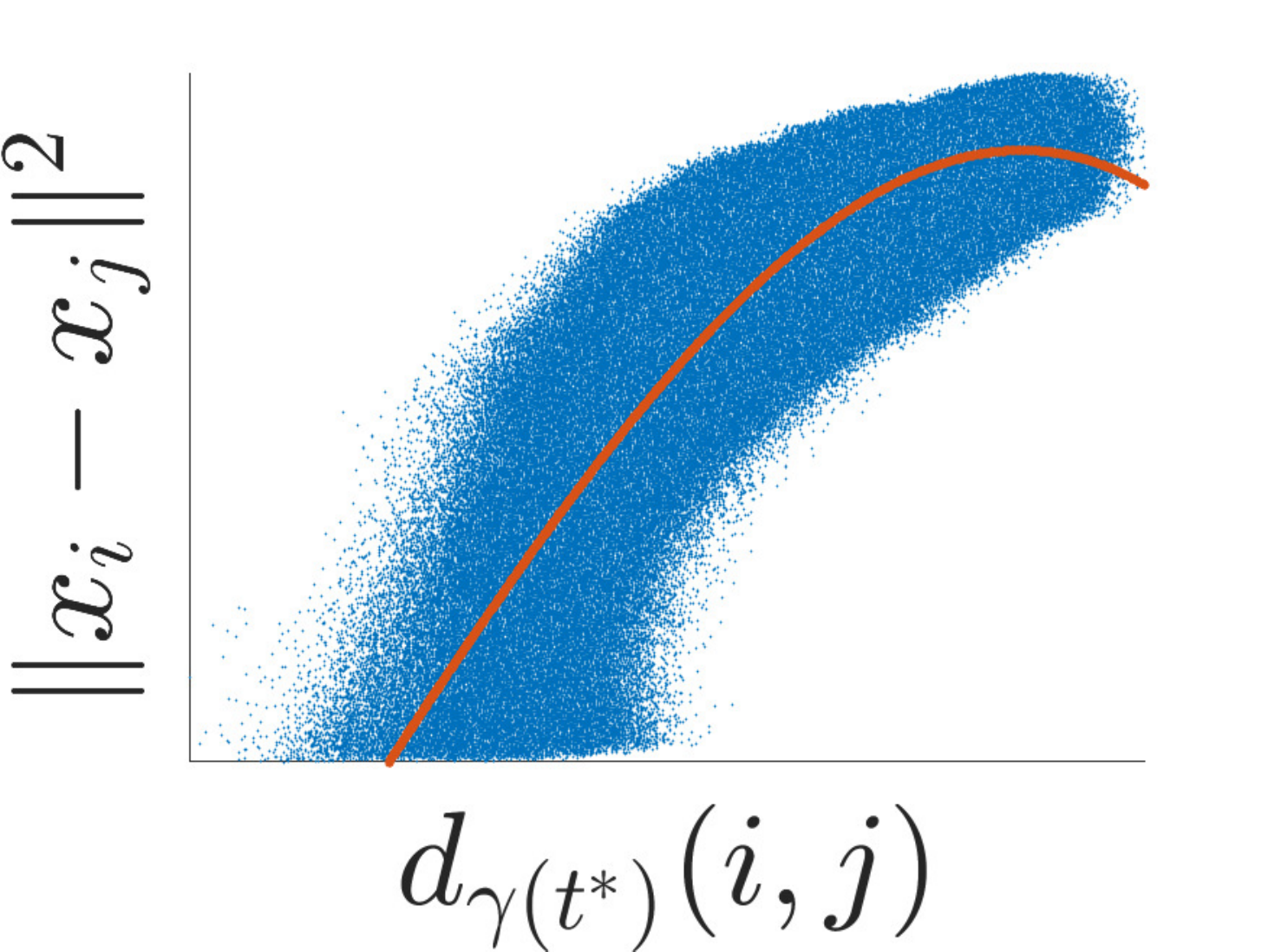}&
	 \includegraphics[width=0.2\textwidth]{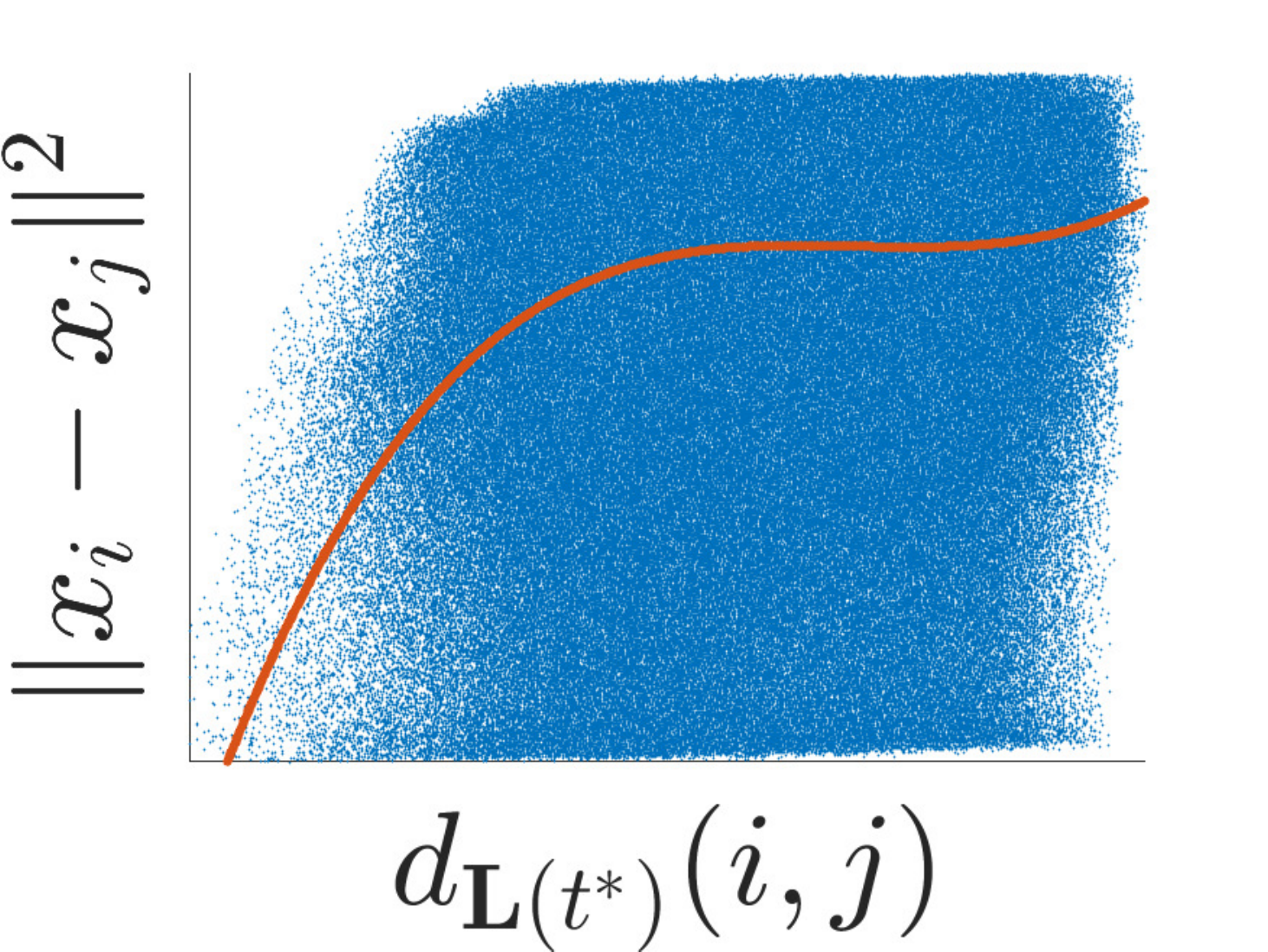}&
	 \includegraphics[width=0.2\textwidth]{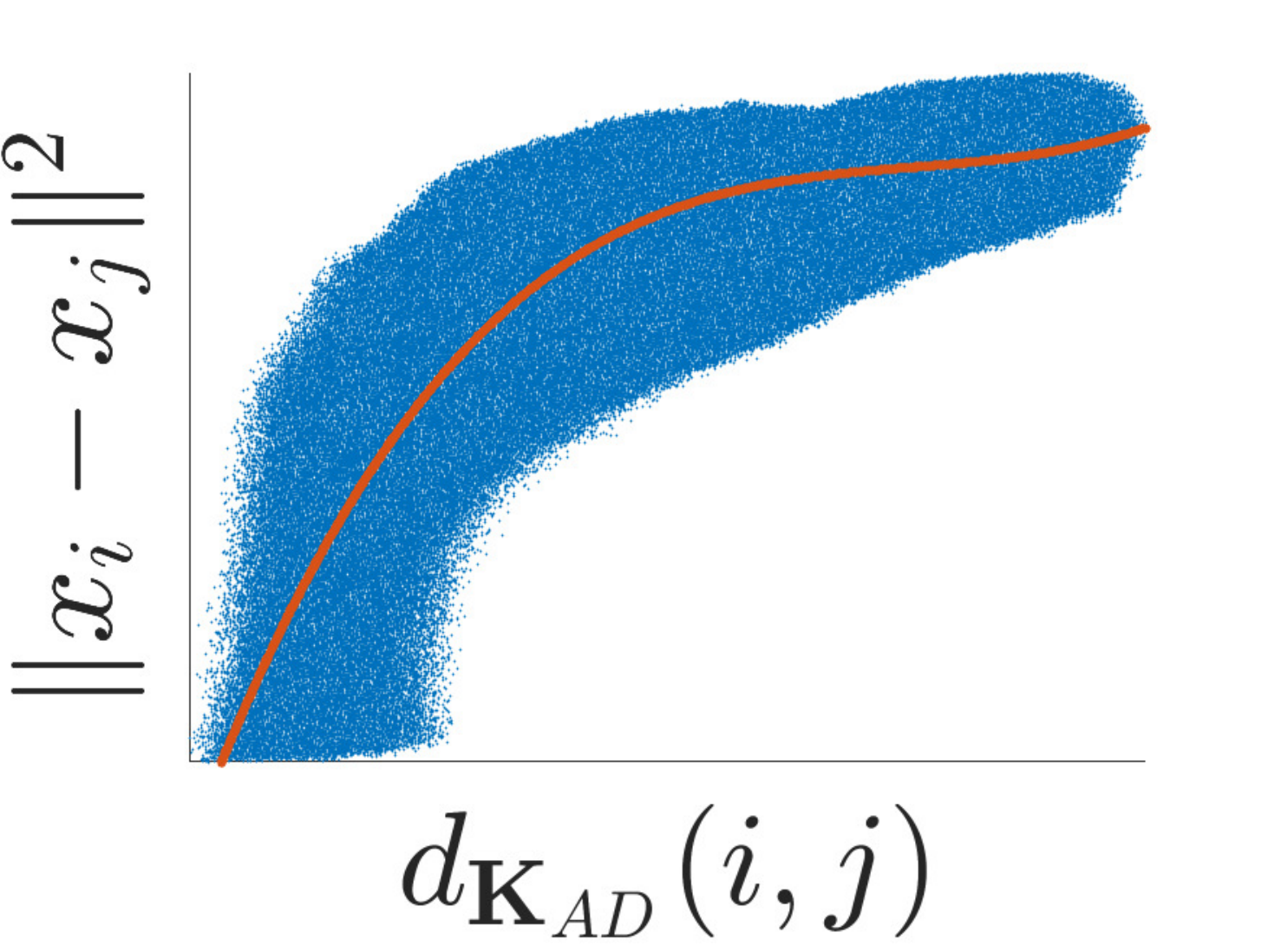}\\
	 (a) & (b) & (c)
	 \end{tabular} \vspace{-0.2cm}
\caption{Scatter plots of $\|x_i-x_j\|_2$ versus $d_{\mathbf{K}}(i,j)$ for $i,j=1,\ldots,n$. The red curve marks the polynomial fit with polynomial of degree $P=3$. In (a) $\mathbf{K}=\gamma(t^*)$, in (b) $\mathbf{K}=\mathbf{L}(t^*)$ and in (c) $\mathbf{K}=\mathbf{K}_{\text{AD}}$.}
\label{fig:PuppetsPolyFittWithADSingle}
\end{figure}

\clearpage

\section{Simulations results}
\label{sec:sup_SimuResults}
In this appendix, we demonstrate our proposed approach on simulated data. 
In \Cref{subsec:sup_convex_hull}, this appendix is concluded with preliminary results demonstrating the ability of the proposed method to handle more than two sets of measurements.

\subsection{2D flat manifolds}
\label{sec:sup_SynthData}

Consider three 2D flat manifolds $\mathcal{M}_x = \mathcal{M}_y = \mathcal{M}_z = [-1/2,1/2]$, and consider tuples $\left\{(x_i,y_i,z_i)\right\}_{i=1}^n$, which are $n$ realizations sampled from some joint distribution on the product manifold $\mathcal{M}_x \times \mathcal{M}_y \times \mathcal{M}_z$. These realizations are measured in the following manner:
\begin{align}\label{eq:2d_flat_measurements}
	s^{(1)}_i = \left( \ell^{(1)}_x x_i , \ell^{(1)}_{y}y_i  \right) \in \mathbb{R}^2 \nonumber \\ 
	s^{(2)}_i = \left( \ell^{(2)}_x x_i , \ell^{(2)}_{z}z_i  \right) \in \mathbb{R}^2 
\end{align}
where $\ell^{(1)}_x,\ell^{(2)}_x,\ell^{(1)}_y,\ell^{(2)}_z$ are positive scaling parameters.
The scaling parameters enable us to control the relative dominance of each of the latent variables on the measurements.
Note that if $\ell^{(1)}_x \neq \ell^{(1)}_y$ or $\ell^{(2)}_x \neq \ell^{(2)}_z$ the embedding to the observable spaces is not isometric.

This example is simple, yet of interest, since it is tractable because the eigenvalues of the Laplace-Beltrami operators of the considered manifolds have closed-form expressions.
Specifically, the eigenvalues of the Laplace-Beltrami operator with Neumann boundary conditions on the manifolds $\mathcal{O}_1= [-\ell^{(1)}_x/2,\ell^{(1)}_x/2] \times [-\ell^{(1)}_y/2,\ell^{(1)}_y/2]$ and $\mathcal{O}_2= [-\ell^{(2)}_x/2,\ell^{(2)}_x/2] \times [-\ell^{(2)}_z/2,\ell^{(2)}_z/2]$ are given by (see Equation (2) in \cite{dsilva2018parsimonious}):
\begin{align}\label{eq:2d_flat_anal_eigv}
\lambda_{1}^{(k_x,k_y)}=\left( \frac{k_x\pi}{\ell^{(1)}_x} \right)^2 + \left( \frac{k_y\pi}{\ell^{(1)}_y} \right)^2 \nonumber \\
\lambda_{2}^{(k_x,k_z)}= \left( \frac{k_x\pi}{\ell^{(2)}_x} \right)^2 + \left( \frac{k_z\pi}{\ell^{(2)}_z} \right)^2,
\end{align}
respectively, where $k_x,k_y,k_z=0,1,2,...$. 
Evidently, the common eigenvalues, i.e. the eigenvalues related only to $\mathcal{M}_x$, are given by $\{\lambda_{1}^{(k_x,k_y=0)}\}_{k_x=0}^{\infty}$ and  $\{\lambda_{2}^{(k_x,k_z=0)}\}_{k_x=0}^{\infty}$.

In this simulation we set the following values to the scaling parameters:
\begin{align*}
\ell^{(1)}_x=2 ,\ell^{(2)}_x=2\\
\ell^{(1)}_{y}=8,\ell^{(2)}_{z}=4,
\end{align*}
making the common manifold more dominant in $\{s^{(2)}_i\}_{i=1}^n$ than in $\{s^{(1)}_i\}_{i=1}^n$ with respect to the corresponding measurement-specific manifolds.
In the sequel, we will present additional results for different choices of the scaling parameters and non-uniform distributions.
We generate $n=1,000$ points $(x_i,y_i,z_i)$ where $x_i \in \mathcal{M}_x$, $y_i \in \mathcal{M}_y$, and $z_i \in \mathcal{M}_z$ are sampled uniformly and independently from each manifold. Each hidden point $(x_i,y_i,z_i)$ gives rise to two (aligned) measurements: $s^{(1)}_i$ and $s^{(2)}_i$ according to \eqref{eq:2d_flat_measurements}.
We apply \Cref{alg:EvfdCalculation} to the two sets of $1,000$ measurements, $\{s^{(1)}_i\},\{s^{(2)}_i\}$, and obtain the \acrshort*{EVFD}. We set the number of points on the geodesic path to $N_t=200$ and the number of eigenvalues to $K=20$. Based on the diagram, using \Cref{alg:COR2}, we estimate the CMR at each point along the geodesic path. We compare the obtained \acrshort*{EVFD} to the diagram obtained based on the linear interpolation $\mathbf{L}(t)$ defined in \eqref{eq:LinearInterpolation} (following the procedure described in \Cref{subsec:sup_CM} where we replace the geodesic path $\gamma(t_j)$ in \cref{eq:geodesic_grid} with the linear interpolation $\mathbf{L}(t_j)$). 
The two diagrams are depicted in Fig. \ref{fig:3DStrip_ModerateSNR_COR}, where the point $t^*$ along the geodesic path at which the CMR estimation is maximal is marked by a horizontal red line. 

By comparing the two diagrams in Fig. \ref{fig:3DStrip_ModerateSNR_COR}, we observe that 
the two diagrams look different, especially in terms of the trajectories of the non-common eigenvalues and the number of leading common eigenvalues. For instance, at $t=0.5$, the logarithm of the largest eigenvalue corresponding to the non-common eigenvalues in the geodesic flow is approximately $-3.5$ whereas in the linear flow it is approximately $-0.7$. In addition, at $t=0.5$ there are $4$ principal common eigenvalues in the geodesic flow, whereas in the linear flow there is only $1$. 
We can also see that the highest $\text{CMR}$ is not obtained at the middle of the geodesic path, but closer to the point $t=1$. This coincides with the fact that $\gamma(1)=\mathbf{K}_2$ is the kernel of the set of measurements $\{s^{(2)}_i\}$, where the common manifold is more dominant.

We analytically compute the eigenvalues of $\mathbf{K}_1=\gamma(0)$ and $\mathbf{K}_2=\gamma(1)$ (by computing the eigenvalues of the counterpart continuous operators according to \eqref{eq:2d_flat_anal_eigv} and then using \eqref{eq:Cont2Discrete}), and we overlay only the common eigenvalues on the diagram in Fig. \ref{fig:3DStrip_ModerateSNR_COR} at the boundaries $t=0$ and $t=1$ (marked by squares). 
Indeed, we observe that the empirical eigenvalues lying on log-linear trajectories coincide with the common eigenvalues at the boundaries (the squares). Note that some trajectories discontinue before reaching the boundaries because we only plot the top $K$ eigenvalues at each vertical coordinate $t$.

\begin{figure}[t]\centering
	\includegraphics[width=1\textwidth]{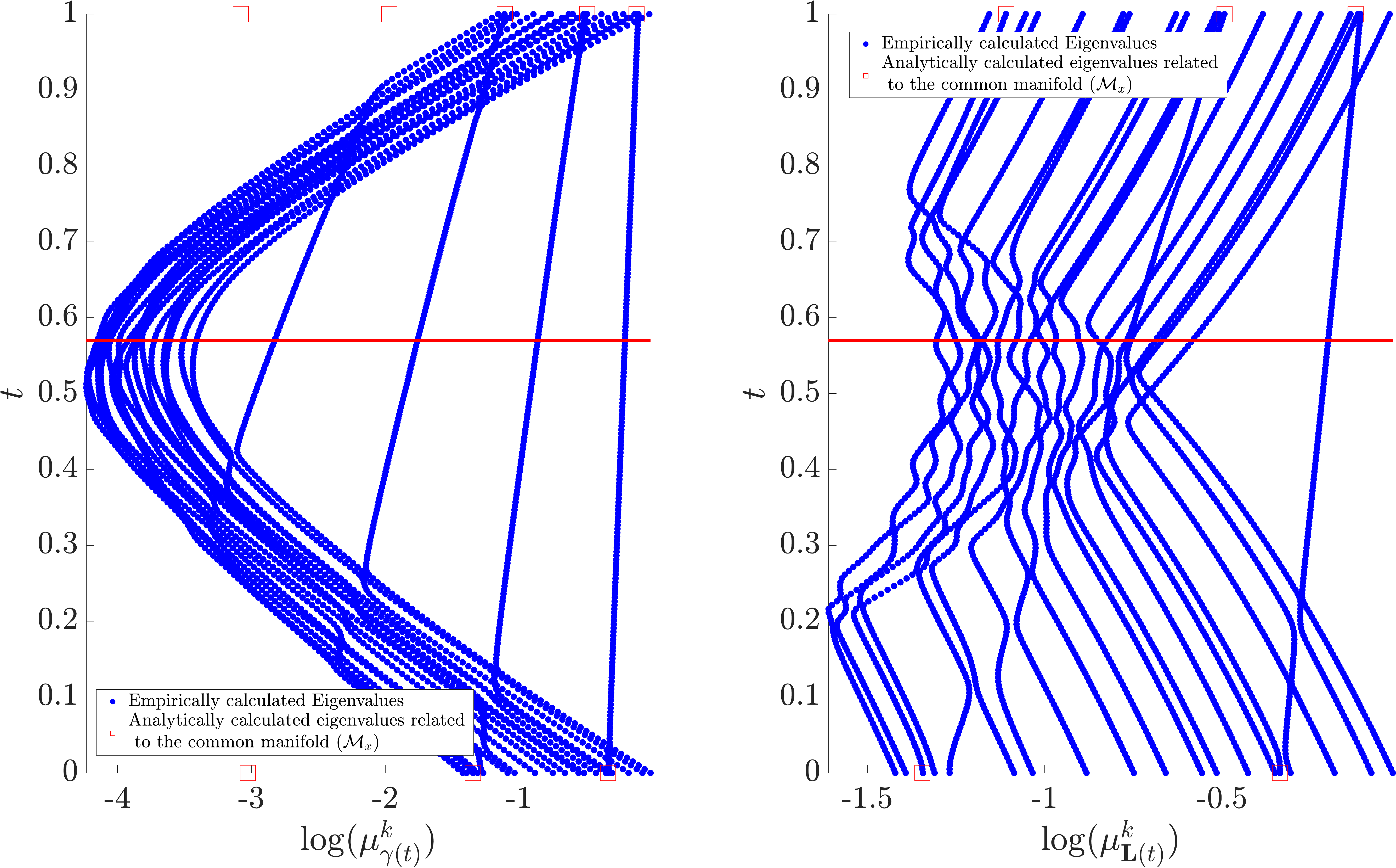}
	\caption{Left: the \acrshort*{EVFD} obtained by \Cref{alg:EvfdCalculation} in the 2D flat manifolds example. Right: the diagram based on the linear interpolation. The point $t$ at which the CMR is maximal is marked by a horizontal red line. The analytically calculated eigenvalues at the boundaries of the diagram (at $t=0$ and $t=1$) are marked by red squares.}
	\label{fig:3DStrip_ModerateSNR_COR}
\end{figure}

To demonstrate the extraction of the common manifold, we show the diffusion propagation resulting from $\gamma(t)$ at different $t$ values, which is computed as follows.
Let $i_0$ be the index of the closest measurement $s_i^{(1)}$ to $(0,0)$, and let $\boldsymbol{\delta}_{i_0} \in \mathbb{R}^{1,000}$ be a vector of all-zeros except the $i_0$th entry which equals $1$. This vector represents a mass on the measurements $\{s_i^{(1)}\}$ (and on $\{s_i^{(2)}\}$) initially concentrated entirely at $s_{i_0}^{(1)}$.
Similarly to the diffusion notion in diffusion maps \cite{coifman2006diffusion} described in \Cref{subsec:dm_doubly}, we use $\gamma(t)$ to propagate this mass. Specifically, the propagated mass after $s \in \mathbb{N}$ propagation steps is given by $\boldsymbol{\delta}_{i_0}^T [\mathbf{A}_{\gamma(t)}]^s$, where $\mathbf{A}_{\gamma(t)}$ is the row-stochastic kernel associated with $\gamma(t)$, and $[\mathbf{A}_{\gamma(t)}]^s$ is the matrix $\mathbf{A}_{\gamma(t)}$ to the power of $s$.

In Fig. \ref{fig:3DStrip_ModerateSNR_Diffs}, we present such propagated masses after $s=2$ steps using four propagation matrices along the geodesic path, namely, we compute $\boldsymbol{\delta}_{i_0}^T[\mathbf{A}_{\gamma(t)}]^s$ for $t=0,0.3,t^*,1$, where $t^* \simeq 0.6$ is the point at which the CMR is maximal. Each propagated mass is a vector of length $1000$ (the number of measurements) and is used to color the measurements. 

We observe that at $t=0$ the mass in $\mathcal{O}_1$ propagates both vertically and horizontally in a rate that is proportional to the scale of the axes, whereas the propagation of the mass in $\mathcal{O}_2$ is incoherent. This coincides with the fact that at $t=0$ the mass is propagated by $\gamma(0)=\mathbf{K}_1$, which is built based on measurements in $\mathcal{O}_1$. A similar trend is observed at $t=1$, where mass propagation proportional to the scale of the axes is observed in $\mathcal{O}_2$ and incoherent mass propagation is observed in $\mathcal{O}_1$, complying with $\gamma(1)=\mathbf{K}_2$ that is based on measurements in $\mathcal{O}_2$.
When marching on the geodesic path from $t=0$ to $t=1$, we observe that the mass propagation changes. Specifically, at $t^*$, the mass propagation both in $\mathcal{O}_1$ and in $\mathcal{O}_2$ saturates along the vertical axis, which represents the measurement-specific manifolds $\mathcal{M}_y$ and $\mathcal{M}_z$. %
This implies that using $\gamma(t^*)$ to propagate masses initially concentrated at different measurements depends only on the initial horizontal coordinate (the common manifold) and ignores the initial vertical coordinate (the measurement-specific manifold). Furthermore, considering distances between such propagated masses, as in the diffusion distance defined in \eqref{eq:diff_dist}, leads to a distance that compares measurements in terms of their common manifold. Concretely, denote $\mathbf{A}_{\gamma(t^*)}$ as the row-stochastic kernel associated with $\gamma(t^*)$, the diffusion distance
\begin{equation}\label{eq:new_ad_distance}
    d_{\gamma(t^*)}(i_1,i_2) =\| \boldsymbol{\delta}_{i_1}^T\mathbf{A}_{\gamma(t^*)}^s - \boldsymbol{\delta}_{i_2}^T\mathbf{A}_{\gamma(t^*)}^s\|_2
\end{equation}
depends only on the horizontal coordinate of $s^{(1)}_{i_1}$ (or $s^{(2)}_{i_1}$) and $s^{(1)}_{i_2}$ (or $s^{(2)}_{i_2}$).
This notion of distance with the AD kernel was extensively explored in \cite{lederman2018learning}. In \Cref{subsec:sup_AD} we elaborate more on AD and further describe the relationship between our method and AD.

We note that comparing the mass propagation obtained by $\gamma(t^*)$ with the mass propagation obtained by $\gamma(0)$,$\gamma(0.3)$ and $\gamma(1)$ demonstrates that a proper choice of $t^*$ along the geodesic path is important. 

\begin{figure}[t]\centering
	\begin{tabular}{ccc}
	    & $\mathcal{O}_1$ & $\mathcal{O}_2$\\ \\
		$t=0$ & \hspace{-0in}  \centered{\includegraphics[scale=0.1]{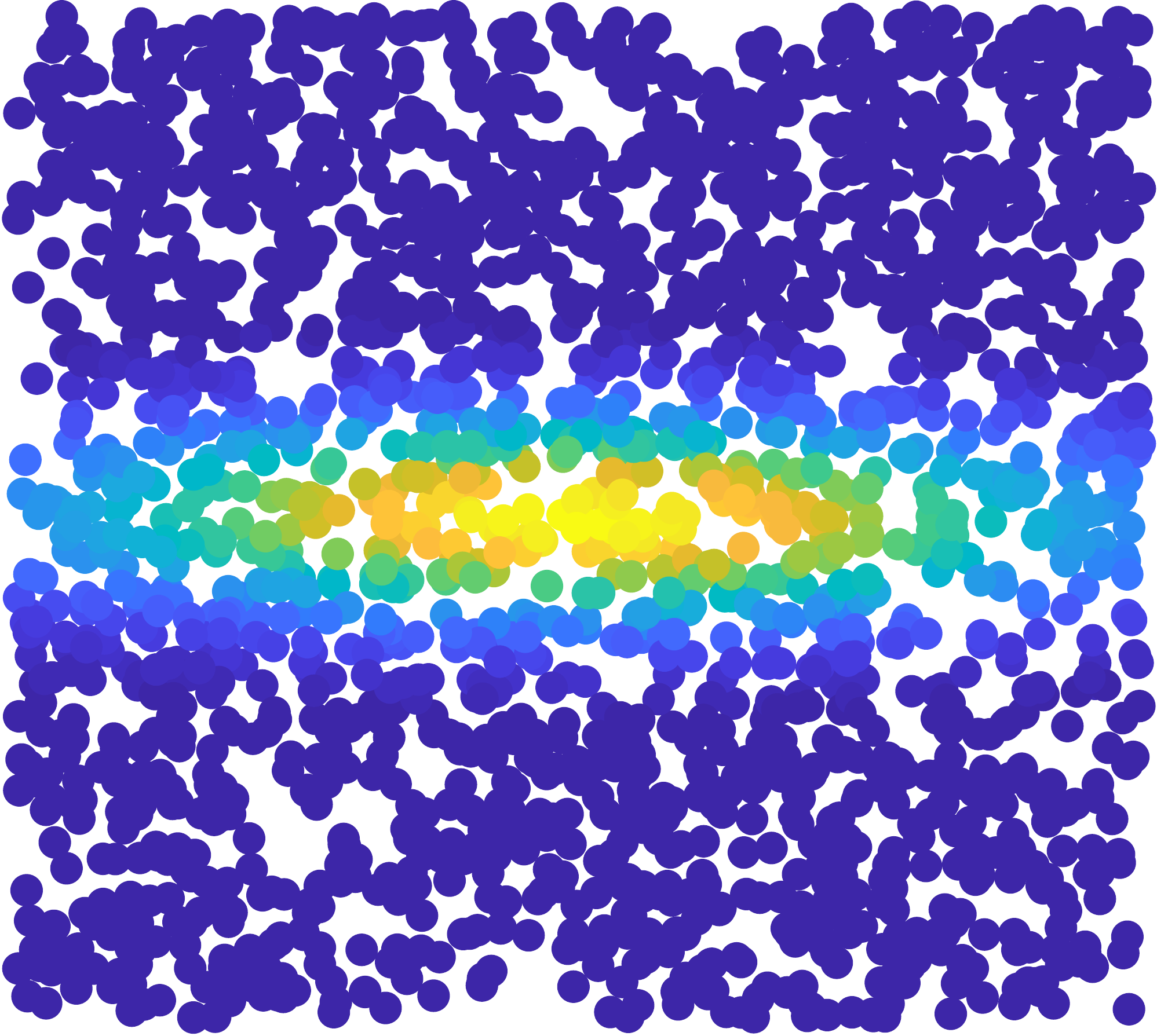}} &
		 \centered{\includegraphics[scale=0.1]{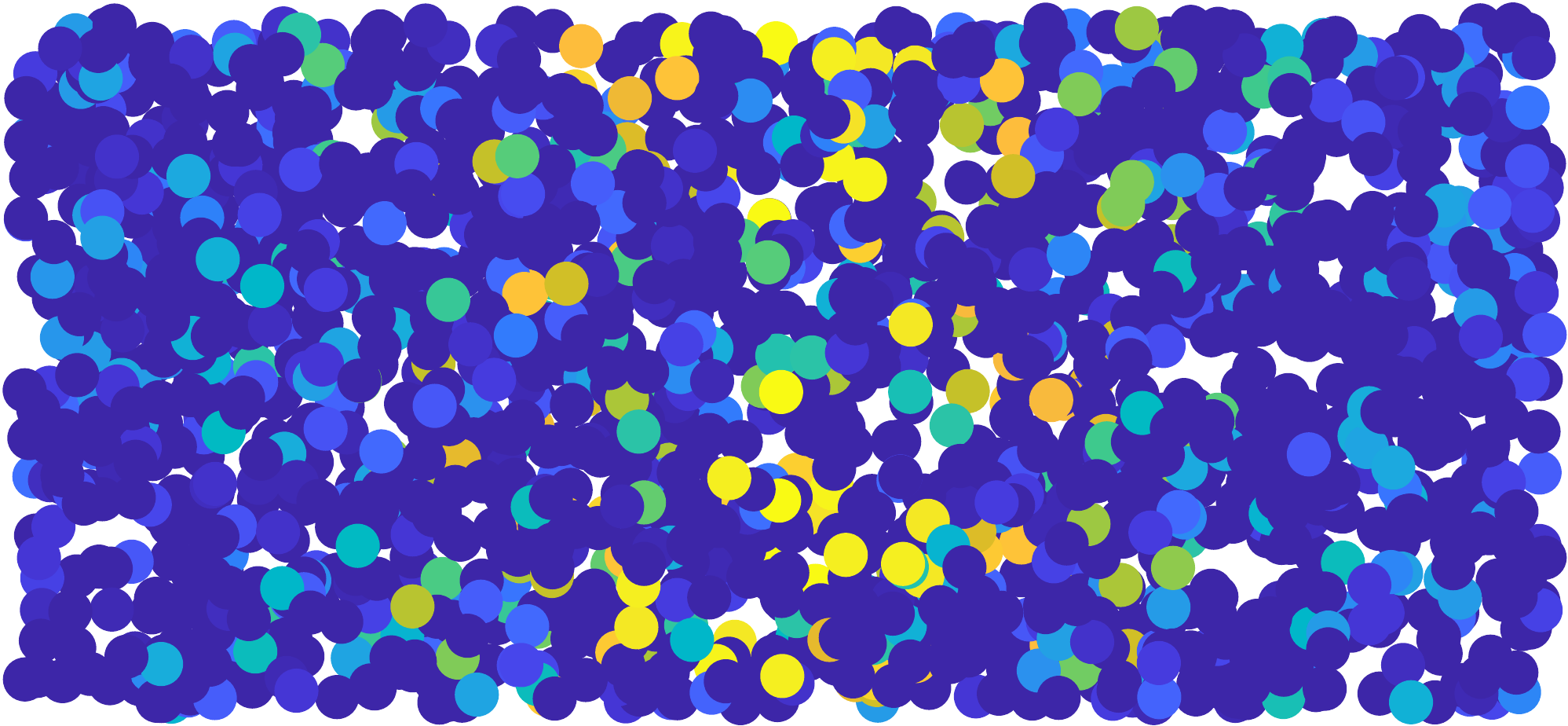}} \\
		$t=0.3$ & \hspace{-0in}  \centered{\includegraphics[scale=0.1]{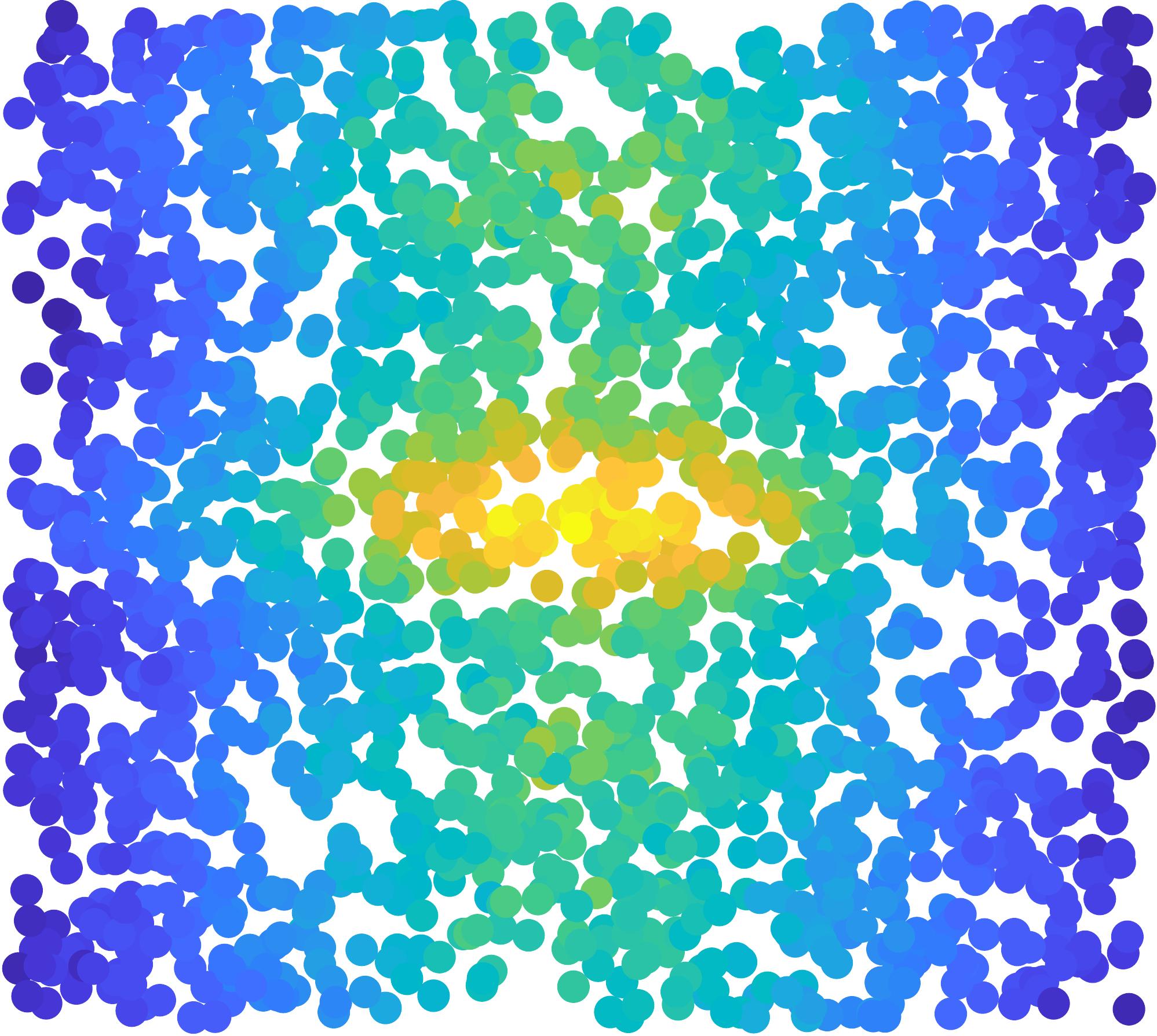}} &
		 \centered{\includegraphics[scale=0.1]{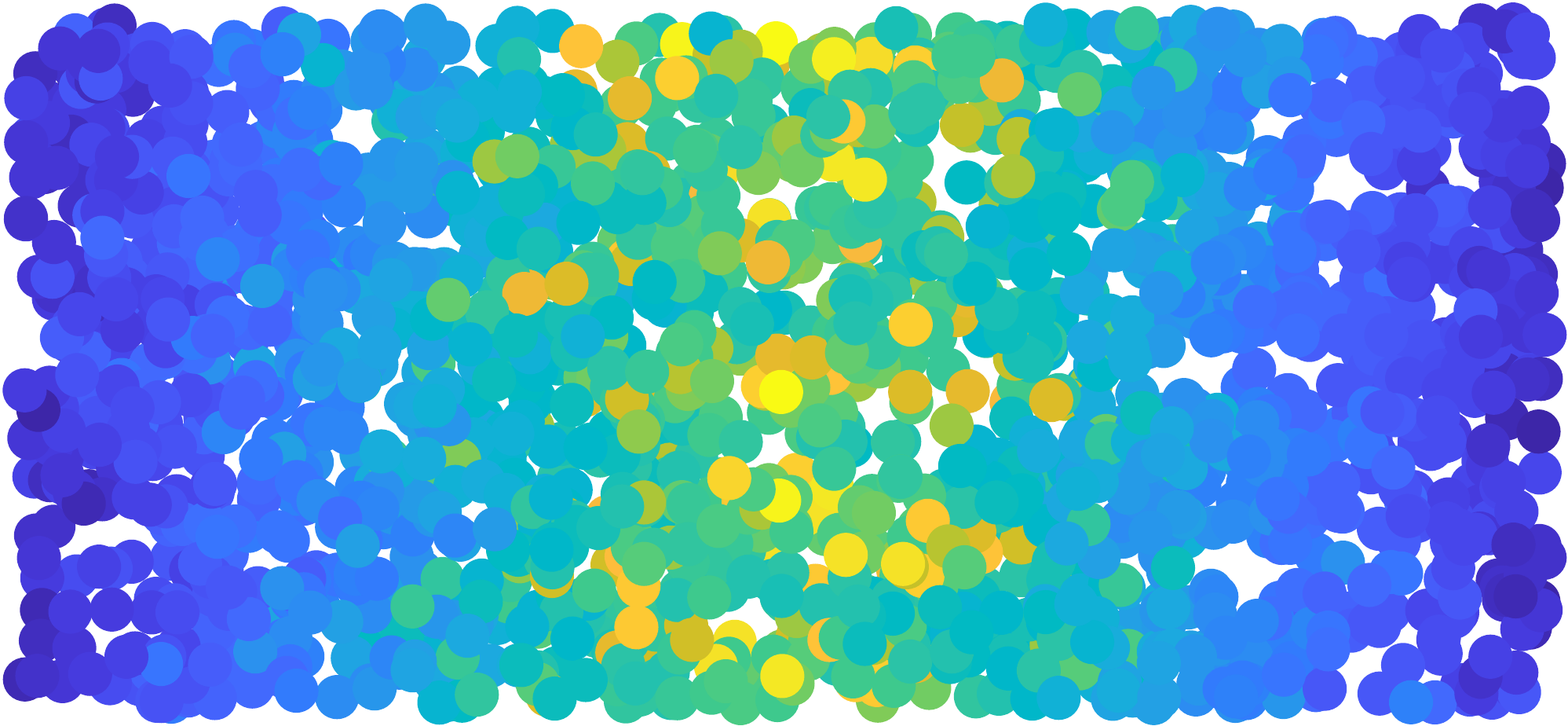}} \\
		$t=t^{*}$ &  \hspace{-0in}  \centered{\includegraphics[scale=0.1]{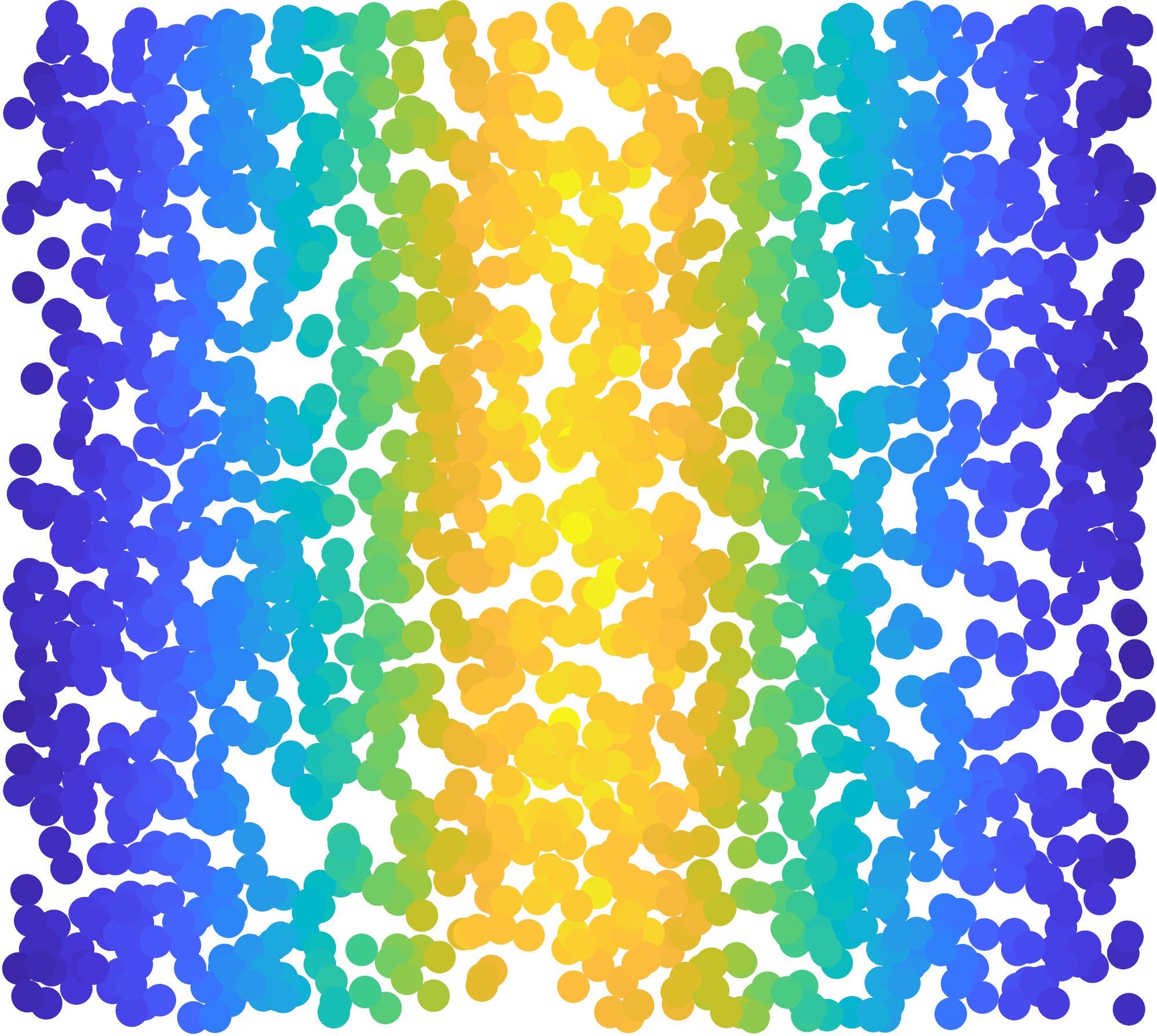}} &
		 \centered{\includegraphics[scale=0.1]{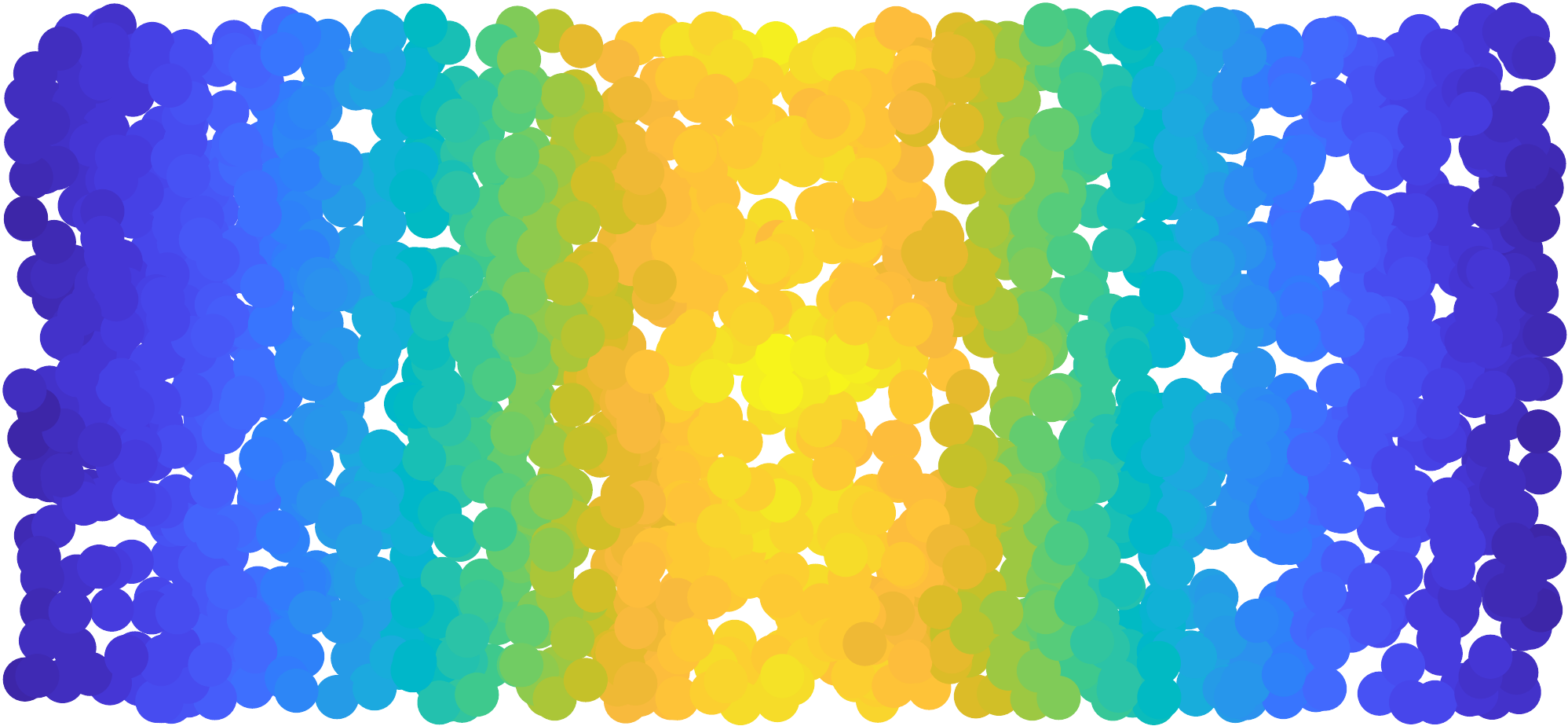}} \\
	    $t=1$ &  \hspace{-0.1in}  \centered{\includegraphics[scale=0.1]{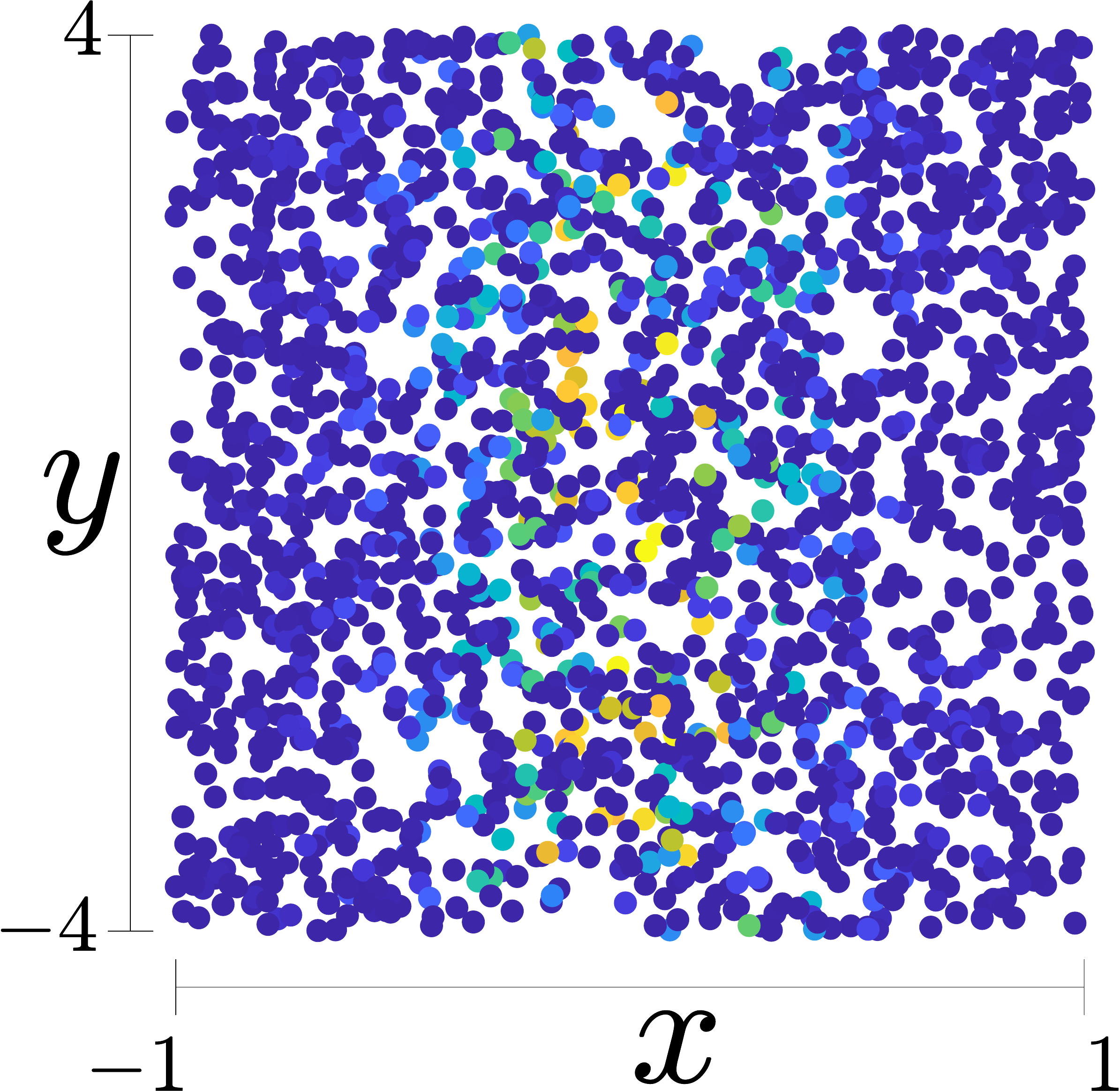}} &
		 \hspace{0.1in} \centered{\includegraphics[scale=0.1]{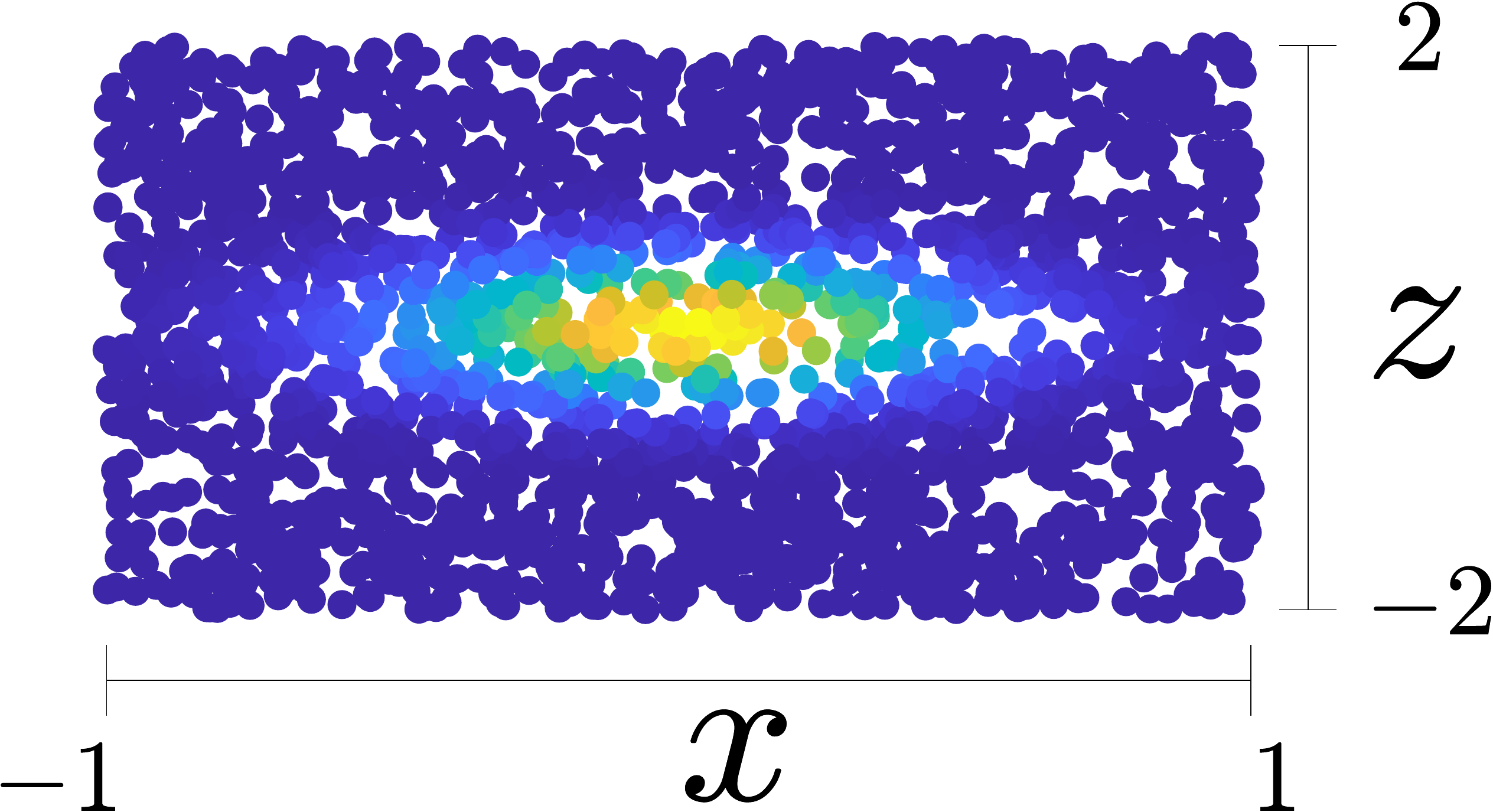}} \\
	\end{tabular}
	\caption{Mass propagation using $\gamma(t)$ at several points $t$ along the geodesic path in the 2D flat manifolds example. At the left column, we plot the measurements from $\mathcal{O}_1$, and at the right column we plot the measurements from $\mathcal{O}_2$. In each column, the horizontal axis is associated with the common manifold and the vertical axis with the measurement-specific manifold. The measurements are colored as follows. The color at the top row corresponds to the mass propagation based on $\gamma(0)=\mathbf{K}_1$, at the second and the third rows based on $\gamma(0.3)$ and $\gamma(t^*)$, and at the bottom row based on $\gamma(1)=\mathbf{K}_2$. All the masses are initially concentrated near the origin.}
	\label{fig:3DStrip_ModerateSNR_Diffs}
\end{figure}

\Cref{prop:MutualEigenValues} states that the common eigenvectors are preserved when marching along the geodesic path, and \Cref{thm:OnlyMutuals} implies that the non-common eigenvectors are not part of the spectrum of the matrices on the geodesic path. 
Our experimental study reveals a stronger result. 
We pick two eigenvectors of $\mathbf{K}_1$, denoted by $v_0^c$ and $v_0^{nc}$. Suppose that $v_0^c$ is a common eigenvector, and suppose that $v_0^{nc}$ is a non-common eigenvector, i.e., it is not an eigenvector of $\mathbf{K}_2$, and as a consequence (by \Cref{thm:OnlyMutuals}) it is not an eigenvector of any other matrix along the geodesic $\gamma(t)$. 
We use a discrete uniform grid of $[0,1]$. For each point on the grid $t$, we calculate the matrix $\gamma(t)$ and its set of eigenvectors: $\{v_{t}^k\}_{k=1}^n$. 
We then calculate the inner products between the vectors $\{v_{t}^{k}\}_{k=1}^n$ and $v_0^c$ or $v_0^{nc}$. The obtained inner products $\{\left\langle v_0^c,v_{t}^k\right \rangle\}_{k=1}^{n}$ and $\{\left\langle v_0^{nc},v_{t}^k\right \rangle\}_{k=1}^{n}$ are the spectral representation of $v_0^c$ and $v_0^{nc}$, respectively, since they are the expansion coefficients when using the eigenvectors of $\gamma(t)$ as a basis. 

In Fig. \ref{fig:NonFixedDispersion}, we present the absolute values of $\{\left\langle v_0^{c},v_{t}^k\right \rangle\}_{k=1}^{n}$ and $\{\left\langle v_0^{nc},v_{t}^{k}\right \rangle\}_{k=1}^{n}$ in columns, where the horizontal axis denotes $t$ and the vertical axis denotes the spectral component index $k$. The spectral components are sorted in an descending order of their respective eigenvalues (high values on top).
\begin{figure}[t]\centering
	\begin{tabular}{cc}
		\includegraphics[width=0.4\textwidth]{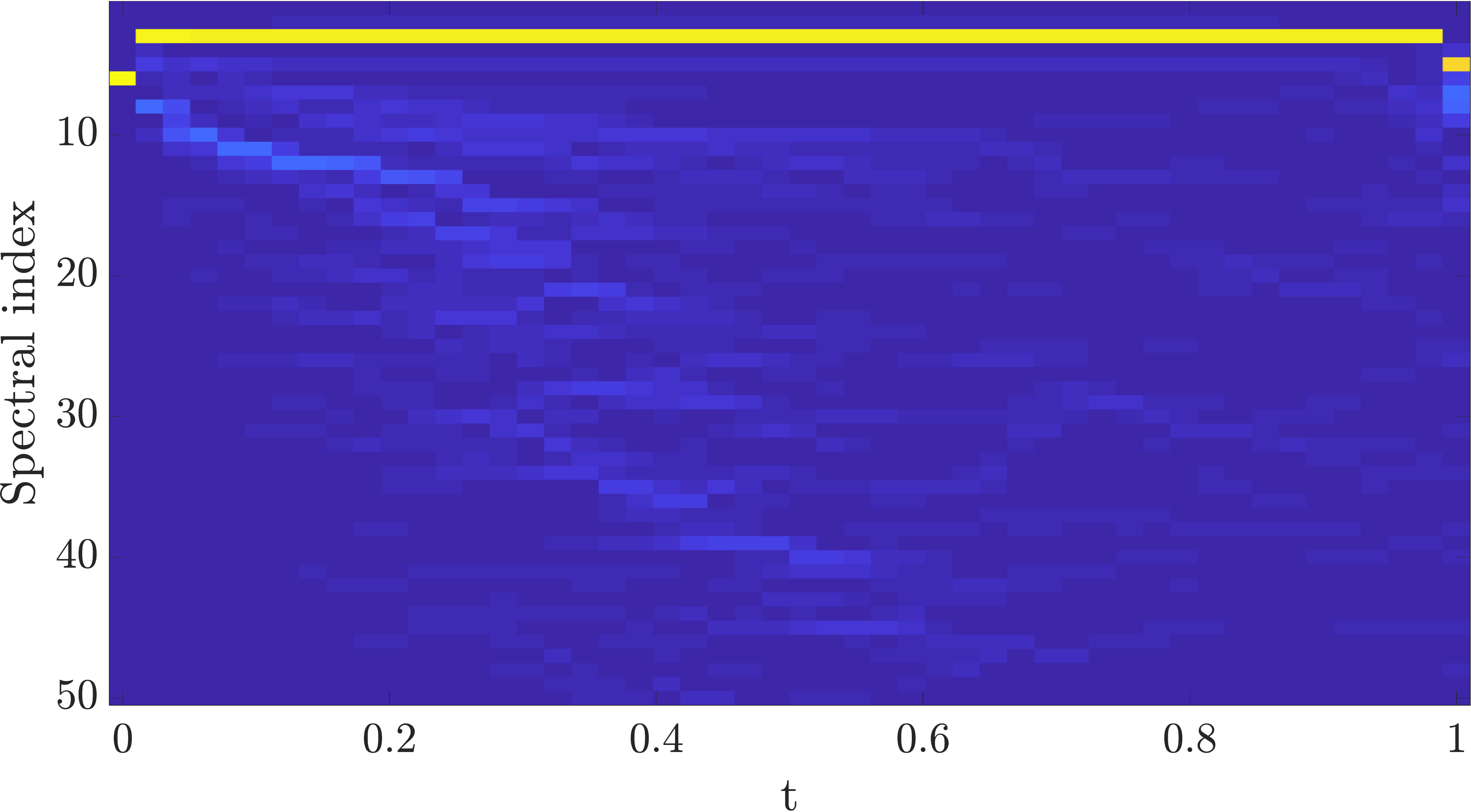} &
		\includegraphics[width=0.4\textwidth]{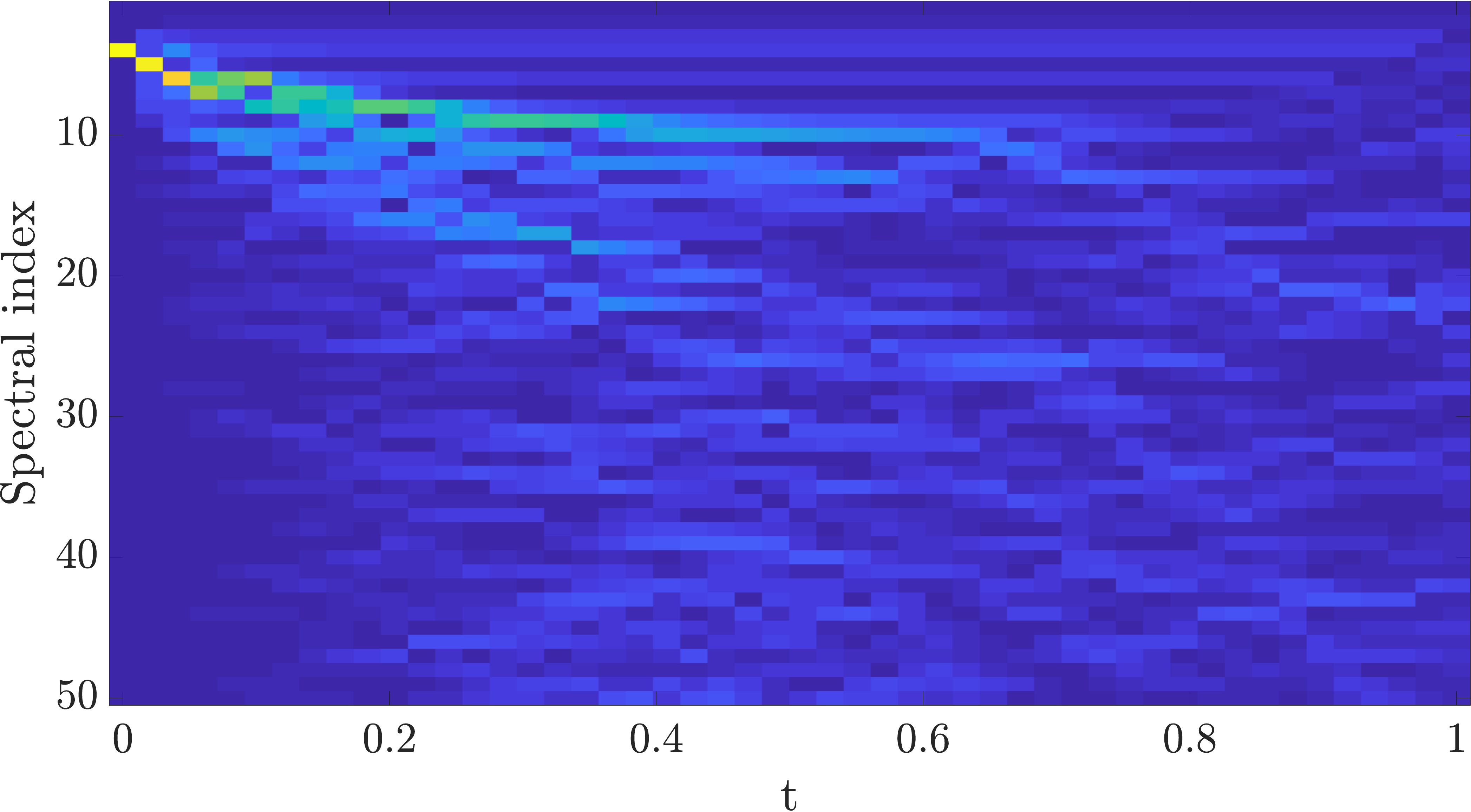} \\
		\hspace{0.4in}(a) &\hspace{0.15in} (b)
	\end{tabular}
	\caption {The spectral representation of (a) $v_0^c$ and (b) $v_0^{nc}$ at each point $t$ along the geodesic path.}
	\label{fig:NonFixedDispersion}
\end{figure}
We observe in Fig. \ref{fig:NonFixedDispersion}(a) that the spectral representation of the common eigenvector $v_0^c$ at every point $t$ along the geodesic path is concentrated only at a single spectral component. Furthermore, this spectral component for $t \in (\delta, 1-\delta)$ for some $\delta>0$ appears higher in the spectrum. This implies that the eigenvalues corresponding to common eigenvectors are more dominant in $\gamma(t)$ compared to $\mathbf{K}_1$ and to $\mathbf{K}_2$, illustrating \Cref{prop:MutualEigenValues}.
In addition, in Fig. \ref{fig:NonFixedDispersion}(b) we observe that the non-common eigenvector $v_0^{nc}$ exhibits a completely different expansion. Not only that $v_0^{nc}$ is not an eigenvector of the geodesic path, its spectral representation $\{\left\langle v_0^{nc},v_{t}^{k}\right \rangle\}_{k=1}^{n}$ quickly spreads over the entire spectrum as $t$ increases. 

We repeat the above examination with the following (different) set of parameters: 
\begin{align*}
\ell^{(1)}_x=2 ,\ell^{(2)}_x=2\\
\ell^{(1)}_{y}=2,\ell^{(2)}_{z}=2\\
\end{align*}
designating the common manifold to be as dominant as the measurement-specific manifold in each of the measurements. 
We compute the \acrshort*{EVFD} by \Cref{alg:EvfdCalculation} with $N_t=200$ and $K=20$. In addition, we estimate the CMR at each point along the geodesic path using \Cref{alg:COR2}. The \acrshort*{EVFD} is depicted in Fig. \ref{fig:3DStrip_HighSNR_COR}, where the point $t$ along the geodesic path at which the CMR estimate is maximal is marked by a horizontal red line. We observe a symmetric diagram, corresponding to the symmetry of the measurements induced by the choice of scale parameters, where the maximal CMR is obtained at $t^{*}=0.5$. In addition, we see here as well that the log-linear trajectories coincide at the boundaries $t=0$ and $t=1$ with the analytically computed eigenvalues marked by red squares. 
\begin{figure}[t]\centering
	\begin{tabular}{cc}
		\includegraphics[scale=0.25]{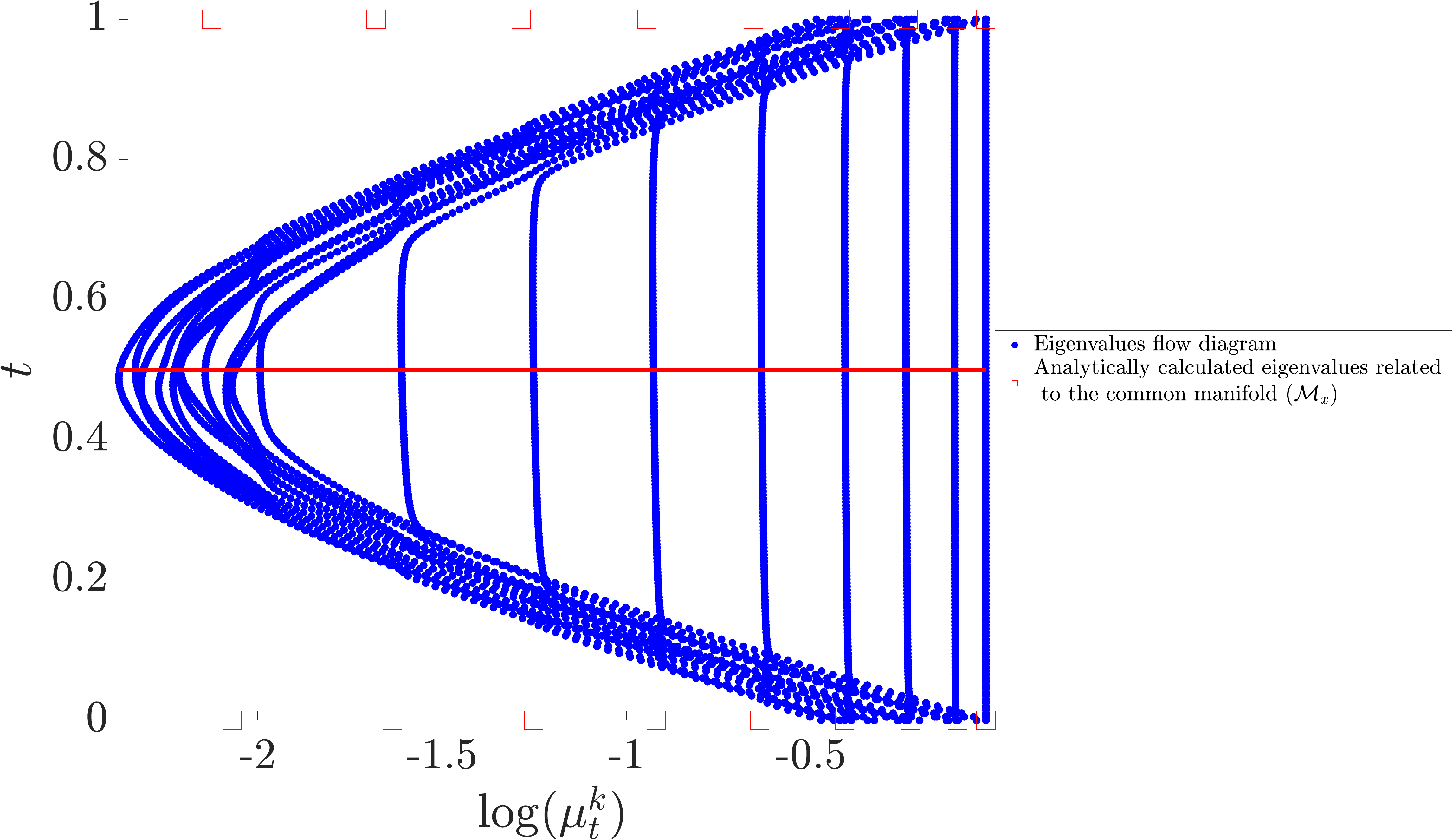}\\
	\end{tabular}
	\caption {Same as Fig. \ref{fig:3DStrip_ModerateSNR_COR} (left) only with a different set of scale parameters.}
	\label{fig:3DStrip_HighSNR_COR}
\end{figure}

Fig. \ref{fig:3DStrip_HighSNR_Diffs} is the same as Fig. \ref{fig:3DStrip_ModerateSNR_Diffs}, depicting the diffusion propagation patterns at $t=0,0.3,t^*,1$ obtained for this set of scale parameters. We observe in this figure as well the saturation along the vertical axes in the diffusion patterns associated with $t=0.3$ and $t^*=0.5$, leading to the invariance of the respective diffusion distances to the measurement-specific variables.

\begin{figure}[t]\centering
    \begin{tabular}{ccc}
	    & $\mathcal{O}_1$ & $\mathcal{O}_2$\\ \\
		$t=0$ & \hspace{-0in}  \centered{\includegraphics[scale=0.1]{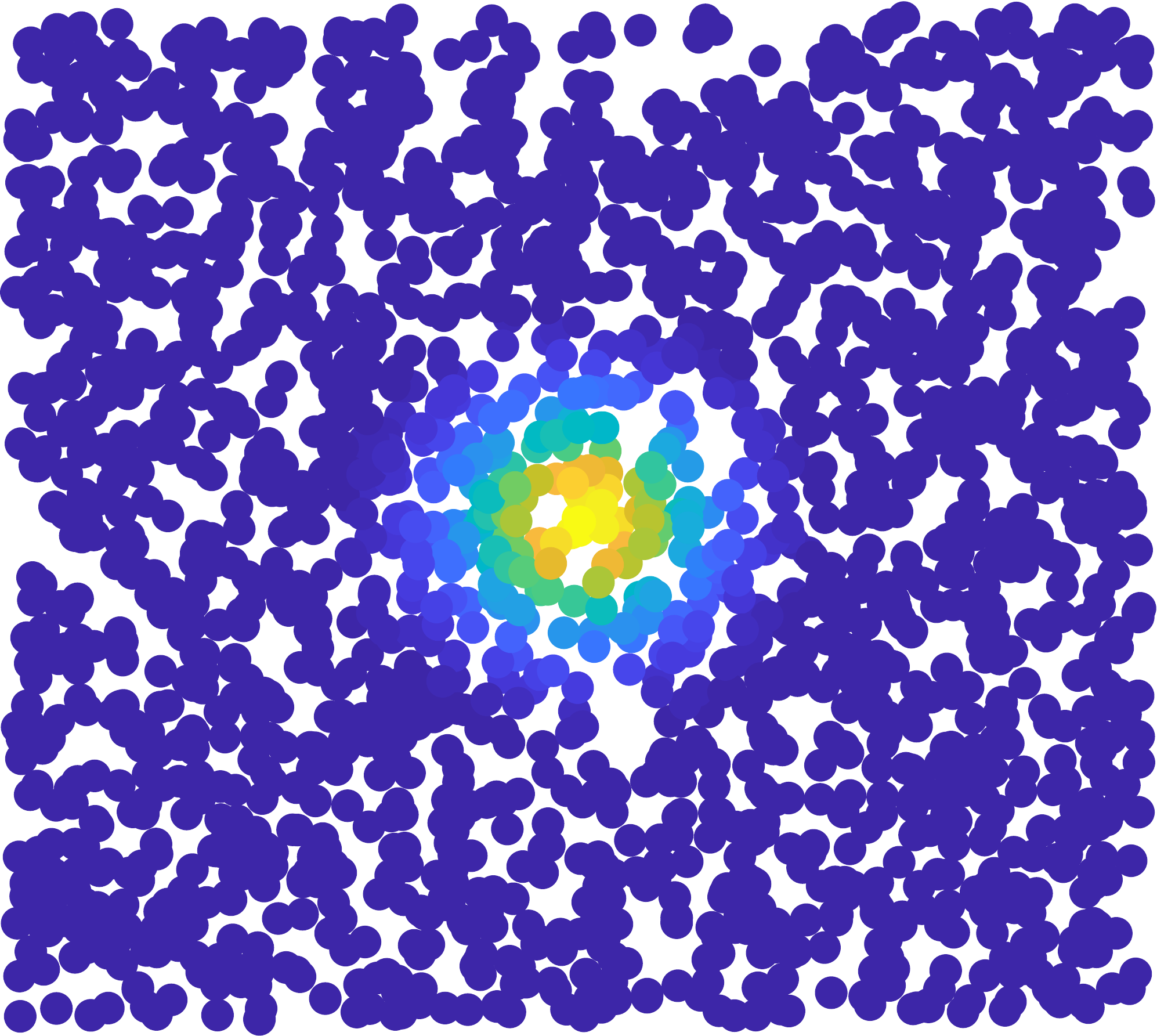}} &
		 \centered{\includegraphics[scale=0.1]{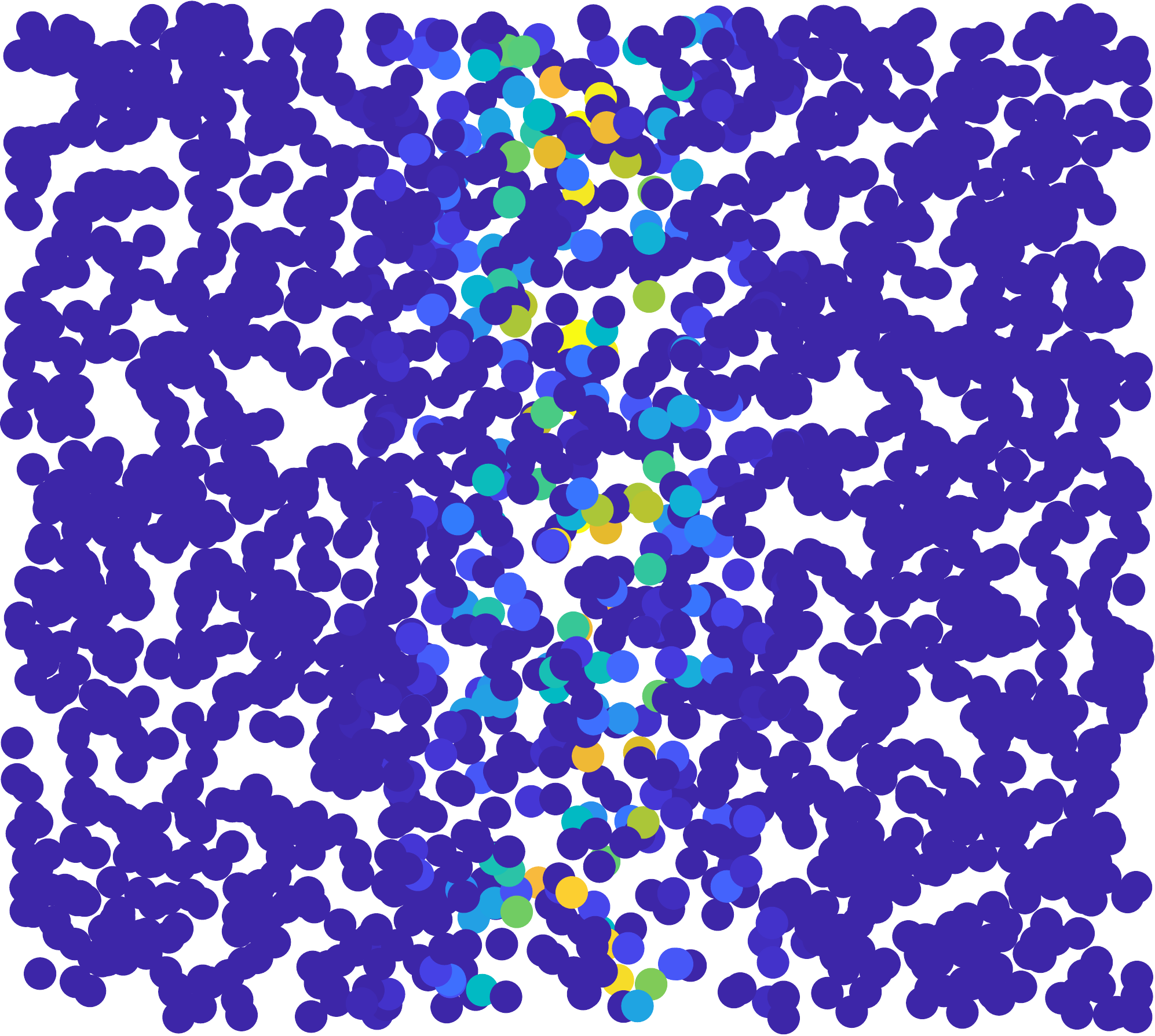}} \\
		$t=0.3$ & \hspace{-0in}  \centered{\includegraphics[scale=0.1]{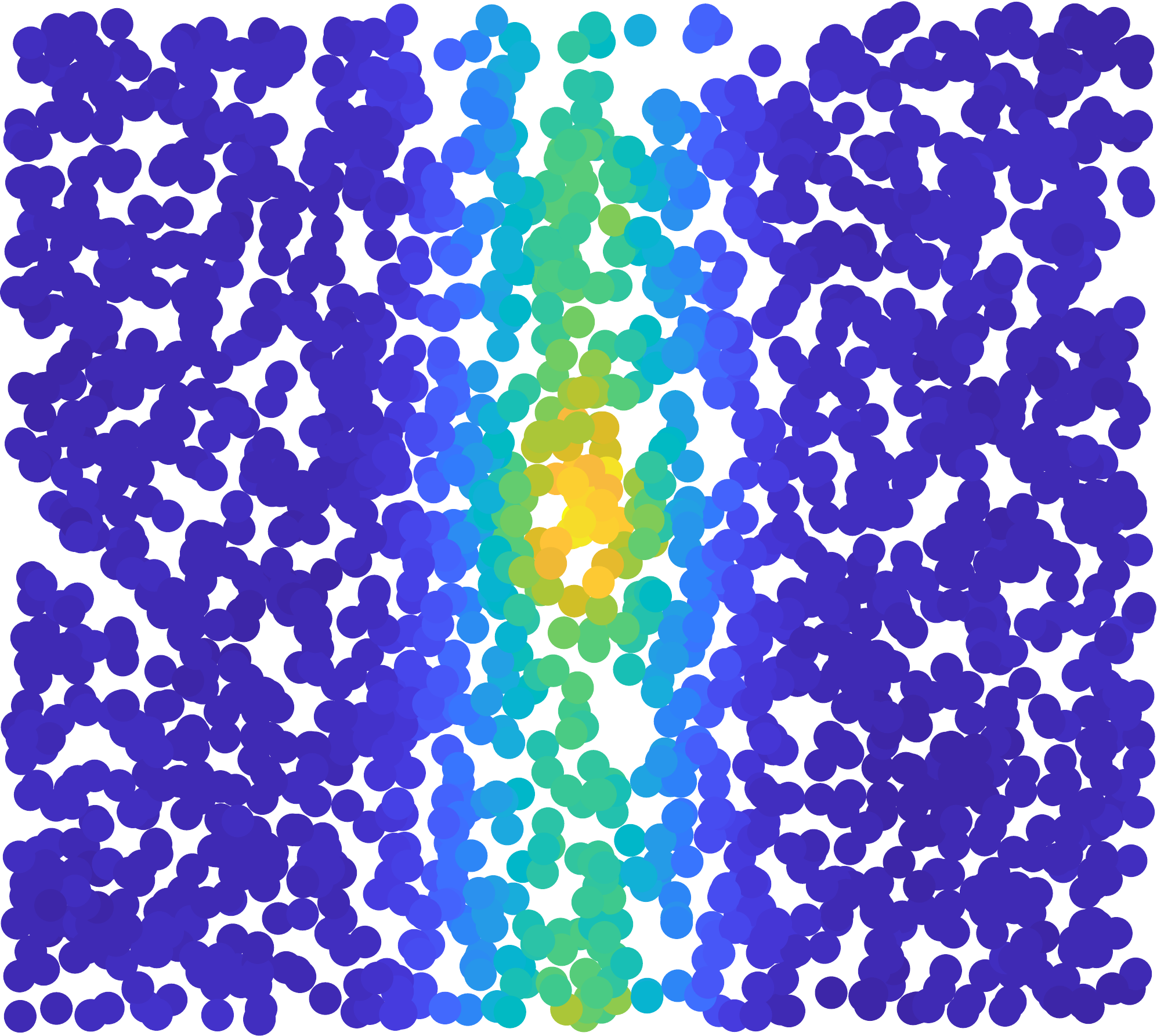}} &
		 \centered{\includegraphics[scale=0.1]{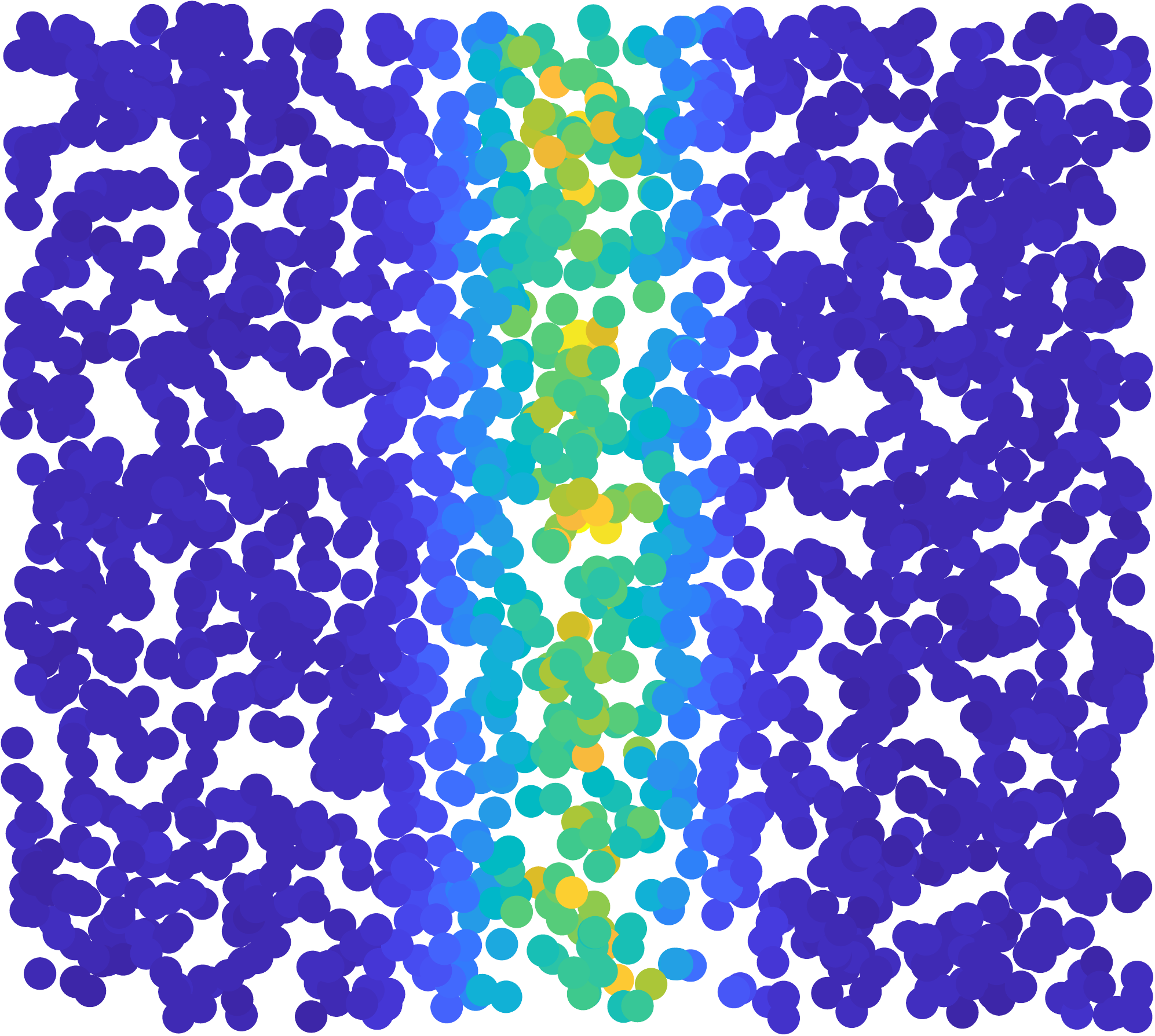}} \\
		$t=t^{*}$ &  \hspace{-0in}  \centered{\includegraphics[scale=0.1]{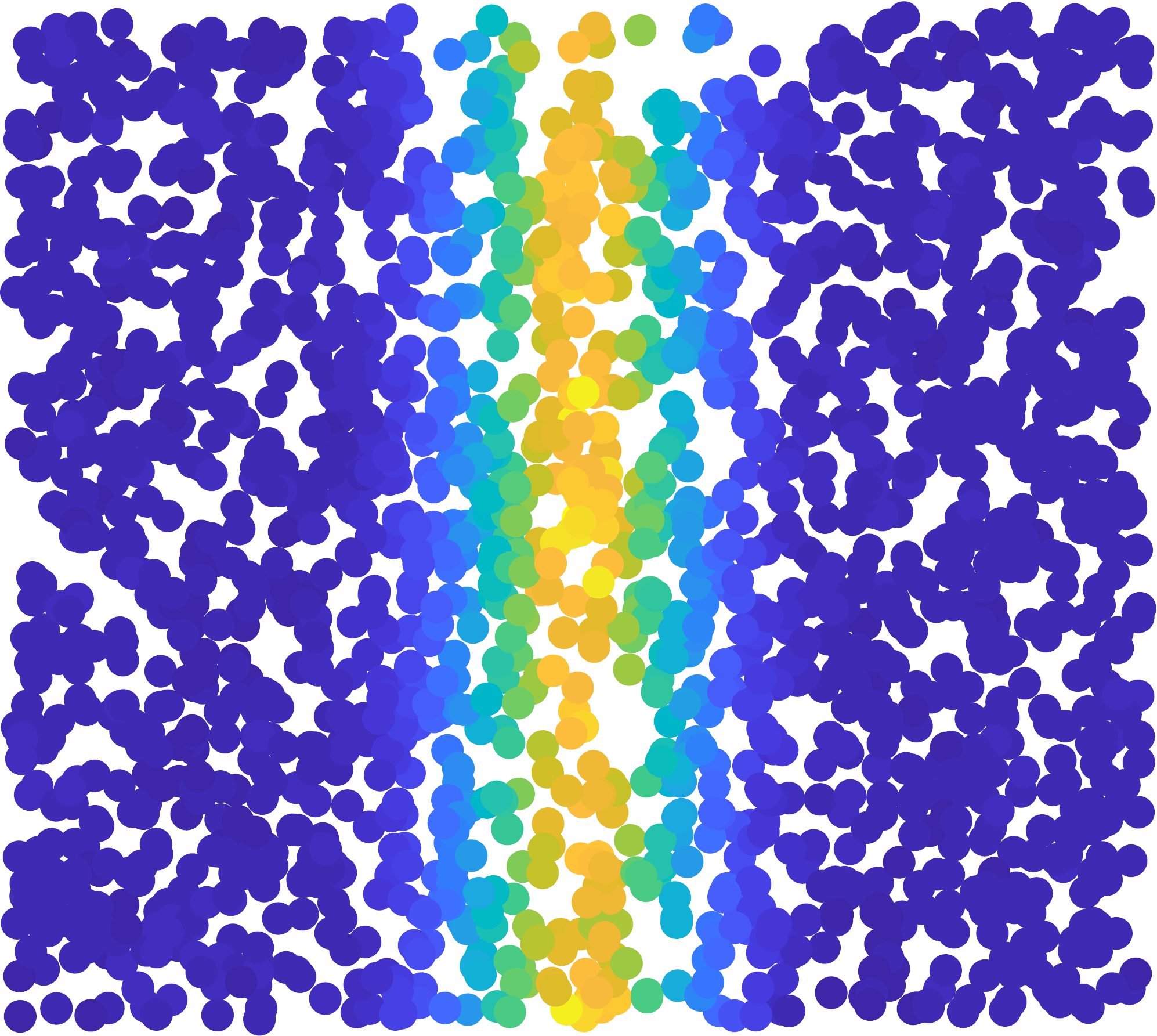}} &
		 \centered{\includegraphics[scale=0.1]{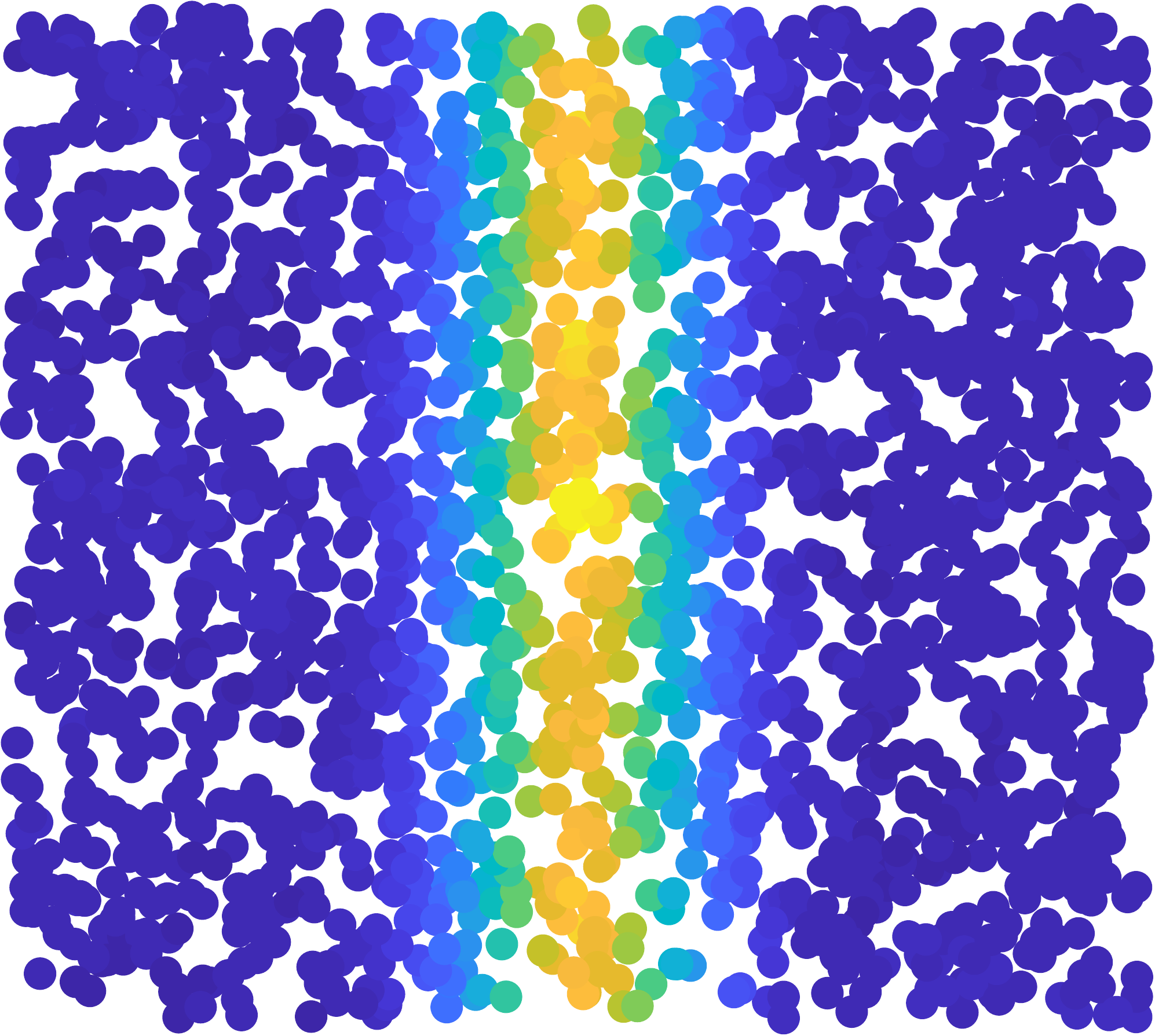}} \\
	    $t=1$ &  \hspace{-0.1in}  \centered{\includegraphics[scale=0.101]{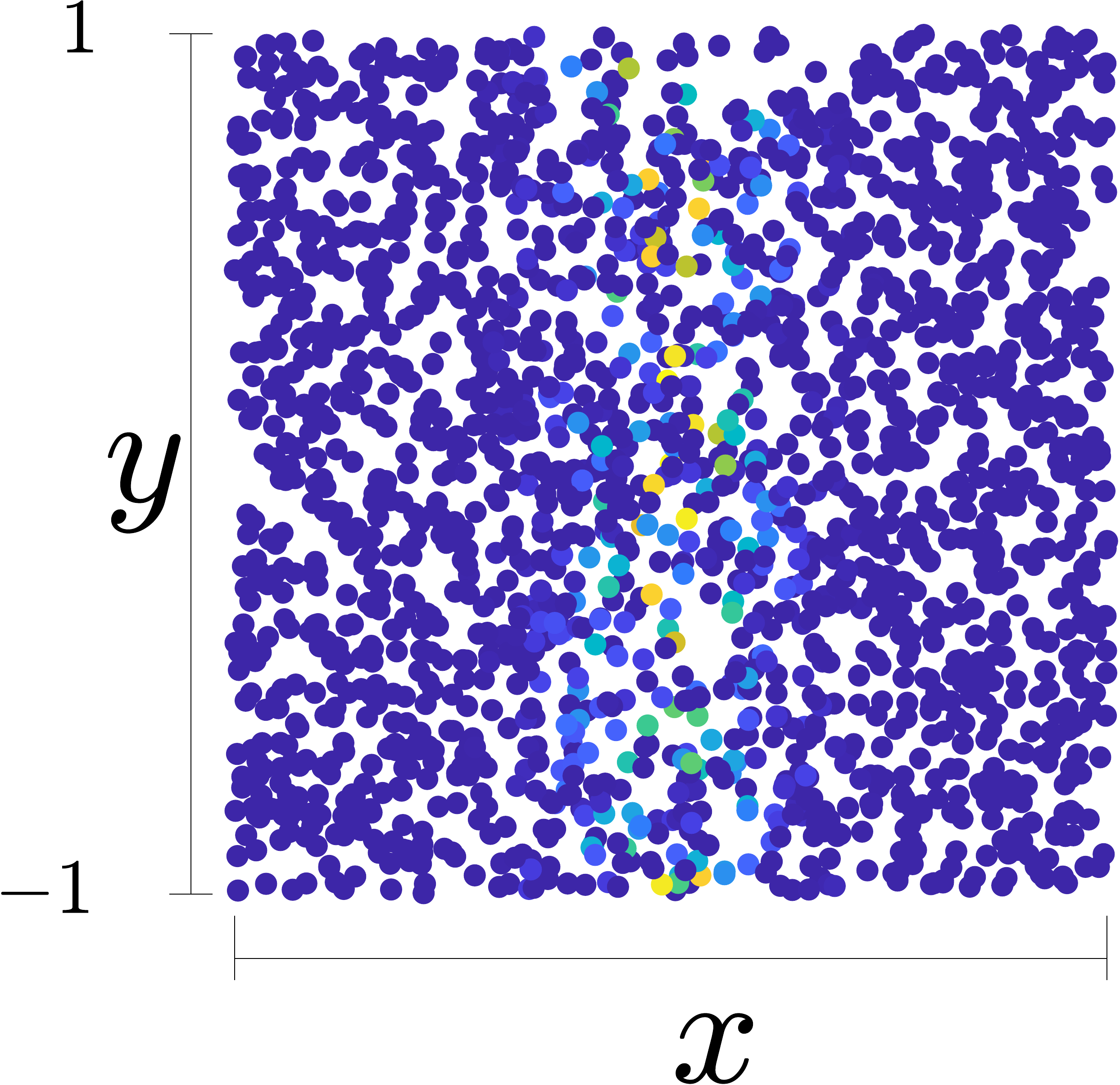}} &
		 \hspace{0.1in} \centered{\includegraphics[scale=0.101]{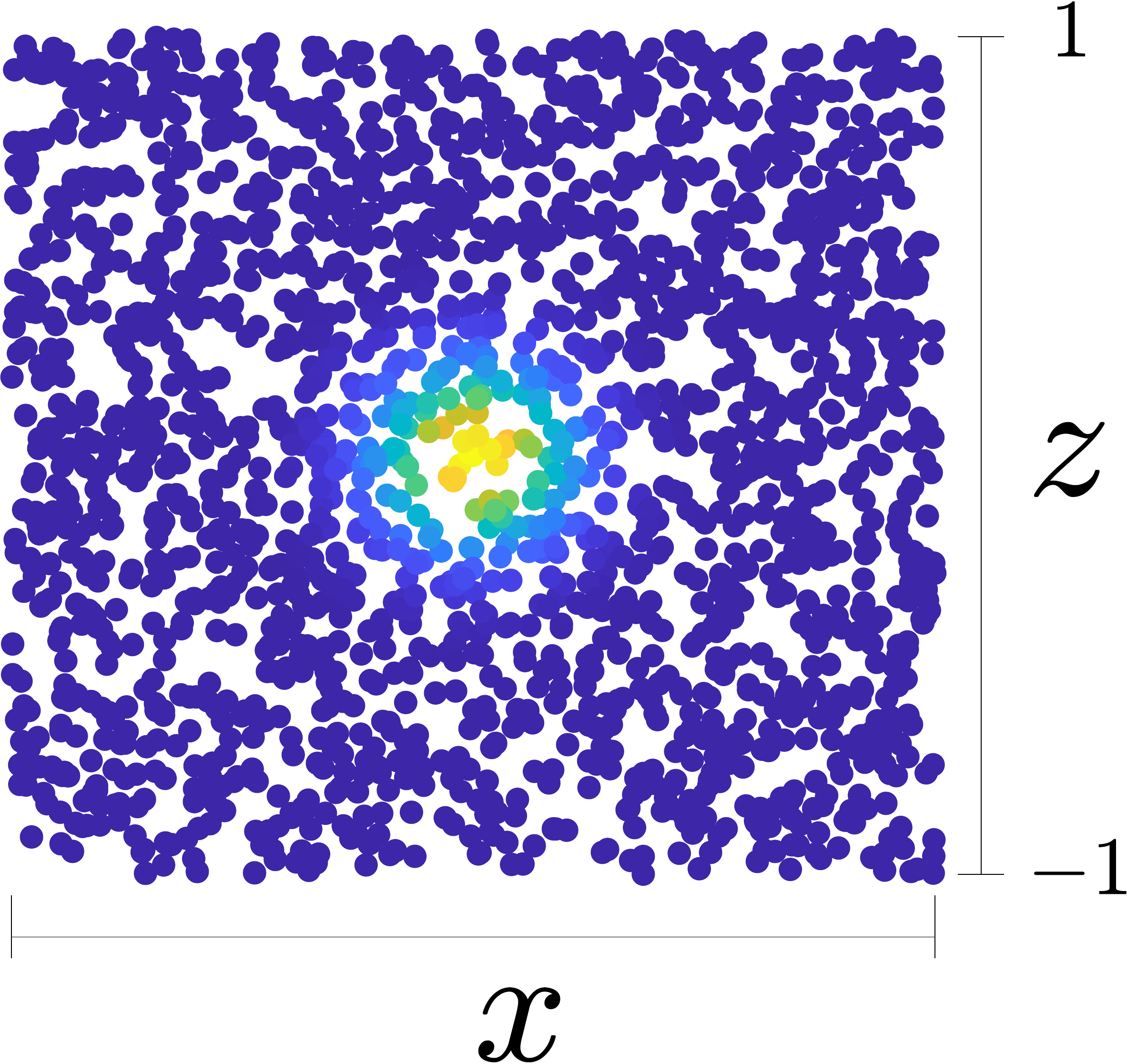}} \\
	\end{tabular}
	\caption{Same as Fig. \ref{fig:3DStrip_ModerateSNR_Diffs} only with a different set of scale parameters.}
	\label{fig:3DStrip_HighSNR_Diffs}
\end{figure}

Next, we consider another set of scale parameters:
\begin{gather*}
	\ell^{(1)}_x=2 ,\ell^{(2)}_x=2\\
	\ell^{(1)}_{y}=30,\ell^{(2)}_{z}=20.
\end{gather*}
In contrast to the previous choice of scales, now the measurement-specific manifold is more dominant than the common manifold and the scales of the two measurements are different.

In Fig. \ref{fig:3DStrip_LowSNR_SNR}, we present the \acrshort*{EVFD} with $N_t=200$ and $K=20$.
We observe that a single common eigenvector is detected as a log-linear trajectory, and that it coincides with the analytic expressions at the boundaries $t=0$ and $t=1$. This clear detection is attained despite the dominance of the measurement-specific manifolds at both measurements. In addition, we observe that the maximal CMR is obtained at $t^{*} \approx 0.6$, complying with ratio between the scales.

\begin{figure}[t]\centering
	\includegraphics[scale=0.25]{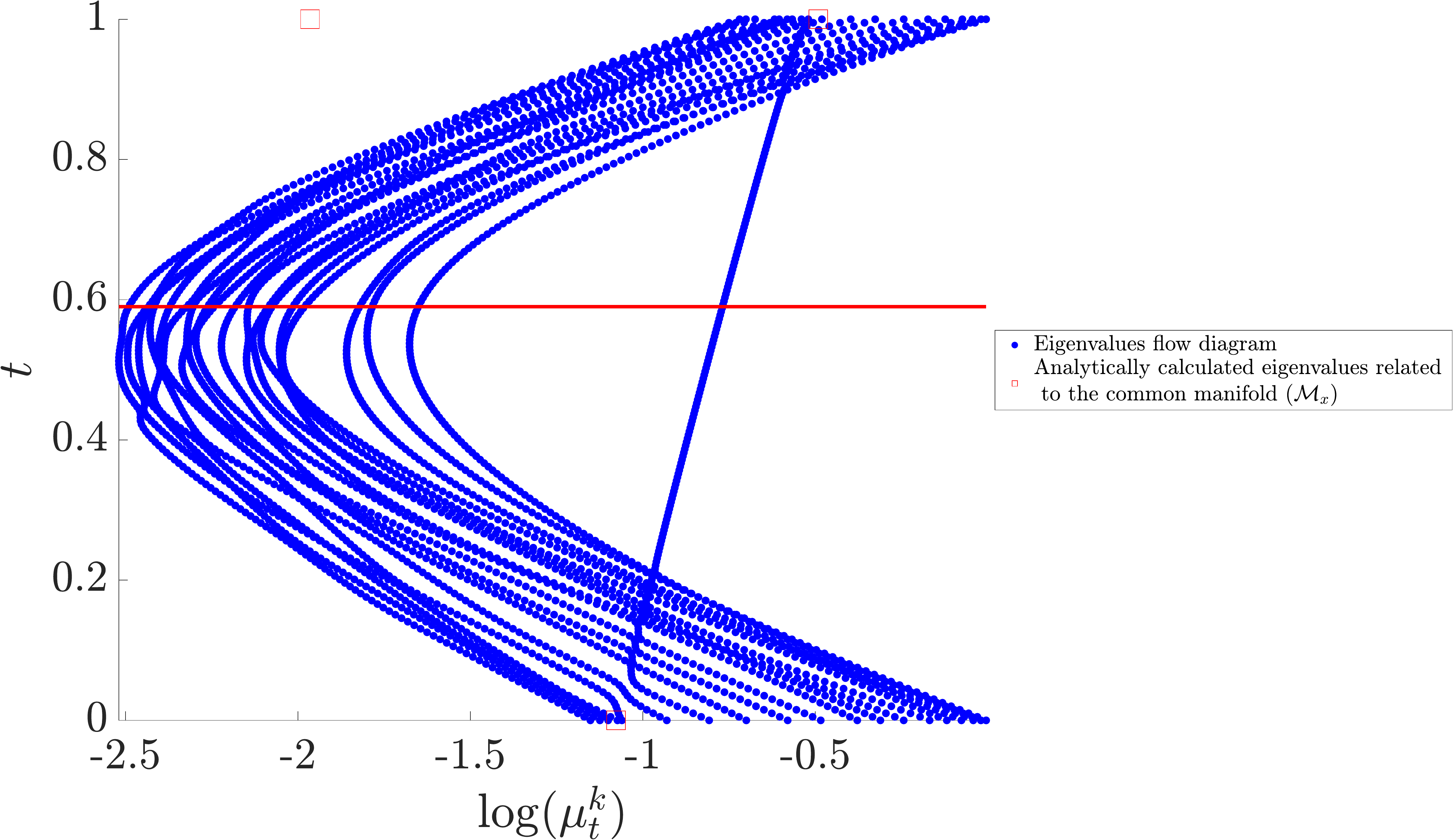} 
	\caption {Same as Fig. \ref{fig:3DStrip_ModerateSNR_COR} (left) only with a different set of scale parameters.}
	\label{fig:3DStrip_LowSNR_SNR}
\end{figure}

Fig. \ref{fig:3DStrip_LowSNR_Diffs} presents the diffusion propagation patterns. Here, we observe the importance of an appropriate choice of $t$ for obtaining a diffusion pattern that saturates along the measurement-specific variable (vertical axis) and diffuses only along the common variable (horizontal axis).

\begin{figure}[ht]\centering
	\begin{tabular}{ccc}
	    & $\mathcal{O}_1$ & $\mathcal{O}_2$\\ \\
		$t=0$ & \hspace{-0in}  \centered{\includegraphics[scale=0.1]{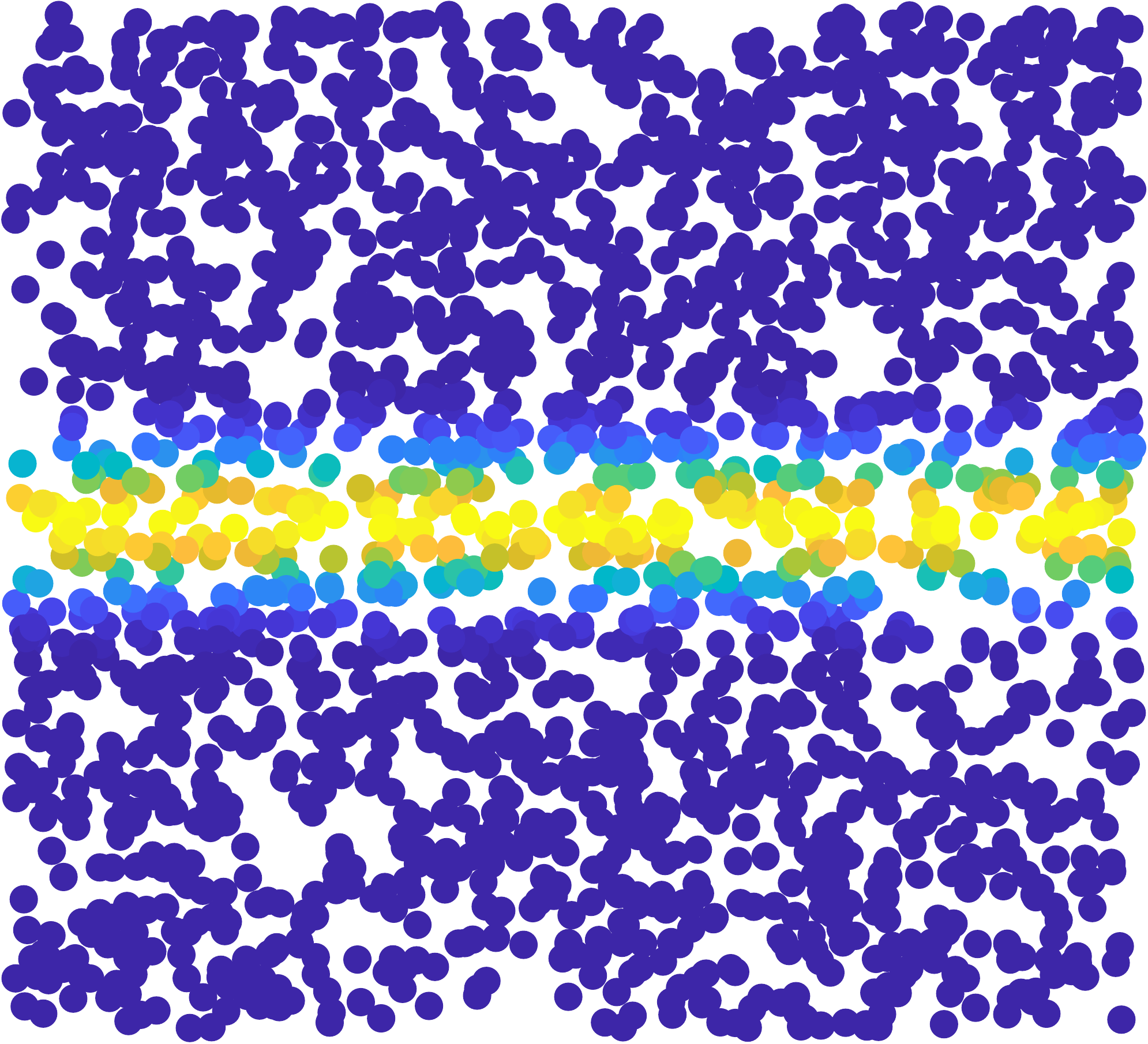}} &
		 \centered{\includegraphics[scale=0.1]{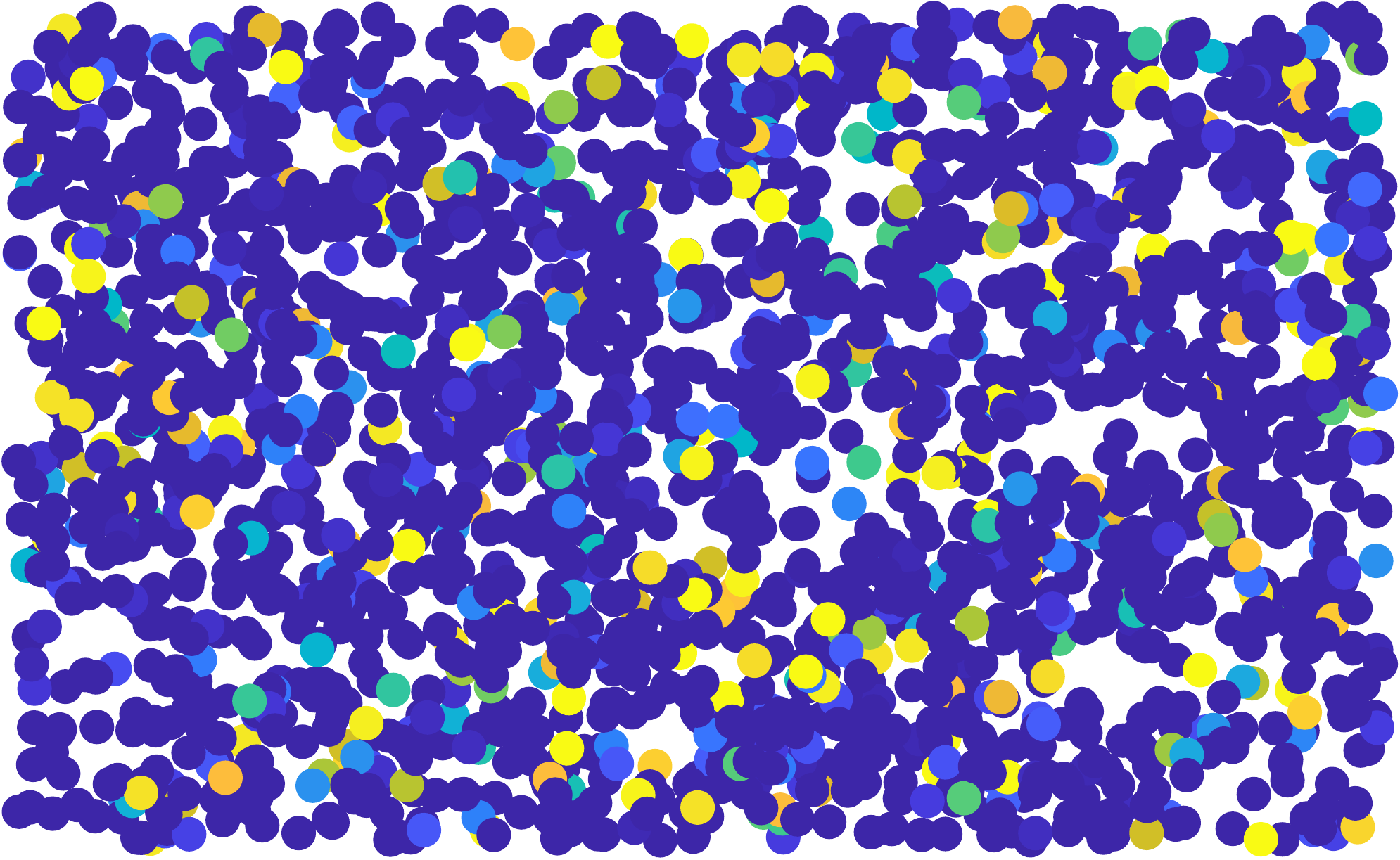}} \\
		$t=0.3$ & \hspace{-0in}  \centered{\includegraphics[scale=0.1]{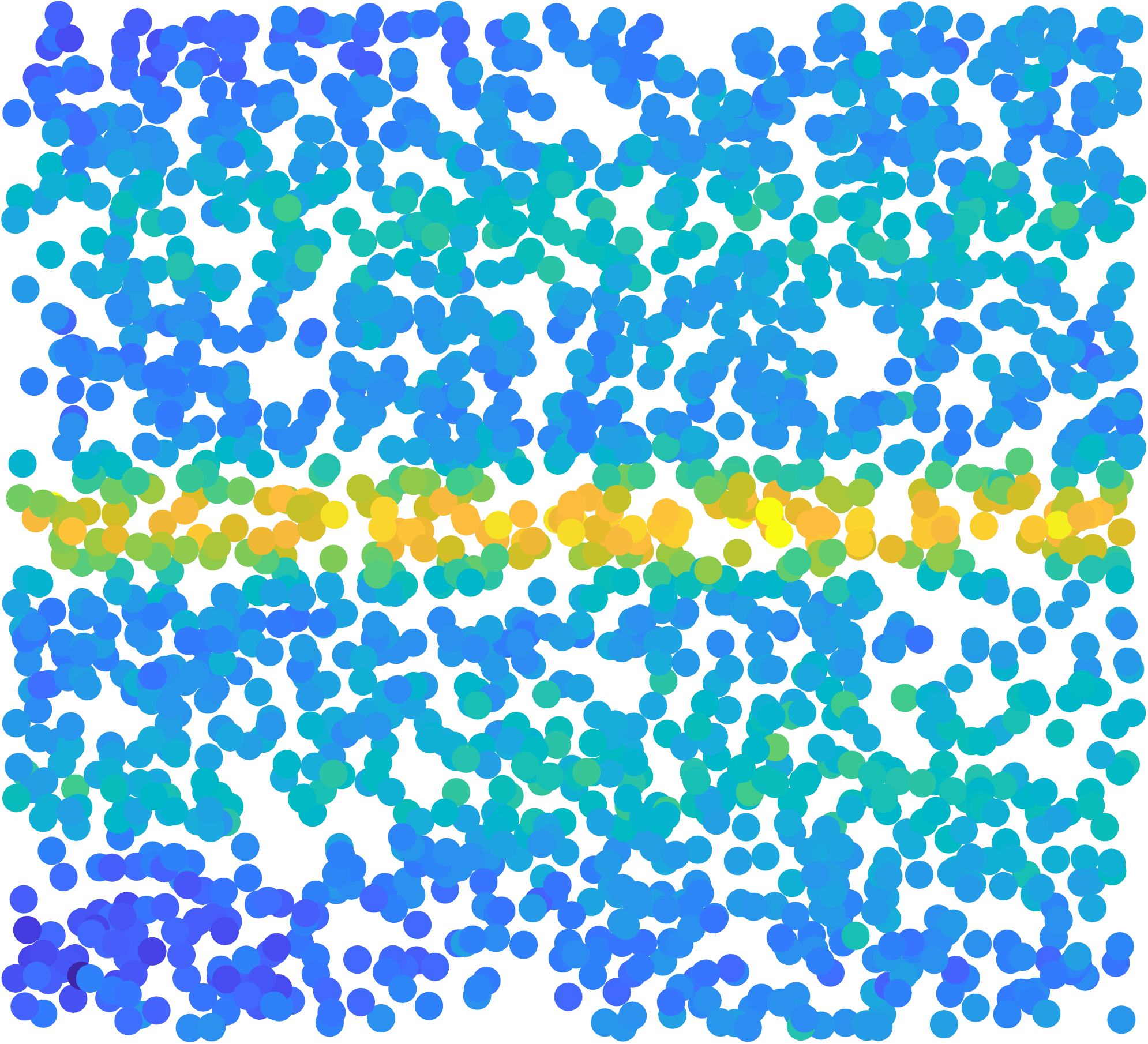}} &
		 \centered{\includegraphics[scale=0.1]{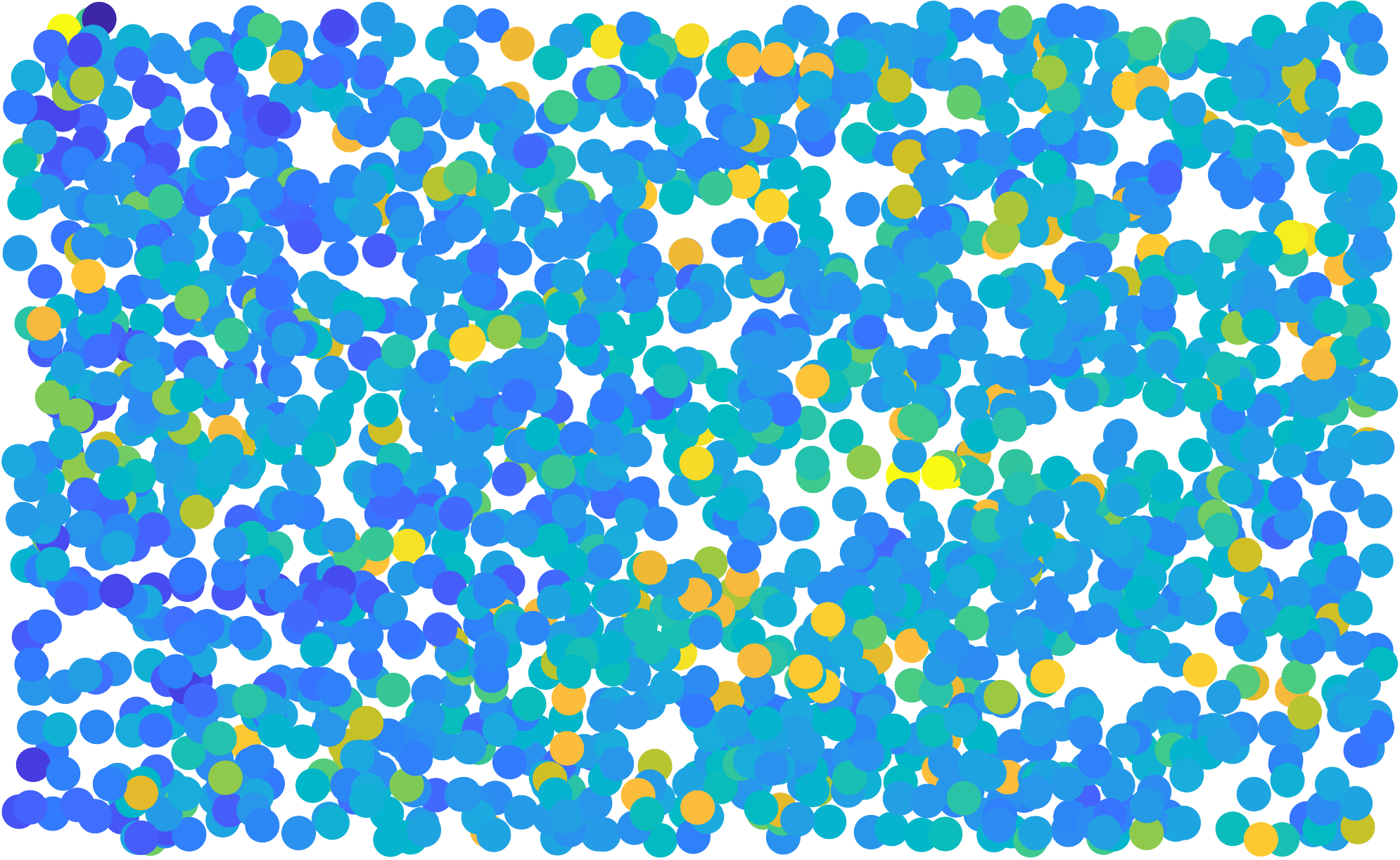}} \\
		$t=t^{*}$ &  \hspace{-0in}  \centered{\includegraphics[scale=0.1]{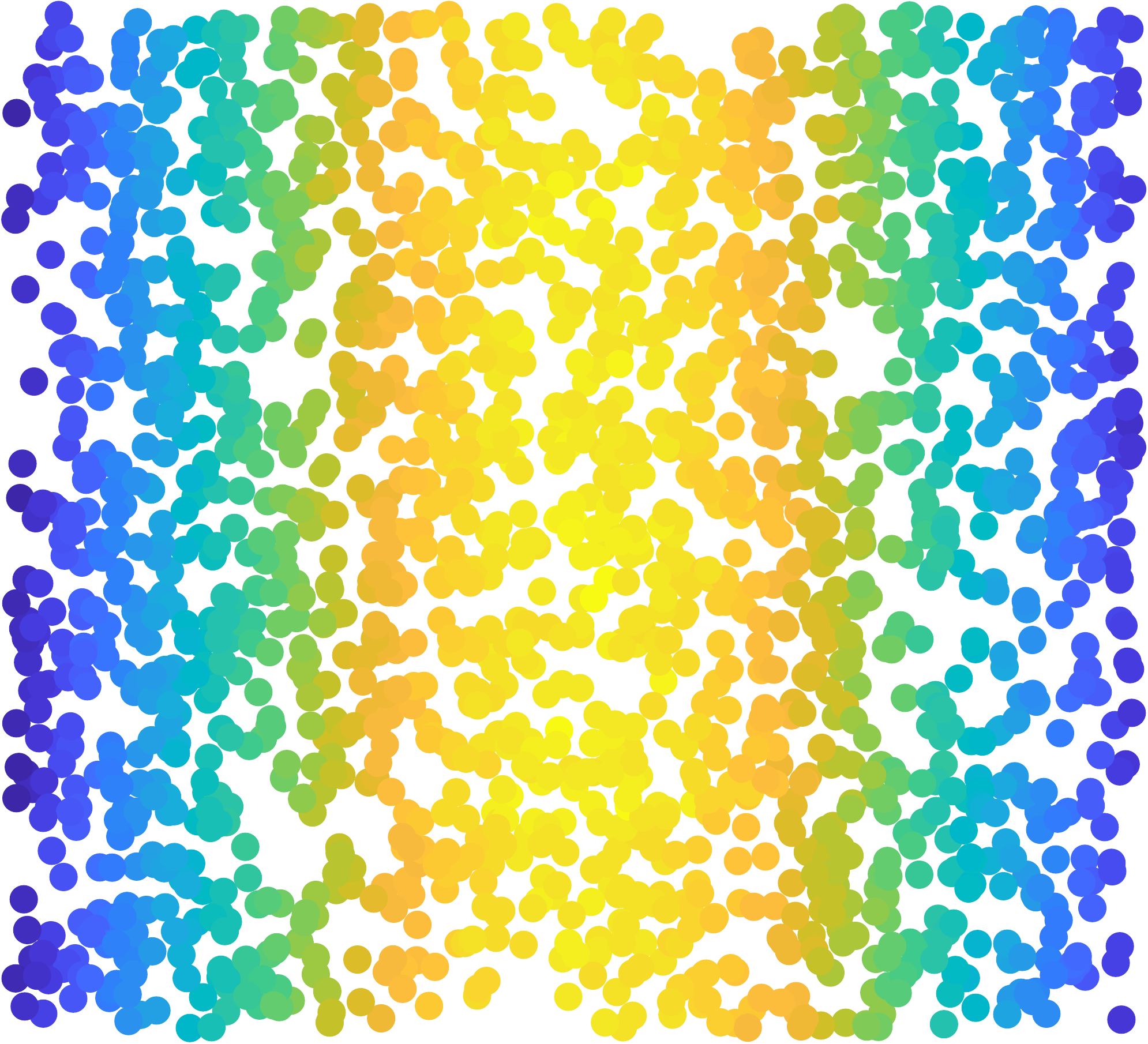}} &
		 \centered{\includegraphics[scale=0.1]{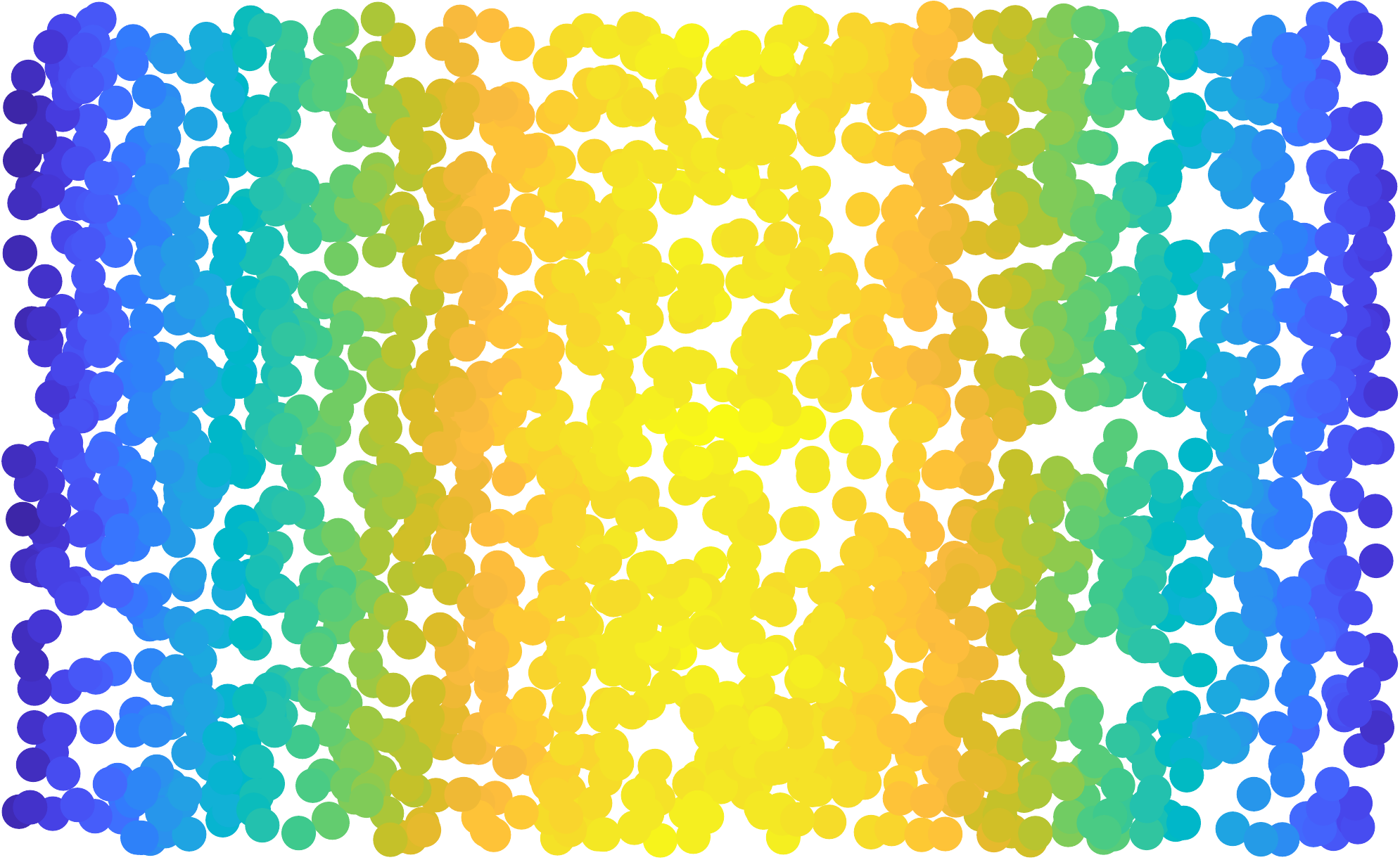}} \\
	    $t=1$ &  \hspace{-0.1in}  \centered{\includegraphics[scale=0.1]{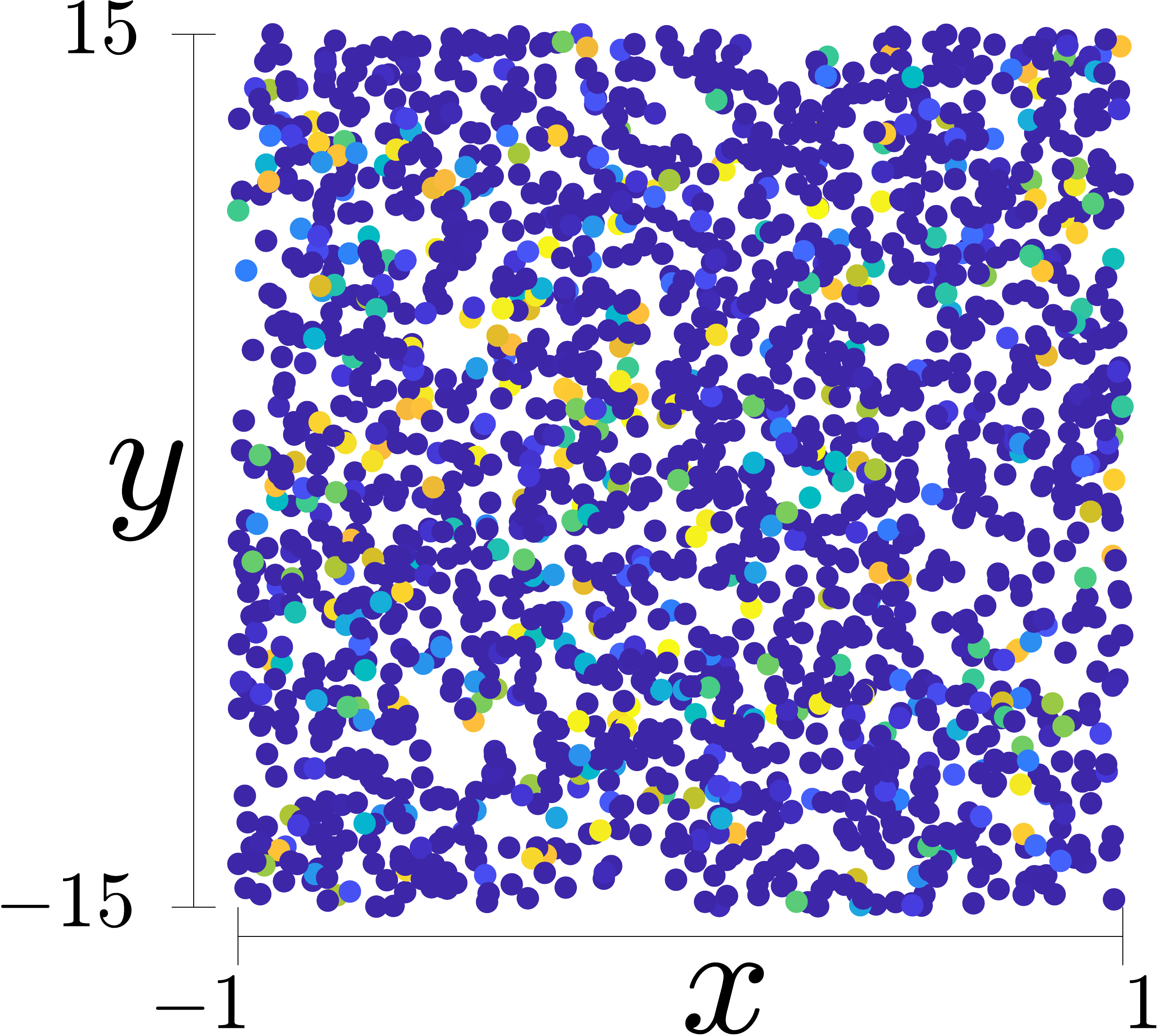}} &
		 \hspace{0.1in} \centered{\includegraphics[scale=0.1]{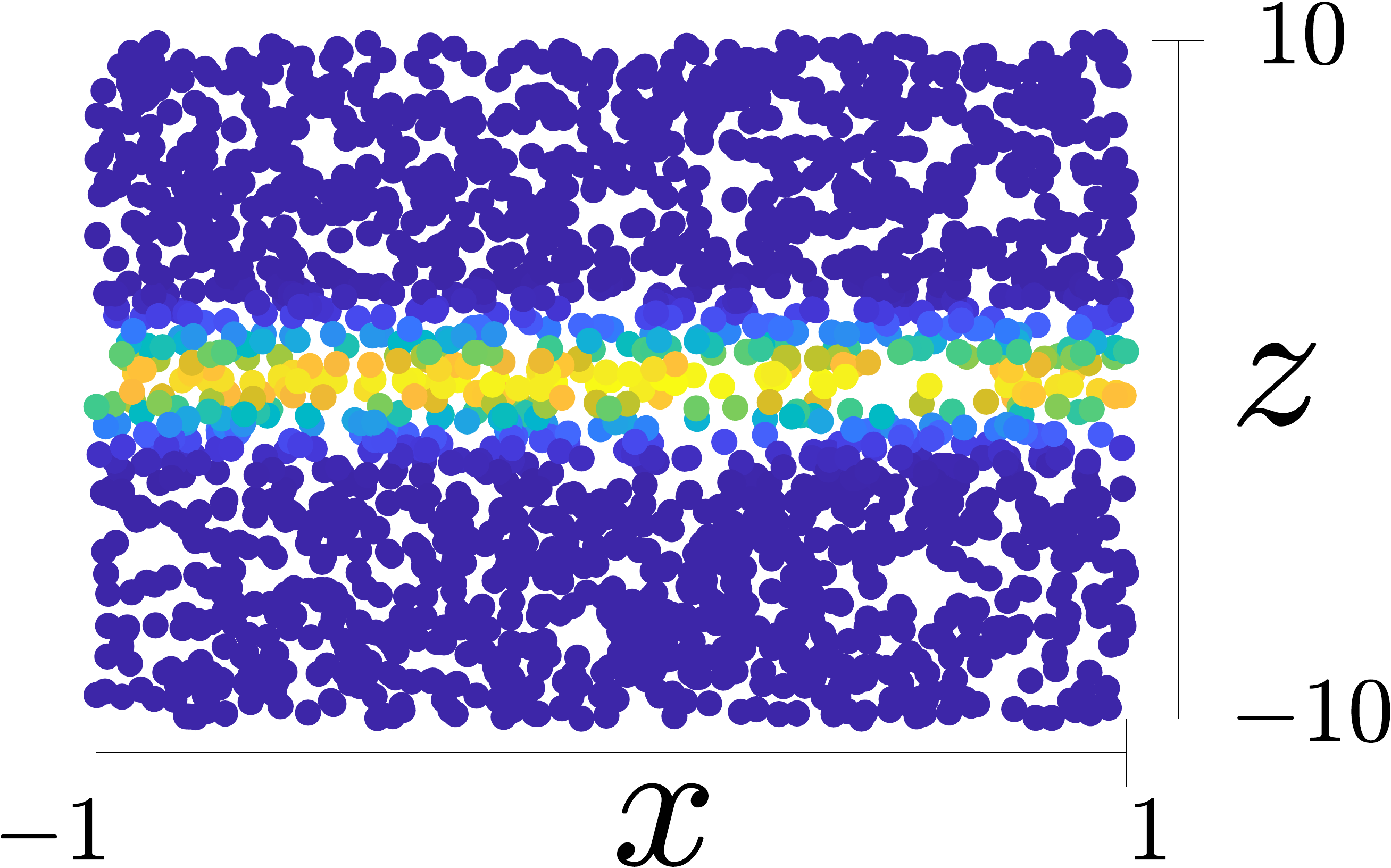}} \\
	\end{tabular}
	\caption{Same as Fig. \ref{fig:3DStrip_ModerateSNR_Diffs} only with a different set of scale parameters.}
	\label{fig:3DStrip_LowSNR_Diffs}
\end{figure}

In this example, the observable spaces $\mathcal{O}_1$ and $\mathcal{O}_2$ are the product of the latent manifolds $\mathcal{M}_x$, $\mathcal{M}_y$, and $\mathcal{M}_z$ (up to scaling).
We exploit the tractability of the spectrum of product manifolds in order to further investigate the differences between the proposed method based on interpolation along the geodesic path and linear interpolation.

Let $\mu_{1}^{(k_x,0)}$ and $\mu_{2}^{(k_x,0)}$ denote the ``discrete'' eigenvalues of $\mathbf{K}_1$ and $\mathbf{K}_2$ corresponding to the ``continuous'' eigenvalues $\lambda_{1}^{(k_x,0)}$ and $\lambda_{2}^{(k_x,0)}$ through the relation in \eqref{eq:Cont2Discrete}. Recall that these eigenvalues are associated with the common eigenvector $v^{(k_x,0)}$ related only to the common manifold $\mathcal{M}_x$.
By their explicit expression given in \eqref{eq:CommonEigenVals}, \Cref{prop:MutualEigenValues} entails that the eigenvalue of $\gamma(t)$ associated with $v^{(k_x,0)}$ is given by
\begin{align}
	\label{eq:2DStripCommonEigenVals_Geodesic}
	\mu_t^{(k_x,0)} &=\left(\mu_{1}^{(k_x,0)}\right)^{t} \left(\mu_{2}^{(k_x,0)}\right)^{1-t} \nonumber \\
	&=\exp\left( -\frac{1}{4} \left( (1-t) \left(\epsilon^{(1)} \right)^2\lambda_1^{(k_x,0)}+t  \left(\epsilon^{(2)}\right)^2 \lambda_2^{(k_x,0)}\right) \right) \nonumber \\
	&=\exp\left(-\left( \frac{ \pi}{4\ell^{(\gamma(t))}} \right)^2 k_x^2 \right),
\end{align}
where $\ell^{(t)}$ is given by
\begin{gather}
	\label{eq:EffectiveLength}
	\ell^{(t)} = \sqrt{ \frac{\bar{\ell}^{(1)}\bar{\ell}^{(2)}}{(1-t)\bar{\ell}^{(1)}+t\bar{\ell}^{(2)}} },
\end{gather}
with $\bar{\ell}^{(1)} = \ell_x^{(1)} / \epsilon^{(1)}$ and $\bar{\ell}^{(2)} = \ell_x^{(2)} / \epsilon^{(2)}$.

Remarkably, the expression obtained in \eqref{eq:2DStripCommonEigenVals_Geodesic} corresponds to the $k_x$th eigenvalue of a kernel associated with the Laplacian of the 1D manifold $[0,\ell^{(t)}]$. 
In other words, the eigenvalues of $\gamma(t)$ residing on log-linear trajectories admit Weyl’s law, and therefore, reconstructing a kernel only from these spectral components by:
\begin{equation*}
    \sum _{k_x \ge 0} \mu_t^{(k_x,0)} v^{(k_x,0)} (v^{(k_x,0)})^T
\end{equation*}
could be viewed as a kernel corresponding only to some effective common manifold\color{black}.

In contrast, if we repeat this derivation for the linear interpolation $\mathbf{L}(t)$, the eigenvalue of $\mathbf{L}(t)$ associated with the common eigenvector $v^{(k_x,0)}$ is given by
\begin{align}
	\label{eq:2DStripCommonEigenVals_Linear}
	\mu_{\mathbf{L}(t)}^{
		(k_x,0)} &= (1-t)\mu_{1}^{(k_x,0)}+t\mu_{2}^{(k_x,0)} \nonumber \\
	&=(1-t) \exp \bigg( - \frac{\left(\epsilon^{(1)}\right)^2}{4} \lambda_1^{(k_x,0)}\bigg) + t \exp \bigg( - \frac{\left(\epsilon^{(2)}\right)^2}{4} \lambda_2^{(k_x,0)}\bigg) \nonumber \\
	&=(1-t) \exp \left( -\left( \frac{\epsilon^{(1)} \pi}{4\ell^{(1)}_x} \right)^2 k_x^2  \right) + t \exp \left( -\left( \frac{\epsilon^{(2)} \pi}{4\ell^{(2)}_x} \right)^2 k_x^2  \right)
\end{align}
which generally cannot be expressed in an exponential form: $\exp(-c  k_x^2)$, where $c$ is some constant corresponding to the scale of the manifold.

We demonstrate this analysis in a simulation. Consider a setting with $\mathcal{M}_x = [-1/2,1/2]$, and suppose samples $x_i$ from $\mathcal{M}_x$ are drawn uniformly at random. The measurements are given by:
\begin{align*}
s^{(v)}_i=\ell^{(v)}_x x_i,
\label{eq:1DStripData}
\end{align*}
where $v=1,2$  and the scales $\ell^{(v)}_x$  are set to $\ell^{(1)}_x=1$ and $\ell^{(2)}_x=5$.

We pick two points, $\gamma(0.4)$ on the geodesic path connecting $\mathbf{K}_1$ and $\mathbf{K}_2$, and $\mathbf{L}(0.4)$ on the linear path connecting $\mathbf{K}_1$ and $\mathbf{K}_2$. We apply \acrshort*{EVD} to $\gamma(0.4)$ and to $\mathbf{L}(0.4)$, and depict the logarithm of the obtained eigenvalues in Fig. \ref{fig:1DStrip_FitDemos}. 
In addition, we plot the analytic expression of the eigenvalues of $\gamma(0.4)$ given in \eqref{eq:2DStripCommonEigenVals_Geodesic} and the fit of the eigenvalues of $\mathbf{L}(0.4)$ to the function $\exp(-\hat{c} k_x^2)$, where $k_x=1,2,\ldots,9$.
\begin{figure}[t]\centering
	\begin{tabular}{cc}
		\includegraphics[scale=0.12]{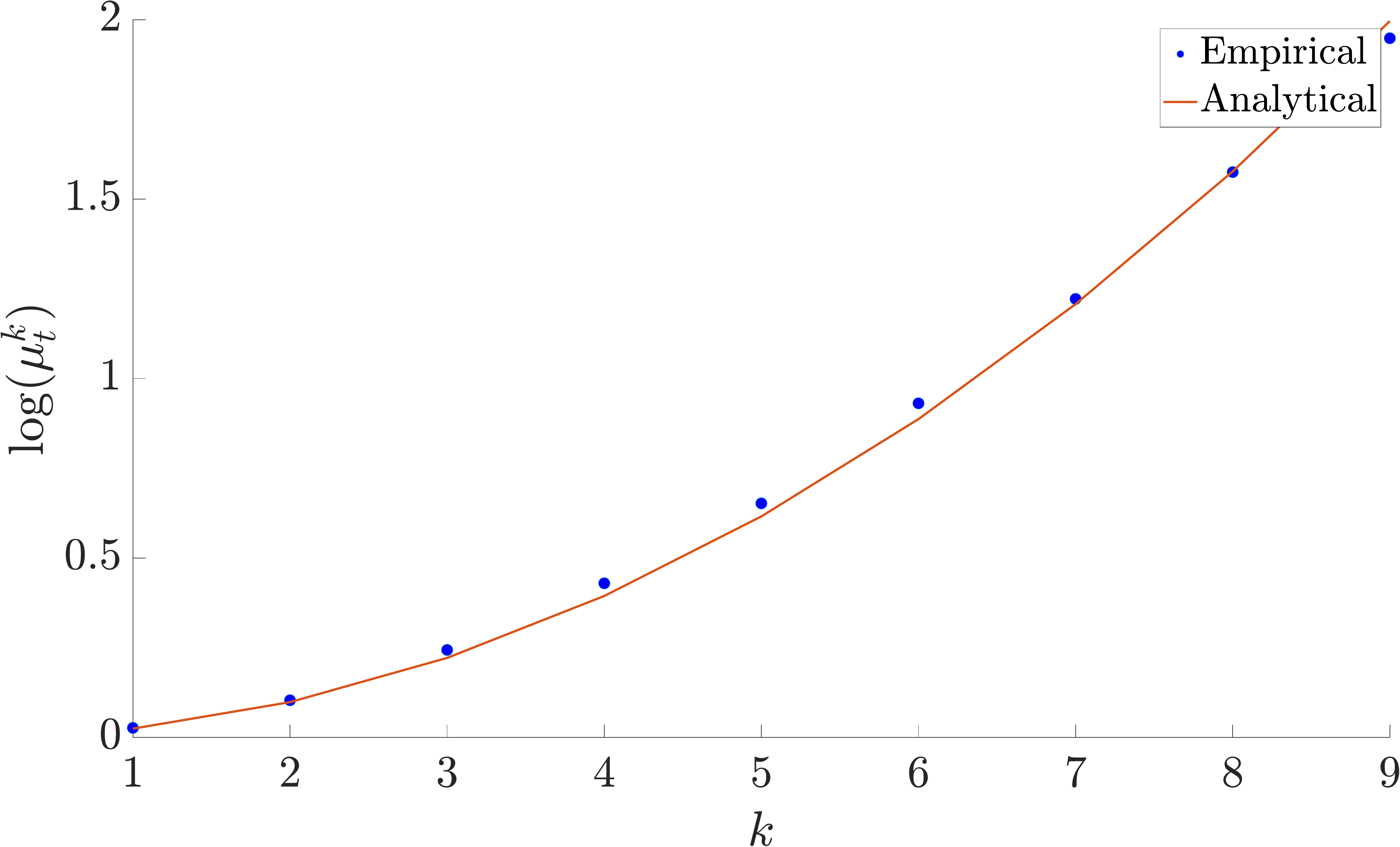} &
		\includegraphics[scale=0.12]{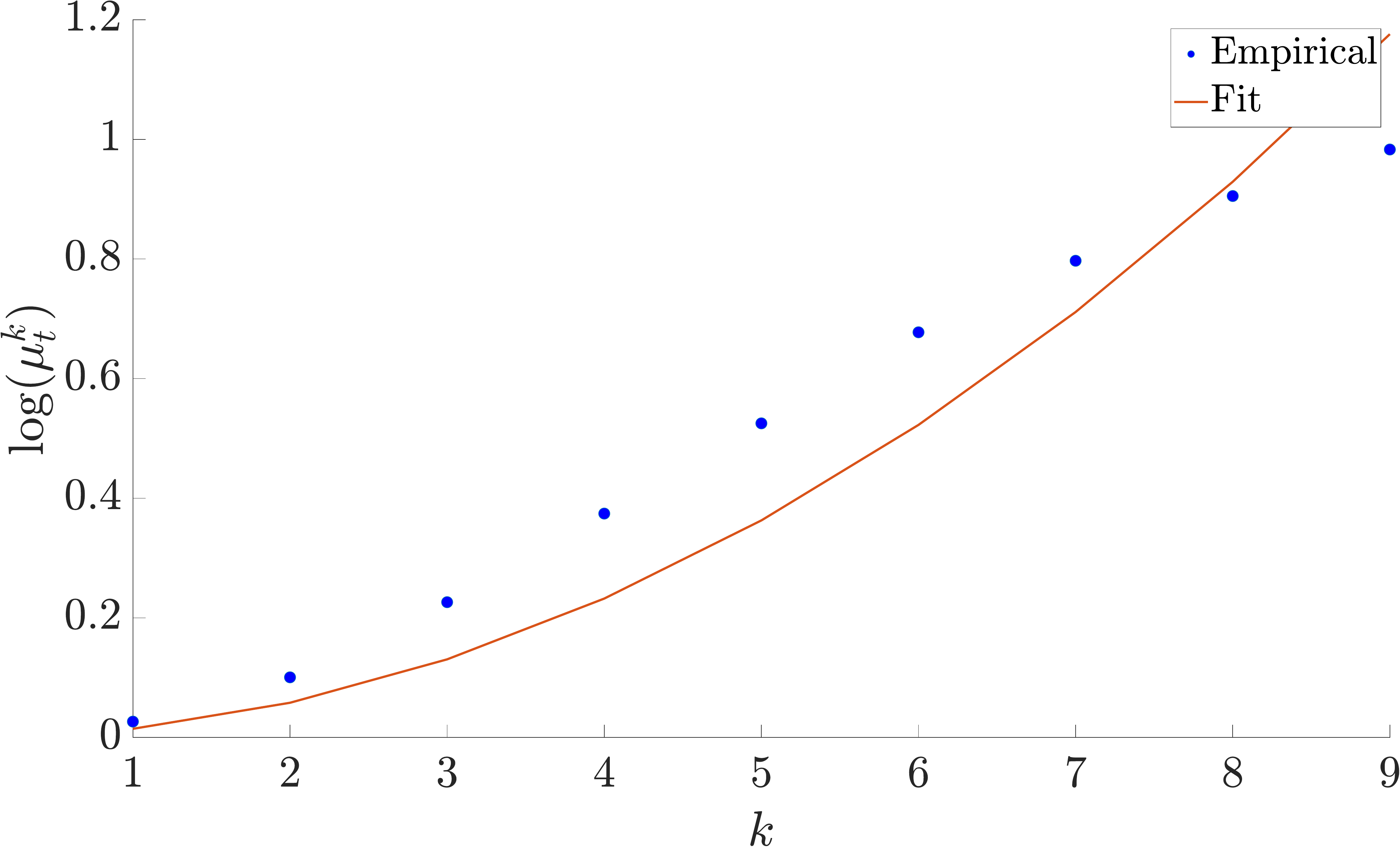}\\
		\hspace{0.2in} (a) & \hspace{0.2in} (b)
	\end{tabular}
	\caption{(a) The fit of the eigenvalues of $\gamma(0.4)$ to the analytic expression. (b) The fit of the eigenvalues of $\mathbf{L}(0.4)$ to $\exp(-\hat{c} k_x^2)$. The blue dots are the logarithm of the $10$ leading eigenvalues, and the red curve is the obtained fit.}
	\label{fig:1DStrip_FitDemos}
\end{figure}
Indeed, we observe that the eigenvalues of $\gamma(0.4)$ admit Weyl’s law, establishing the correspondence to a diffusion operator of a $1$D interval. In contrast, the eigenvalues of $\mathbf{L}(0.4)$ do not admit such a form. 
\color{black}

\subsection{Non-linear measurement functions}
\label{subsec:sup_nonlinear}
Thus far, the simulations involved only linear measurement functions $g$ and $h$
In this section, we consider nonlinear measurement functions $h$, so that the embedding of the product manifold $\mathcal{M}_x \times \mathcal{M}_z$ is no longer isometric, and as a consequence, the density of the corresponding measurements $\{s^{(2)}_i\}$ is no longer uniform. 

\subsubsection{Non-linear transformation of 2D flat manifolds}
\label{subsubsec:supp_nonlinear_2D}
We consider two measurement functions denoted by $h_1$ and $h_2$.
The first function $h_1$ is given by:
\begin{eqnarray*}
	h_1 (x_i,z_i)= \left( \ell^{(2)}_x x_i , \ell^{(2)}_{z} \left(1-\sqrt{z_i}\right)  \right)
\end{eqnarray*}
where the measurement-specific variable $z_i$ is non-linearly mapped, and the common variable $x_i$ is scaled.
The second function $h_2$ is given by:
\begin{eqnarray*}
	h_2 (x_i,z_i)=\left( \ell^{(1)}_x \left(1-\sqrt{x_i}\right) , \ell^{(2)}_{z} z_i  \right)
\end{eqnarray*}
where the common variable $x_i$ is non-linearly mapped, and the measurement-specific variable $z_i$ is scaled.
The observation function $g$ remains linear, and is given by
\begin{eqnarray*}
	s^{(1)}_i = g(x_i,y_i) = \left( \ell^{(1)}_x x_i , \ell^{(1)}_y y_i  \right)
\end{eqnarray*}

\begin{figure}[t]\centering
	\begin{tabular}{cc}
		\hspace{-0in} 
		\includegraphics[scale=0.12]{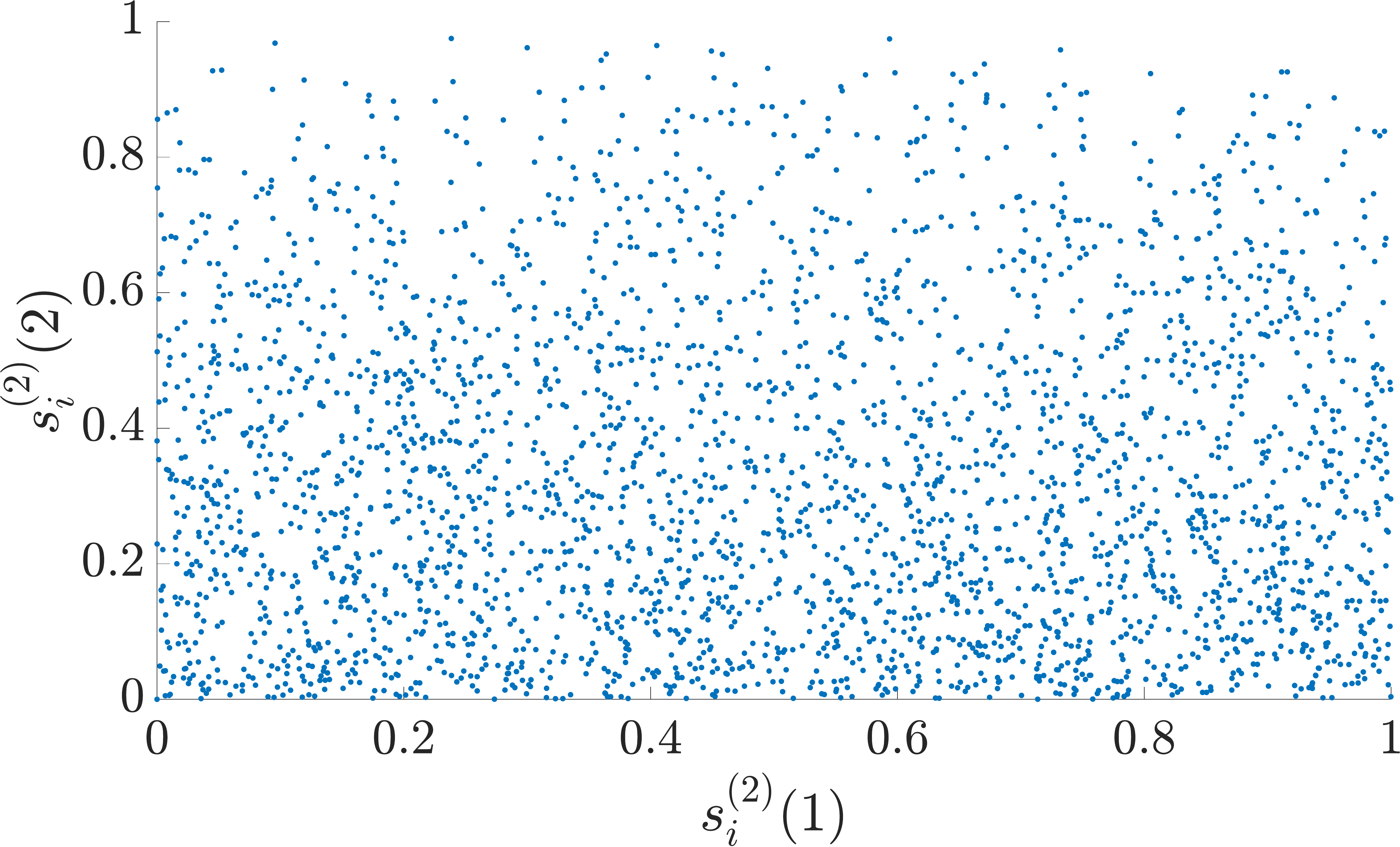} &
		\includegraphics[scale=0.12]{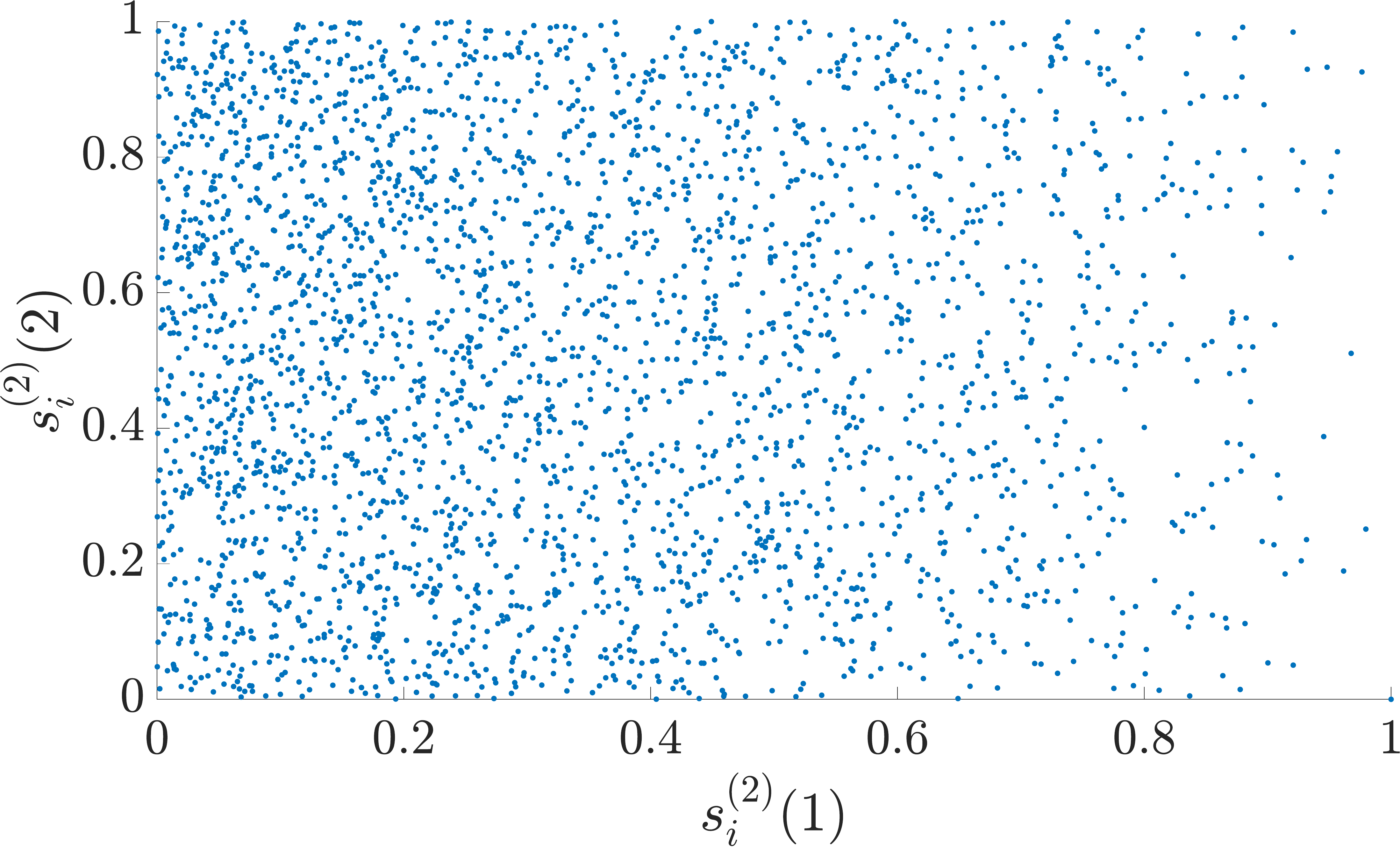}
		\\
		\hspace{0.4in}(a) &\hspace{0.15in} (b)
	\end{tabular}
	\caption {The obtained measurements $\{s^{(2)}_i\}_{i=1}^n$ via (a) $h_1$ and (b) $h_2$.}
	\label{fig:3DStrip_Density_Samples}
\end{figure}

As before, we consider the following three manifolds $\mathcal{M}_x = \mathcal{M}_y = \mathcal{M}_z = [0,1]$ and the following scaling parameters:
\begin{align*}
	\ell^{(1)}_x=1 ,\ell^{(2)}_x=1\\
	\ell^{(1)}_{y}=1,\ell^{(2)}_{z}=1
\end{align*}
We generate $n=2,000$ points $(x_i,y_i,z_i)$ where $x_i \in \mathcal{M}_x$, $y_i \in \mathcal{M}_y$, and $z_i \in \mathcal{M}_z$ are sampled uniformly and independently from each manifold.
Then, we calculate the corresponding measurements: $\{ s^{(1)}_i,s^{(2)}_i\}_{i=1}^n$ twice, once with $h_1$ and once with $h_2$. 
The obtained two sets of measurements $\{s^{(2)}_i\}_{i=1}^n$ are depicted in Fig. \ref{fig:3DStrip_Density_Samples}. Indeed, we observe the non-uniform density along the vertical axis in Fig. \ref{fig:3DStrip_Density_Samples}(a) (corresponding to the measurement-specific variable) and along the horizontal axis in Fig. \ref{fig:3DStrip_Density_Samples}(b) (corresponding to the common variable).

We apply \Cref{alg:EvfdCalculation} to each set with $N_t=200$ and $K=10$, and obtain the corresponding \acrshort*{EVFD}s, which are depicted in Fig. \ref{fig:3DStrip_Density_Evfd}.
\begin{figure}[t]\centering
	\begin{tabular}{cc}
		\hspace{-0in} 
		\includegraphics[scale=0.12]{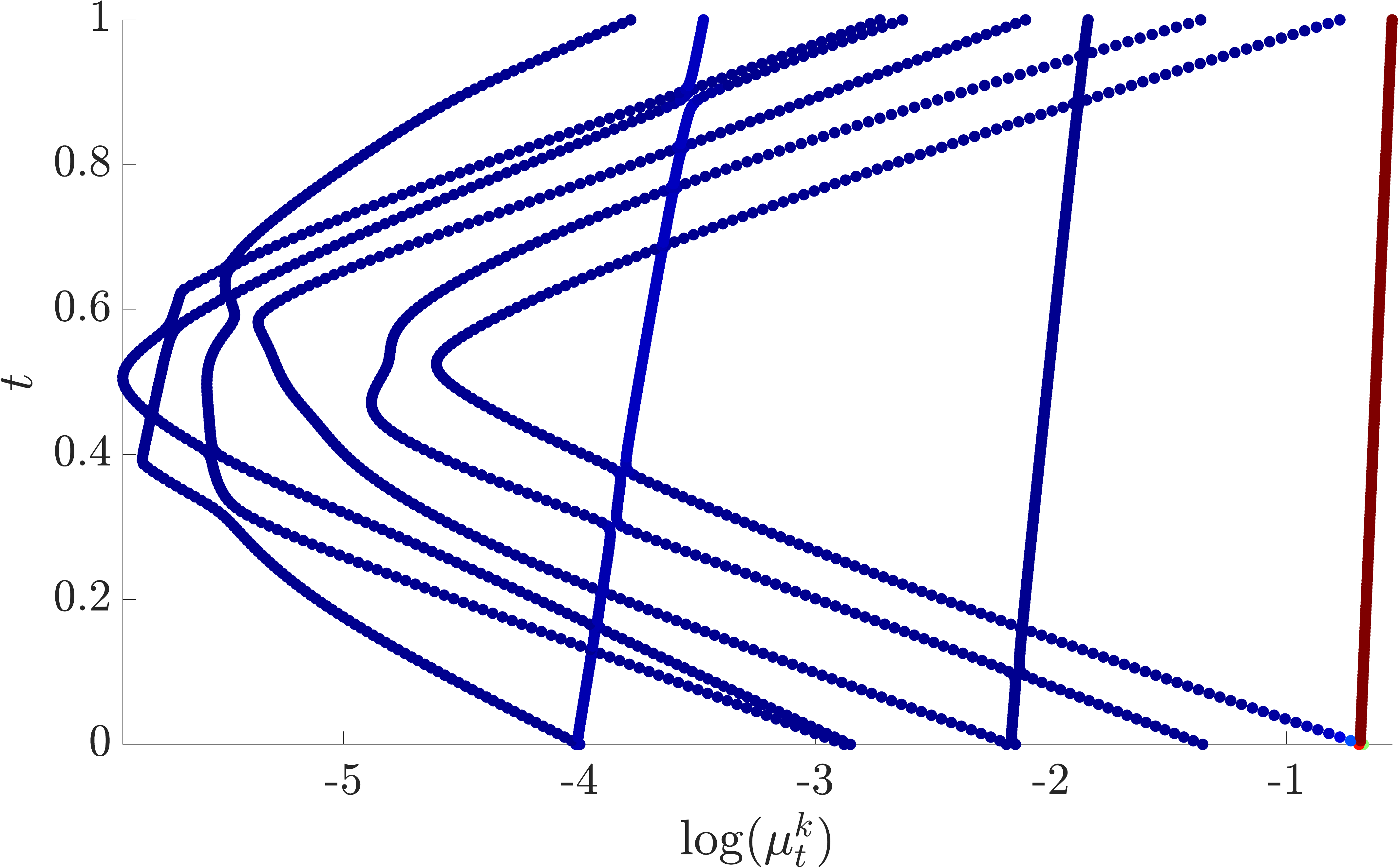} &
		\includegraphics[scale=0.12]{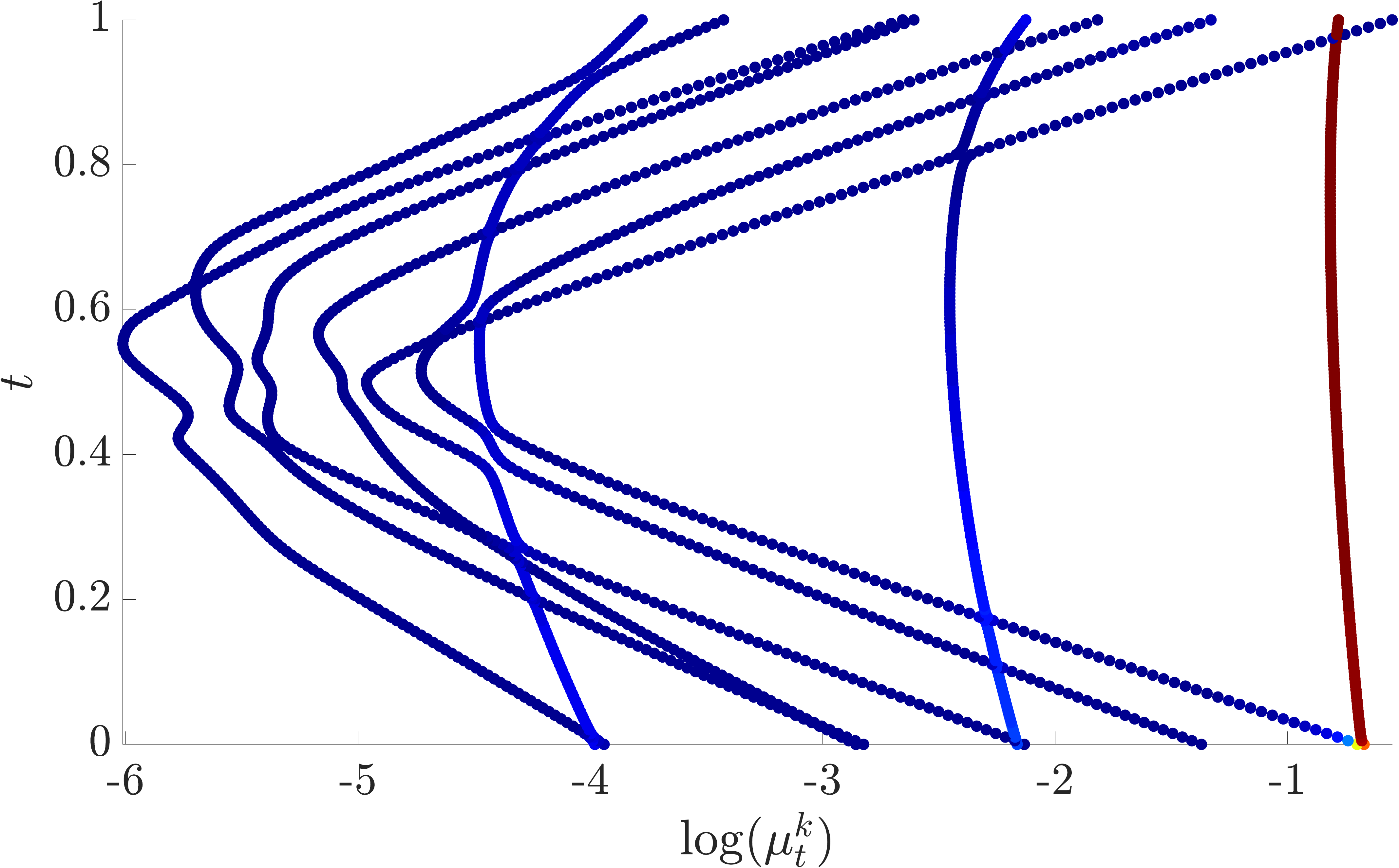}
		\\
		\hspace{0.4in}(a) &\hspace{0.15in} (b)
	\end{tabular}
	\caption{The \acrshort*{EVFD}s based on the measurements with (a) $h_1$ and (b) $h_2$. Each point (eigenvalue) in the diagrams is colored according to the correlation between its corresponding eigenvector and the vector of common variable realizations $(x_i)_{i=1,\ldots,n}$.}
	\label{fig:3DStrip_Density_Evfd}
\end{figure}
We observe in Fig. \ref{fig:3DStrip_Density_Evfd}(a) that the deformation caused by $h_1$ does not affect the log-linear trajectories of eigenvalues corresponding to common eigenvectors. In contrast, we observe in Fig. \ref{fig:3DStrip_Density_Evfd}(b) that the deformation caused by $h_2$ results in an \acrshort*{EVFD}, where the ``common'' eigenvectors (which are not strictly common eigenvectors in this case) are characterized by a slightly curved trajectories of eigenvalues. Yet, these trajectories are still very distinct compared to the trajectories corresponding to measurement-specific eigenvectors, and the high correlation of the principal common eigenvector with the common variable is maintained.

\subsubsection{2D tori in \texorpdfstring{$\mathbb{R}^3$}{Lg}}
\label{subsubsec:2DTorus}

Consider samples that are drawn uniformly and independently from three manifolds $\mathcal{M}_x = \mathcal{M}_y = \mathcal{M}_z = [0,2\pi]$, giving rise to the tuples $\left\{(x_i,y_i,z_i)\right\}_{i=1}^n$, where $x_i \in \mathcal{M}_x$, $y_i \in \mathcal{M}_y$, and $z_i \in \mathcal{M}_z$.
Suppose the measurements are given by:
\begin{align*}
    s^{(1)}_{i} &=\left(\left(R_1 + r_1  \cos\left(x_i\right)\right)  \cos\left({y}_i\right), \left(R_1 + r_1  \cos\left(x_i\right)\right)  \sin\left({y}_i\right), r_1  \sin\left(x_i\right)\right) \\
    s^{(2)}_{i} &=\left(\left(R_2 + r_2  \cos\left(x_i\right)\right)  \cos\left({z}_i\right),\left(R_2 + r_2  \cos\left(x_i\right)\right)  \sin\left({z}_i\right),r_2 \sin\left(x_i\right)\right),
\end{align*}
where the parameters $R_v$ and $r_v$ denote the major and minor radii in the measurement $v=1,2$. Specifically, here we set:
\begin{gather*}
R_1=10 ; r_1=5 \\
R_2=10 ; r_2=3.
\end{gather*}
In other words, we consider measurement functions that form 2D tori embedded in $\mathbb{R}^3$, where the common $x_i \in \mathcal{M}_x$ is the poloidal angle and measurement-specific $y_i\in \mathcal{M}_y$ and $z_i\in\mathcal{M}_z$ are the toroidal angles. The dominance of the measurement-specific manifold in each measurement depends on the ratio between the radii of the torus. 

\begin{figure}[t]\centering
	\label{subsubsec:2DTorusPolodial}
		\includegraphics[scale=0.25]{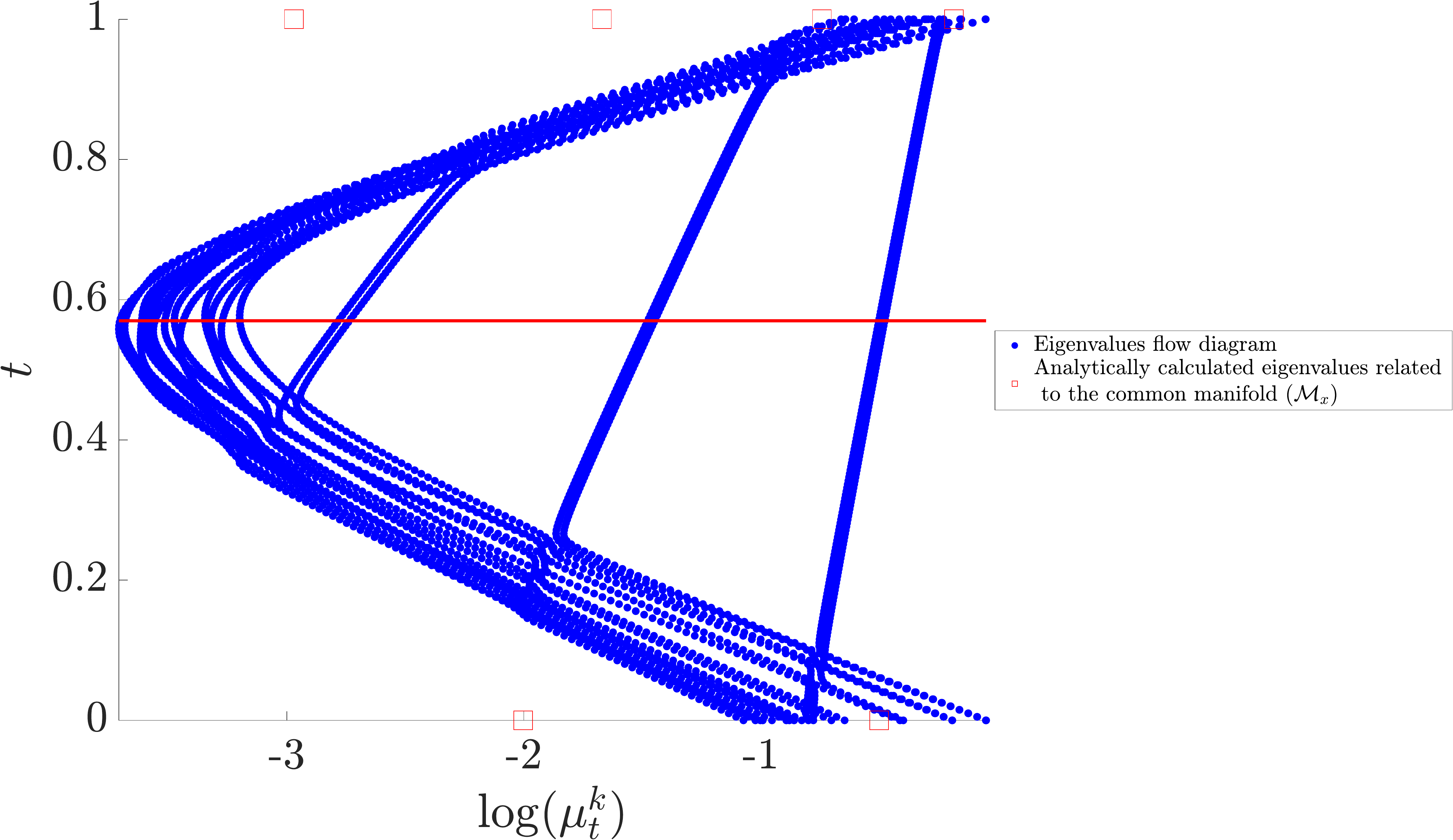}\\
	\caption {The \acrshort*{EVFD} in the 2D tori example. The point $t$ at which the CMR is maximal is marked by a horizontal red line. The analytically calculated eigenvalues at the boundaries of the diagram (at $t=0$ and $t=1$) are marked by squares.}
	\label{fig:3DTorus_PolodialCommon_LowSNR_EVDiagram}
\end{figure}
We compute the \acrshort*{EVFD} by applying \Cref{alg:EvfdCalculation} to the two sets of measurements and present it in Fig. \ref{fig:3DTorus_PolodialCommon_LowSNR_EVDiagram}.
Same as in the previous example, we analytically compute the eigenvalues of the Laplace-Beltrami operator defined on the observable manifolds (2D-tori), and we overlay the eigenvalues associated with the common manifold on the diagram. The eigenvalues of the Laplace-Beltrami operator (with Neumann boundary conditions) on the manifolds $\mathcal{O}_1 = \mathcal{O}_2 = \mathcal{T}^2$ are given by (see Lemma 6.5.1 in \cite{SpectralGraphsSpeilman}):
\begin{align}\label{eq:2d_tori_anal_eigv}
	\lambda_{1}^{(k_x,k_y)}&=\left( \frac{\left\lfloor \frac{k_x}{2} \right\rfloor}{2r_1} \right)^2 + \left( \frac{\left\lfloor \frac{k_y}{2} \right\rfloor}{2R_1} \right)^2 \nonumber \\
	\lambda_{2}^{(k_x,k_z)}&=\left( \frac{\left\lfloor \frac{k_x}{2} \right\rfloor}{2r_2} \right)^2 + \left( \frac{\left\lfloor \frac{k_z}{2} \right\rfloor}{2R_2} \right)^2,
\end{align}
and the eigenvalues of the discrete operator are computed through the relation in \eqref{eq:Cont2Discrete} (according to equation 7 \cite{dsilva2018parsimonious}). We note that the multiplicity of the eigenvalues is $2$ due to the periodic boundaries of the torus, which is apparent in the \acrshort*{EVFD}. 
In addition, we observe that the log-linear trajectories of eigenvalues coincide with the analytically computed eigenvalues corresponding to the common manifold (with indices $k_y=0$ and $k_z=0$).

In Fig. \ref{fig:3DTorus_PolodialCommon_LowSNR_Diffs}, we present the diffusion propagation patterns at $t=0,0.3,t^*,1$. We observe that only at $t^*$, the diffusion is along the poloidal angle, which corresponds to the common manifold. 

\begin{figure}[t]\centering
	\begin{tabular}{ccc}
	    & $\mathcal{O}_1$ & $\mathcal{O}_2$\\ \\
		$t=0$ & \hspace{-0in} \centered{\includegraphics[width=0.3\textwidth]{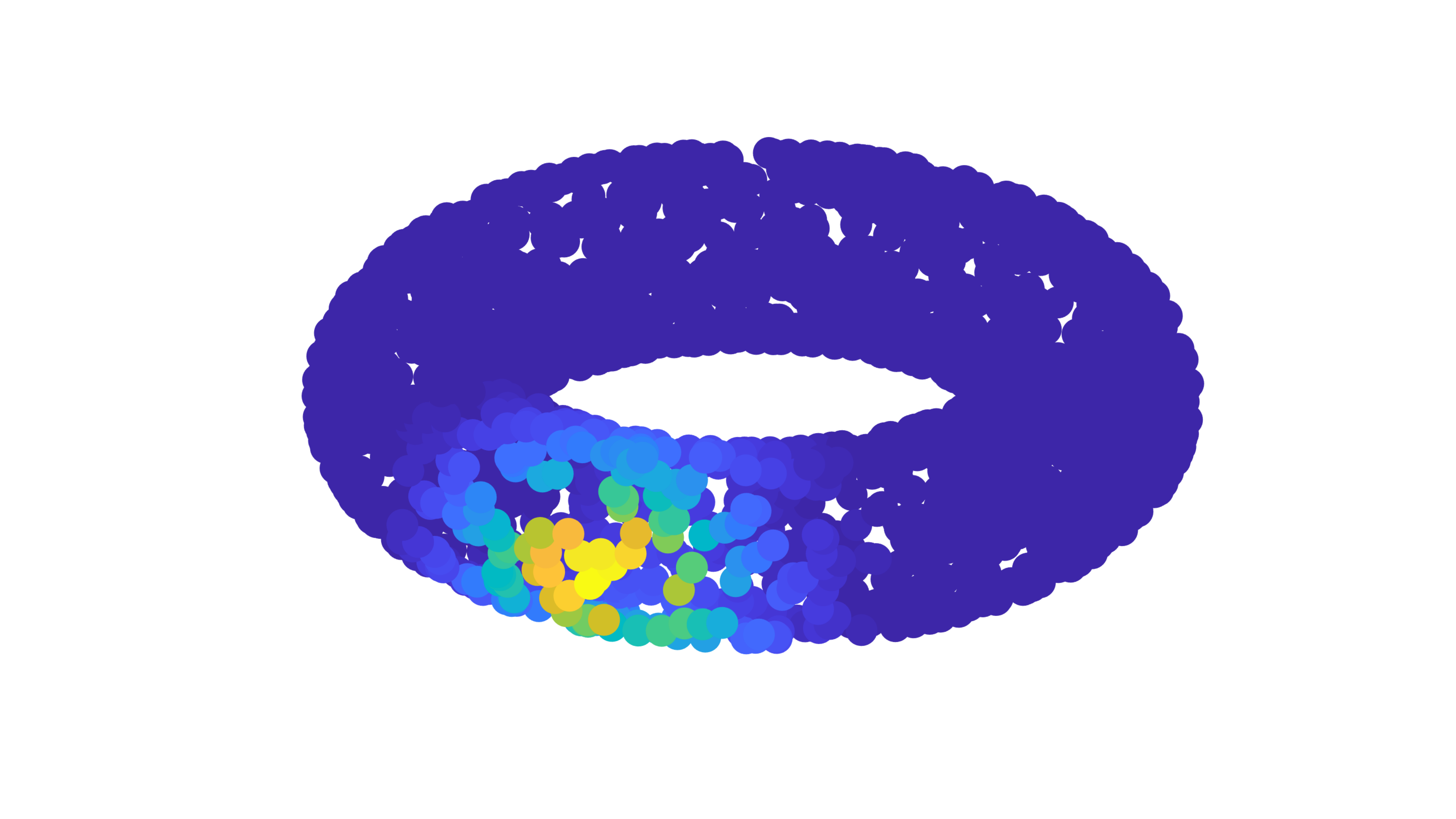}} &  \centered{\includegraphics[width=0.3\textwidth]{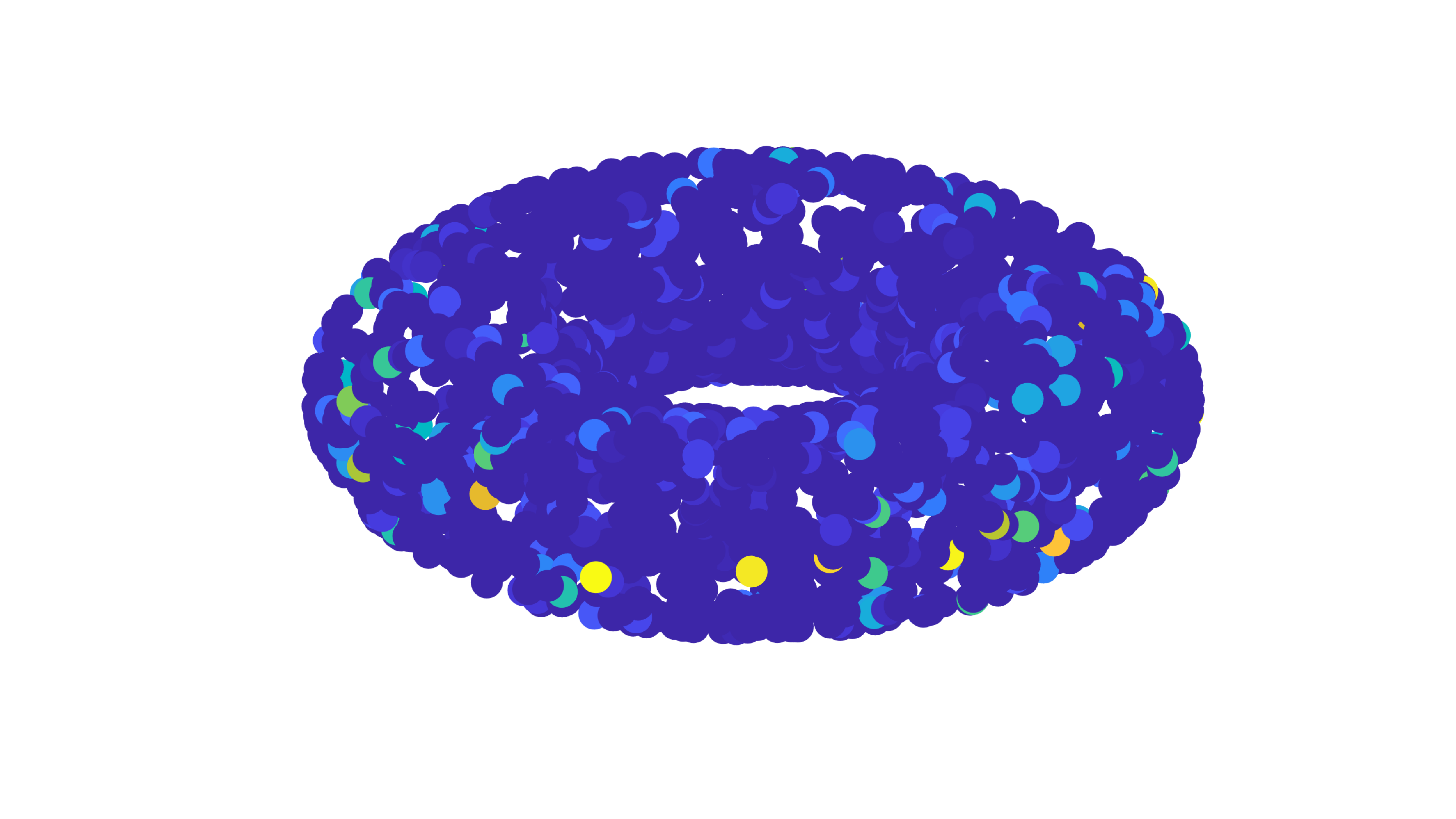}} \\
		$t=0.3$ & \centered{\includegraphics[width=0.3\textwidth]{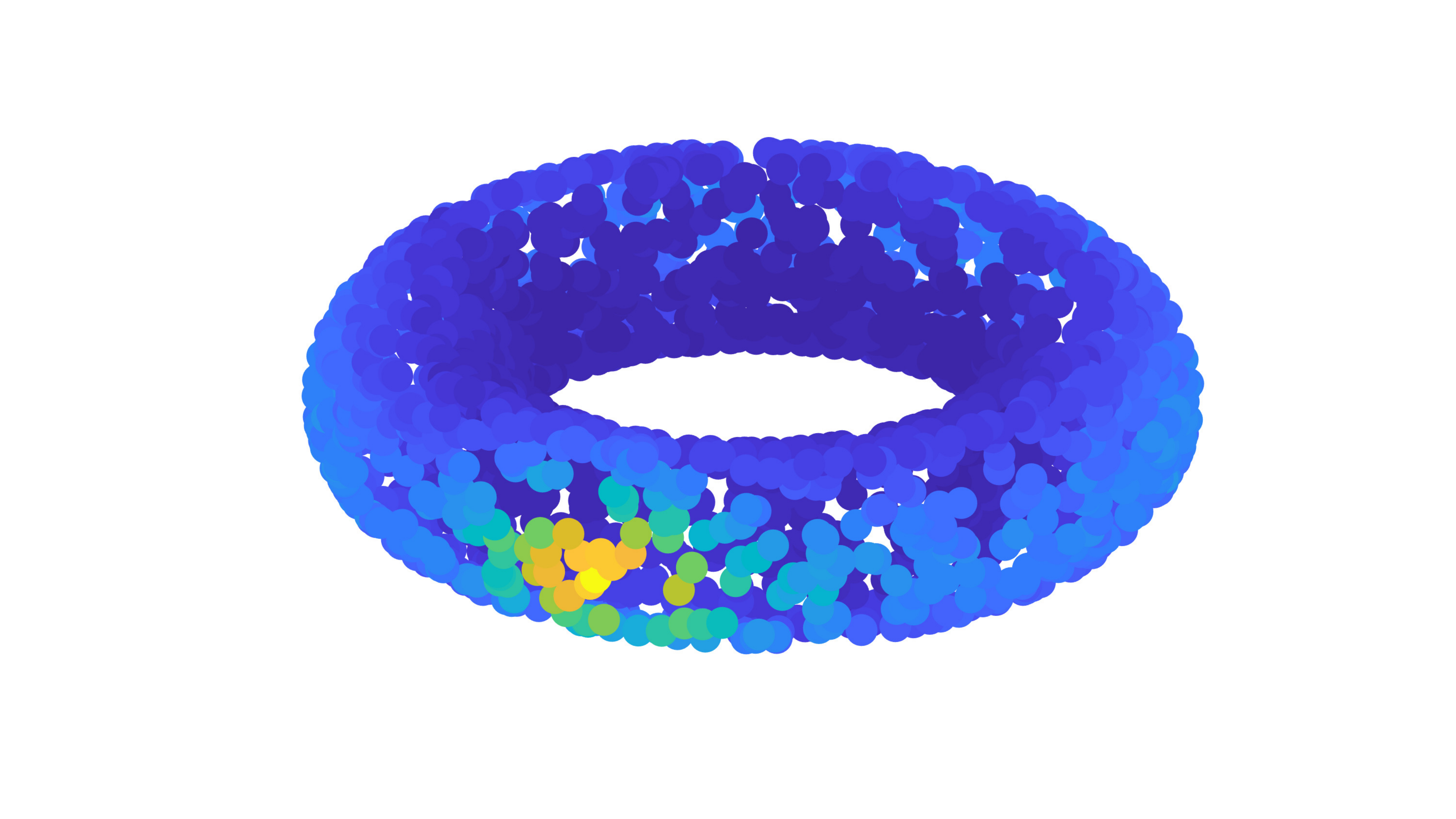}} &
		\centered{\includegraphics[width=0.3\textwidth]{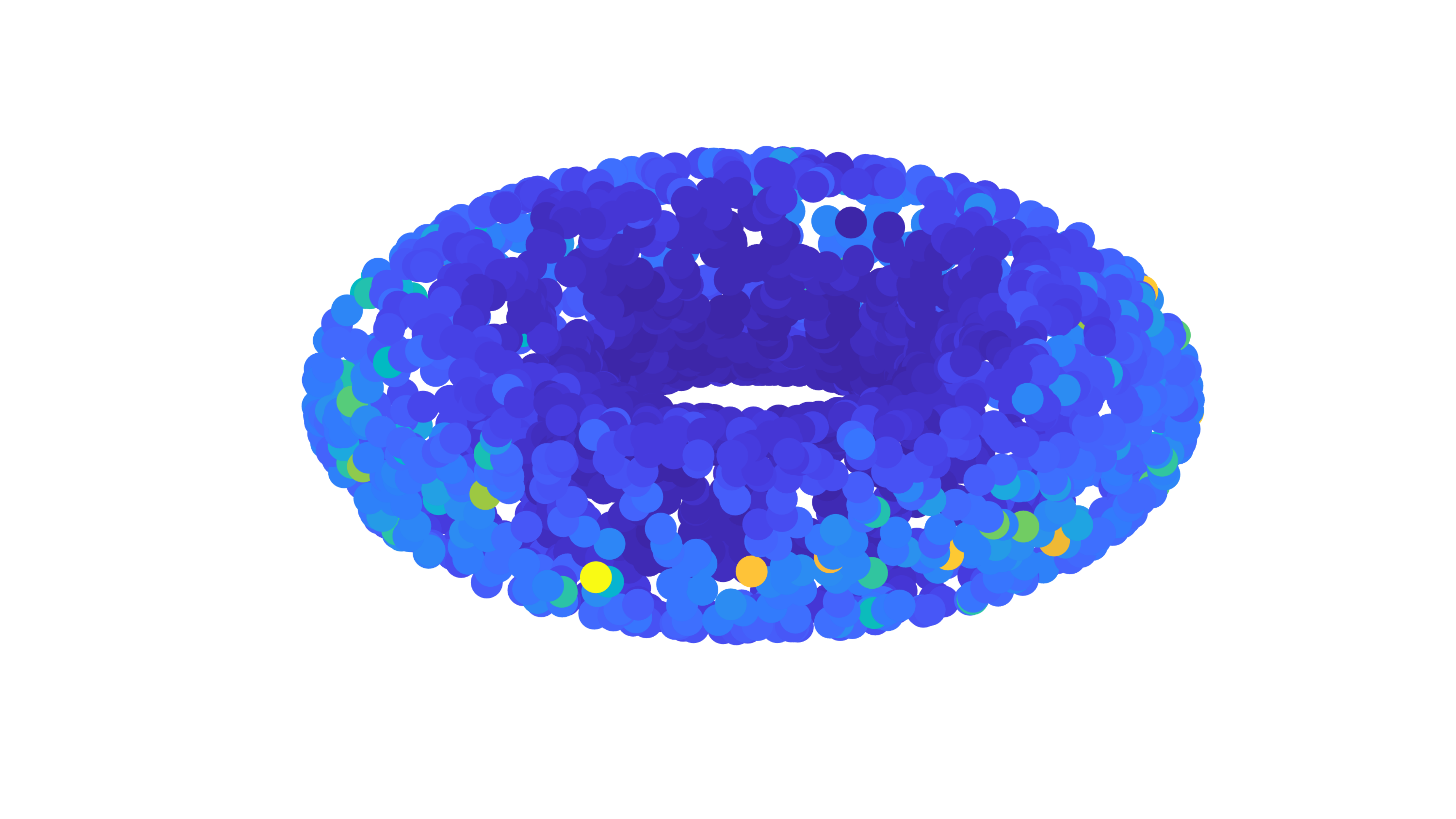}} \\
		$t=t^{*}$ & \centered{\includegraphics[width=0.3\textwidth]{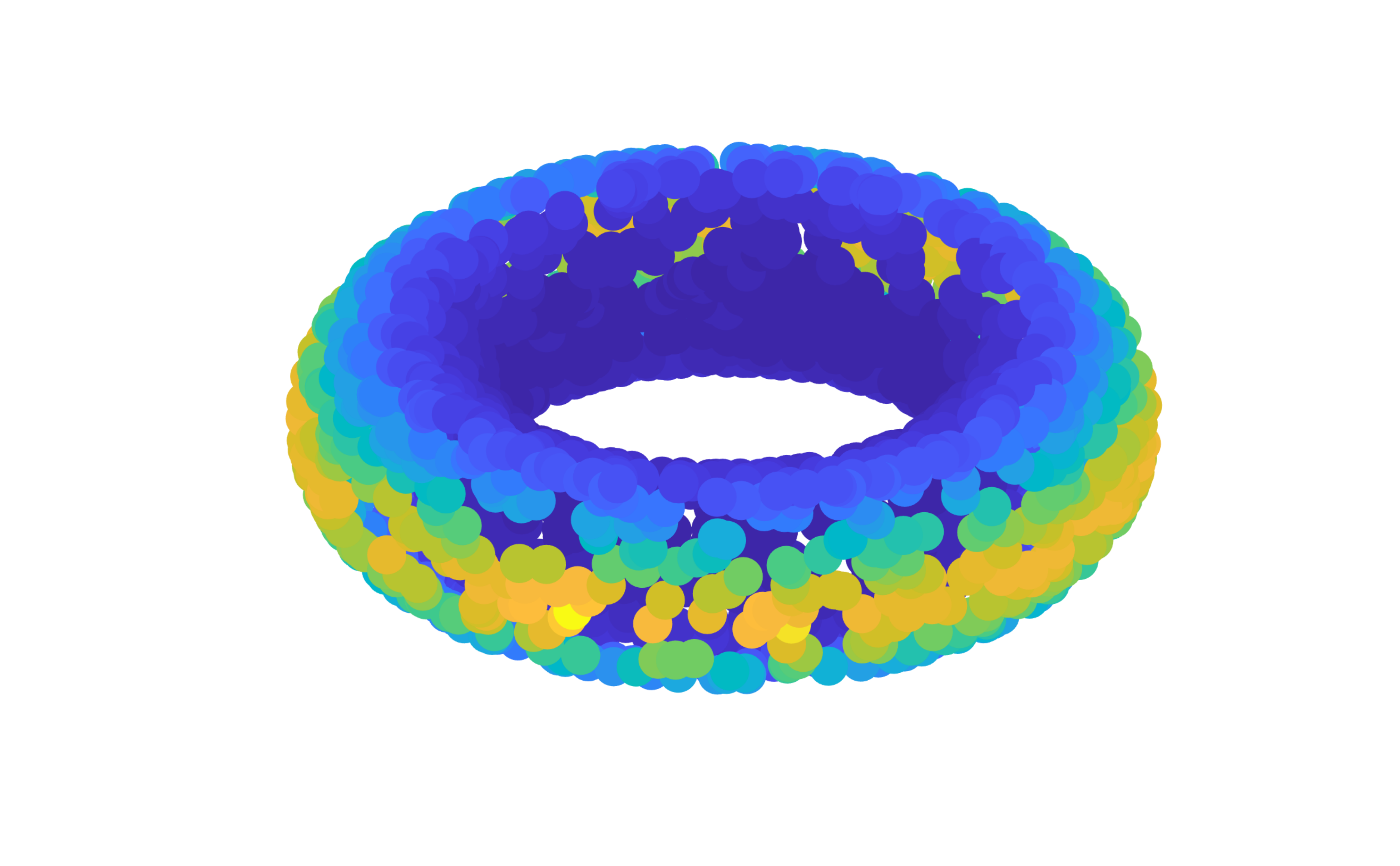}} &
		\centered{\includegraphics[width=0.3\textwidth]{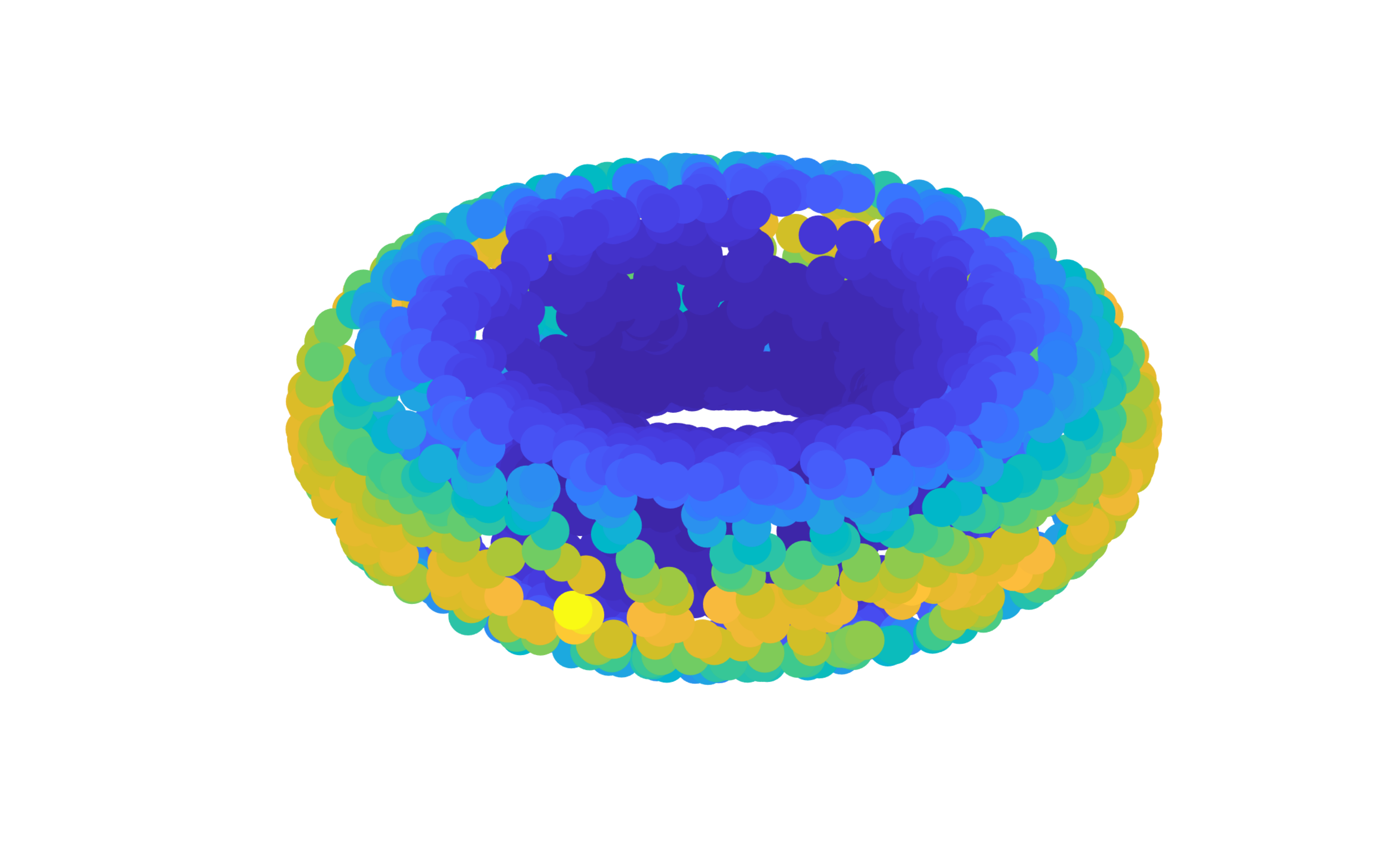}} \\
		$t=1$ &\centered{\includegraphics[width=0.3\textwidth]{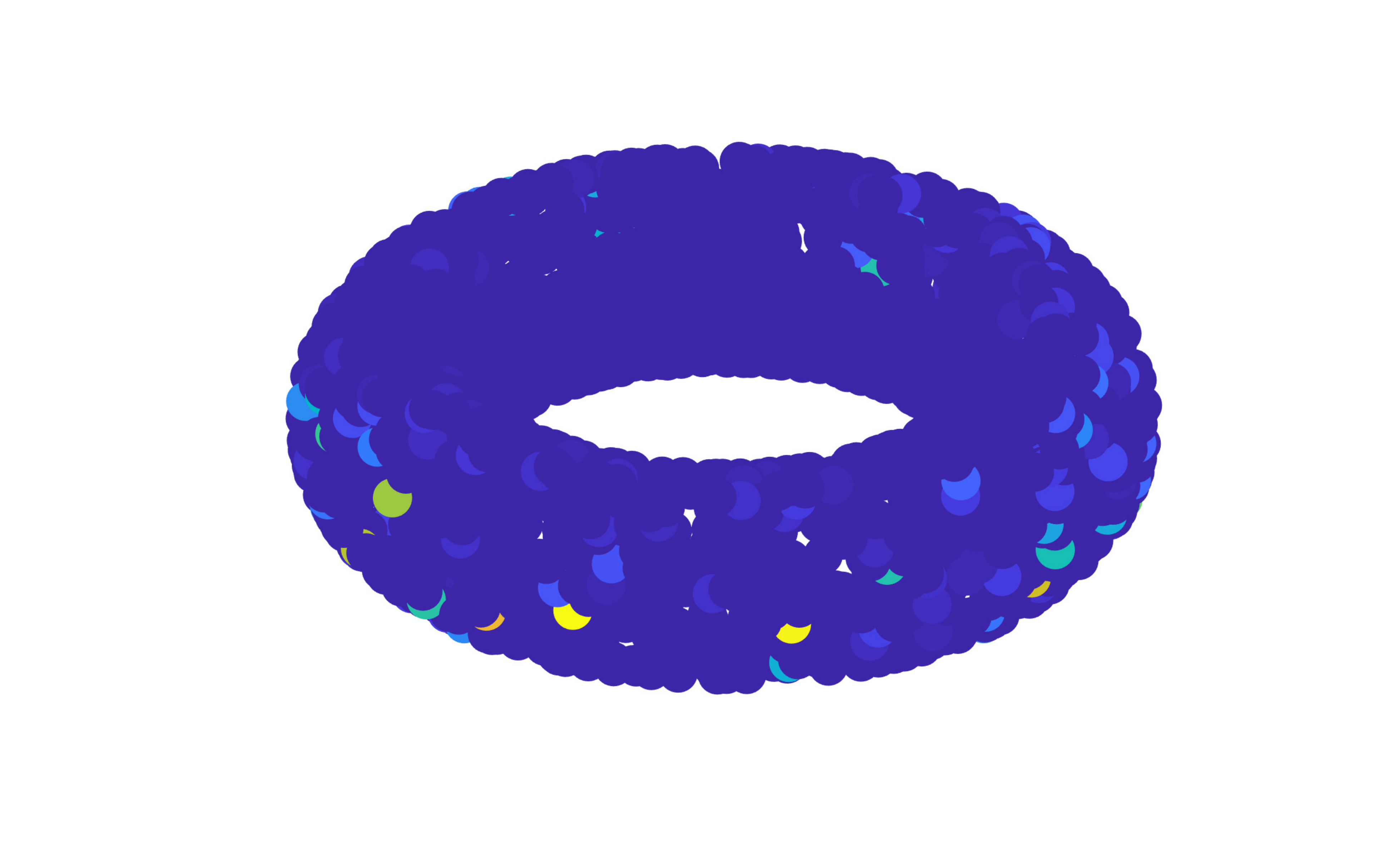}} &
		\centered{\includegraphics[width=0.3\textwidth]{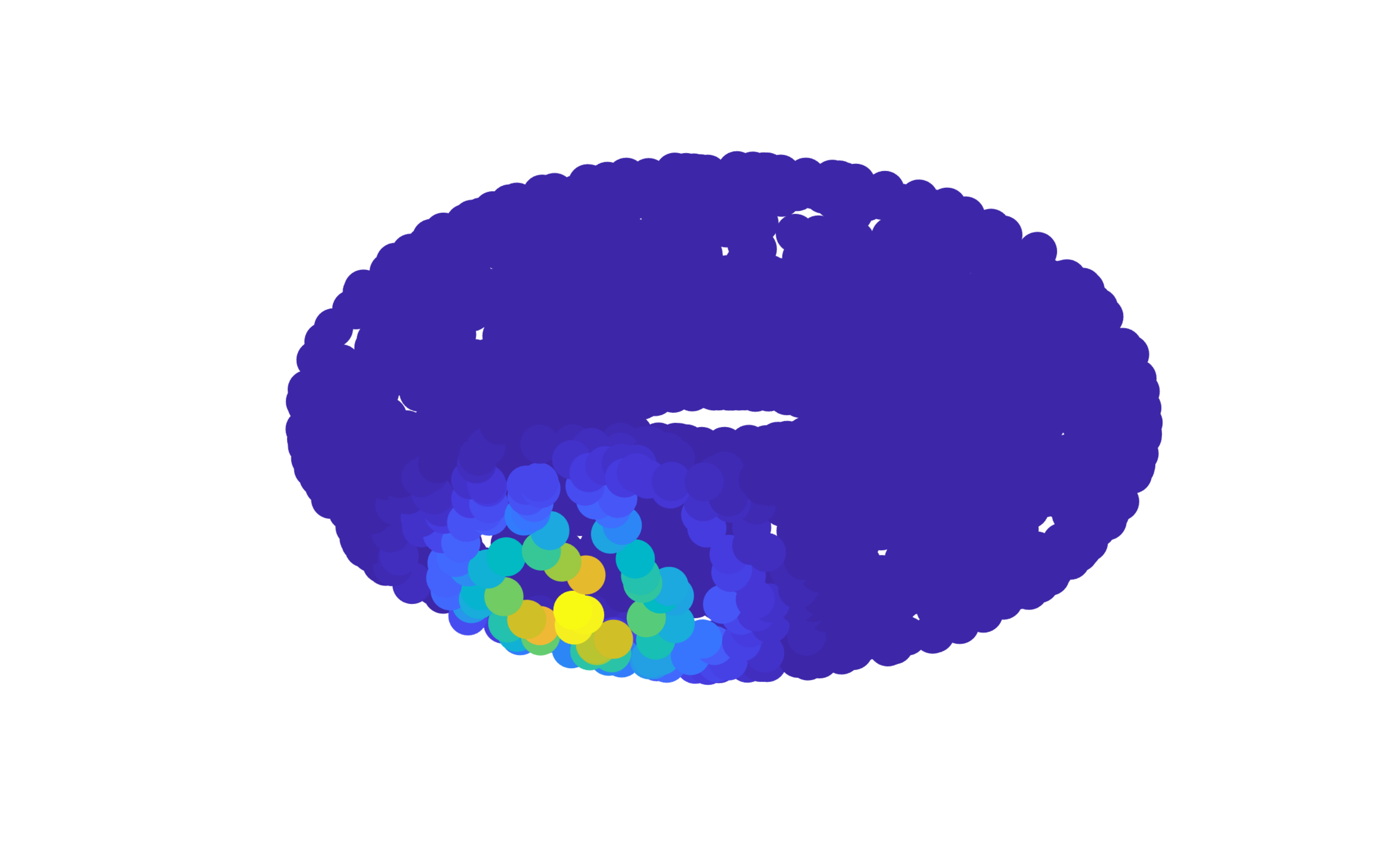}}\\
	\end{tabular}
	\caption{Same as Fig. \ref{fig:3DStrip_ModerateSNR_Diffs} only for the 2D tori example.}
	\label{fig:3DTorus_PolodialCommon_LowSNR_Diffs}
\end{figure}

We repeat the latter simulation, but with a setting in which the common $x_i \in \mathcal{M}_x$ represents the torodial angle in each measurement, i.e., the data in each measurement is given by the following expressions:
\begin{eqnarray*}
	s^{(1)}_{1,i}=\left(R_1 + r_1 \cos\left(y_i\right)\right)  \cos\left({x}_i\right)  &;& 
	s^{(2)}_{1,i}=\left(R_2 + r_2 \cos\left(z_i\right)\right)  \cos\left({x}_i\right)  \\		s^{(1)}_{2,i}=\left(R_1 + r_1 \cos\left(y_i\right)\right) \sin\left({x}_i\right) &;& 
	s^{(2)}_{2,i}=\left(R_2 + r_2 \cos\left(z_i\right)\right)  \sin\left({x}_i\right) \\
	s^{(1)}_{3,i}=r_1 \sin\left(y_i\right) &;&
	s^{(2)}_{3,i}=r_2 \sin\left(z_i\right)
\end{eqnarray*}
The resulting \acrshort{EVFD} is depicted in Fig. \ref{fig:3DTorus_TorodialCommon_HighSNR_SNR}.
\begin{figure}[t]\centering
	\begin{tabular}{cc}
		\includegraphics[scale=0.2]{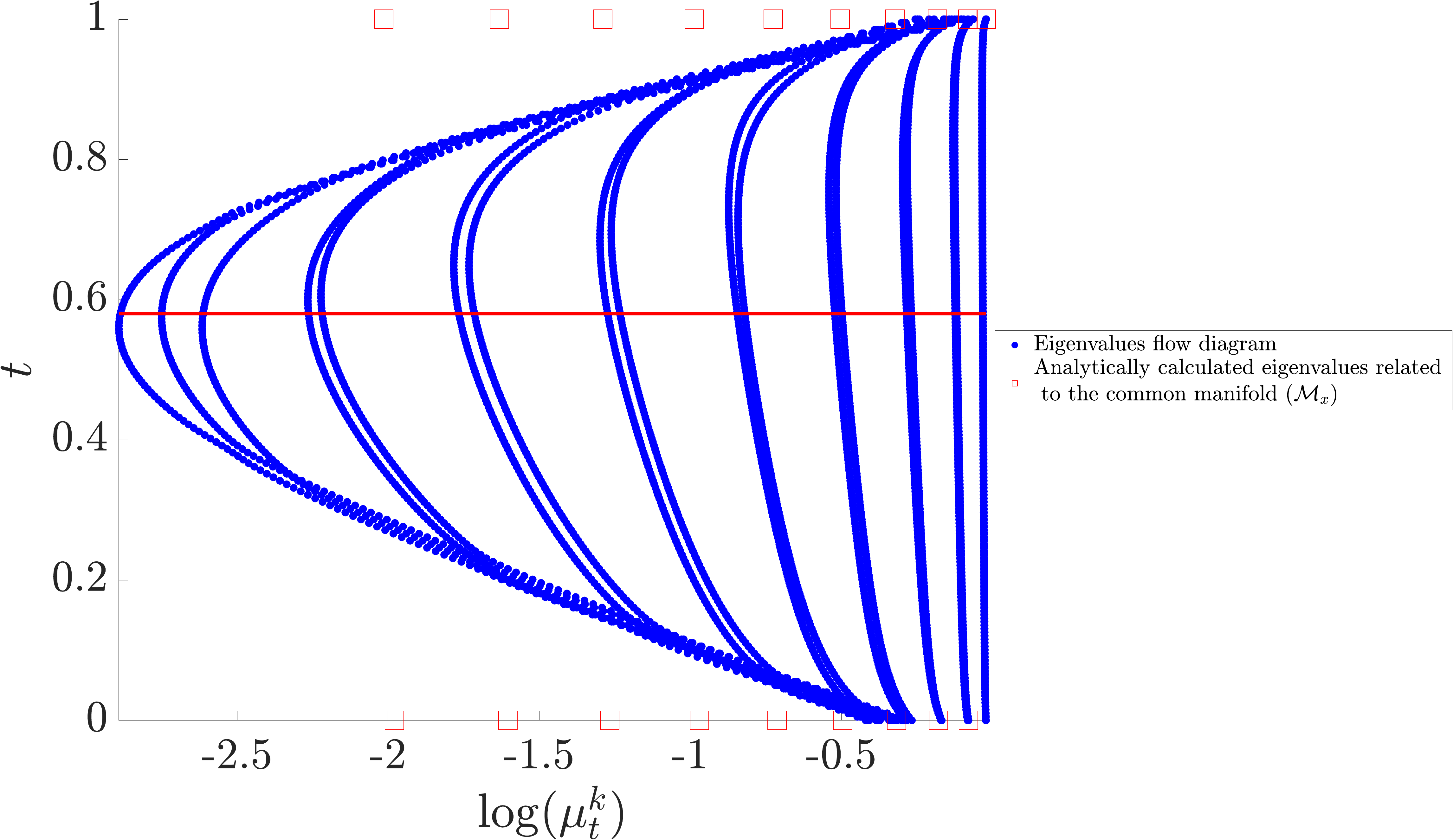}
		\\
	\end{tabular}
	\caption {Same as Fig. \ref{fig:3DTorus_PolodialCommon_LowSNR_EVDiagram} only for the setting where the common variable is the torodial angle.}
	\label{fig:3DTorus_TorodialCommon_HighSNR_SNR}
\end{figure}
As before, we see that the common eigenvectors appear as log-linear trajectories of eigenvalues. In addition, we clearly observe the multiplicity of the eigenvalues in this diagram.
For completeness, the corresponding diffusion patterns at $t=0, t=0.3, t=t^*$ and $t=1$ are shown in Fig. \ref{fig:3DTorus_TorodialCommon_HighSNR_Diffs}.
\begin{figure}[t]\centering
	\begin{tabular}{ccc}
	    & $\mathcal{O}_1$ & $\mathcal{O}_2$\\ \\
		$t=0$ & \hspace{-0in} \centered{\includegraphics[scale=0.1]{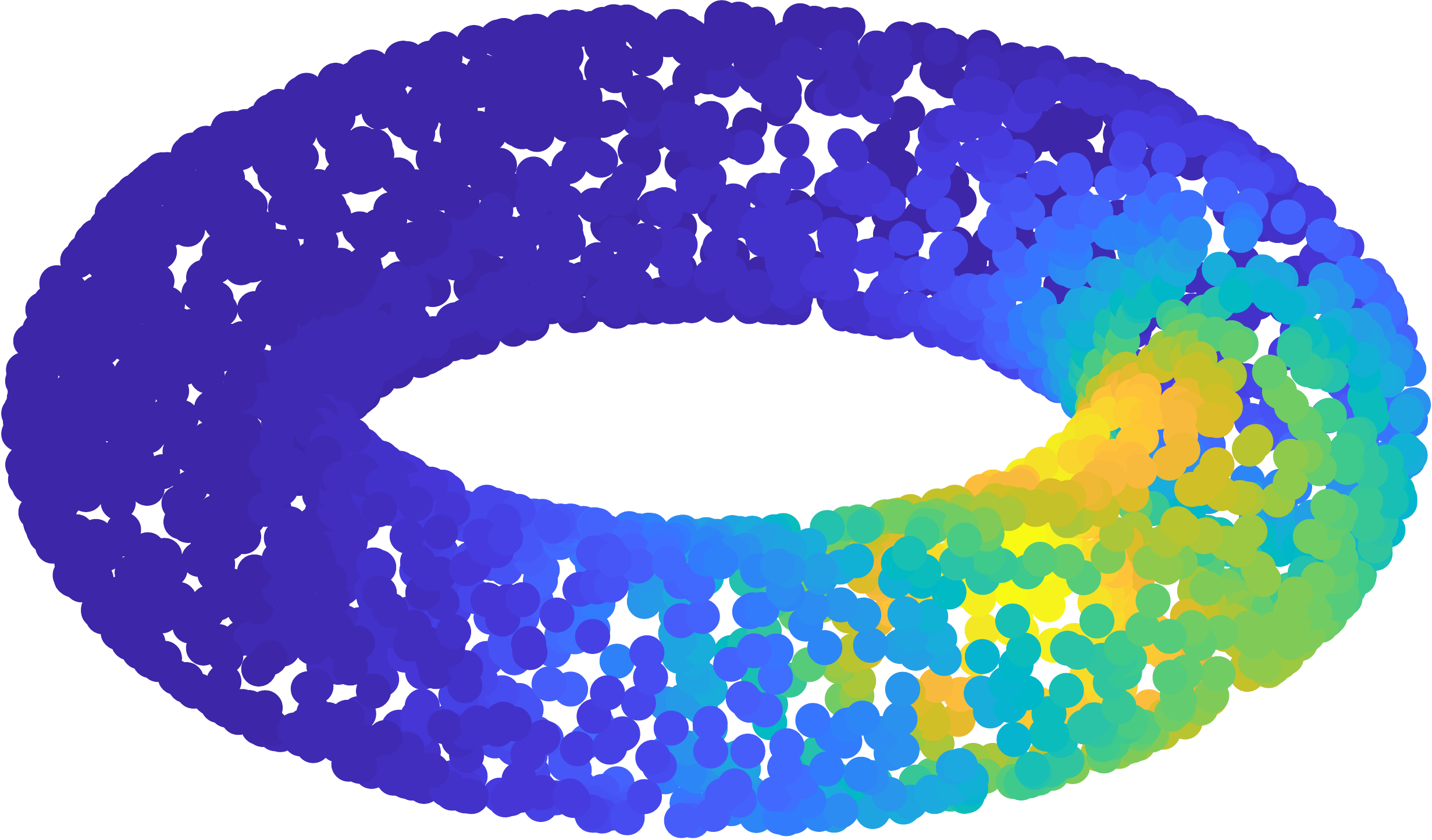}} &  \centered{\includegraphics[scale=0.1]{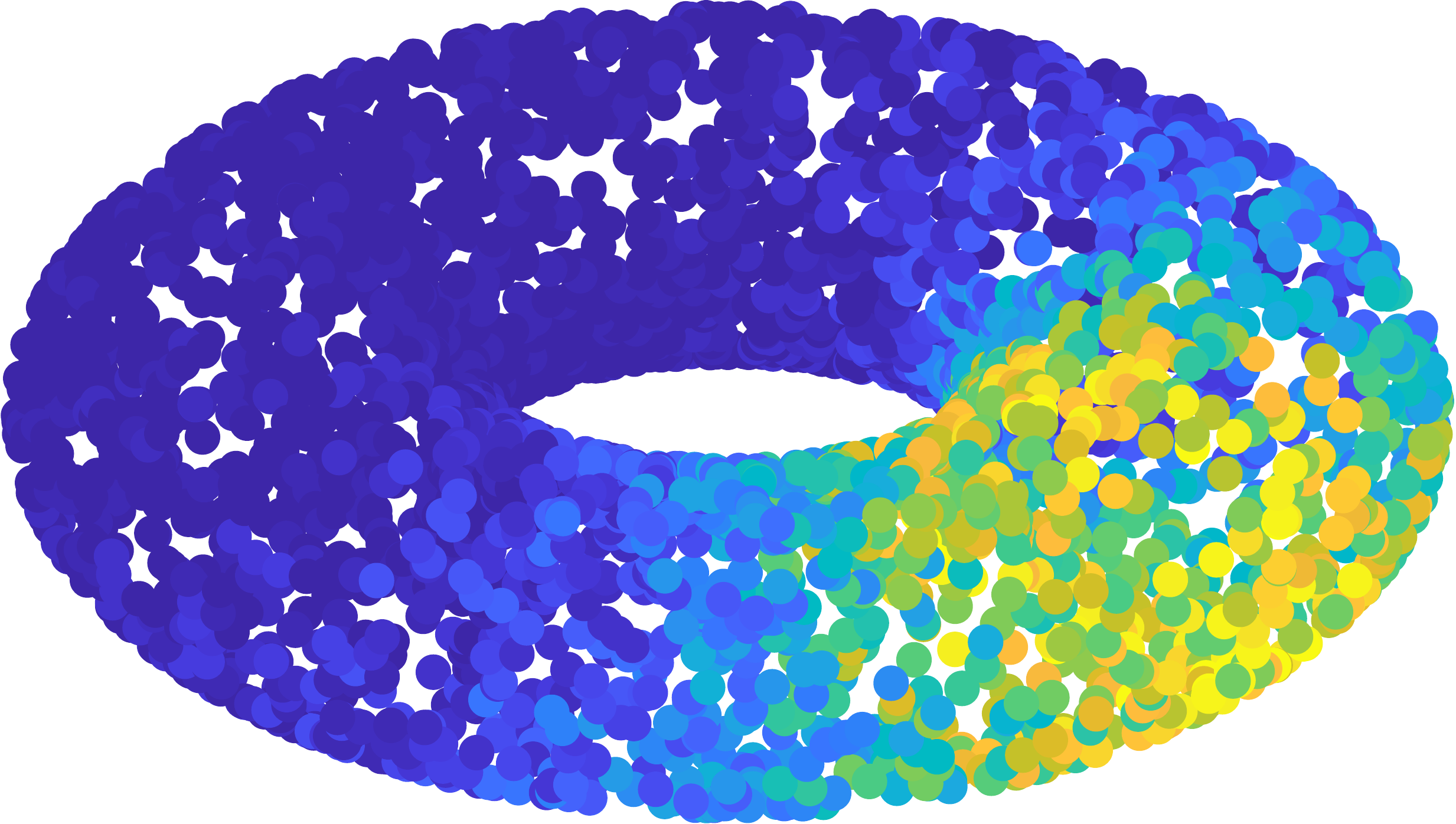}} \\
		$t=0.3$ & \centered{\includegraphics[scale=0.1]{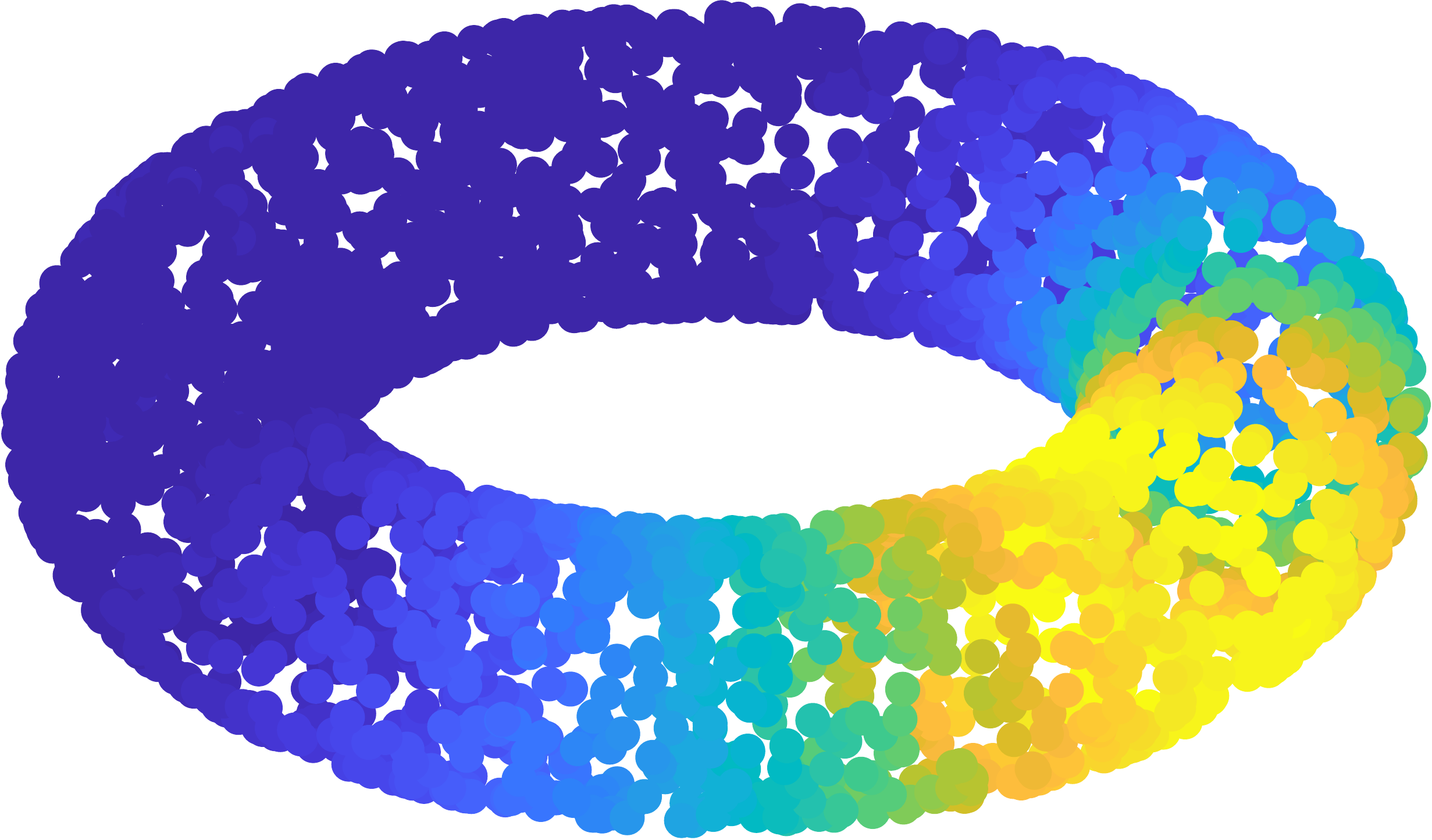}} &
		\centered{\includegraphics[scale=0.1]{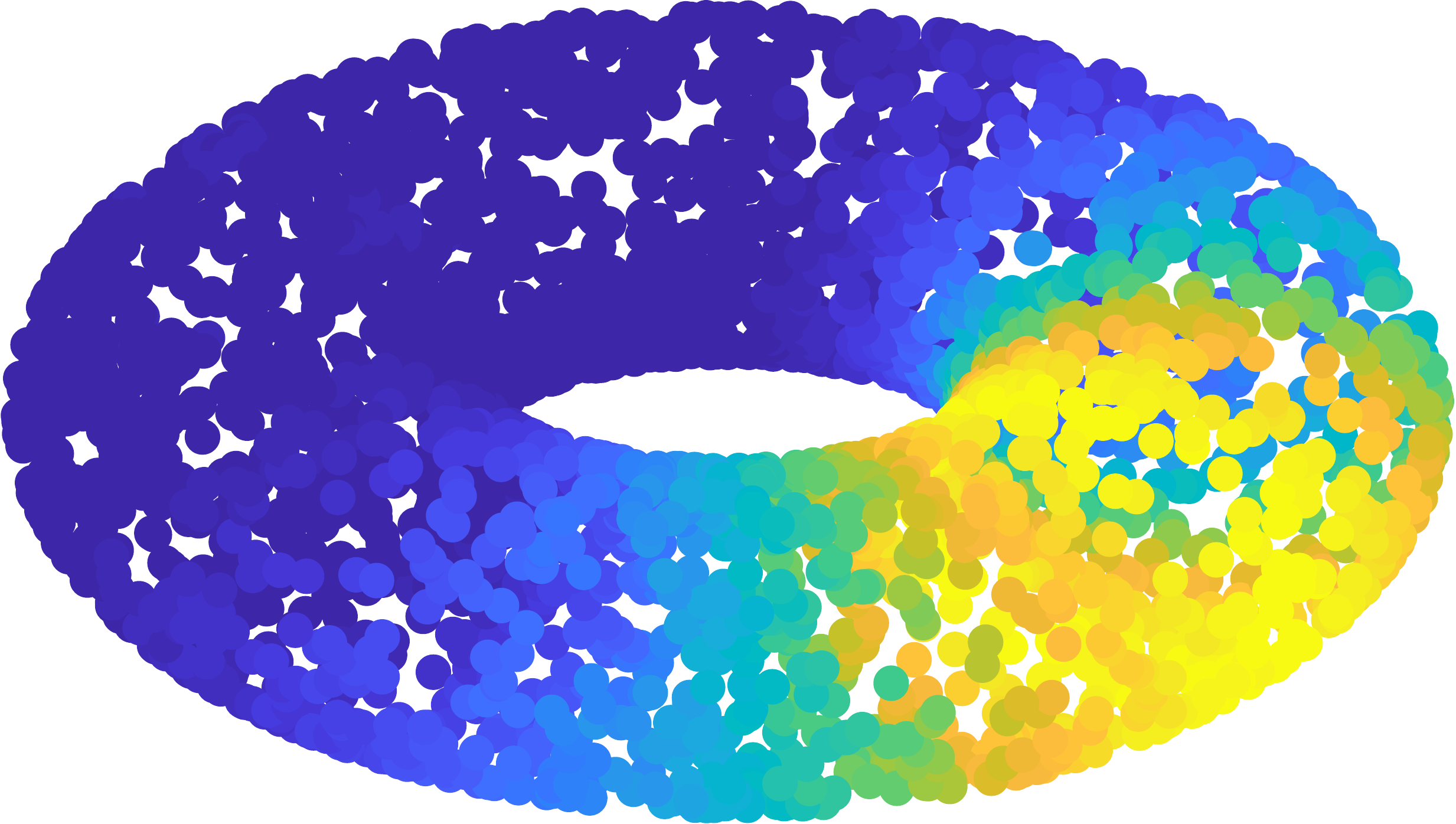}} \\
		$t=t^{*}$ & \centered{\includegraphics[scale=0.1]{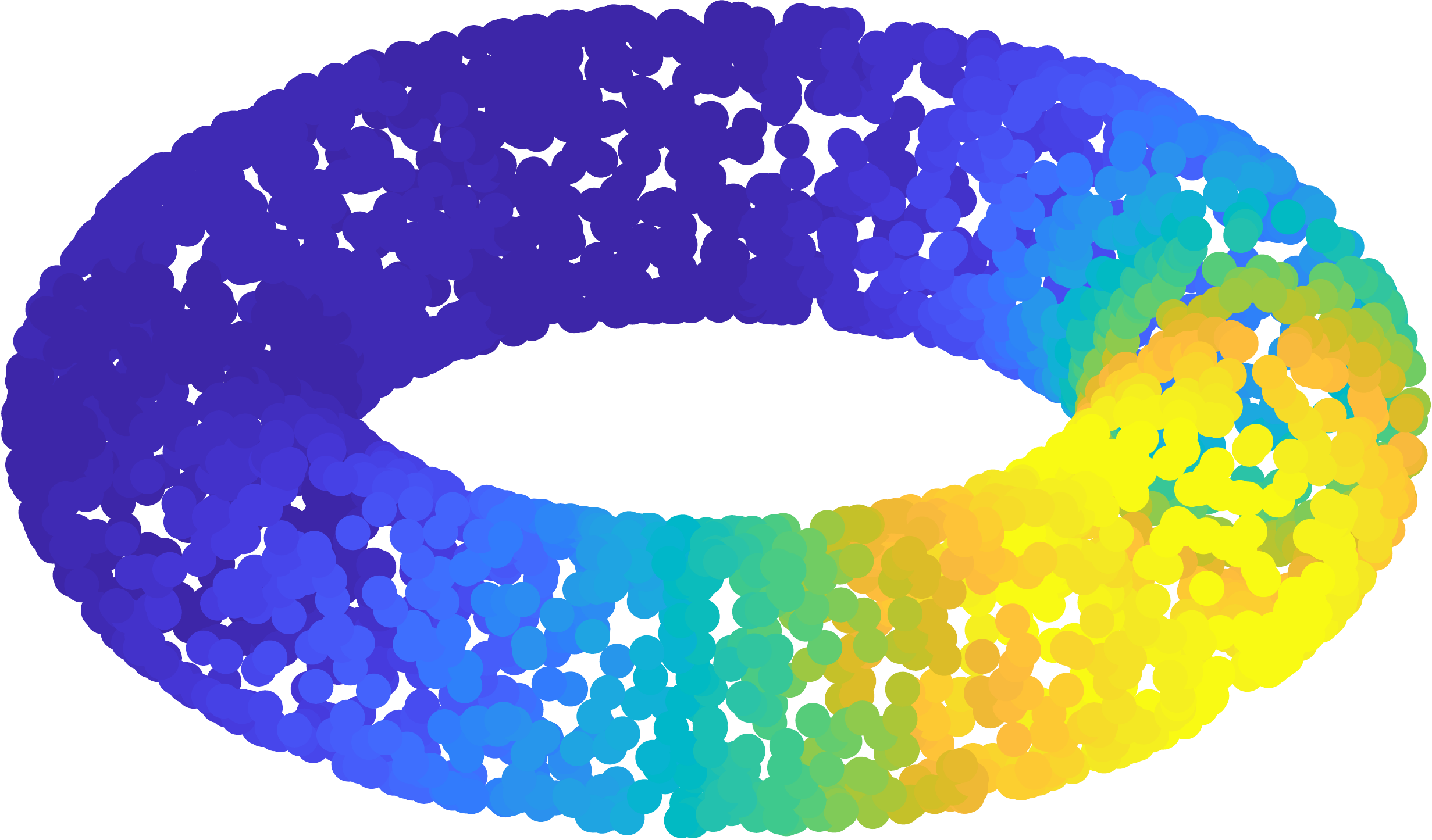}} &
		\centered{\includegraphics[scale=0.1]{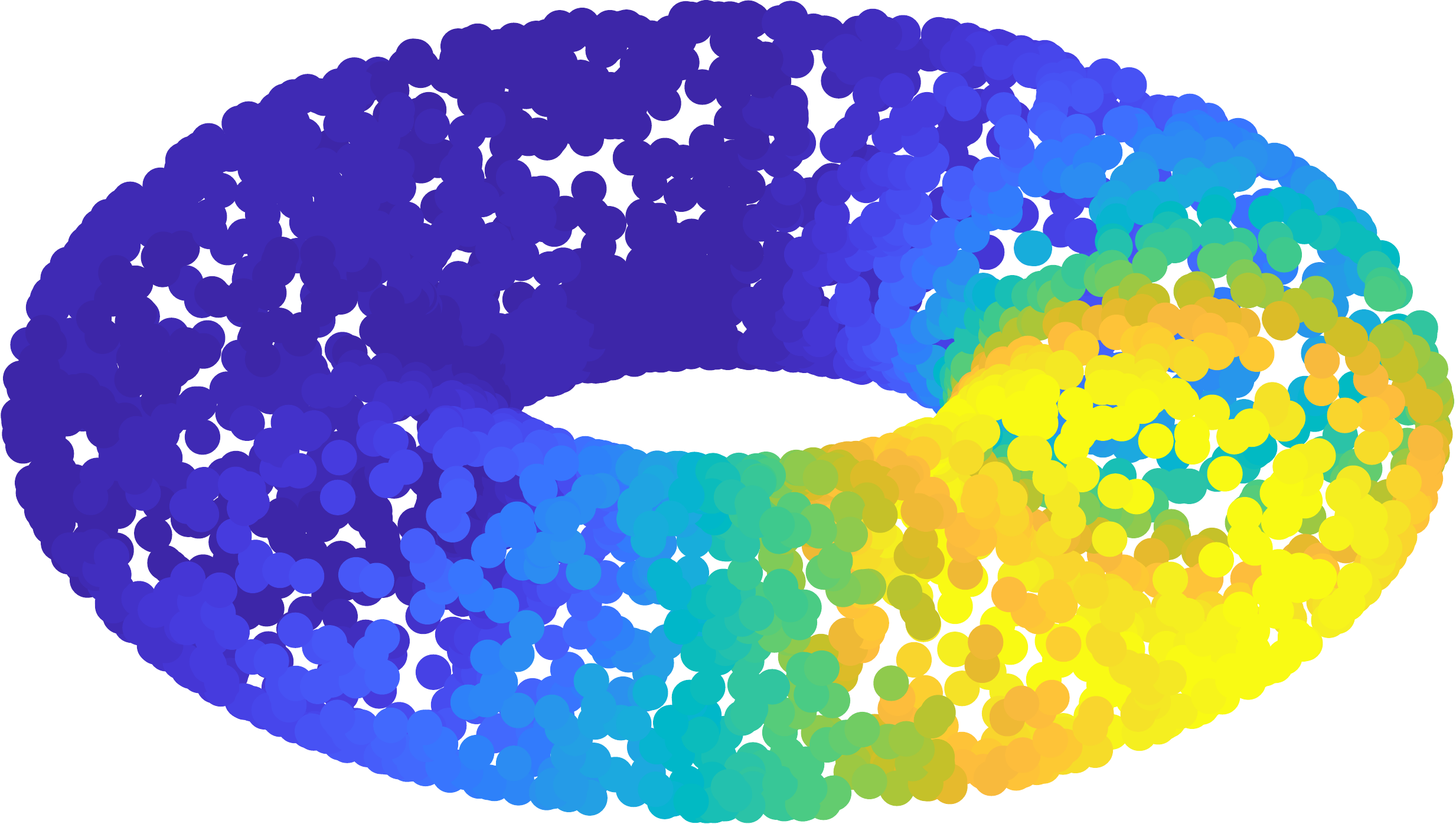}} \\
		$t=1$ &\centered{\includegraphics[scale=0.1]{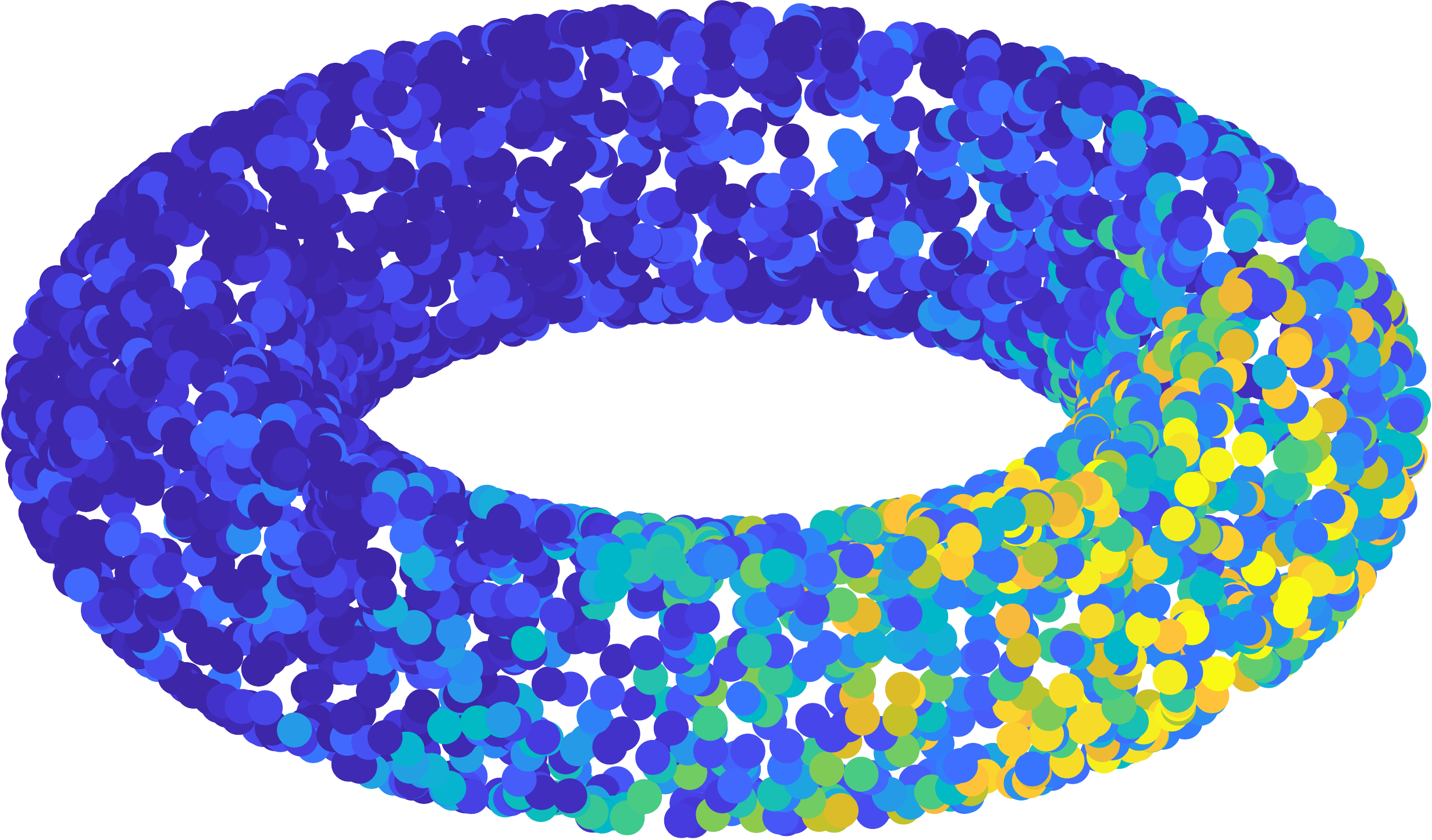}} &
		\centered{\includegraphics[scale=0.1]{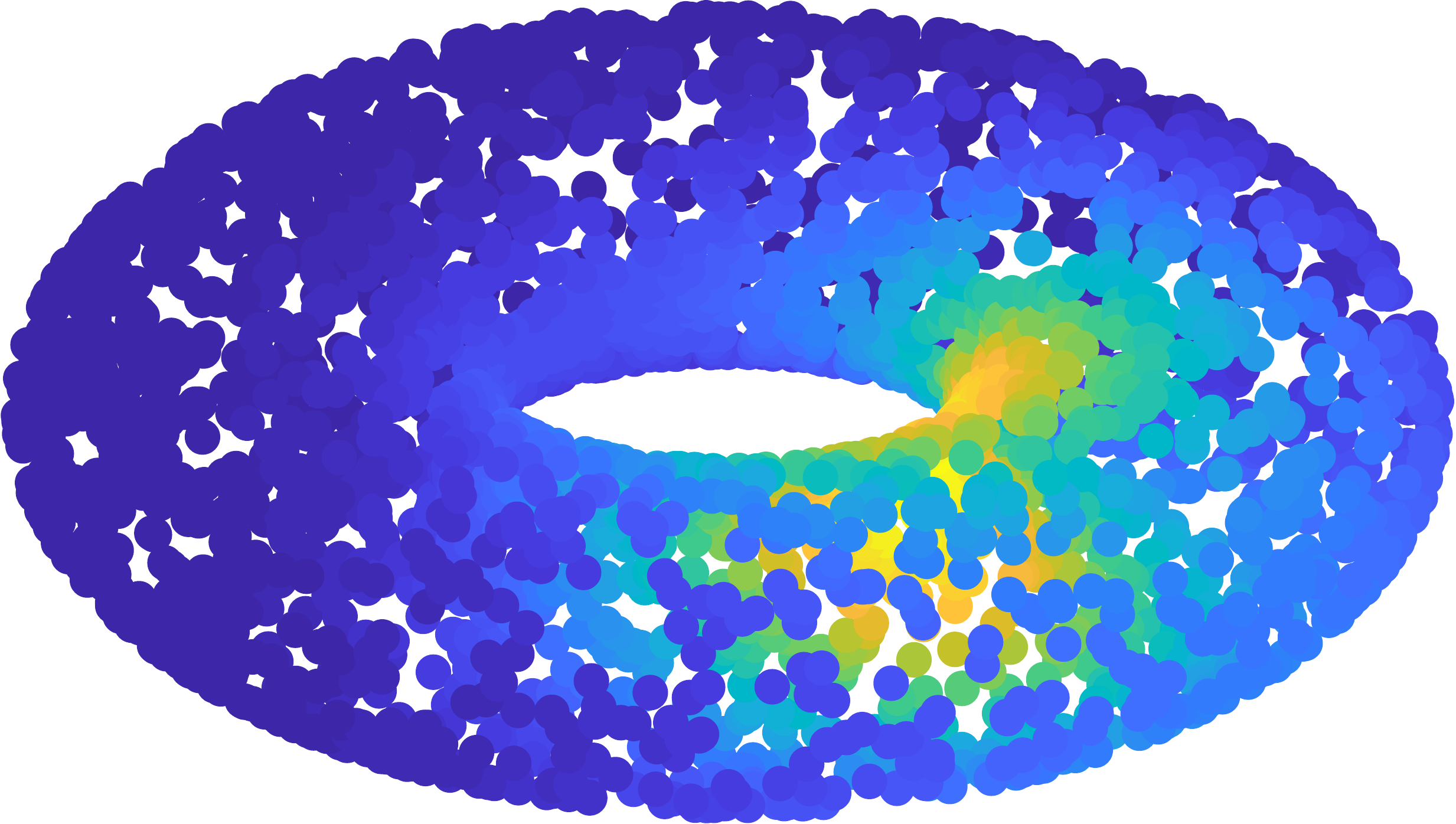}}\\
	\end{tabular}
	\caption{Same as Fig. \ref{fig:3DTorus_PolodialCommon_LowSNR_Diffs} but for the setting where the common variable is the torodial angle.}
	\label{fig:3DTorus_TorodialCommon_HighSNR_Diffs}
\end{figure}

\subsection{Extension to more than two sets of measurements}
\label{subsec:sup_convex_hull}

Here we present a straight-forward extension of \Cref{alg:EvfdCalculation} for three sets of aligned measurements: $\{s^{(1)}_i\}_{i=1}^{n},\{s^{(2)}_i\}_{i=1}^{n},\{s^{(3)}_i\}_{i=1}^{n}$. 
First, we compute three SPD kernels: $\mathbf{K}_1,\mathbf{K}_2$ and $\mathbf{K}_3$ for each set of measurements.
Then, instead of considering the geodesic path between two kernels:
\begin{equation*}
\gamma_{1\rightarrow2}(t)=\mathbf{K}_1^{1/2}\left( \mathbf{K}_1^{-1/2}\mathbf{K}_2\mathbf{K}_1^{-1/2} \right)^t \mathbf{K}_1^{1/2}
\end{equation*}
as proposed in Step 2 in \Cref{alg:EvfdCalculation}, we consider the following convex hull of three kernels: 
\begin{equation*}
\gamma_{1,2,3}(t_1,t_2)=\left( \gamma_{1\rightarrow2}(t_1) \right)^{1/2}\left( \left(\gamma_{1\rightarrow2}(t_1)\right)^{-1/2}\mathbf{K}_3\left(\gamma_{1\rightarrow2}(t_1) \right)^{-1/2} \right)^{t_2} \left( \gamma_{1\rightarrow2}(t_1)\right)^{1/2}
\end{equation*}
This convex hull gives rise to a 3D eigenvalues flow diagram with similar properties to the proposed \acrshort*{EVFD} for two sets of measurements.

We demonstrate this extension in a simulation. 
Consider four 2D flat manifolds: $\mathcal{M}_x = \mathcal{M}_y = \mathcal{M}_z=\mathcal{M}_w = [-1/2,1/2]$. These manifolds give rise to the following measurements:
\begin{align}\label{eq:three_2d_flat_measurements}
	s^{(1)}_i = \left(  x_i , y_i  \right) \in \mathbb{R}^2 \nonumber \\ 
	s^{(2)}_i = \left(  x_i , z_i  \right) \in \mathbb{R}^2 \nonumber \\ 
	s^{(3)}_i = \left(  x_i , w_i  \right) \in \mathbb{R}^2, \nonumber 
\end{align}
where $\left\{(x_i,y_i,z_i,w_i)\right\}_{i=1}^n$ are $n$ realizations sampled from some joint distribution on the product manifold $\mathcal{M}_x \times \mathcal{M}_y \times \mathcal{M}_z\times \mathcal{M}_w$.

In Fig \ref{fig:ConvexHull_EVDiagramTracking}, we depict the 3D eigenvalues flow diagram that results from the geodesic convex hull. Comparing the 3D diagram in Fig. \ref{fig:ConvexHull_EVDiagramTracking} with the 2D diagram in Fig. \ref{fig:3DStrip_ModerateSNR_COR}, we see that the log-linear trajectories in Fig.  \ref{fig:3DStrip_ModerateSNR_COR} are extended to 2D flat hyperplanes in Fig. \ref{fig:ConvexHull_EVDiagramTracking}, thereby, providing a natural extension to the proposed method.
\begin{figure}[t]\centering
	\includegraphics[scale=0.2]{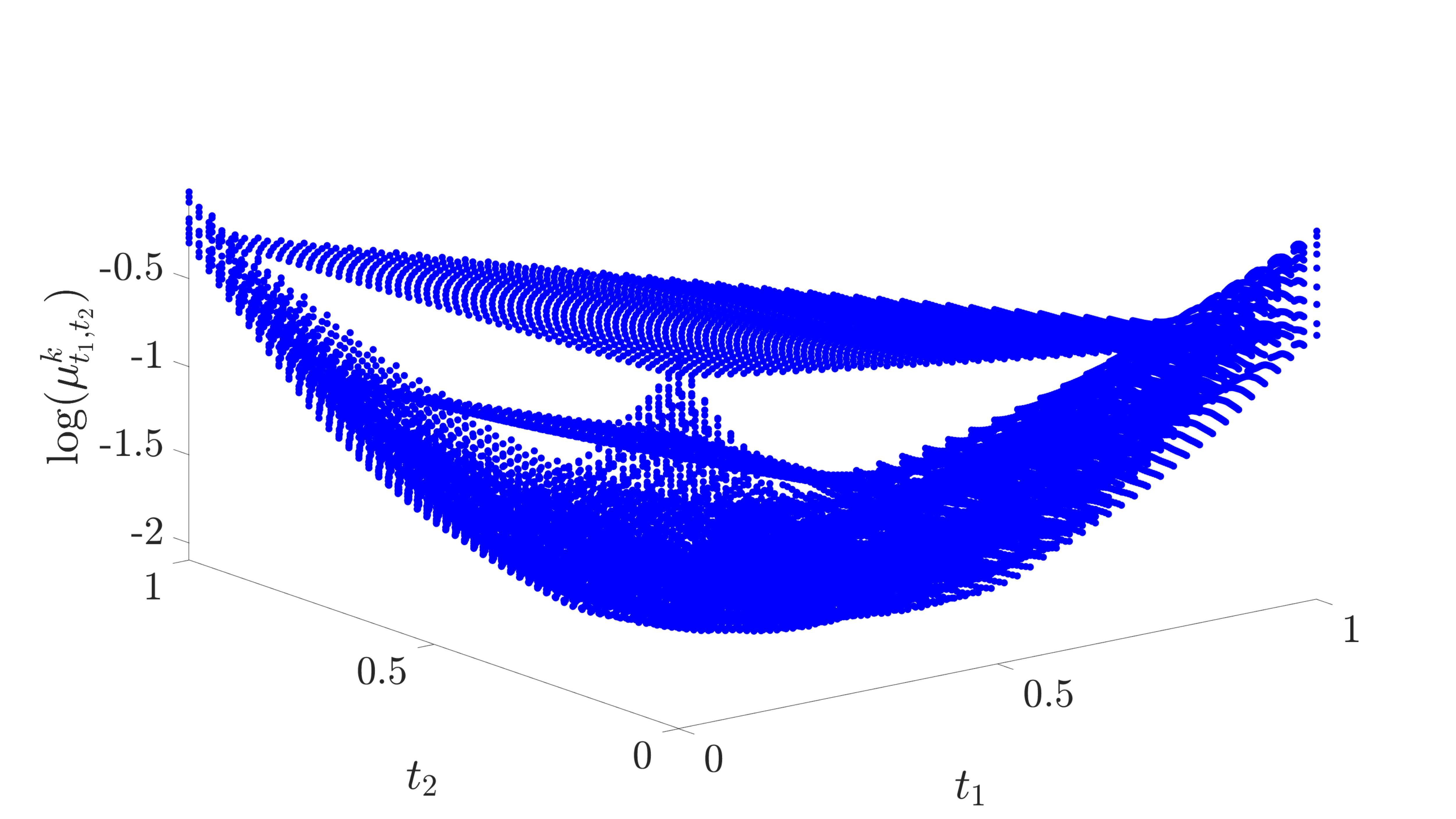}
	\caption{	The 3D \acrshort*{EVFD} based on three sets of aligned measurements.}
	\label{fig:ConvexHull_EVDiagramTracking}
\end{figure}
In future work we plan to study and generalize our results for this convex hull.

\clearpage

\section{Data description and availability}
\label{sec:sup_Code}

The experimental study in this paper is based on two datasets. 
Both datasets are publicly available in the UCI Machine Learning Repository \cite{Dua:2019}.

\subsection{Condition monitoring}

The condition monitoring dataset was contributed by the authors of \cite{helwig2015condition} and is publicly available in this link\footnote{\url{https://archive.ics.uci.edu/ml/datasets/Condition+monitoring+of+hydraulic+systems}}.

The dataset consists of measurements of a hydraulic test rig. The system periodically repeats constant load cycles, each of duration of $60$ seconds, and monitors values using $17$ sensors of different types: pressure, volume flow, power, vibration, and temperature sensors. During the monitored period, the condition of the hydraulic test rig quantitatively varies. 
Each record in the dataset consists of signal acquisitions and their labels. The signals are $17$ time-series representing the sensor measurements during one load cycle. 
The labels are $4$-dimensional vectors representing the condition of four components of the hydraulic system: the valve condition, the cooler condition, the internal pump leakage, and the hydraulic accumulator (bar). For more details on the dataset, see Section 2 in \cite{helwig2015condition}.

\subsection{Artificial olfaction for gas identification}

The E-Nose dataset was contributed by the authors of \cite{vergara2012chemical,rodriguez2014calibration}.
The dataset is publicly available in this link\footnote{\url{http://archive.ics.uci.edu/ml/datasets/Gas+Sensor+Array+Drift+Dataset+at+Different+Concentrations}}. 

This dataset was recorded during a controlled experiment that examined the reaction of gas sensors to the injection of pure gaseous substances. 
The experiment included six distinct pure gaseous substances: Ammonia, Acetaldehyde, Acetone, Ethylene, Ethanol, and Toluene.
These gases were measured by four gas sensor devices that have different sensitivity to the concentration of these six gases. 
Each one of the sensor devices consists of a four metal-oxide sensor array.
The four devices (sixteen sensors) were placed in a test chamber of $60$ml volume.  

The experiment was carried out in trials. In each trial, a single gas of interest was injected to the test chamber at a certain level of concentration and was measured by the metal-oxide gas sensors. The response of each sensor is a read-out of the resistance across its active layer, forming an $8$ dimensional vector. Each record in the dataset corresponds to a single trial, and it includes: (i) the read-outs of the sensors, and (ii) the concentration of the injected gas (viewed as a label).

\newpage

\section{Implementation}
\label{sec:Implementations}

In this appendix, we present algorithms that make concrete and systematic the procedures mentioned in the paper.
In \Cref{subsec:evfd_alg}, we present the algorithm for computing the \acrshort*{EVFD}. 
In \Cref{subsec:sup_FixedRank}, we consider the case where the kernels $\mathbf{K}_1$ and $\mathbf{K}_2$ from \eqref{eq:dm_K} are not strictly positive or have very small eigenvalues.
In this case, the computation of the matrices along the geodesic path in \cref{eq:geodesic_grid} is numerically unstable, because it involves the inverse of a (near) low rank matrix. To circumvent this problem, we propose to use an approximation of the geodesic path on the manifold of symmetric positive semidefinite (SPSD) matrices rather than the geodesic path of the manifold of SPD matrices.
Finally, in \Cref{subsec:sup_post}, we present post-processing procedures of the \acrshort*{EVFD}, aiming to automatically resolve the trajectories, required for the embedding presented in \Cref{sec:embedding}.

\subsection{EVFD algorithm}
\label{subsec:evfd_alg}
\begin{algorithm}[H]

	\hspace*{\algorithmicindent} \textbf{Input}: Two sets of $n$ (aligned) measurements:
	\begin{equation}\label{eq:observations}
	\left\{s^{(1)}_i\right\}_{i=1}^n, \left\{s^{(2)}_i\right\}_{i=1}^n
	\end{equation}
	\hspace*{\algorithmicindent} where $s^{(1)}_i \in \mathcal{O}_1$ and $s^{(2)}_i \in \mathcal{O}_2$.\\
	\hspace*{\algorithmicindent} \textbf{Output}: Eigenvalues flow diagram.\\
	\hspace*{\algorithmicindent} \textbf{Parameters}: 
	\begin{itemize}
		\item $N_t$ -- The number of points on the geodesic path determining the resolution of the flow.
		\item $K$ -- The number of eigenvalues comprising the diagram.
	\end{itemize}
	\begin{algorithmic}[1]
		\State Build two kernels for the two input sets of measurements:
		\begin{algsubstates}
			\State Construct two affinity matrices $\mathbf{W}^{(1)}$ and $\mathbf{W}^{(2)}$ using Gaussian kernels:
                \begin{equation}
                	W_{i,j}^{(1)}=\exp\left(-\frac{\|s^{(1)}_i-s^{(1)}_j\|_{M_1}^2}{\varepsilon^{(1)}}\right) ; W_{i,j}^{(2)}=\exp\left(-\frac{\|s^{(2)}_i-s^{(2)}_j\|_{M_2}^2}{\varepsilon^{(2)}}\right)
                \end{equation}
                for all $i,j=1,\hdots,n$, where $\varepsilon^{(1)}$ and $\varepsilon^{(2)}$ are tunable kernel scales, and $\|\cdot\|_{M_1}$ and $\|\cdot\|_{M_2}$ are two norms induced by two metrics corresponding to the two observable spaces $\mathcal{O}_1$ and $\mathcal{O}_2$.
			\State Compute the row-stochastic matrices $\mathbf{A}_1$ and $\mathbf{A}_2$ by \cref{eq:dm_A}.
			\State Compute the SPD kernels $\mathbf{K}_1$ and $\mathbf{K}_2$ by \cref{eq:dm_K}.
		\end{algsubstates}	
		\State Consider a discrete uniform grid of $N_t$ points $\{t_i\}_{i=1}^{N_t}$ in the interval $[0,1]$. For each $t_i$:
		\begin{algsubstates}
			\State Calculate the matrix $\gamma(t_i)$ on the geodesic path at $t_i$ according to \cref{eq:geodesic}:
			\begin{equation}
			\gamma(t_i)= \mathbf{K}_1^{1/2}\left( \mathbf{K}_1^{-1/2}\mathbf{K}_2\mathbf{K}_1^{-1/2} \right)^{t_i} \mathbf{K}_1^{1/2}
			\end{equation}
			\State Apply \acrshort*{EVD} to $\gamma(t_i)$ and obtain the largest $K+1$ eigenvalues of $\gamma(t_i)$, $\{\mu_{t_i}^{k}\}_{k=1}^{K+1}$.
			\State Scatter plot the logarithm of the obtained non-trivial eigenvalues $\{\log(\mu_{t_i}^{k})\}_{k=2}^{K+1}$ as a function of $t_i$.
		\end{algsubstates}	
	\end{algorithmic}
	\caption{Computing the \acrshort*{EVFD}.}
	\label{alg:EvfdCalculation}
\end{algorithm}

\subsection{Geodesic paths on \texorpdfstring{$\mathcal{S}^{+}(p,n)$}{Lg}}
\label{subsec:sup_FixedRank}
In practice, the dimension of the kernel matrices is determined by the size of the dataset $n$, which is typically higher than the intrinsic dimensionality of the data. As a result, the kernel matrices might consist of small, negligible eigenvalues, and as a consequence, would not, in effect, be strictly positive-definite matrices in $\mathcal{P}(n)$. 
In such scenarios, the computation of the affine-invariant metric defined in \eqref{eq:metric} and the geodesic path defined in \eqref{eq:geodesic} become unstable, because they involve the inverse of a near low rank matrix. 
The remedy we propose is to view the small positive eigenvalues as zeros, and consider the manifold of SPSD matrices instead of the manifold of SPD matrices.
We remark that restricting the rank of an $n \times n$ matrix to $p<n$ reduces the complexity of typical matrix operations such as SVD and EVD from $O(n^3)$ to $O(n^2 p)$.

In \cite{bonnabel2009riemannian} the authors introduced a metric that extends the affine-invariant metric defined on strictly positive matrices to the set of symmetric positive semi-definite $n \times n$ matrices $\mathcal{S}^{+}(p,n)$ with a fixed rank $p<n$. The proposed metric inherits most of the useful properties of the affine-invariant metric, e.g., it is invariant to rotations, scaling, and pseudo-inversion.

Although the exact closed-form expression of the geodesic path is not provided, the authors utilized the fact that the set $\mathcal{S}^{+}(p,n)$ admits a quotient manifold representation:
\[
\mathcal{S}^{+}(p,n) \eqsim \left( V_{n,p} \times \mathcal{P}(p) \right) /O(p)
\]
where $V_{n,p}$ is the Stiefel manifold, i.e., the set of $n \times p$ matrices with orthonormal columns: $U^TU = I_p$ and $O(p)$ is the orthogonal group in dimension $p$, and propose the following approximation based on the horizontal geodesic path in this space.
Consider two matrices $\mathbf{A}$ and $\mathbf{B}$ in $\mathcal{S}^{+}(p,n)$, and let $\mathbf{V}_A$ and $\mathbf{V}_B$ be two matrices that span the range of $\mathbf{A}$ and the range of $\mathbf{B}$, respectively. The approximation of the geodesic path between $\mathbf{A}$ and $\mathbf{B}$ is computed as follows:
\begin{enumerate}
	\item Calculate the SVD decomposition of $\mathbf{V}_A^T \mathbf{V}_B$, such that:
	\[
	\mathbf{V}_A^T\mathbf{V}_B = \mathbf{O}_A \mathbf{S}_{AB}\mathbf{O}_B^T
	\]
	\item Denote:
	\subitem $\mathbf{U}_A \triangleq \mathbf{V}_A \mathbf{O}_A$
	\subitem $\mathbf{U}_B \triangleq \mathbf{V}_B \mathbf{O}_B$
	\subitem $\mathbf{\Theta} \triangleq \arccos(\mathbf{S}_{AB})$
	\subitem $\mathbf{X}=(\mathbf{I}-\mathbf{U}_A \mathbf{U}_A^T)\mathbf{U}_B(\sin(\mathbf{\Theta}))^{-1}$
	\subitem $\mathbf{R}_A^2=\mathbf{U}_A^T \mathbf{A} \mathbf{U}_A$
	\subitem $\mathbf{R}_B^2=\mathbf{U}_B^T \mathbf{B} \mathbf{U}_B$
	\item The Grassman geodesic path connecting range($\mathbf{A}$) and range($\mathbf{B}$) is given by:
	\[
	    \mathbf{U}(t)=\mathbf{U}_A \cos(\mathbf{\Theta} t)+ \mathbf{X} \sin(\mathbf{\Theta} t)
	\]
	for $t \in [0,1]$.
	\item The associated geodesic path in $\mathcal{P}(p)$ connecting $R_A^2$ and $R_B^2$ is given by:
	\[
	\mathbf{R}^2(t)=\mathbf{R}_A \exp\left( t \log \left( \mathbf{R}_A^{-1} \mathbf{R}_B^2 \mathbf{R}_A^{-1} \right) \right) \mathbf{R}_A
	\]
	\item Finally, the approximation of the geodesic path between $\mathbf{A}$ and $\mathbf{B}$ is given by the following curve:
	\begin{equation}\label{eq:spsd_geod}
	\hat{\gamma}(t)=\mathbf{U}(t) \mathbf{R}^2(t) \mathbf{U}^T(t)
	\end{equation}
\end{enumerate}
According to Theorem 2 in \cite{bonnabel2009riemannian}, $\gamma(t)$ admits the following properties:
\begin{itemize}
	\item $\gamma(t) \in \mathcal{S}^{+}(p,n)$ for every $t\in[0,1]$.
	\item The curve $(\mathbf{U}(t),\mathbf{R}^2(t))$ is a horizontal lift of $\gamma(t)$ and it is a geodesic path in $V_{n,p} \times \mathcal{P}(p)$.
	\item The total length of $\gamma(t)$ is given by:
	\[
	\ell^2(\gamma)= ||\mathbf{\Theta}||_F^2+ k ||\log({\mathbf{R}_A}^{-1}{\mathbf{R}_B}^2{\mathbf{R}_A}^{-1})||_F^2
	\]
	where $k$ is a positive constant. The resulting curve is invariant to pseudo-inversion and to the group action by congruence of orthogonal transformations and scaling.
\end{itemize}
Viewing matrices in $\mathcal{S}^{+}(p,n)$ as flat ellipsoids in $\mathbb{R}^n$, the length of $\gamma(t)$, $\ell^2(\gamma)$, consists of two independent components: a distance on the Grassman manifold and a distance on $\mathcal{P}(p)$.
It is worthwhile noting that $\gamma(t)$ is not necessarily a geodesic path in $\mathcal{S}^{+}(p,n)$ because its length is not the Riemannian distance between $\mathbf{A}$ and $\textbf{B}$ since it does not necessarily satisfy the triangle inequality. However, its length provides a meaningful measure of proximity between $\mathbf{A}$ and $\mathbf{B}$.

In this work, when the kernel matrices $\mathbf{K_1}$ and $\mathbf{K_2}$ have small negligible eigenvalues, we assume a fixed rank $p$ that accommodates the kernels significant eigenvalues. In turn, in Step {\tt 2a} in \cref{alg:EvfdCalculation}, we replace the geodesic on the SPD manifold $\gamma(t)$ with the approximate geodesic path $\hat{\gamma}(t)$ \cref{eq:spsd_geod}.

\subsection{Post-processing the diagram} 
\label{subsec:sup_post}
We propose three post-processing algorithms that can be applied to the \acrshort*{EVFD} in order to make the empirical characterizations presented in \Cref{subsec:ToyExample} systematic. 

The first algorithm aims to extract sequences of eigenvalues (trajectories) along the geodesic path (along the parameter $t$). 
This algorithm is presented in \Cref{subsubsec:sup_ComponentsIdentification}.
Extracting such sequences provides important information, enabling us to identify common and non-common eigenvalues according to the results presented in \Cref{subsec:TheorericFoundations}. Specifically, by \Cref{prop:MutualEigenValues}, eigenvalues lying on log-linear trajectories in $t$ are associated with the common manifold, and the remaining eigenvalues are considered non-common (associated with either the measurement-specific manifolds exclusively or with both the common and measurement-specific manifolds). 
We propose two algorithms for testing whether an eigenvalue is common or not. 
Importantly, the two algorithms support a configurable deviation from the strict log-linear trajectory form, thereby allowing to identify eigenvalues that do not meet the strict definition in \Cref{prop:MutualEigenValues}, but still can be considered common in practice (see for example the largest eigenvalue of $\gamma(0.5)$ in Fig. \ref{fig:Puppets_EVDiagrams_Geodesic}).
The first algorithm requires to resolve the eigenvalues trajectories, by applying the algorithm mentioned above as a prerequisite, and then tests whether the extracted trajectory of the eigenvalue of interest assumes a log-linear form. The second algorithm utilizes an empirical trend, demonstrated in \Cref{sec:sup_SynthData}, that makes use of the variation of the expansion of the eigenvector (corresponding to the eigenvalue in question) in the eigenvectors of $\gamma(t)$ as $t$ changes along the geodesic path.
The two algorithms are described in detail in \Cref{subsubsec:cmr_estimation}.

\subsubsection{Resolving the eigenvalues trajectories}\label{subsubsec:sup_ComponentsIdentification}

We use the results presented in \Cref{subsec:TheorericFoundations} and in \Cref{sec:sup_SimuResults} in order to identify sequences of the different eigenvalues along the geodesic path.
At each point $t_0$ along the geodesic path we would like to assign to each eigenvalue of $\gamma(t_0)$ a unique identifier which is consistent with adjacent points along the geodesic path $\gamma(t)$, where $t\in[t_0-\Delta t,t_0+\Delta t]$ and $\Delta t$ is a locality parameter. This tracking problem can be defined as a matching or as an assignment problem. For convenience, rather than finding the assignment at each point $t_0$, we will find the permutation at each $t_0$ with respect to a reference order of eigenvalues. This change, which is only semantic, allows us to formulate the problem in a concise and elegant way.
In order to illustrate our goal, consider the synthetic \acrshort*{EVFD} depicted in Fig. \ref{fig:PermToyExampleEVDiagram_Geodesic}. 
\begin{figure}[t]
	\centering
	\includegraphics[scale=0.2]{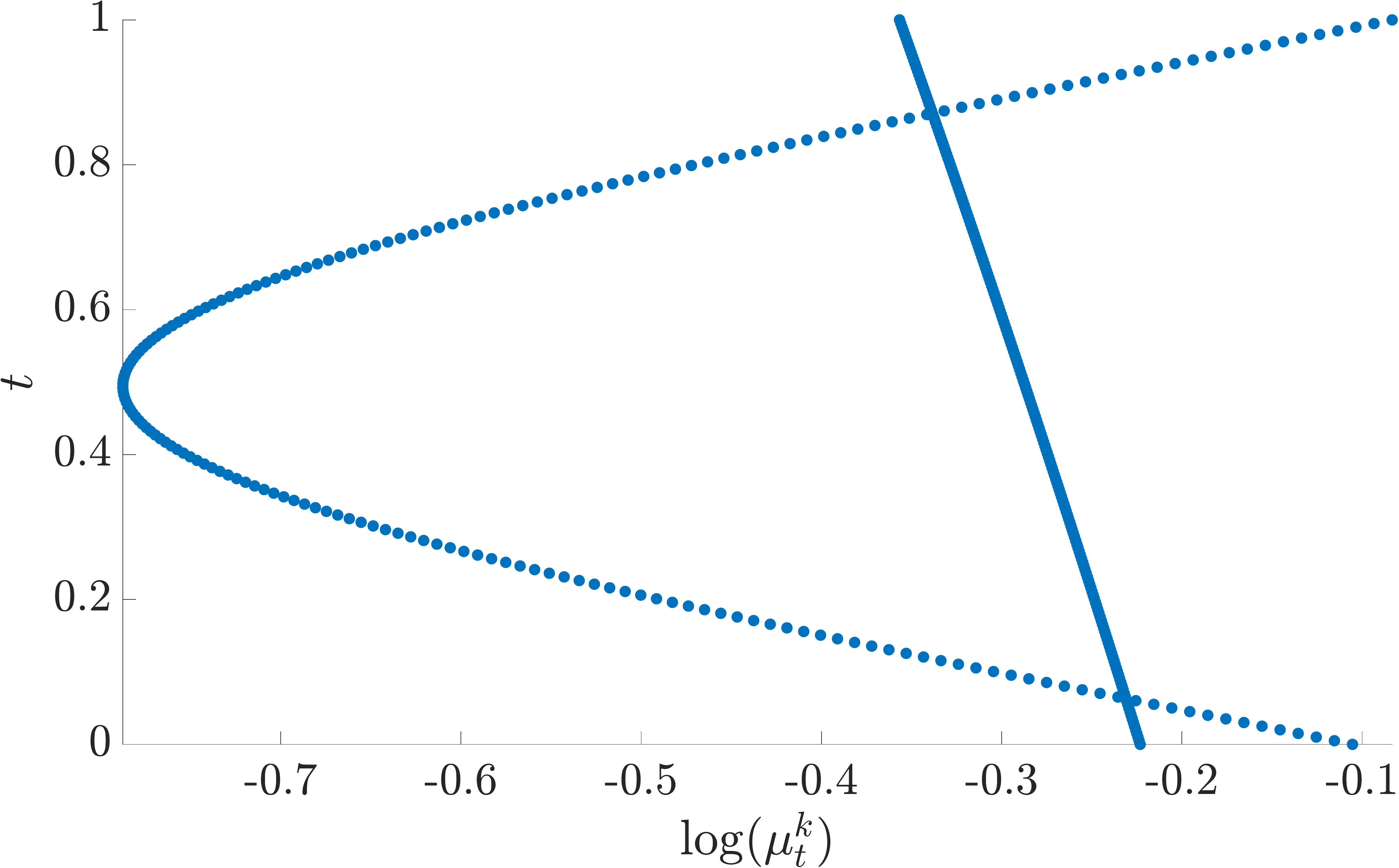} 
	\caption {An illustrative \acrshort*{EVFD} consisting of two components.}
	\label{fig:PermToyExampleEVDiagram_Geodesic}
\end{figure}
Let $\mu_{t_i}^{k}$ denote the $k$th eigenvalue of $\gamma(t_i)$. 
For simplicity, we set the order of the eigenvalues at $t=0$ as our reference order, i.e., $1=\mu_{0}^{0} > \mu_{0}^{1} > \ldots >\mu_{0}^{n}$. Note that the reference order consists of an implicit assumption, implying that the eigenvalues at $t=0$ have no multiplicity.
Let $\Pi_t: \{1,\ldots,N_e\} \rightarrow \{1,\ldots,N_e\} $ denote the permutation at point $t$ with respect to the reference at $t=0$. Accordingly, the permutations we aim to find in Fig. \ref{fig:PermToyExampleEVDiagram_Geodesic} are given by:
\begin{equation*}
\Pi_t=
\begin{cases}
(1,2), & 0 \leq t \leq 0.06 \\
(2,1), & 0.06 < t < 0.88 \\ 
(1,2), & 0.88 \leq t \leq 1 
\end{cases}
\end{equation*}
In general, when considering $N_e>2$ eigenvalues, the eigenvalue at point $t$ corresponding to the $k$th eigenvalue ($1\leq k \leq N_e$) at $t=0$ is specified by $\Pi_t(k)$.
Consider a discrete uniform grid of $t \in [0,1]$ along the geodesic path: $\{t_i\}_{i=1}^{N_t}$, and their corresponding components at each point $\left\{\mu_{t_i}^{k},v_{t_i}^{k} \right\}$, where $i=1, \ldots, N_t$ is the index of the point on the grid along the geodesic path, and $k=1,\ldots,N_e$ is the index of the eigenvalue. When estimating $\Pi_t$ the following considerations should be taken.
The first consideration is a local one and concerns the relation between the permutation at a certain point $t_i$ and the following or preceding permutations at $t_{i-1}$ and $t_{i+1}$, respectively. According to the smoothness of the flow we impose that $\Pi_{t_{i+1}}$ should be equal to $\Pi_{t_{i}}$ and $\Pi_{t_{i-1}}$ or it may contain only one or two swaps.
Second, the difference between two consecutive permutations $\Pi_{t_{i+1}}$ and $\Pi_{t_{i}}$ can only contain swaps between eigenvalues with similar values.
The third consideration is global and addresses the relationship between the permutation at a certain point $t_i$ and the permutations at the other points on the grid. Given two permutations $\Pi_{t_i}$ and $\Pi_{t_j}$, we encode the correspondence between the associated eigenvectors by requiring high correlations 
	\begin{equation*}
	\langle v_{t_i}^{\Pi_{t_i}(k)}, v_{t_j}^{\Pi_{t_j}(k)} \rangle,
	\end{equation*}
for any $k=1,\ldots, N_e$.

Since the log-linear behavior is guaranteed only when there exists a common eigenvector, the proposed implementations do not rely on this log-linear assumption. 
However, the log-linear flow can be easily incorporated as an additional prior within the proposed framework.
Based on these considerations, we propose a criterion and select the set of permutations that maximize this criterion. Since the permutation at each point depends on the permutations in adjacent points, the proposed criterion is evaluated on the entire set of permutations on the grid. Considering $N_t$ points on the grid of $t \in [0,1]$, and $N_e$ eigenvalues, the number of possible sets to be evaluated is $\left(N_e!\right)^{N_t}$; calculating a certain criterion $\left(N_e!\right)^{N_t}$ times is not a feasible.
Instead, we propose to find $\Pi_t$ by solving a dynamic programming problem using a small variant of the seminal Viterbi algorithm \cite{forney1973viterbi}. 
Importantly, the ability to use standard dynamic programming facilitates efficient solutions to the problem, which might have not been feasible to be resolved otherwise. 
We use a terminology which is commonly used in the context of the Viterbi algorithm.
Let $s$ denote a hidden state which is defined as a tuple of $L$ successive permutations:
\begin{equation*}
s=(\Pi_1,\Pi_2,\ldots,\Pi_L)
\end{equation*}
where $L$ is a tunable scale parameter.
Let $\mathcal{S}$ denote the set of all hidden states, where the total number of states is given by $N_s = |\mathcal{S}|=(N_e!)^L$. Access to the $l$th permutation in the state (tuple) $s$ is given by $s(l)$.
Consider a sequence of hidden states, where $x_{t_i} \in \mathcal{S}$ is the state at $t_i$, and suppose $x_{t_i} = (\Pi_{t_{i-L+1}},\Pi_{t_{i-L+2}},\ldots,\Pi_{t_i})$.
The transition probability $\Pr(x_{t_i},x_{t_{i+1}})$ between two successive hidden-states in the sequence $x_{t_i}$ and $x_{t_{i+1}}$ is determined according to the above considerations.
First, given $x_{t_{i}}$, for the consistency of the sequence, $x_{t_{i+1}}$ must be of the form $(\Pi_{t_{i-L+2}},\ldots,\Pi_{t_i}, \Pi_{t_{i+1}})$, i.e, the sequence of hidden states should correspond to a sequence of permutations. Otherwise, $\Pr(x_{t_i},x_{t_{i+1}})=0$.
Then, non-zero transition probabilities are defined with respect to the first permutation of $x_{t_{i}}(1)$, namely, $\Pi_{t_{i-L+1}}$, and the last permutation of $x_{t_{i+1}}(L)$, namely $\Pi_{t_{i+1}}$.
Particularly, we allow only one or two swaps between $\Pi_{t_{i-L+1}}$ and $\Pi_{t_{i+1}}$ in order to impose smoothness.

Explicitly, the transition probability is defined by:
\begin{equation*}
\Pr(x_{t_i},x_{t_{i+1}})=
\begin{cases}
0 , & \parbox[t]{.6\textwidth}{if $\exists j=1,\ldots,L-1$ such that $x_{t_i}(j) \neq x_{t_{i+1}}(j+1)$}\\
p_1, & \parbox[t]{.6\textwidth}{if $\forall j=1,\ldots,L-1$, $x_{t_i}(j) = x_{t_{i+1}}(j+1)$ and the difference between $x_{t_i}(1)$ and $x_{t_{i+1}}(L)$ consist of only one swap}\\
p_2, & \parbox[t]{.6\textwidth}{if $\forall j=1,\ldots,L-1$, $x_{t_i}(j) = x_{t_{i+1}}(j+1)$ and the difference between $x_{t_i}(1)$ and $x_{t_{i+1}}(L)$ consist of two swaps}\\
1-p_1-p_2, & \text{Otherwise}
\end{cases}
\label{eq:TransitionProb}
\end{equation*}
The proposed transition probability, which depends on $2$ hyper-parameters ($p_1$ and $p_2$), serves as a regularization that dominates smoothness of the most likely sequences of permutations.
Now, for each hidden state in the sequence $x_{t_i}$, we assign a likelihood score.
Let $V_{t_i}$ be an $N_e \times n$ matrix, whose rows are the eigenvectors $v_{t_i}^{k}$, ordered according to the reference order. 
We denote by $\Pi(V_{t_i})$ an $N_e \times n$ matrix, which consists of the rows of $V_{t_i}$ permuted according to $\Pi$.
Namely, the $l$th row of $\Pi(V_{t_i})$ is the $\Pi(l)$th row of $V_{t_i}$.
Let $y_{t_i}$ denote an observation of a hidden state in the sequence $x_{t_i}$, defined by a tuple of two matrices of eigenvectors:
\begin{equation*}
(\Pi_1(V_{t_{i-L+1}}), \Pi_L(V_{t_i}))
\end{equation*}
where $\Pi_1=x_{t_i}(1)$ and $\Pi_L=x_{t_i}(L)$.
The likelihood of an observation $y_{t_i}$ is defined based on the pairwise correlations between the eigenvectors:
\begin{equation*}
\Pr(y_{t_i} | x_{t_i}) = \frac{1}{c} \text{Tr} \left(\Pi_1(V_{t_{i-L+1}})  \Pi_L(V_{t_{i}}))^{T}\right) = \frac{1}{c}\sum_{k=1}^{N_e} \langle v_{t_{i-L+1}}^{\Pi_1(k)}, v_{t_{i}}^{\Pi_L(k)}\rangle
\end{equation*}
where $c$ is a normalization factor.
Given this formulation we can use the forward-backward algorithm in order to find the most likely sequence of hidden states. 

The entire algorithm is summarized in \Cref{alg:TrackingAlg}.
By construction, since we impose consistency, the most likely sequence of hidden states forms a sequence of permutations $\Pi_{t}$, which resolves the trajectories of eigenvalues in the \acrshort*{EVFD}.
\begin{algorithm}[t]
	\hspace*{-0.1in} \textbf{Input}: The geodesic path $\gamma(t)$ \\
	\hspace*{-0.1in} \textbf{Output}: $\Pi_t$ - the permutation at point $t$ with respect to the reference at $t=0$. \\
	\begin{algorithmic}[1]
		\State { Initialize $T_1$ - a $N_s \times N_t$ matrix. $T_1[i,j]$ denotes the log-probability of the most likely sequence of $j$ states ending at the $i$th state - $s_i$, this path will be denoted by: $\hat{X}=(\hat{x}_{t_1},\hat{x}_{t_2},\ldots,\hat{x}_{t_j})$ where $\hat{x}_{t_j}=s_i$.}
		\State {Initialize $T_2$ - a $N_s \times N_t$ matrix. $T_2[i,j]$ denotes the most likely previous state index. I.e.  $T_2[i,j]$ is the index of the state of $\hat{x}_{t_{j-1}}$.}
		\State { For each state $j=1,\ldots,N_s$:
			\begin{algsubstates}
				\State{$T_1[j,1]\leftarrow -1$ }
				\State{$T_2[j,1]\leftarrow 0 $ }
		\end{algsubstates}}
		\State For each observation index $i=1,\ldots,N_t$,:
		\Statex  For each state $s_j$, $j=1,\ldots,N_s$:
		\begin{algsubstates}
				\State{$T_1[j,i]\leftarrow \underset{k}{\operatorname{max}}{\{T_1[k,i-1] + \log(P(y_{t_i}|s_j)+ \log(P(x_{t_i}=s_k|x_{t_{i-1}}=s_j))\}}$}
				\State{$T_2[j,i]\leftarrow \underset{k}{\operatorname{argmax}}{\{T_1[k,i-1] + \log(P(y_{t_i}|s_j)+ \log(P(x_{t_i}=s_k|x_{t_{i-1}}=s_j))\}}$}
		\end{algsubstates}
	\end{algorithmic}
	Then compute the backward steps:
	\begin{algorithmic}[1]
		\State Initialize $Z$ - an $1 \times N_t$ array. $Z[i]$ denotes the index of the $i$th state within the most likely sequence of states: $\hat{X}$.
		\State{$Z[N_t]\leftarrow \underset{k}{\operatorname{argmax}}{  \text{\space} T_1[k,N_t]}$}
		\State For each $i=N_t,\ldots,2$:
		\begin{algsubstates}
			\State{$Z[i-1]\leftarrow T_2[Z[i-1],i]$}
		\end{algsubstates}	
		\State Using the obtained sequence $Z$ derive the corresponding ''hidden-states" sequence:
		\[
		\hat{X}=(\hat{x}_{t_1},\hat{x}_{t_2},\ldots,\hat{x}_{t_{N_t}})= (s_{Z[1]},s_{Z[2]},\ldots,s_{Z[N_t]})
		\]
		\State {The permutation at point $t_i$ with respect to the origin: $t=0$ is given by:
			\[
			\Pi_{t_i} = \hat{x}_{t_i}(L) 
			\]}
		\caption{Resolving the eigenvalues trajectories}
		\label{alg:TrackingAlg}
	\end{algorithmic}
	
\end{algorithm}

We conclude this section with few remarks on the differences between \Cref{alg:TrackingAlg} and the standard Viterbi algorithm.
First, a na\"{i}ve implementation requires $N_s=(N_e!)^L$ states. However, the restriction of the transition probabilities \eqref{eq:TransitionProb} allows to effectively reduce the number of states considerably. As a consequence, an efficient beam search implementation, which does not include all possible states, is used.
Second, in the standard Viterbi algorithm, the likelihood is typically fixed and depends on the observation and the hidden state. In contrast, according to our choice of formulation, the likelihood varies and depends on the point along the geodesic path ($\gamma(t_i)$). As a consequence, the likelihood cannot be computed a-priori, and must be re-computed in each forward step.

\subsubsection{CMR estimation}
\label{subsubsec:cmr_estimation}

We describe two approaches for assessing the ``commonality'' of a given eigenvector. The first approach relies on the log-linear trajectories of common eigenvectors. This requires resolving the trajectories using \Cref{alg:TrackingAlg}. The second approach utilizes the coherence of each eigenvector along the geodesic path.

\paragraph{CMR estimation using \Cref{alg:TrackingAlg}.}
\label{subsubsubsec:COR_tracking}

By \Cref{prop:MutualEigenValues}, the logarithm of eigenvalues that are related to the common eigenvectors varies linearly along the geodesic path with respect to $t$; we propose to utilize this log-linear behavior in order to identify the common eigenvectors and to estimate the CMR \eqref{eq:COR}. 

First, we apply \Cref{alg:TrackingAlg}. Let $ \tilde{\mu}_{t}^{k}$ be the eigenvalues of the $k$th eigenvector at $t=0$ (i.e. $\tilde{\mu}_{t}^{k}=\mu_{t}^{\Pi_t(k)}$). Then, for each eigenvector we compute the length of the linear line connecting $\log(\tilde{\mu}_{0}^{k})$ and $\log(\tilde{\mu}_{1}^{k})$ and compare it to the arc-length of the computed eigenvalues: $\log(\tilde{\mu}_{t}^{k})$. Using this comparison we can get an assessment for the ``commonality'' of the $k$th eigenvector. 
If the $k$th eigenvector is a common eigenvector, then it exhibits a log-linear behavior, i.e., the length of trajectory of eigenvalues $\log(\tilde{\mu}_{t}^{k})$ is close to the length of the straight line. 
Conversely, if the $k$th eigenvector is a non-common eigenvector, it is quickly suppressed in a non-linear  manner and the above calculation yields a high difference.
The algorithm is presented in \Cref{alg:COR}.

\begin{algorithm}[t]
	\hspace*{\algorithmicindent} \textbf{Input}: The geodesic path $\gamma(t)$. \\
	\hspace*{\algorithmicindent} \textbf{Output}: $\text{CMR}(t)$ estimation for any $t$ within a discrete uniform grid in $[0,1]$. 
	\begin{algorithmic}[1]
		\State Calculate the trace of each eigenvalue: $\{ \tilde{\mu}_{t}^{k} \}_{k=1}^{N_e}$, where $\tilde{\mu}_{t}^{k}=\mu_{t}^{\Pi_t(k)}$, and $\Pi_t$ is computed according to  \Cref{alg:TrackingAlg}.
		\State For each eigenvalue $k=1,2,\ldots,N_e$ compute:
		\begin{algsubstates}
			\State The arc-length of $\log(\tilde{\mu}_{t}^{k})$ in $t\in[0,1]$ using the following expression: 
			\[
			a_k = \int_{0}^{1} \sqrt{1+\frac{d}{dt}\log(\tilde{\mu}_{t}^{k})} dt.
			\]
			where $\frac{d}{dt}\log(\tilde{\mu}_{t}^{k})$ is estimated by the discrete differentiation along the uniform grid of $t\in [0,1]$, and similarly - $\int_{0}^{1} (\cdot) dt$ is estimated by the discrete summation.
			\State The length of the straight line connecting $\log(\tilde{\mu}_{0}^{k})$ with  $\log(\tilde{\mu}_{1}^{k})$: 
			\[ 
			l_i = \sqrt{1+\big( \log(\tilde{\mu}_{1}^{k})-\log(\tilde{\mu}_{0}^{k})\big)^2}
			\]
			\State A score defined by: $w_k=\frac{a(k)-l(k)}{a(k)}$. Note that $0 \leqslant w_k \leqslant 1$, and that for common eigenvectors: $k \in \mathcal{S}_{c}$, $a(k) \rightarrow l(k)$ and  $w_k \simeq 0$. 
			This score is viewed as a measure of the ``commonality'' of the $k$th eigenvector. 
		\end{algsubstates}
		\State The estimated CMR at each point $t$ along the geodesic path is computed by:
		\[
		\widehat{\text{CMR}}(t)=\frac{\sum_{k=1}^{N_e} (1-w_k)\cdot \tilde{\mu}_{t}^{k}  }{\sum_{k=1}^{N_e} w_k\cdot \tilde{\mu}_{t}^{k} }.
		\]
	\end{algorithmic}
	\caption{CMR estimation along the geodesic flow using eigenvectors identification.}
	\label{alg:COR}
\end{algorithm}

\paragraph{CMR estimation using eigenvectors dispersion}
\label{subsubsubsec:COR_dispersion}

\Cref{alg:COR} requires resolving the trajectories in the \acrshort*{EVFD}. However, in cases where extracting these trajectories is not explicitly required, estimating the CMR can be done in a simpler manner by relying on the empirical results presented in \Cref{sec:sup_SimuResults}. 

Rather than first resolving the trajectories of the eigenvalues along the geodesic path and then assessing their ``commonality'' by examining their lengths, we can directly assess the ``commonality'' of a single eigenvector by observing its ``dispersion'' along the geodesic path. 
In \Cref{sec:sup_SimuResults} we demonstrated the following. Given a certain eigenvector $v \in \mathbb{R}^n$, and the \acrshort*{EVD} of $\gamma(t)$ at a certain point along the geodesic path, $\{v_{t}^{k}\}_{k=1}^{N_e}$, the inner products between $v$ and the eigenvectors of $\gamma(t)$ satisfy:
\begin{equation}\label{eq:eigv_inner_product}
\left\langle v,v_{t}^{k}\right \rangle =
\begin{cases}
\delta(k-k_0) , & \text{if $v$ is a common eigenvector}	\\
C, & \text{otherwise}	
\end{cases}, k=1, \ldots,N_e
\end{equation}
where $k_0$ denotes the (unknown) index of the common eigenvector in $\{v_{t}^{k}\}_{k=1}^{N_e}$, and $C$ is a constant. 
Based on \eqref{eq:eigv_inner_product}, we can assess the ``commonality'' of each eigenvector by observing its inner-products with $\{v_{t}^{k}\}_{k=1}^{N_e}$ as described in \Cref{alg:COR2}.

\begin{algorithm}[t]
	\hspace*{\algorithmicindent} \textbf{Input}: The geodesic path $\gamma(t)$ \\
	\hspace*{\algorithmicindent} \textbf{Output}: Estimation of $\text{CMR}(t_0)$ for $t_0 \in[0,1]$ 
	\begin{algorithmic}[1]
		\State Calculate:
		\[
		\Gamma=\prod_{i=1}^{N_t} \gamma(t_i) = \gamma(t_{N_t}) \cdot \ldots \cdot \gamma(t_{1})
		\]
		where $\{t_i\}_{i=1}^{N_t} $ is a discrete uniform grid in $[0,1]$
		\State  Compute the eigenvectors of $\Gamma$, denoted by $\{u_{\Gamma}^j\}_{j=1}^n$. 
		\State For each eigenvector $v_{t_0}^{k}$ of $\gamma(t_0)$ for $k=1,2,\ldots,N_e$:
		\begin{algsubstates}
			\State Calculate the inner-products between $v_{t_0}^{k}$ and all the eigenvectors of $\Gamma$:
			\begin{equation*}
			c_j = \left\langle v_{t_0}^{k},u_{\Gamma}^j\right \rangle, j=1,\ldots,n.
			\end{equation*}
			\State Compute the following ``commonality'' score of the $k$th eigenvector $v_{t_0}^{k}$: $w_k=-\sum_{j=1}^{n}c_j \log(c_j)$, which is nothing but  the entropy of the vector of inner-products. Therefore, $0 \leqslant w_k \leqslant \log(n)$. If $v_{t_0}^{k}$ is a common eigenvector, i.e., $k \in \mathcal{S}_{\text{c}}$, then $w_k \rightarrow 0$. Conversely, if $k \notin \mathcal{S}_{\text{c}}$, then $w_k \rightarrow \log(n)$. 
		\end{algsubstates}
		\State The estimated $\text{CMR}(t_0)$ is computed by:
		\[
		\widehat{\text{CMR}}(t_0)=\frac{\sum_{k=1}^{N_e} \left(\log\left(n\right)-w_k\right) \mu_{t_0}^{k}  }{\sum_{k=1}^{N_e} w_k \mu_{t_0}^{k} }.
		\]
		where $\mu_{t_0}^{k}$ is the corresponding eigenvalue of $v_{t_0}^{k}$.
	\end{algorithmic}
	\caption{CMR estimation using eigenvectors inner products.}
	\label{alg:COR2}
\end{algorithm}

\clearpage

\section{Related work and additional background}
\label{sec:sup_related}

In this appendix, we present additional background on diffusion maps and describe alternating diffusion in detail.

\subsection{Diffusion maps}
\label{subsec:sup_DM}

In \Cref{sec:prelim}, we presented a formulation of diffusion maps (\acrshort*{DM}) in a discrete space.
Here, we present \acrshort*{DM} in a continuous space, providing additional theoretical background and preliminaries.

Let $(\mathcal{M},d\mu)$ be a measure space, where $\mathcal{M}$ is a compact smooth Riemannian manifold and $d\mu(x) = p(x)dx$ is a measure with density $p(x) \in C^3(\mathcal{M})$, which is a positive definite function with respect to the volume measure $dx$ on $\mathcal{M}$. 
Assume that $\mathcal{M}$ is isometrically embedded in $\mathbb{R}^d$, and let $w_{\epsilon}:\mathcal{M} \times \mathcal{M} \rightarrow \mathbb{R}$ be a symmetric positive definite kernel given by
\begin{equation*}
    w_{\epsilon}(x,y) = \exp\left( - \frac{\| x-y\|^2}{\epsilon}\right),
\end{equation*}
where $\| \cdot \|$ denotes the Euclidean norm in $\mathbb{R}^d$.

Let $p_\epsilon:\mathcal{M} \rightarrow \mathbb{R}$ be the local measure of the volume, given by
\begin{equation*}
    p_\epsilon(x) = \int _{\mathcal{M}} w_{\epsilon}(x,y)p(y)dy,
\end{equation*}
and form a new kernel  $k_{\epsilon}:\mathcal{M} \times \mathcal{M} \rightarrow \mathbb{R}$ by
\begin{equation*}
    k_{\epsilon}(x,y) = \frac{w_{\epsilon}(x,y)}{p_\epsilon(x) p_\epsilon (y)}.
\end{equation*}
This particular kernel normalization allows to mitigate non-uniform distributions (see \cite{coifman2006diffusion} for details).

Now, apply another normalization to this kernel and form $a_{\epsilon}:\mathcal{M} \times \mathcal{M} \rightarrow \mathbb{R}$
\begin{equation*}
    a_\epsilon (x,y) = \frac{k_{\epsilon}(x,y)}{d_{\epsilon}(x)},
\end{equation*}
where $d_{\epsilon}: \mathcal{M} \rightarrow \mathbb{R}$ is given by
\begin{equation*}
   d_{\epsilon}(x) = \int _{\mathcal{M}} k_{\epsilon}(x,y)p(y)dy.
\end{equation*}

Define the diffusion operator $A_{\epsilon}:L^2(\mathcal{M},d\hat{\mu}) \rightarrow L^2(\mathcal{M},d\hat{\mu})$ by
\begin{equation*}
    A_{\epsilon} f(x) = \int _\mathcal{M} a_{\epsilon}(x,y) f(y) p(y) d y.
\end{equation*}
Theorem 2 in \cite{coifman2006diffusion} states that the infinite generator corresponding to $A_\epsilon$ converges to the Laplace-Beltrami operator on $\mathcal{M}$, i.e.,
\begin{equation}
    \frac{1}{\epsilon}(I - A_\epsilon)f \rightarrow \Delta f
    \label{eq:cor1_marshall}
\end{equation}
in $L^2$ norm as $\epsilon \rightarrow 0$, where $\Delta$ is the positive semi-definite Laplace-Beltrami operator on $\mathcal{M}$, for any smooth function $f$\footnote{$f$ is in the span of the dominant Neumann eigenfunctions of $\Delta$.}.

This convergence result implies that $A_{\epsilon}$ may be used for approximating the Laplace-Beltrami operator $\Delta$.
Since the eigenfunctions of $\Delta$ bear all the geometric information on the manifold $\mathcal{M}$ and can be used for embedding the points on the manifold $x \in \mathcal{M}$ \cite{berard1994embedding,jones2008manifold}, the eigenfunctions of $A_{\epsilon}$ may be used for the same purpose.

In \Cref{alg:dm} we present the discrete approximation of the diffusion operator $A_{\epsilon}$ from a finite dataset of samples (as described in \Cref{sec:prelim}).
The row-stochastic affinity matrix $\mathbf{A}$ in \Cref{alg:dm} can be viewed as the discrete counterpart of $A_\epsilon$, and by Theorem 2 in \cite{coifman2006diffusion}, its eigenvalues and eigenvectors can be used to approximate the eigenvalues and eigenfunctions of the Laplace-Beltrami operator $\Delta$ on $\mathcal{M}$.
Specifically, if $\lambda$ is an eigenvalue of $\Delta$, then, 
\begin{equation}
\label{eq:Cont2Discrete}
\mu = \exp \bigg( -\frac{\epsilon ^2}{4} \lambda\bigg)
\end{equation}
is its corresponding eigenvalue of $\mathbf{A}$ (see Equation (7) in \cite{dsilva2018parsimonious}).

\begin{algorithm}[t]	
	\hspace*{\algorithmicindent} \textbf{Input}: A set of $n$ samples: $\left\{ s_i\right\}_{i=1}^n$, where $s_i \in \mathcal{M} \subset \mathbb{R}^d$.\\
	\hspace*{\algorithmicindent} \textbf{Output}: A column stochastic diffusion operator: $\mathbf{A}\in \mathbb{R}^{n\times n}$.\\

	\begin{algorithmic}[1]
	    \State Construct an affinity matrix  $\mathbf{W}$ using a Gaussian kernel:
	        \begin{equation*}
	        W_{i,j}=\exp\left(-\frac{\|s_i-s_j\|^2}{\varepsilon}\right) 
	        \end{equation*}
	        for all $i,j=1,\hdots,n$, where $\varepsilon$ is a tunable kernel.
	   \State Compute a diagonal matrix $\overline{\mathbf{D}} = \text{diag}(\mathbf{W}\mathbf{1})$, where $\mathbf{1}$ is a column vector of all ones.
	   \State Construct a normalized affinity matrix $\overline{\mathbf{K}} = \overline{\mathbf{D}}^{-1} \mathbf{W} \overline{\mathbf{D}}^{-1}$.
	   \State Compute a diagonal matrix  $\mathbf{D} = \text{diag}(\overline{\mathbf{K}}\mathbf{1})$.
	   \State Construct a row stochastic diffusion operator:
	   \begin{equation*}
             \mathbf{A} =\mathbf{D}^{-1} \overline{\mathbf{K}}
        \end{equation*}
	\end{algorithmic}
	\caption{Diffusion Operator}
	\label{alg:dm}
\end{algorithm}

\subsection{Alternating diffusion}
\label{subsec:sup_AD}

In standard canonical correlation analysis (CCA) \cite{hotelling1936relations}, given two sets, the goal is to find two linear projections of the sets, such that the correlation between the projections is maximized. 
In this work, we approach the problem from a manifold learning perspective. Broadly, by using this approach, the statistical considerations are replaced by geometric considerations. Specifically, the commonality between two sets is modeled by the existence of a common manifold (as formulated in \Cref{sec:ProblemFormulation}) rather than by a statistical correlation criterion. As described in \Cref{sec:related}, a large number of multi-manifold learning techniques (as well as related kernel methods) have been proposed recently. Of particular interest in the context of this work is alternating diffusion (AD) \cite{lederman2018learning} for two reasons. First, unlike most of the other methods for multi-manifold learning, AD attempts to extract the common manifold of two sets rather than just to find a fusion of the sets. Second, alternating diffusion could be viewed as an extension of diffusion maps \cite{coifman2006diffusion}, which, as described in \Cref{sec:prelim}, introduces the highly useful notions of diffusion on data and diffusion distance.

Alternating diffusion is based on alternating applications of diffusion maps steps on the two sets. In \cite{lederman2018learning}, it was shown that such an alternating diffusion procedure, implemented based on the product of two diffusion operators, gives rise to a new distance, termed alternating diffusion distance, that captures a distance only between common variables and is independent of other non-common variables. In \cite{talmon2018latent}, the analysis was extended to the setting presented in \Cref{sec:ProblemFormulation} that includes hidden manifolds, and it was shown that the alternating diffusion kernel converges to the Laplace-Beltrami operator of the latent common manifold in a similar fashion to the convergence of diffusion maps described in \Cref{subsec:sup_DM}.

Alternating diffusion, as implemented in the results reported in this paper, is described in \Cref{alg:ad}.

\begin{algorithm}[t]	
    \hspace*{\algorithmicindent} \textbf{Input}: Two sets of $n$ (aligned) measurements:
	\begin{equation*}
	\left\{s^{(1)}_i\right\}_{i=1}^n, \left\{s^{(2)}_i\right\}_{i=1}^n
	\end{equation*}
	\hspace*{\algorithmicindent} where $s^{(1)}_i \in \mathcal{O}_1$ and $s^{(2)}_i \in \mathcal{O}_2$.\\
	\hspace*{\algorithmicindent} \textbf{Output}: The kernel $\mathbf{K}_{\text{AD}}$.\\
	\hspace*{\algorithmicindent} \textbf{Parameters}: 
	\begin{itemize}
		\item $s$ -- The number of alternations to perform.
	\end{itemize}
	
	\begin{algorithmic}[1]
	    \State Apply \Cref{alg:dm} to $\left\{s^{(1)}_i\right\}_{i=1}^n$ and to $\left\{s^{(2)}_i\right\}_{i=1}^n$, twice separately, and obtain the two diffusion operators: $\mathbf{
	    A}_1$ and $\mathbf{A}_2$.
	    \State Build a column-stochastic kernel:
	        \begin{equation*}
	            \label{eq:ad_kernel}
	            \mathbf{A}_{\text{AD}}=\mathbf{A}_2^T\mathbf{A}_1^T
	        \end{equation*}
	   \State Compute the diffusion distances with respect to $\mathbf{A}_{\text{AD}}$ by \eqref{eq:diff_dist}:
	   	        \begin{equation*}
	                 \label{eq:ad_dist}
	                 D_{i,j}=\sqrt{\sum_{k=1}^{n} \left( \left(\mathbf{A}_{\text{AD}}^s\right)_{k,i} - \left(\mathbf{A}_{\text{AD}}^s\right)_{k,j} \right)^2}, \ i,j=1,\ldots,n
	        \end{equation*}
	   \State Construct an affinity matrix  $\mathbf{W}$ using a Gaussian kernel:
	        \begin{equation*}
	            W_{i,j}=\exp\left(-\frac{D^2_{i,j}}{\varepsilon}\right), \ i,j=1,\ldots,n
	        \end{equation*}
	        where $\varepsilon$ is a tunable scale.
	   \State Follow steps 2-5 in \Cref{alg:dm}, and denote the resulting kernel by $\mathbf{K}_{\text{AD}}$.
	\end{algorithmic}
	\caption{Alternating diffusion}
	\label{alg:ad}
\end{algorithm}

Next, we discuss the relationship between the proposed method and AD.
We begin with two important remarks. First, the main goal of the two methods is different -- AD aims to build an embedding that represents the latent common variables between the two sets of aligned measurements and is not restricted to models where the common variable corresponds to common eigenvectors, whereas the proposed method primarily provides an analysis tool. Second, in order to establish a tractable connection, we consider in the remainder of this appendix a variant of AD.

Let $t \in \mathbb{Q} \cap (0,1)$, and let $s_1 , s_2 \in \mathbb{N}$ be two positive integers such that $t=\frac{s_2}{s_1+s_2}$. 
Assume there exists a common eigenvector $v \in \mathbf{R}^n$ of $\mathbf{K}_1$ and $\mathbf{K}_2$ with corresponding eigenvalues $\mu_1$ and $\mu_2$.
By \eqref{eq:Cont2Discrete}, we have:
\begin{gather*}
\mathbf{K}_1 v = \mu_1 v = \exp\bigg( -\frac{\left(\epsilon^{(1)}\right) ^2}{4}\lambda_1 \bigg) v \\
\mathbf{K}_2 v = \mu_2 v = \exp\bigg( -\frac{\left(\epsilon^{(2)}\right) ^2}{4}\lambda_2 \bigg) v ,
\end{gather*}
where $\lambda_1$ and $\lambda_2$ are the corresponding eigenvalues of the continuous counterparts of $\mathbf{K}_1$ and $\mathbf{K}_2$, respectively.
By \Cref{prop:MutualEigenValues}, $v$ is also an eigenvector of $\gamma(t)$ with the following eigenvalue:
\begin{equation*}
\mu_{\gamma(t)} = \mu_{1}^{1-t} \mu_{2}^{t} = \mu_{1}^{\frac{s1}{s_1+s_2}} \mu_{2}^{\frac{s2}{s_1+s_2}}.
\end{equation*}

Now, construct $\widetilde{\mathbf{K}}_1$ and $\widetilde{\mathbf{K}}_2$ with finer scales given by:
\begin{align*}
\tilde{\epsilon}^{(1)}=\frac{\epsilon^{(1)}}{\sqrt{s_1+s_2}} && \tilde{\epsilon}^{(2)}=\frac{\epsilon^{(2)}}{\sqrt{s_1+s_2}}
\end{align*}
similarly to the construction of $\mathbf{K}_1$ and $\mathbf{K}_2$.
Namely, we build two kernels with scales smaller by a factor of $\sqrt{s_1+s_2}$.
Then, we define a variant of the AD kernel $\gamma_{\text{AD}}(t)$ by:
\begin{align*}
\gamma_{\text{AD}}(t)=\left(\widetilde{\mathbf{K}}_1\right)^{s_1}\left(\widetilde{\mathbf{K}}_2\right)^{s_2}.
\end{align*}
In the context of AD, this operator consists of $s_2$ diffusion steps using $\widetilde{\mathbf{K}}_2$ followed by $s_1$ diffusion steps using $\widetilde{\mathbf{K}}_1$.

Revisiting the common eigenvector $v$ considered above, we have:
\begin{gather*}
\widetilde{\mathbf{K}}_1 v = \exp\bigg( -\frac{\left(\epsilon^{(1)}\right) ^2}{4(s_1+s_2)}\lambda_1 \bigg) v = \mu_1^{1/(s1+s2)} v\\
\widetilde{\mathbf{K}}_2 v = \exp\bigg( -\frac{\left(\epsilon^{(2)}\right) ^2}{4(s_1+s_2)}\lambda_2 \bigg) v = \mu_2^{1/(s1+s2)} v,
\end{gather*}
namely, $v$ is an eigenvector of $\widetilde{\mathbf{K}}_1$ and $\widetilde{\mathbf{K}}_2$ with corresponding eigenvalues $\mu_1^{1/(s1+s2)}$ and $\mu_2^{1/(s1+s2)}$.
Therefore, $v$ is also an eigenvector of $\gamma_{\text{AD}}(t)$:
\begin{gather*}
\gamma_{\text{AD}}(t) v = \left(\widetilde{\mathbf{K}}_1\right)^{s_1}\left(\widetilde{\mathbf{K}}_2\right)^{s_2} v =
 \mu_{1}^{\frac{s1}{s_1+s_2}} \mu_{2}^{\frac{s2}{s_1+s_2}} v .
\end{gather*}
with eigenvalue given by
$$
\mu_{1}^{\frac{s1}{s_1+s_2}} \mu_{2}^{\frac{s2}{s_1+s_2}}  = \mu_{1}^{1-t} \mu_{2}^{t} v .
$$
We obtain that $v$ is an eigenvector of both $\gamma(t)$ and $\gamma_{\text{AD}}(t)$ with the same eigenvalue, where the difference is the scales used in the construction of the respective kernels.

When $\mathbf{K}_1$ and $\mathbf{K}_2$ share the same set of eigenvectors, i.e.  $[\mathbf{K}_1,\mathbf{K}_2]=0$, then the relationship to AD becomes tighter.
\begin{proposition}
	If $[\mathbf{K}_1,\mathbf{K}_2]=0$, then $\gamma(t)=\gamma_{AD}(t)$ for every $t\in  \mathbb{Q}$.
\end{proposition}
\begin{proof}
	If $[\mathbf{K}_1,\mathbf{K}_2]=0$, then $\mathbf{K}_1$ and $\mathbf{K}_2$ have the same set of eigenvectors, i.e., we can write:
	\begin{eqnarray*}
		\mathbf{K}_1 = \mathbf{U} \mathbf{S}_1  \mathbf{U}^T \\ 
		\mathbf{K}_2 = \mathbf{U} \mathbf{S}_2  \mathbf{U}^T \\ 
	\end{eqnarray*}
	Therefore:
	\begin{align*}
	\gamma(t)= & \mathbf{K}_1^{1/2}( \mathbf{K}_1^{-1/2}  \mathbf{K}_2  \mathbf{K}_1^{-1/2})^t  \mathbf{K}_1^{1/2} =\\
	& \mathbf{U} \mathbf{S}_1^{1/2}  \mathbf{U}^T (\mathbf{U} \mathbf{S}_1^{-1/2}  U^T  \mathbf{U} \mathbf{S}_2^{1/2}  \mathbf{U}^T  \mathbf{U} \mathbf{S}_1^{-1/2}  \mathbf{U}^T )^t  \mathbf{U} \mathbf{S}_1^{1/2}  \mathbf{U}^T = \\
	& \mathbf{U} \mathbf{S}_1^{(1-t)}  \mathbf{S}_2^{t} \mathbf{U}^T.
	\end{align*}
	Similarly:
	\begin{eqnarray*}
		\widetilde{\mathbf{K}}_1 = \mathbf{U} \mathbf{S}_1^{1/(s_1+s_2)}  \mathbf{U}^T \\ 
		\widetilde{\mathbf{K}}_2 = \mathbf{U} \mathbf{S}_2^{1/(s_1+s_2)}  \mathbf{U}^T \\ 
	\end{eqnarray*}
	Therefore:
	\begin{align*}
	\gamma_{AD}(t)&= \left(\widetilde{\mathbf{K}}_1\right)^{s_1}\left(\widetilde{\mathbf{K}}_2\right)^{s_2}\\
	&=\mathbf{U} \mathbf{S}_1^{(1-t)}  \mathbf{S}_2^{t} \mathbf{U}^T=\gamma(t)
	\end{align*}
\end{proof}

We end this section with a couple of remarks.
First, the case in which $[\mathbf{K}_1,\mathbf{K}_2]=0$ is typically not of interest, since it implies that the two kernels have only common eigenvectors.
Second, the alternating diffusion kernel ($\mathbf{A}_{\text{AD}}$ in \Cref{alg:ad}) is not designed to be symmetric (as $\gamma_{\text{AD}}(t)$); AD in its original formulation in \cite{lederman2018learning} includes an additional step of constructing a new kernel $\mathbf{K}_{\text{AD}}$ from a kernel analogous to $\gamma_{\text{AD}}(0.5)$ here. Conversely, $\gamma(t)$ is symmetric and positive, and therefore, may facilitate spectral embedding directly (see \Cref{sec:embedding}). 
In our experimental study in \Cref{sec:RealData}, we compare the proposed method based on $\gamma(t)$ to the original formulation of AD.

\newpage

\section{Proofs}
\label{sec:sup_TheorericFoundations}

\paragraph*{Proof of \cref{prop:MutualEigenValues}.}
	Let $\mathbf{M} = \mathbf{K}_1^{-1/2}\mathbf{K}_2\mathbf{K}_1^{-1/2}$. By definition, $v$ is an eigenvector of $\mathbf{M}$ with eigenvalue: $(\mu_1)^{-1}\mu_2$. Therefore, $v$ is also an eigenvector of $\mathbf{M}^t$ with eigenvalue: $\mu_1^{-t}\mu_2^{t}$.
	Recalling that $\gamma(t)$, defined in \cref{eq:geodesic}, can be recast as $\gamma(t) = \mathbf{K}_1^{1/2}\mathbf{M}^t\mathbf{K}_1^{1/2} $ yields:
		\begin{equation*}
		\begin{aligned}
		\gamma(t) v &= \mathbf{K}_1^{1/2}\mathbf{M}^t\mathbf{K}_1^{1/2}  v \\
		&=	\mu_1^{1/2}\mu_1^{-t}\mu_2^t \mu_1^{1/2} v= 
		& \mu_1^{1-t} \mu_2^{t} v
		\end{aligned}
		\end{equation*}

\begin{proposition}
\label{prop:GeodesicEquivalences}
    The following definitions of the operator $\gamma(t)$ are equivalent:
    \begin{enumerate}
        \item $\gamma(t)=\mathbf{K}_1^{\frac{1}{2}}  \left(\mathbf{K}_1^{-\frac{1}{2}}\mathbf{K}_2 \mathbf{K}_1^{-\frac{1}{2}} \right)^t \mathbf{K}_1^{\frac{1}{2}}$
        \item  $\gamma(t)=\mathbf{K}_1(\mathbf{K}_1^{-1}\mathbf{K}_2)^t$
        \item  $\gamma(t)=(\mathbf{K}_2\mathbf{K}_1^{-1})^t \mathbf{K}_1$
        \item  $\gamma(t)=\mathbf{K}_1 \exp\{t \log{\left(\mathbf{K}_1^{-1}\mathbf{K}_2\right)}\}$
         \item  $\gamma(t)=\exp\{t \log{\left(\mathbf{K}_2\mathbf{K}_1^{-1}\right)}\} \mathbf{K}_1$ 
    \end{enumerate}
\end{proposition}
\begin{proof}
Denote: $\mathbf{C} \triangleq \mathbf{K}_1^{-\frac{1}{2}}\mathbf{K}_2 \mathbf{K}_1^{-\frac{1}{2}}$ and $\mathbf{M}\triangleq \mathbf{K}_1^{-1}\mathbf{K}_2$. Notice that $\mathbf{C}$ and $\mathbf{M}$ are similar, specifically - if $v$ is an eigenvector of $\mathbf{M}$ then $\mathbf{K}_1^{\frac{1}{2}}v$ is an eigenvector of $\mathbf{C}$. 
Denote the eigen-decomposition of $\mathbf{M}$ by $\mathbf{M}= USU^T$. Since $\mathbf{K}_1$ is SPD, the  eigen-decomposition of $\mathbf{C}$ is given by $\mathbf{C}= \mathbf{K}_1^{\frac{1}{2}}USU^T\mathbf{K}_1^{-\frac{1}{2}}$. Therefore: $\mathbf{C}^t= \mathbf{K}_1^{\frac{1}{2}}US^tU^T\mathbf{K}_1^{-\frac{1}{2}}= \mathbf{K}_1^{\frac{1}{2}} \mathbf{M}^t \mathbf{K}_1^{-\frac{1}{2}}$.
Plugging this term in the first expression for $\gamma(t)$ yields the equivalence to the second expression:
\[
\gamma(t)=\mathbf{K}_1^{\frac{1}{2}}  \left(\mathbf{K}_1^{-\frac{1}{2}}\mathbf{K}_2 \mathbf{K}_1^{-\frac{1}{2}} \right)^t \mathbf{K}_1^{\frac{1}{2}}  = \mathbf{K}_1^{\frac{1}{2}} \mathbf{C}^t \mathbf{K}_1^{\frac{1}{2}} =\mathbf{K}_1^{\frac{1}{2}} \mathbf{K}_1^{\frac{1}{2}} \mathbf{M}^t \mathbf{K}_1^{-\frac{1}{2}} \mathbf{K}_1^{\frac{1}{2}}=\mathbf{K}_1\mathbf{M}^t.
\]

The same approach can be used in order to show the equivalence between the first expression for $\gamma(t)$ and the third expression.  The equivalence to the fourth (fifth) expression is a direct consequence of the equivalence to the second (third) expression.
\end{proof}

\paragraph*{Proof of \cref{prop:WeaklyCommon}.}
The proof of this proposition follows the technique introduced in Theorem 4 in \cite{shnitzer2022spatiotemporal}.
According to \cref{prop:GeodesicEquivalences} we can rewrite the geodesic term as follows: $\gamma(t)=(\mathbf{K}_2\mathbf{K}_1^{-1})^t \mathbf{K}_1$. Therefore:
\begin{align*}
    \norm{\left(\gamma(t) - \left(\mu^j_1\right)^{1-t}\left(\mu^j_2\right)^t \mathbf{I}\right)v_1^j} &=  
    \norm{\left((\mathbf{K}_2\mathbf{K}_1^{-1})^t \mathbf{K}_1 - \left(\mu^j_1\right)^{1-t}\left(\mu^j_2\right)^t \mathbf{I}\right)v_1^j} \\
    &= \mu^j_1 \norm{\left((\mathbf{K}_2\mathbf{K}_1^{-1})^t- \left( \frac{\mu^j_2}{\mu^j_1}\right)^t \mathbf{I}\right)v_1^j}. \numberthis\label{eq:equivalence_norm}
\end{align*}
In the following steps we will upper-bound $\norm{\left((\mathbf{K}_2\mathbf{K}_1^{-1})^t- \left( \frac{\mu^j_2}{\mu^j_1}\right)^t \mathbf{I}\right)v_1^j}$.

Note that $\mathbf{K}_2\mathbf{K}_1^{-1}$ is similar to $\mathbf{K}_1^{-\frac{1}{2}}\mathbf{K}_2 \mathbf{K}_1^{-\frac{1}{2}}$, hence:
\[
f(\mathbf{K}_2\mathbf{K}_1^{-1}) = \mathbf{K}_1^{\frac{1}{2}} f(\mathbf{K}_1^{-\frac{1}{2}}\mathbf{K}_2 \mathbf{K}_1^{-\frac{1}{2}}) \mathbf{K}_1^{-\frac{1}{2}},
\]
for every analytic function $f$ defined over an open set in $\mathbb{R}$ contains the spectrum of $\mathbf{K}_2\mathbf{K}_1^{-1}$. Specifically, considering $f(x)=x^t$, where $0 \le x \le 1$,  we have: 
\[
\left(\mathbf{K}_2\mathbf{K}_1^{-1} \right)^t  = \mathbf{K}_1^{\frac{1}{2}} \left(\mathbf{K}_1^{-\frac{1}{2}}\mathbf{K}_2 \mathbf{K}_1^{-\frac{1}{2}}\right)^t \mathbf{K}_1^{-\frac{1}{2}} 
\]
Therefore:
\begin{align*}
\left(\mathbf{K}_2\mathbf{K}_1^{-1}\right)^t -  \left( \frac{\mu^j_2}{\mu^j_1}\right)^t \mathbf{I}  &= 
\mathbf{K}_1^{\frac{1}{2}} \left(\mathbf{K}_1^{-\frac{1}{2}}\mathbf{K}_2 \mathbf{K}_1^{-\frac{1}{2}}\right)^t \mathbf{K}_1^{-\frac{1}{2}}- \mathbf{K}_1^{\frac{1}{2}} \mathbf{K}_1^{-\frac{1}{2}}   \left(\frac{\mu^j_2}{\mu^j_1}\right)^t \mathbf{I}   \\
&=\mathbf{K}_1^{\frac{1}{2}} \left(   \left(\mathbf{K}_1^{-\frac{1}{2}}\mathbf{K}_2 \mathbf{K}_1^{-\frac{1}{2}}\right)^t -  \left(\frac{\mu^j_2}{\mu^j_1}\right)^t \mathbf{I}  \right) \mathbf{K}_1^{-\frac{1}{2}},
\end{align*}
and:
\[
\left(  \left(\mathbf{K}_2\mathbf{K}_1^{-1}\right)^t -  \left( \frac{\mu^j_2}{\mu^j_1}\right)^t \mathbf{I}  \right) v_1^j =
\mathbf{K}_1^{\frac{1}{2}} \left(   \left(\mathbf{K}_1^{-\frac{1}{2}}\mathbf{K}_2 \mathbf{K}_1^{-\frac{1}{2}}\right)^t -  \left(\frac{\mu^j_2}{\mu^j_1}\right)^t \mathbf{I}  \right) \frac{1}{\sqrt{\mu_1^j}} v_1^j.
\]
Let $\{\lambda_i\}_{i=1}^n$ and $\{u_i\}_{i=1}^n$ be the real and positive eigenvalues and corresponding orthonormal set of eigenvectors of the SPD matrix $\mathbf{K}_1^{-\frac{1}{2}}\mathbf{K}_2 \mathbf{K}_1^{-\frac{1}{2}}$. 
Expanding $v_1^j$ in the basis $\{u_i\}_{i=1}^n$ yields:
\begin{align}
\label{eq:expansion}
\left(  \left(\mathbf{K}_2\mathbf{K}_1^{-1}\right)^t -  \left( \frac{\mu^j_2}{\mu^j_1}\right)^t \mathbf{I}  \right) v_1^j =\frac{1}{\sqrt{\mu_1^j}}  \mathbf{K}_1^{\frac{1}{2}}  \sum_{i=1}^n \alpha_{ij} \left( \lambda_i^t - \left(\frac{\mu^j_2}{\mu^j_1}\right)^t \right) u_i 
\end{align}
where $\alpha_{ij}= \langle u_i, v_1^j \rangle = u_i^\top v_1^j $ are the expansion coefficient.
Taking the squared norm of \cref{eq:expansion} yields:
\begin{align*}
\norm{\left(  \left(\mathbf{K}_2\mathbf{K}_1^{-1}\right)^t -  \left( \frac{\mu^j_2}{\mu^j_1}\right)^t \mathbf{I}  \right) v_1^j }^2 & = 
 \frac{1}{\mu_1^j} \norm{ \mathbf{K}_1^{\frac{1}{2}}  \sum_{i=1}^n \alpha_{ij} \left( \lambda_i^t - \left(\frac{\mu^j_2}{\mu^j_1}\right)^t \right) v_i  }^2 \\
 &\leq  \frac{1}{\mu_1^j} \norm{\mathbf{K}_1^{\frac{1}{2}}}^2   \sum_{i=1}^n |\alpha_{ij}|^2 \left| \lambda_i^t - \left(\frac{\mu^j_2}{\mu^j_1}\right)^t \right|^2 \\
&=\frac{1}{\mu_1^j}  \sum_{i=1}^n |\alpha_{ij}|^2 \left| \lambda_i^t - \left(\frac{\mu^j_2}{\mu^j_1}\right)^t \right|^2. \numberthis \label{eq:norm_expansion2}
\end{align*}
where the last transition is due to the fact that $\norm{\mathbf{K}^{\frac{1}{2}}_1}=1$. Plugging this result in \cref{eq:equivalence_norm} we get:
\begin{align*}
\norm{\left(\gamma(t) - \left(\mu^j_1\right)^{1-t}\left(\mu^j_2\right)^t \mathbf{I}\right)v_1^j}^2 &\leq \mu_1^j   \sum_{i=1}^n |\alpha_{ij}|^2 \left| \lambda_i^t - \left(\frac{\mu^j_2}{\mu^j_1}\right)^t \right|^{2} \\
&= \mu_1^j   |\alpha_{jj}|^2 \left| \lambda_j^t - \left(\frac{\mu^j_2}{\mu^j_1}\right)^t \right|^{2} + \mu_1^j   \sum_{i \neq j} |\alpha_{ij}|^2 \left| \lambda_i^t - \left(\frac{\mu^j_2}{\mu^j_1}\right)^t \right|^{2}. \numberthis \label{eq:norm_expansion}
\end{align*}

Now, we bound the expansion coefficients $|\alpha_{ij}|$ using classical perturbation theory \cite{rellich1969perturbation}. 
Specifically, using the EVD of $\mathbf{K}_i= \mathbf{V}_i \mathbf{M}_i \mathbf{V}^{\top}_i, i=1,2$ and $\mathbf{V}_2=\mathbf{V}_1+\epsilon\mathbf{E}$, we can recast $\mathbf{K}_1^{-\frac{1}{2}}\mathbf{K}_2 \mathbf{K}_1^{-\frac{1}{2}}$ as the perturbation:
\[
 \mathbf{K}_1^{-\frac{1}{2}}\mathbf{K}_2 \mathbf{K}_1^{-\frac{1}{2}} = \mathbf{V}_1 \mathbf{L} \mathbf{V}^{\top}_1 + \epsilon \mathbf{F},
\]
where $\mathbf{L} = \mathbf{M}_1^{-1} \mathbf{M}_2:=\text{diag}\left({\ell_1,\ell_2,\ldots,\ell_n}\right)$ and:
\[
 \mathbf{F} = \mathbf{V}_1 \mathbf{M}_1^{-\frac{1}{2}} \mathbf{V}_1 ^{\top} 
 \left( \mathbf{E} \mathbf{M}_2 \mathbf{V}_1 ^{\top}  + \mathbf{V}_1  \mathbf{M}_2  \mathbf{E}^{\top} \right)
 \mathbf{V}_1 \mathbf{M}_1^{-\frac{1}{2}} \mathbf{V}_1 ^{\top} 
 +\epsilon  \mathbf{V}_1 \mathbf{M}_1^{-\frac{1}{2}} \mathbf{V}_1 ^{\top} 
 \mathbf{E}  \mathbf{M}_2  \mathbf{E}^{\top}
 \mathbf{V}_1 \mathbf{M}_1^{-\frac{1}{2}} \mathbf{V}_1 ^{\top}.
\]
The perturbed eigenvectors and eigenvalues are given by:
\begin{align}
    \label{eq:pertubed_eigvec}
    u_i &= v_1^i+\epsilon \sum_{k\neq i} \frac{\langle v_1^k , \mathbf{F}v_1^i\rangle}{\ell_i-\ell_k} v_1^k +\mathcal{O}(\epsilon^2)\\
    \lambda_i &= \ell_i + \epsilon \langle v_1^i , \mathbf{F}v_1^i\rangle +\mathcal{O}(\epsilon^2).\label{eq:pertubed_eigval}
\end{align}

By \cref{eq:pertubed_eigvec}, we immediately get that the expansion coefficient for $i=j$ is:
\begin{align}\label{eq:alpha_jj}
    \alpha_{jj}= \left(v_1^j+\epsilon \sum_{k\neq j} \frac{\langle v_1^k , \mathbf{F}v_1^j\rangle}{\ell_j-\ell_k} v_1^k +\mathcal{O}(\epsilon^2)\right)^{\top} v_1^j = 1 +\mathcal{O}(\epsilon^2).
\end{align}

Note that:
\[
\mathbf{I}=  \mathbf{V}_2  \mathbf{V}_2^{\top} = \left( \mathbf{V}_1+  \epsilon \mathbf{E} \right)\left( \mathbf{V}_1+  \epsilon \mathbf{E} \right)^{\top} = \mathbf{I} +\epsilon \left( \mathbf{E}_1^{\top}\mathbf{V}_1 + \mathbf{V}_1^{\top}\mathbf{E}  \right)  +\epsilon^2  \mathbf{E} \mathbf{E}^{\top},
\]
therefore, we get that:
\begin{align}
\label{eq:pertubation_transpose}
\mathbf{E}^{\top}\mathbf{V}_1 = -  \mathbf{V}_1^{\top}\mathbf{E} -\epsilon^2  \mathbf{E} \mathbf{E}^{\top}.
\end{align}
We have:
\begin{align*}
\langle v_1^k , \mathbf{F}v_1^i\rangle &=   \frac{1}{\sqrt{\mu^k_1\mu^i_1}} \left( e_k^{\top}\mathbf{V}_1 ^{\top} 
 \left( \mathbf{E} \mathbf{M}_2 \mathbf{V}_1 ^{\top}  + \mathbf{V}_1  \mathbf{M}_2  \mathbf{E}^{\top} \right)
 \mathbf{V}_1 e_i +\epsilon  e_k^{\top} \mathbf{E} \mathbf{M}_2 \mathbf{E}^{\top} e_i \right) \\ 
 &=   \frac{1}{\sqrt{\mu^k_1\mu^i_1}} e_k^{\top}
 \left(  \mathbf{V}_1 ^{\top} \mathbf{E} \mathbf{M}_2  + \mathbf{M}_2  \mathbf{E}^{\top}\mathbf{V}_1  + \epsilon \mathbf{E} \mathbf{M}_2 \mathbf{E}^{\top}\right) e_i \\ 
  &=   \frac{1}{\sqrt{\mu^k_1\mu^i_1}} e_k^{\top}
 \left(  \mathbf{V}_1 ^{\top} \mathbf{E} \mu^i_2 + \mu^k_2   \mathbf{E}^{\top}\mathbf{V}_1  + \epsilon\mathbf{E} \mathbf{M}_2 \mathbf{E}^{\top} \right) e_i \\ 
 &=   \frac{1}{\sqrt{\mu^k_1\mu^i_1}} e_k^{\top}
 \left(  \mathbf{V}_1 ^{\top} \mathbf{E} \mu^i_2 -  \mu^k_2   \mathbf{V}_1^{\top}\mathbf{E}   -\epsilon^2 \mu^k_2 \mathbf{E} \mathbf{E}^{\top}  + \epsilon \mathbf{E} \mathbf{M}_2 \mathbf{E}^{\top}  \right) e_i \\ 
  &=   \frac{\left(\mu^i_2 -  \mu^k_2 \right)}{\sqrt{\mu^k_1\mu^i_1}}  e_k^{\top}\mathbf{V}_1 ^{\top} \mathbf{E} e_i
  +  \frac{\epsilon}{\sqrt{\mu^k_1\mu^i_1}} e_k^{\top} \mathbf{E} \left( \mathbf{M}_2 -\mu^k_2\epsilon\mathbf{I}\right) \mathbf{E}^{\top} e_i,
\end{align*}
where the third transition is by substituting \cref{eq:pertubation_transpose}. \\
For $k=i$, we get:
\begin{align}\label{eq:inner_product_ii}
\left| \langle v_1^i , \mathbf{F}v_1^i\rangle \right| & = \left| \epsilon \frac{1}{\mu^i_1} e_k^{\top} \mathbf{E} \left( \mathbf{M}_2 -\mu^k_2\epsilon\mathbf{I}\right) \mathbf{E}^{\top} e_i \right| \leq  \frac{\epsilon}{\mu^i_1} \norm{E} \norm{\mathbf{M}_2 -\mu^k_2\epsilon\mathbf{I}} \norm{E} \leq \frac{\epsilon}{\mu^i_1},
\end{align}
and for $k\neq i$, we get:
\begin{align*}
\left| \frac{\langle v_1^k , \mathbf{F}v_1^i\rangle}{\ell_i-\ell_k} \right| & =\left| \frac{\left(\mu^i_2 -  \mu^k_2 \right)}{\sqrt{\mu^k_1\mu^i_1} \left(\ell_i-\ell_k\right)} e_k^{\top}\mathbf{V}_1^{\top} \mathbf{E} e_i + \epsilon \frac{  e_k^{\top} \mathbf{E} \left( \mathbf{M}_2 -\mu^k_2\epsilon\mathbf{I}\right) \mathbf{E}^{\top} e_i}{\sqrt{\mu^k_1\mu^i_1} \left(\ell_i-\ell_k\right)} \right|  \\
&\leq \left| \frac{\left(\mu^i_2 -  \mu^k_2 \right)}{\sqrt{\mu^k_1\mu^i_1}\left(\ell_i-\ell_k\right)}  e_k^{\top}\mathbf{V}_1 ^{\top} \mathbf{E} e_i \right|  +   \epsilon \left| \frac{  e_k^{\top} \mathbf{E} \left( \mathbf{M}_2 -\mu^k_2\epsilon\mathbf{I}\right) \mathbf{E}^{\top} e_i}{\sqrt{\mu^k_1\mu^i_1} \left(\ell_i-\ell_k\right)} \right|\\
&\leq \frac{1}{\gamma_i} \left (\left|  \frac{\left(\mu^i_2 -  \mu^k_2 \right)}{\sqrt{\mu^k_1\mu^i_1}}  \right|
      \left| e_k^{\top}\mathbf{V}_1^{\top} \mathbf{E} e_i \right|  + \epsilon \left|\frac{ e_k^{\top} \mathbf{E} \left( \mathbf{M}_2 -\mu^k_2\epsilon\mathbf{I}\right) \mathbf{E}^{\top} e_i}{\sqrt{\mu^k_1\mu^i_1}} \right|\right)\\ 
&\leq \frac{1}{\gamma_i} \left (\left|  \frac{\left(\mu^i_2 -  \mu^k_2 \right)}{\sqrt{\mu^k_1\mu^i_1}}  \right|\left| e_k^{\top}\mathbf{V}_1^{\top} \mathbf{E} e_i \right|
       + \epsilon\frac{1}{\sqrt{\mu^k_1\mu^i_1}} \right) 
       \leq \frac{1}{\gamma_i} \frac{1}{\sqrt{\mu^k_1\mu^i_1}} \left(\left| e_k^{\top}\mathbf{V}_1^{\top} \mathbf{E} e_i \right|+\epsilon\right),
\end{align*}
where $|\mu^i_2 -  \mu^k_2|\leq 1$.

Combining this result with \cref{eq:pertubed_eigvec}, we get that the expansion coefficients for $i\neq j$ are given by:
\begin{align}\label{eq:alpha_ij}
    |\alpha_{ij}| &=  |u_i^{\top} v_1^j| = \epsilon \left| \frac{\langle v_1^j , \mathbf{F}v_1^i\rangle}{\ell_i-\ell_j}\right|  +\mathcal{O}(\epsilon^2) \leq \epsilon \frac{\left| e_j^{\top}\mathbf{V}_1^{\top} \mathbf{E} e_i \right|}{\gamma_i\sqrt{\mu^j_1\mu^i_1}} + \mathcal{O}(\epsilon^2),
\end{align}
where the constant depends on $\frac{1}{\gamma_i \sqrt{\mu^j_1\mu^i_1} }$. 

The Taylor expansion of $\lambda_j^t$ at point $\ell_j$ is given by  $\lambda_j^{t} = \ell_j^t +t\ell_j^{t-1}\left(\lambda_j-\ell_j\right)+\mathcal{O}\left((\lambda_j-\ell_j)^2\right)$.
Using this Taylor expansion, we get:
\begin{align*}
    \left| \lambda_j^t - \ell_j^t \right|^{2}     &= \left| t\ell_j^{t-1}\left(\lambda_j-\ell_j\right)+\mathcal{O}\left((\lambda_j-\ell_j)^2\right) \right|^2 \\
    &= \left| t\ell_j^{t-1}\epsilon \langle v_1^j , \mathbf{F}v_1^j\rangle +\mathcal{O}(\epsilon^2) \right|^2 = O(\epsilon^4),
\end{align*}
where the first transition is due to \cref{eq:pertubed_eigval} and the second transition is due to \cref{eq:inner_product_ii}. Combining this with \cref{eq:alpha_jj}, we have that the first term in \cref{eq:norm_expansion} is negligible.

As for the second term in \cref{eq:norm_expansion}, by similar arguments, we have
that for $i\neq j$:
\begin{align*}
    \left| \lambda_i^t - \ell_j^t \right|^2 &=
     \left| \ell_i^t + t \ell_i^{t-1}\epsilon \langle v_1^i , \mathbf{F}v_1^i\rangle  +\mathcal{O}(\epsilon^2) - \ell_j^t \right|^2 \\
     &\leq \left| \ell_i^t-\ell_j^t \right|^2+ O(\epsilon^4)\\
      &\leq c^{2t}+\mathcal{O}(\epsilon^4). \numberthis \label{eq:diff_ij}
\end{align*}
Therefore, combining \cref{eq:alpha_ij} and \cref{eq:diff_ij}, the second term is bounded by $O(\epsilon^2)$ because:
\begin{align}\label{eq:last_eq_prop_weakly}
    \mu_1^j \sum_{i \neq j} |\alpha_{ij}|^2 \left| \lambda_i^t - \left(\frac{\mu^j_2}{\mu^j_1}\right)^t \right|^{2} 
    &\leq  2 \epsilon^2 \mu_1^j c^{2t}  \sum_{i \neq j}  \frac{\left| e_j^{\top}\mathbf{V}_1^{\top} \mathbf{E} e_i \right|^2}{\gamma_i^2 \mu^j_1\mu^i_1} 
    = 2 \epsilon^2 c^{2t}  \sum_{i \neq j}  \frac{\left| e_j^{\top}\mathbf{V}_1^{\top} \mathbf{E} e_i \right|^2}{\gamma_i^2 \mu^i_1} \nonumber \\
    &\leq 4 \epsilon^2 c^{2t} \frac{1}{\min_i (\gamma_i^2\mu_1^i)}, 
\end{align}
where $\sum_{i\neq j} \left|e_j^\top\mathbf{V}_1^\top\mathbf{E}e_i \right|^2 \leq 2 \| \mathbf{E} e_j \|^2 \leq 2$ due to $\| \mathbf{E}\|=1$.

\paragraph*{On \cref{rem:weakly}.}
We empirically observe that in the expansion $v_1^j=\sum_j \alpha_{ij}u_i$ used on the proof of \Cref{prop:WeaklyCommon}, the expansion coefficients $\alpha_{ij}=\langle u_i, v_1^j \rangle$ are concentrated at indices $i$ close to $j$ and nearly vanish otherwise. This empirical observation is demonstrated in Fig. \ref{fig:NonFixedDispersion}. By this observation, the upper bound in \Cref{eq:last_eq_prop_weakly} depends only on $\min_{i \sim j} (\gamma_i^2\mu_1^i)$, i.e., on eigenvalues $\mu_1^i$ close to the eigenvalue of interest $\mu_1^j$, which is typically large.

\paragraph*{Proof of \cref{thm:OnlyMutuals}.}

	We prove by contradiction. Let $\mu_1$ be the eigenvalue of $\mathbf{K}_1$ associated with the eigenvector $v$.
	Assume that $v$ is also an eigenvector of $\gamma(t)$ with a corresponding eigenvalue $\mu_t$. That is
	\begin{equation*}
	\mathbf{K}_1^{1/2}\left(\mathbf{K}_1^{-1/2}\mathbf{K}_2\mathbf{K}_1^{-1/2}\right)^t\mathbf{K}_1^{1/2} v = \mu_t v
	\end{equation*}
	Since $v$ is an eigenvector of $\mathbf{K}_1$ we have
	\begin{equation*}
	\mu_1^{1/2}\mathbf{K}_1^{1/2}\left(\mathbf{K}_1^{-1/2}\mathbf{K}_2\mathbf{K}_1^{-1/2}\right)^t v = \mu_t v,
	\end{equation*}
	which can be recast as
	\begin{equation*}
	\left(\mathbf{K}_1^{-1/2}\mathbf{K}_2\mathbf{K}_1^{-1/2}\right)^t v = \mu_1^{-1/2} \mathbf{K}_1^{-1/2} \mu_t v = \mu_1^{-1}\mu_t v,
	\end{equation*}
	This entails that $\mu_1^{-1}\mu_t$ is an eigenvalue of $\left(\mathbf{K}_1^{-1/2}\mathbf{K}_2\mathbf{K}_1^{-1/2}\right)^t$ with the corresponding eigenvector $v$. Consequently, $(\mu_1^{-1}\mu_t)^{1/t}$ is an eigenvalue of $\mathbf{K}_1^{-1/2}\mathbf{K}_2\mathbf{K}_1^{-1/2}$ , i.e.
	\begin{equation*}
	\mathbf{K}_1^{-1/2}\mathbf{K}_2\mathbf{K}_1^{-1/2} v = (\mu_1^{-1}\mu_t)^{1/t} v
	\end{equation*}
	By exploiting again that $v$ is an eigenvector of $\mathbf{K}_1$ we have
	\begin{equation*}
	\mu_1^{-1/2}\mathbf{K}_1^{-1/2}\mathbf{K}_2 v = (\mu_1^{-1}\mu_t)^{1/t} v
	\end{equation*}
	which gives
	\begin{equation*}
	\mathbf{K}_2 v = \mu_1^{1/2}\mathbf{K}_1^{1/2}(\mu_1^{-1}\mu_t)^{1/t} v = \mu_1(\mu_1^{-1}\mu_t)^{1/t} v
	\end{equation*}
	Therefore $v$ is an eigenvector of $\mathbf{K}_2$ with the corresponding eigenvalue $\mu_1^{1-1/t}\mu_t^{1/t}$.
	
\newpage

\section{Analyzing the \texorpdfstring{\acrshort*{EVFD}}{} in a discrete model}
\label{sec:CaseStudy}
\input{supplement_case_study}

%% file: supplement_case_study.tex
To provide additional insight, we examine the \acrshort*{EVFD} in a specific use case, where the setting includes graphs instead of manifolds, and the analysis is based on spectral graph theory.
In the considered use case, the spectral decompositions of the kernels are analytically known and have closed-form expressions, making the eigenvalues ``flow'' mathematically tractable. 
More concretely, it enables us to present the march along the geodesic path as a spectral filtering that preserves the common eigenvectors and attenuates the influence of the measurement-specific eigenvectors.
In addition, we show that the diffusion induced by the kernel matrices on the geodesic path is anisotropic, such that the induced diffusion distance exhibits some degree of invariance to the measurement-specific graph. We remark that all the proofs of the statements presented in this appendix appear in \Cref{subsec:proofs}.

We begin with the description of the setting.
Let $G_x$ be an undirected weighted graph with vertex set $V_X=\{u^x_1,u^x_2,\ldots,u^x_n\}$ and affinity matrix $\mathbf{A}_{x}$, and let $G_y$ and $G_z$ be two undirected weighted graphs with vertex sets $V_Y=\{u^y_1,u^y_2,\ldots,u^y_m\}$ and $V_Z=\{u^z_1,u^z_2,\ldots,u^z_m\}$ and with affinity matrices $\mathbf{A}_{y}$ and $\mathbf{A}_{z}$, respectively. Consider a permutation: $\pi: V_Y \rightarrow V_Z$, so that the vertices of $G_z$ can be viewed as a permutation of the vertices of $G_y$, namely, any $u^z \in V_Z$ can be written as $u^z = \pi (u^y)$ for $u^y \in V_Y$.
Finally, let $G_{xy}$ be the product graph of $G_x$ and $G_y$ with the following affinity matrix
\begin{equation*}
	\label{eq:ProductGraphAffinity}
	\mathbf{A}_{xy}=\mathbf{A}_x \times \mathbf{1}_m  + \mathbf{1}_n \times \mathbf{A}_y,  
\end{equation*}
and let $G_{xz}$ be the product graph of $G_x$ and $G_z$ with the following affinity matrix
\begin{equation*}
	\label{eq:ProductGraphAffinity2}
	\mathbf{A}_{xz}=\mathbf{A}_x \times \mathbf{1}_m  + \mathbf{1}_n \times \mathbf{A}_z.
\end{equation*}

The considered setting here is different than the continuous manifold setting considered in the paper.
Yet, we make the following analogy with the problem setting described in Section \ref{sec:ProblemFormulation}. The graphs $G_x$, $G_y$, and $G_{z}$ correspond to the hidden manifolds $\mathcal{M}_x$, $\mathcal{M}_y$, and $\mathcal{M}_z$, and the dependency between the hidden manifolds conveyed by the joint distribution of the samples $(x,y,z) \in \mathcal{M}_x \times \mathcal{M}_y \times \mathcal{M}_z$ is encoded here by the permutation $\pi$.
In this discrete setting, a measurement is represented by a triplet $(x,y,z)$ of nodes, where $x \in V_X$, $y \in V_Y$ and $z \in V_Z$, giving rise to a pair of nodes $(x,y)$ and $(x,z)$ in the product graphs $G_{xy}$ and $G_{xz}$, respectively.
We remark that for a tighter relationship between the discrete and continuous models, a discrete model consisting of multiple permutations should have been considered, allowing for dependencies between the nodes of $G_y$ and $G_z$ with the nodes of $G_x$. 
However, \color{black} for simplicity \color{black}, we assume there exists a single fixed permutation $\pi$, which does not depend on $V_X$. 
As we will discuss in the sequel, this relaxation makes the problem at hand more challenging. 

To make the analysis simple to present, we focus on a use case involving cycle (ring) graphs. Namely, we consider the following affinity matrix $\mathbf{A}_x\in\mathbb{R}^{n \times n}$ 
\begin{equation}
    \mathbf{A}_x = \left(\begin{array}{cccccc}
         1 & 1/2 & 0 & \cdots & 0 & 1/2 \\
         1/2 & 1 & 1/2 & 0 & \cdots & 0 \\
         0 & 1/2 & 1 & 1/2 & \cdots & 0 \\
         \vdots & \ddots & \ddots & \ddots & \ddots & \vdots \\
         0 & \cdots & 0 & 1/2 & 1 & 1/2 \\
         1/2 & 0 & \ldots & 0 & 1/2 & 1
    \end{array}\right)
\end{equation}
where $\mathbf{A}_y,\mathbf{A}_z \in \mathbb{R}^{m \times m}$ are defined in a similar manner.
We remark that our analysis can be extended with only minor modifications to d-regular graphs as well.

We conclude the description of the setting with a convenient notation.
For brevity, we will refer to the vertices of $G_x$, $G_y$ and $G_z$ by their respective indices, so that $V_X=\{1, \ldots, n\}$, $V_Y=\{1, \ldots, m\}$, $V_Z=\{1, \ldots, m\}$, and $\pi: \{1, \ldots, m\} \rightarrow \{1, \ldots, m\}$.
Consequently, each vertex in the product graphs can be indexed by a pair $(x,y) \in \{1,2,\ldots,n\}\times \{1,2,\ldots,m\}$ or by a single index using this column-stack bijective mapping function $r(x,y): \{1,2,\ldots,n\}\times \{1,2,\ldots,m\} \rightarrow \{1,2,\ldots,mn\}$, given by:
\begin{gather}
	r(x,y)=(x-1)m+y.
	\label{eq:mapping}
\end{gather}
The inverse mapping $x(r): \{1,2,\ldots,mn\} \rightarrow \{1,2,\ldots,n\}$ and $y(r): \{1,2,\ldots,mn\} \rightarrow \{1,2,\ldots,n\}$ are given by:
\begin{align}
	x(r)&=1+\left\lfloor \frac{r}{m} \right\rfloor\\
	y(r)&=1+\mod(r,m).
	\label{eq:inv_mapping}	
\end{align}
Throughout this appendix, we will use both indices interchangeably.

The first step in the computation of the \acrshort*{EVFD} is the construction of the kernels $\mathbf{K}_{xy}$ and $\mathbf{K}_{xz}$ by applying a two-step normalization to the affinity matrices $\mathbf{A}_{xy}$ and $\mathbf{A}_{xz}$, respectively, as described in \Cref{subsec:dm_doubly}.
In the considered cycle graphs, the sum of the rows and the columns of $\mathbf{A}_{x},\mathbf{A}_{y}$ and $\mathbf{A}_{z}$ is fixed and equals $2$.
By construction, we get that the sum of the rows and the columns of $\mathbf{A}_{xy}$ and $\mathbf{A}_{xz}$ is fixed and equals $4$. Accordingly, the kernels $\mathbf{K}_{xy}$ and $\mathbf{K}_{xz}$ are given by:
\begin{equation*}
\mathbf{K}_{xy}=\frac{1}{4} \mathbf{A}_{xy}; \quad
\mathbf{K}_{xz}=\frac{1}{4} \mathbf{A}_{xz}.
\end{equation*}

Applying EVD to $\mathbf{K}_{xy}$ yields
\begin{equation*}
	\mathbf{K}_{xy} = \mathbf{V}_{xy} \mathbf{S}_{xy} \mathbf{V}_{xy}^T,
\end{equation*}
where $\mathbf{S}_{xy}$ is a diagonal matrix whose entries are the eigenvalues of $\mathbf{K}_{xy}$, denoted by $\mu^{(r)}_{xy}$, and $\mathbf{V}_{xy}$ is a matrix whose columns consist of the eigenvectors of $\mathbf{K}_{xy}$, denoted by $v^{(r)}_{xy}$, for $r=1,\dots,nm$.
Note that according to our notation, the index $r$ can be replaced by a double index $(k,l)$ such that $k=x(r)$ and $l=y(r)$.

\begin{lemma}
	\label{lemma:ProductEvecsEvalsDiscrete}
	The eigenvectors and eigenvalues of $\mathbf{K}_{xy}$ are given by:
	\begin{align}
	v^{(k,l)}_{xy}(r)&=v^k_x\left(x\left(r\right)\right)v^l_y\left(y\left(r\right)\right) \nonumber \\
	\mu^{(k,l)}_{xy}&=\frac{1}{2}(\mu_x^k+\mu_y^l),
	\label{eq:ProductEvecsEvalsDiscrete}
	\end{align}
for $r=1,\dots,nm$, where $\{(\mu^{k}_x,v^{k}_x)\}_{k=1}^n$ is the set of eigenvalues and eigenvectors of $\mathbf{K}_{x}=\frac{1}{2}\mathbf{A}_{x}$, and $\{(\mu^{l}_y,v^{l}_y)\}_{l=1}^m$ is the set of eigenvalues and eigenvectors of $\mathbf{K}_{y}=\frac{1}{2}\mathbf{A}_{y}$.
\end{lemma}

Since $G_x$ and $G_y$ are cycle graphs, the eigenvalues of $\mathbf{K}_x$ and $\mathbf{K}_y$ are given by:
\begin{gather*}
	\mu_x^k= \frac{1}{2} \left(1+\cos\left(\frac{2\pi\left(k-1\right)}{n}\right) \right)\\
	\mu_y^u= \frac{1}{2} \left(1+\cos\left(\frac{2\pi\left(u-1\right)}{m}\right) \right)
\end{gather*}
and the eigenvectors are given by:
\begin{gather*}
v^k_x(x)= \alpha(k)\sqrt{\frac{1}{n}}\cos\left(\frac{\pi\left(k-1\right)}{n}(x-1) +\phi(k,n) )\right) \\ 
v^l_y(y)= \alpha(l) \sqrt{\frac{1}{m}}\cos\left(\frac{\pi\left(l-1\right)}{m}(y-1)+\phi(l,m) \right)
\end{gather*}
where $\alpha(k)$ is a scaling function given by:
\begin{gather}
	\alpha(k)=
	\begin{cases}
	1, & k=1 \\
	\sqrt{2}, & k>1 
	\end{cases}
	\label{eq:alpha}
\end{gather}
and $\phi(k,n)$ is a phase function given by:
\begin{gather}
	\phi(k,n)=\frac{\pi}{2}\left\lfloor\frac{k}{\frac{n+1}{2}+1} \right\rfloor
	\label{eq:phi}
\end{gather}
Plugging these explicit expressions into \eqref{eq:ProductEvecsEvalsDiscrete} gives explicit expressions to the eigenvectors and eigenvalues of $\mathbf{K}_{xy}$.

Suppose that the eigenvalues of $\mathbf{K}_{xy}$ are ordered according to $r(k,l)$ in an ascending order.
In this order, the eigenvectors with indices $r=m+1,2m+1,\ldots, (n-1)m+1$, such that $y(r)=1$, are the (common) eigenvectors of $\mathbf{K}_x$, whereas all other eigenvectors of $\mathbf{K}_{xy}$ consist of non-trivial (measurement-specific) eigenvectors of $\mathbf{K}_y$.

Similarly, the EVD of $\mathbf{K}_{xz}$ is given by:
\begin{equation*}
	\mathbf{K}_{xz} = \mathbf{V}_{xz} \mathbf{S}_{xz} \mathbf{V}_{xz}^T,
\end{equation*}
where $\mathbf{S}_{xz}=\mathbf{S}_{xy}$ and 
\begin{equation*}
	\mathbf{V}_{xz} = \Pi \mathbf{V}_{xy},
\end{equation*}
where $\Pi$ is defined by:
\begin{equation*}
	\Pi = \mathbf{1}_n \otimes \mathbf{P_{\pi}} .
\end{equation*}

The computation of the spectral flow diagram continues with the construction of a point on the geodesic path on the SPD manifold connecting $\mathbf{K}_{xy}$ and $\mathbf{K}_{xz}$, denoted by $\gamma(t)$ for $t \in [0,1]$.
By definition, $\gamma(t)$ can be recast as:
\begin{equation*}
	\gamma (t) = \mathbf{K}_{xy}^{1/2} \mathbf{M}^{t} \mathbf{K}_{xy}^{1/2}
\end{equation*}
where $\mathbf{M} \triangleq \mathbf{K}_{xy}^{-1/2} \mathbf{K}_{xz} \mathbf{K}_{xy}^{-1/2}$.
By the EVD of $\mathbf{K}_{xy}$ and $\mathbf{K}_{xz}$, $\mathbf{M}$ can be written as: 
\begin{equation*}
	\mathbf{M} = \mathbf{V}_{xy} {\mathbf{S}_{xy}}^{-\frac{1}{2}} {\mathbf{V}_{xy}}^T \mathbf{V}_{xz}\mathbf{S}_{xy}{\mathbf{V}_{xz}}^T \mathbf{V}_{xy} {\mathbf{S}_{xy}}^{-\frac{1}{2}} {\mathbf{V}_{xy}}^T
\end{equation*}
Recalling that $\mathbf{V}_{xz} = \Pi \mathbf{V}_{xy}$, we can further simplify the expression and recast it only using $\Pi$ and the \acrshort{EVD} of $\mathbf{K}_{xy}$:
\begin{equation*}
	\mathbf{M} = \mathbf{V}_{xy} {\mathbf{S}_{xy}}^{-\frac{1}{2}} {\mathbf{V}_{xy}}^T \Pi \mathbf{V}_{xy}\mathbf{S}_{xy}{\mathbf{V}_{xy}}^T \Pi^T \mathbf{V}_{xy} {\mathbf{S}_{xy}}^{-\frac{1}{2}} {\mathbf{V}_{xy}}^T
\end{equation*}
For brevity, we discard the subscripts $xy$, so that we get: 
\begin{equation}\label{eq:M}
	\mathbf{M} = \mathbf{V}\underbrace{{\mathbf{S}}^{-\frac{1}{2}} \underbrace{\overbrace{{\mathbf{V}}^T \Pi \mathbf{V}}^{\triangleq {\mathbf{B}}}\mathbf{S} \overbrace{ {\mathbf{V}}^T \Pi^T \mathbf{V}}^{\triangleq {\mathbf{B}^T}}}_{\triangleq {\mathbf{C}}} {\mathbf{S}}^{-\frac{1}{2}}}_{\triangleq \mathbf{C_\mathbf{S}}} {\mathbf{V}}^T
\end{equation}
The notation in \eqref{eq:M} prescribes a $5$-step calculation of $\gamma(t)$:
\begin{enumerate}
		\item Calculate $\mathbf{B}\triangleq \mathbf{V}^T \Pi \mathbf{V}$.
		\item Calculate $\mathbf{C}\triangleq \mathbf{B} \mathbf{S} \mathbf{B}^T$.
		\item Perform rows and columns multiplication according to the diagonal matrix $\mathbf{S}^{-\frac{1}{2}}$ in order to compute: $\mathbf{C_{\mathbf{S}}}\triangleq \mathbf{S}^{-\frac{1}{2}} \mathbf{C}  \mathbf{S}^{-\frac{1}{2}}$.
		\item Rotate $\mathbf{C_{\mathbf{S}}}$ according to $\mathbf{V}$ and obtain $\mathbf{M}= \mathbf{V} \mathbf{C_{\mathbf{S}}} \mathbf{V}^T$.
		\item Raise $\mathbf{M}$ to the power of $t$ and multiply it by $\mathbf{K}^{-\frac{1}{2}}$ from the left and right, and obtain $\gamma(t)=\mathbf{K}^{-\frac{1}{2}} \mathbf{M}^{t} \mathbf{K}^{-\frac{1}{2}}$.
\end{enumerate}

In the sequel, we follow this procedure step-by-step. Throughout the derivation, in addition to general results, we present an illustration based on a specific simulation with the following parameters:
\begin{equation*}
	m=31 ; n=11
\end{equation*}
and a permutation $\pi$ that was drawn uniformly at random. 
The randomly drawn realization of the permutation matrix $\mathbf{P_{\pi}}$ is depicted in Fig. \ref{fig:ProofIllustration_Pi}.	
\begin{figure}[t]\centering
		\includegraphics[scale=0.15]{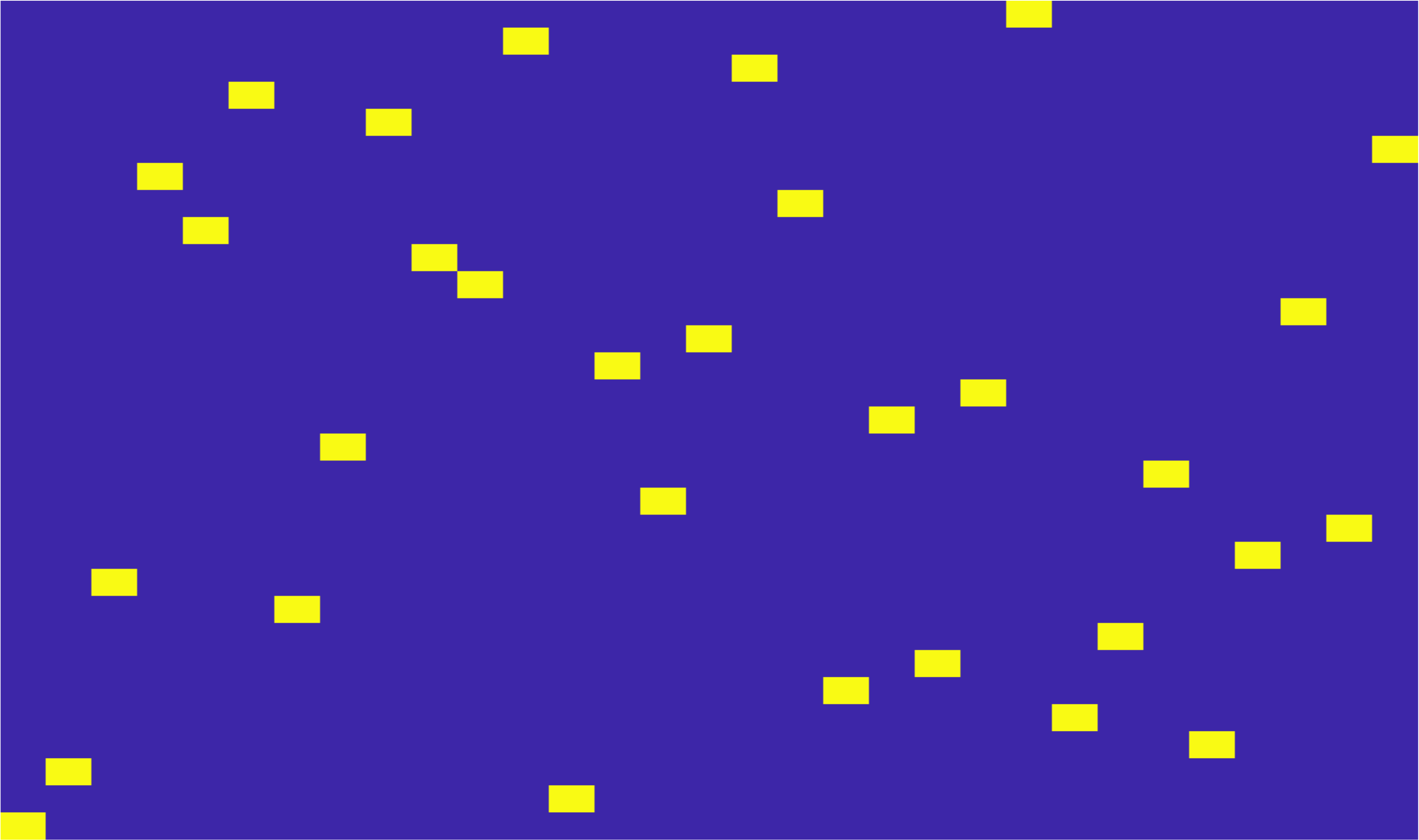} 
		\caption{The randomly drawn permutation matrix $\mathbf{P_{\pi}}$.}
		\label{fig:ProofIllustration_Pi}
\end{figure}
	
The elements of the matrix $\mathbf{B} \triangleq \mathbf{V}^T \Pi \mathbf{V}$ are given by: 
\begin{equation*}
\mathbf{B}_{i,j} = \langle v^{(x(i),y(i))}_{xy} ,  v^{(x(j),z(j))}_{xz} \rangle.
\end{equation*}
\begin{proposition}
	\label{prop:ExpressionForB}
	The matrix $\mathbf{B} \in \mathbb{R}^{nm \times nm}$ is block diagonal consisting of $n$ identical $m \times m$ matrices $\mathbf{b}$ so that $\mathbf{b}_{1,1} = 1$, $\mathbf{b_{1,i}} = 0$ and $\mathbf{b}_{i,1}=0$ for $i>1$. In addition, for any $1<l,l'\leq m$, the probability to randomly pick a permutation for which $|\mathbf{b}_{l,l'}| \geq \alpha$ is upper-bounded by:
	\begin{equation*}
	P(|\mathbf{b}_{l,l'}| \geq \alpha ) \leq \frac{1}{\alpha}\sqrt{\frac{2}{m}}.
	\end{equation*}
\end{proposition}
The obtained matrix $\mathbf{B}$ in the simulation is presented in Fig. \ref{fig:ProofIllustration_B} (a) and one of its blocks $\mathbf{b}$ is presented in Fig. \ref{fig:ProofIllustration_B} (b). We observe that $\mathbf{B}$ is indeed a block-diagonal matrix consisting of $n$ identical blocks, where each block admits the structure specified in \Cref{prop:ExpressionForB}. Namely, the first element of each block is $1$, the remaining elements in the first row and column are $0$, and the rest of the entries are unstructured and small. 

\begin{figure}[H]\centering
	\begin{tabular}{cc}
			\hspace{-0in} \includegraphics[scale=0.15]{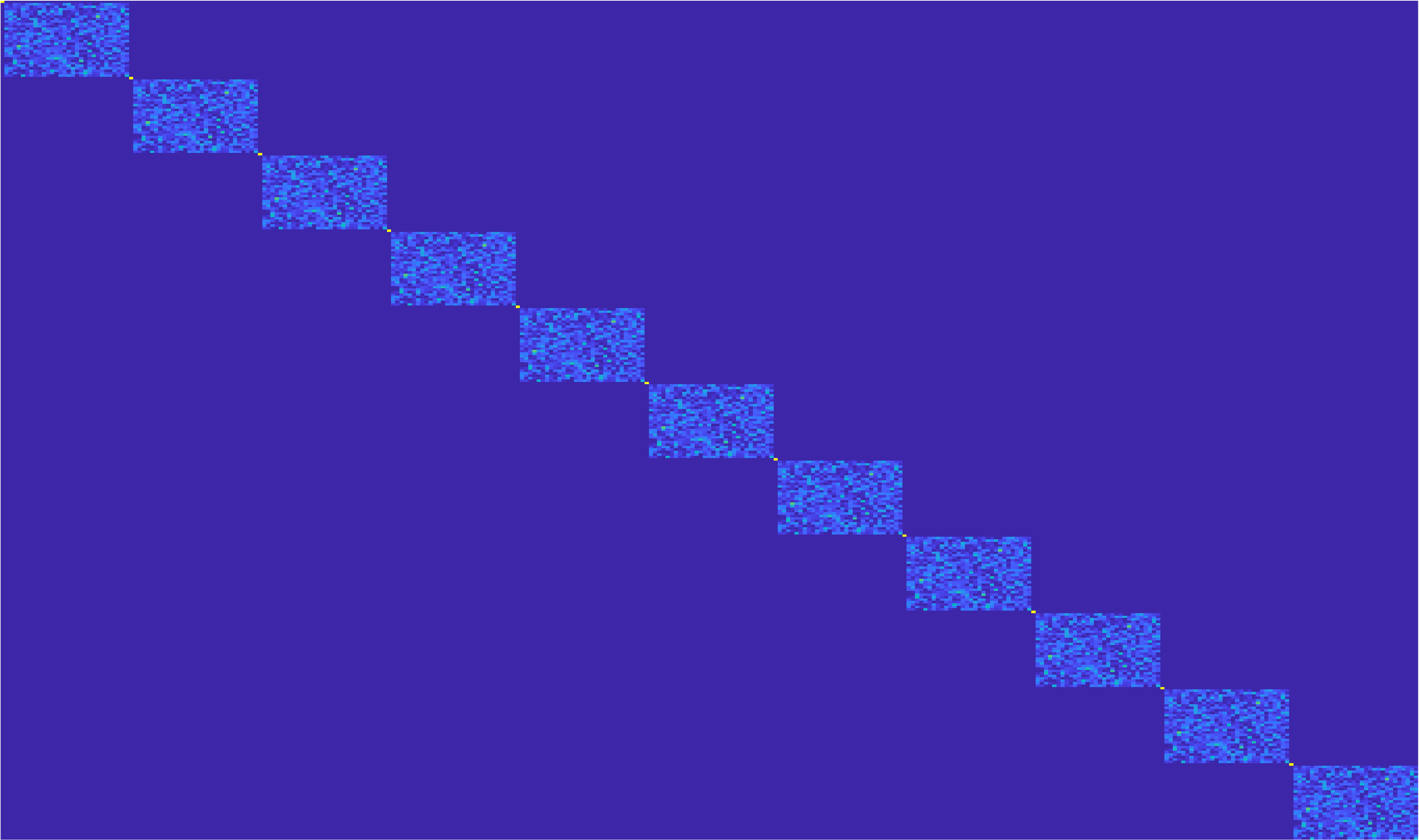} &
			\includegraphics[scale=0.15]{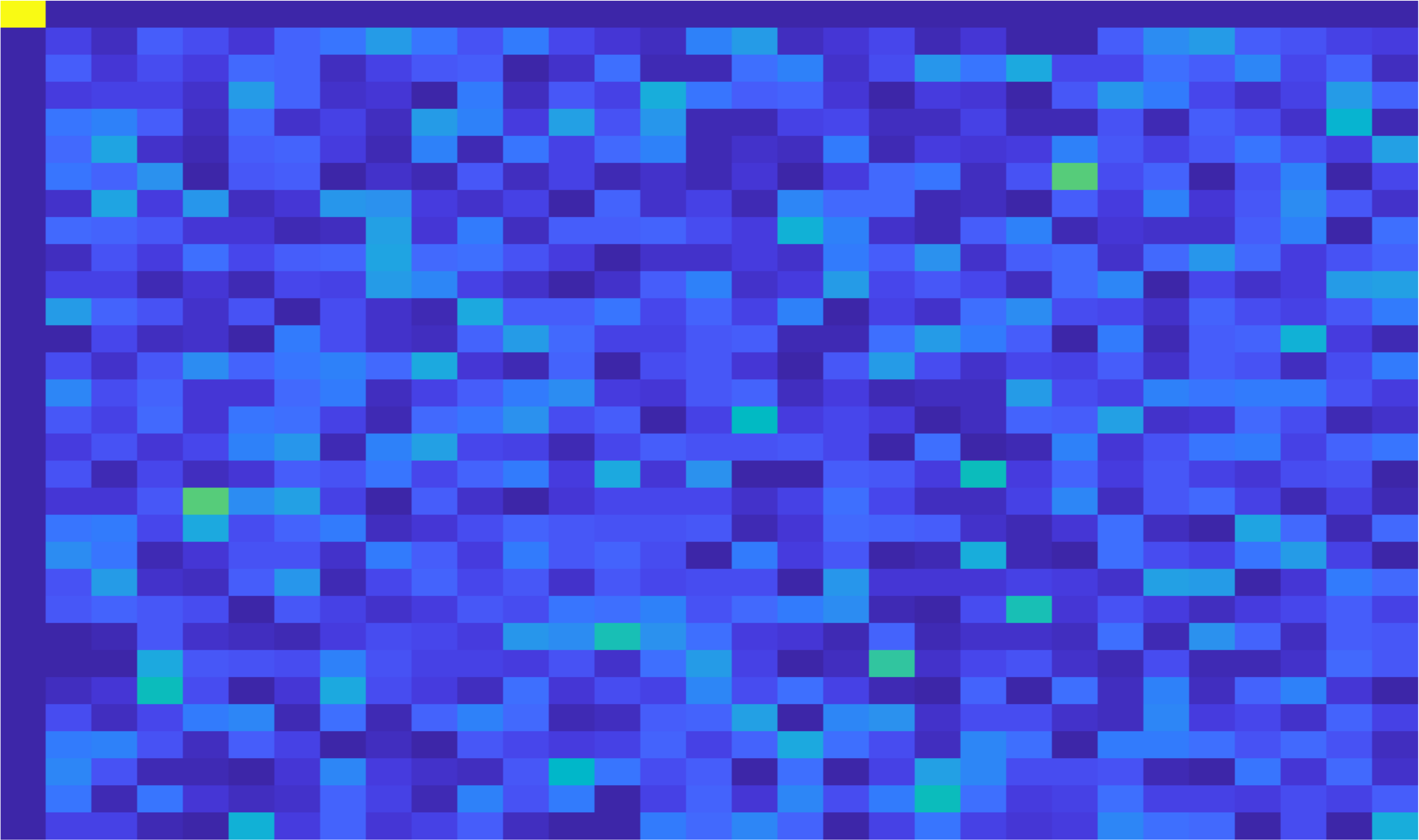} \\ 
			(a) & (b)
	\end{tabular}
	\caption{(a) The matrix $\mathbf{B}$. (b) One of its blocks.}
	\label{fig:ProofIllustration_B}
\end{figure}

Next, we compute $\mathbf{C} \triangleq \mathbf{B}^T \mathbf{S} \mathbf{B}$. Recall that $\mathbf{S}$ is a diagonal matrix consisting of the eigenvalues of $\mathbf{K}_{xy}$ and its elements are given according to \eqref{eq:ProductEvecsEvalsDiscrete} by:
\begin{align}
\label{eq:ProductEvecsEvalsDiscrete2}
	\mathbf{S}(r,r) &= \mu^{(x(r),y(r)}_{xy}=\frac{1}{2}(\mu_x^{x(r)}+\mu_y^{y(r)}) \\ 
	&= \frac{1}{2} + \frac{1}{4} \left(\cos\left(\frac{2\pi\left(x(r)-1\right)}{n} \right) + \cos\left(\frac{2\pi\left(y(r)-1\right)}{m} \right) \right).
\end{align}
Fig. \ref{fig:ProofIllustration_S_Diag} presents the diagonal of $\mathbf{S}$ in the simulation.
Since the multiplicity of the eigenvalues of $\mathbf{K}_x$ and $\mathbf{K}_y$ is $2$, i.e.:
\begin{gather*}
\mu_x^{k}=\mu_x^{n-(k-2)} \\ 
\mu_y^{l}=\mu_y^{m-(l-2)} 
\end{gather*}
\color{black}the  multiplicity of the eigenvalues of $\mathbf{K}_{xy}$ is $4$ \color{black}:
\begin{gather*}
\mu^{(x(r),y(r))}_{xy} = \mu^{(n-(x(r)-2),y(r))}_{xy} = \mu^{(x(r),m-(y(r)-2))}_{xy}= \mu^{(n-(x(r)-2),m-(y(r)-2))}_{xy}
\end{gather*}
\color{black}
Therefore, in Fig. \ref{fig:ProofIllustration_S_Diag}, as well as in all subsequent similar figures, we depict only the unique $1/4$ of the eigenvalues of $\mathbf{K}_{xy}$.
More specifically, we show the diagonal entries of $\mathbf{S}$ with indices $r(x,y)$ sorted in ascending order, where $x=1,\ldots,\frac{n+1}{2}$ and $y=1,\ldots,\frac{m+1}{2}$.
\color{black}

\begin{figure}[H]\centering
	\includegraphics[scale=0.2]{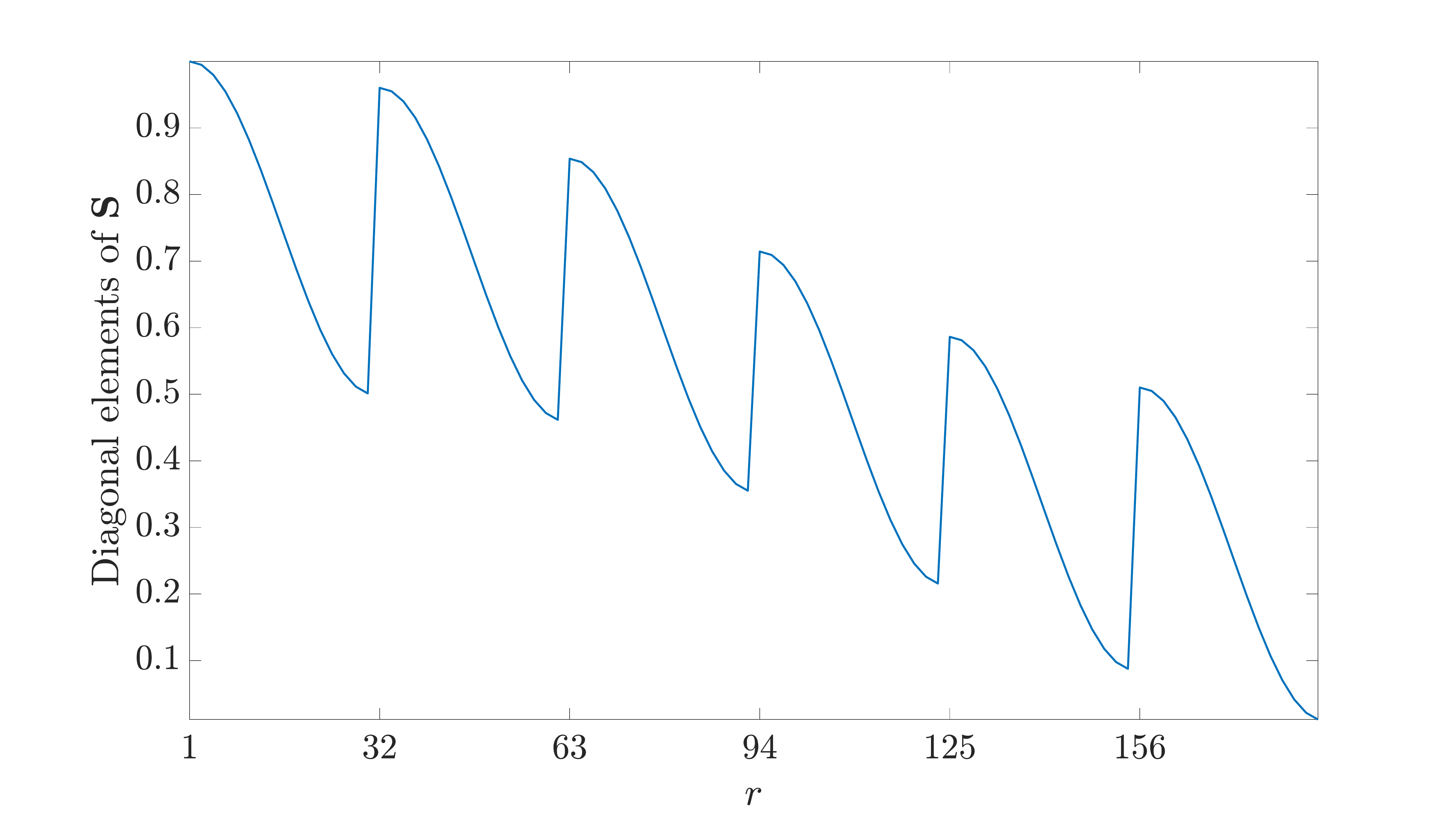}
	\caption{The diagonal entries of $\mathbf{S}$, i.e., the eigenvalues of $\mathbf{K}_{xy}$ sorted according to the mapping $r(x,y)$ in an ascending order.}
		\label{fig:ProofIllustration_S_Diag}
\end{figure}

Fig. \ref{fig:ProofIllustration_C} (a) depicts the obtained matrix $\mathbf{C}$ in the simulation, and Fig. \ref{fig:ProofIllustration_C} (b) depicts its diagonal elements. 
We observe that the off-diagonal elements of $\mathbf{C}$ are significantly smaller than the diagonal elements. Following this empirical results, we approximate $\mathbf{C}$ as a diagonal matrix.

\begin{figure}[H]\centering
	\begin{tabular}{cc}
			\hspace{-0in} \includegraphics[scale=0.18]{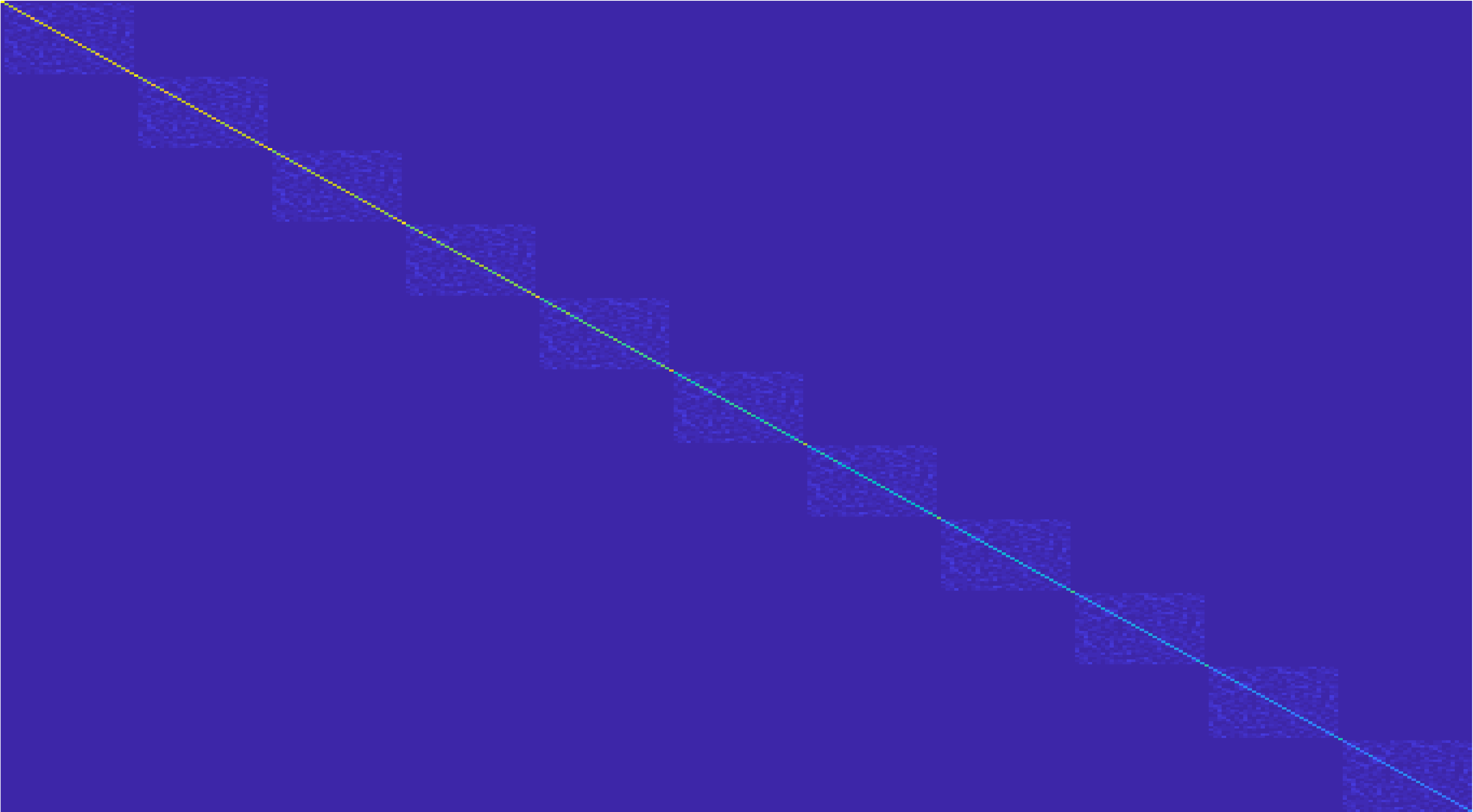} &
			\includegraphics[scale=0.14]{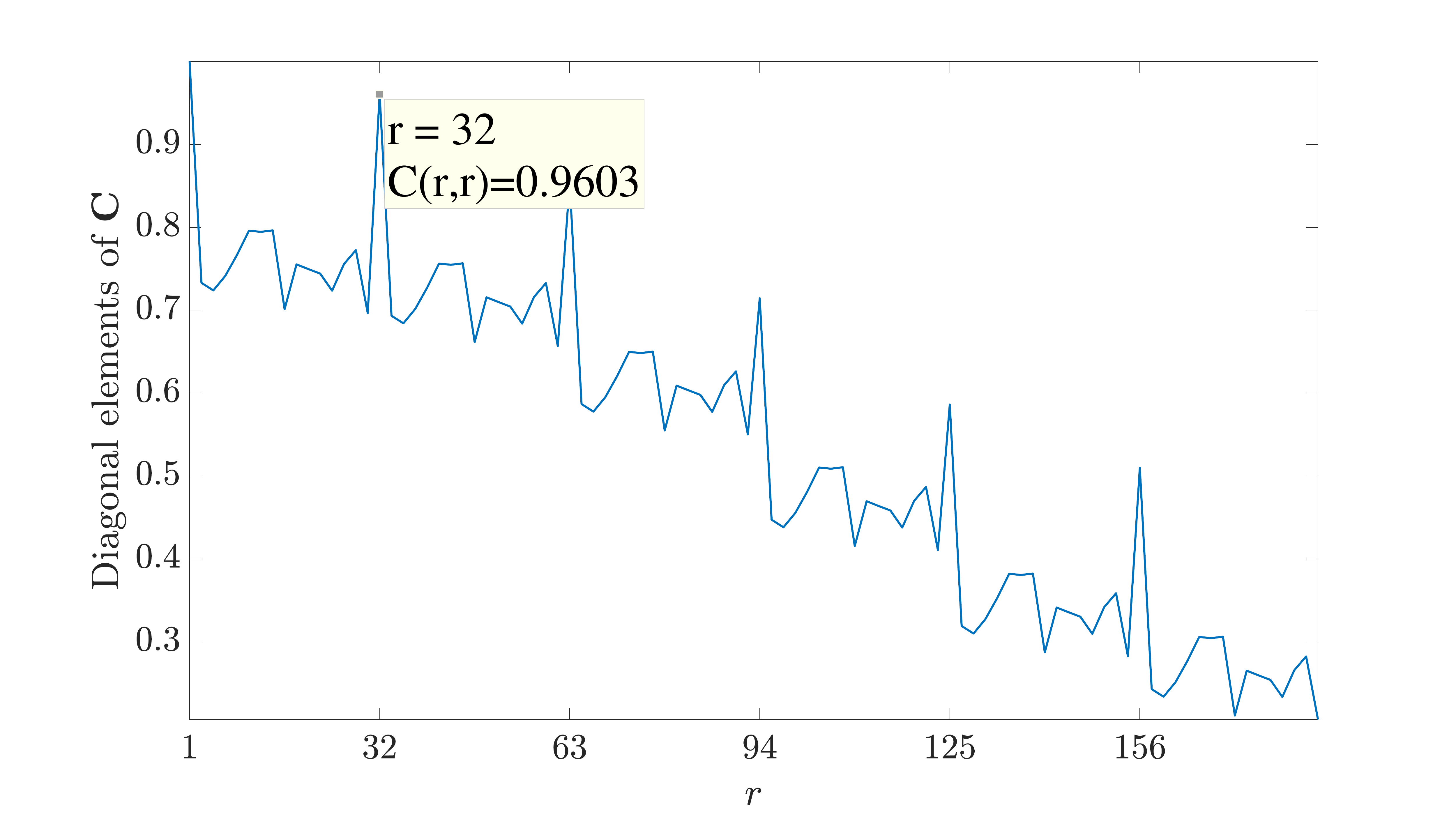} \\ 
			(a) & (b)
	\end{tabular}
	\caption{(a) The matrix $\mathbf{C}$. (b) The diagonal elements of $\mathbf{C}$.}
	\label{fig:ProofIllustration_C}
\end{figure}
Focusing on the diagonal of $\mathbf{C}$ reveals that the elements corresponding to the common \color{black} eigenvectors \color{black} at indices $m+1,2m+1,\ldots, (n-1)m+1$ are identical to the values in $\mathbf{S}$. For example, the empirical value of the diagonal element at index $m+1=32$ is $\mathbf{C}(32,32)=0.9603$. This value admits the analytical expression for $\mathbf{S}(32,32)$, which is given by substituting $r=32$ in \eqref{eq:ProductEvecsEvalsDiscrete2}:
\begin{align*}
	\mathbf{S}(32,32) &=   \frac{1}{2} + \frac{1}{4} \left(\cos\left(\frac{2\pi\left(x(32)-1\right)}{11} \right) + \cos\left(\frac{2\pi\left(y(32)-1\right)}{31} \right) \right) \\ 
	&= \frac{1}{2} + \frac{1}{4} \left(\cos\left(\frac{2\pi\left(2-1\right)}{11} \right) + 1\right) = 0.9603.
\end{align*}

\begin{figure}[t]\centering
	\includegraphics[scale=0.2]{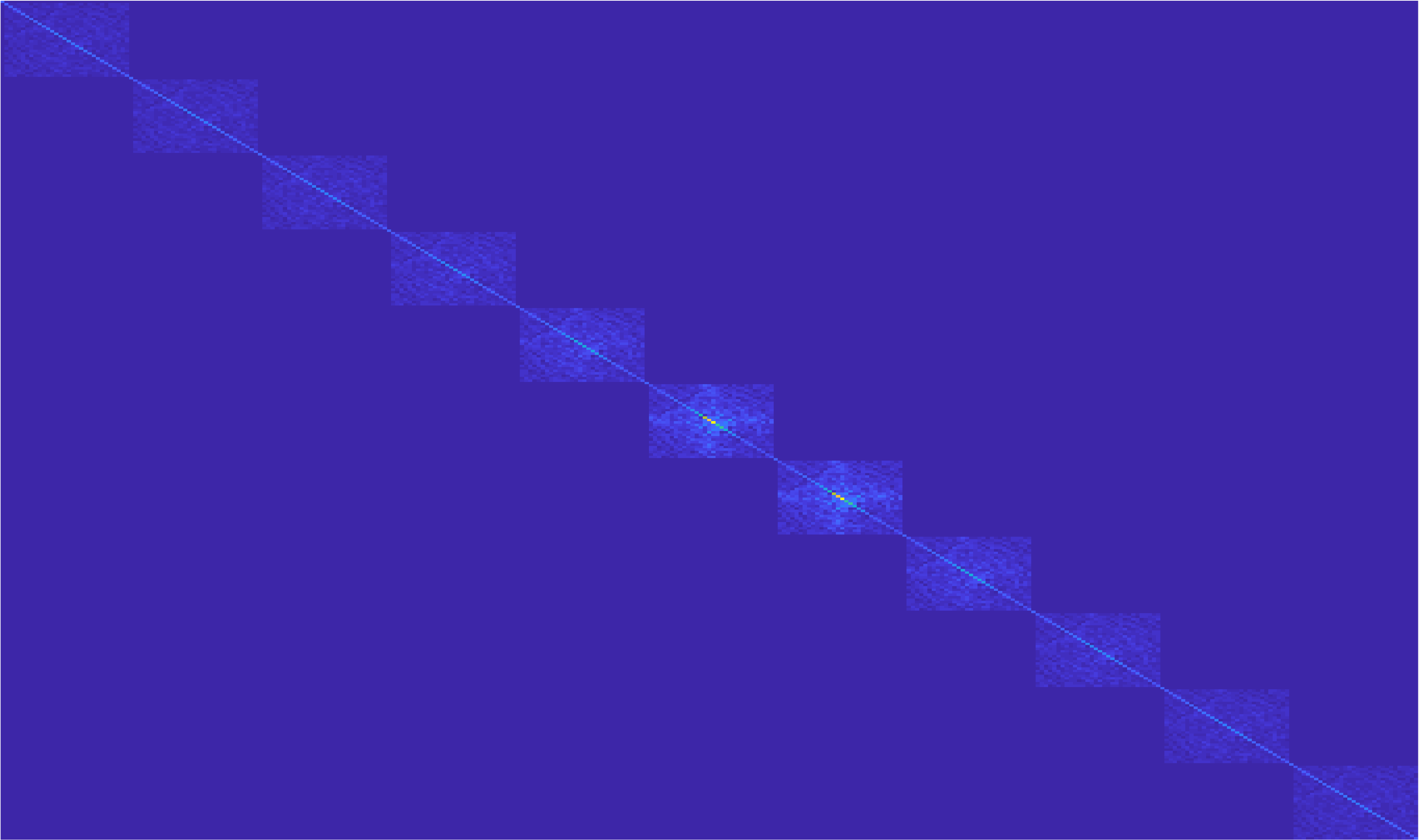}
	\caption{The matrix $\mathbf{C_S}$.}
	\label{fig:ProofIllustration_TildeC}
\end{figure}

The final steps are to compute the matrix $\mathbf{C_S}\triangleq \mathbf{S}^{-\frac{1}{2}} \mathbf{C}  \mathbf{S}^{-\frac{1}{2}}$.
The resulting matrix is presented in Fig. \ref{fig:ProofIllustration_TildeC}.
Similarly to $\mathbf{C}$, the matrix $\mathbf{C_{\mathbf{S}}}$ has small off-diagonal elements and is close to a diagonal matrix as well.

The next step is to compute $\mathbf{M}=\mathbf{V} \mathbf{C_{\mathbf{S}}} \mathbf{V}^T$, raise it to the power of $t$, and then compute the geodesic path $\gamma(t) = \mathbf{K}^{-\frac{1}{2}}\mathbf{M}^t \mathbf{K}^{-\frac{1}{2}}$. 
\begin{proposition}
	\label{prop:GammaAlternativeForm}
	The geodesic path $\gamma(t)$ can be recast as follows:
	\begin{equation*}
	\gamma(t) = \mathbf{V} \mathbf{F_t} \mathbf{V}^T
	\end{equation*}
	where $\mathbf{F_t}$ is given by:
	\begin{equation*}
	\mathbf{F_t} \triangleq \mathbf{S}^{\frac{1}{2}}\mathbf{C}^t_{\mathbf{S}} \mathbf{S}^{\frac{1}{2}}
	\end{equation*}
\end{proposition}

In light of the empirical result, assuming that $\mathbf{C_{\mathbf{S}}}$ is approximately diagonal, by Taylor expansion, we have that $\mathbf{C}^t_{\mathbf{S}}$ is approximately diagonal for small $t$ values. Therefore, for small $t$, $\gamma(t)$ can be approximated by: 
\begin{equation}\label{eq:gamma_approx}
	\gamma({t}) \simeq \mathbf{V}\mathbf{S} \mathbf{C}^t_{\mathbf{S}}\mathbf{V}^T
\end{equation}
This representation implies that marching along the geodesic path (starting from $t=0$) is analogous to spectral graph filtering with a filter $f_{t}(\cdot)$ that is applied to the graph spectrum: $f_{t}(\mathbf{S})=\mathbf{S}\mathbf{C}^t_{\mathbf{S}} $. 
For illustration, in Fig. \ref{fig:ProofIllustration_FilterOutputs}, we pick $5$ points along the geodesic path at $t \in \{ 0, 0.125, 0.25, 0.375, 0.5\}$ and depict $f_{t}(\mathbf{S})$.
\begin{figure}[t]\centering
		\hspace{10in}
		\includegraphics[scale=0.35]{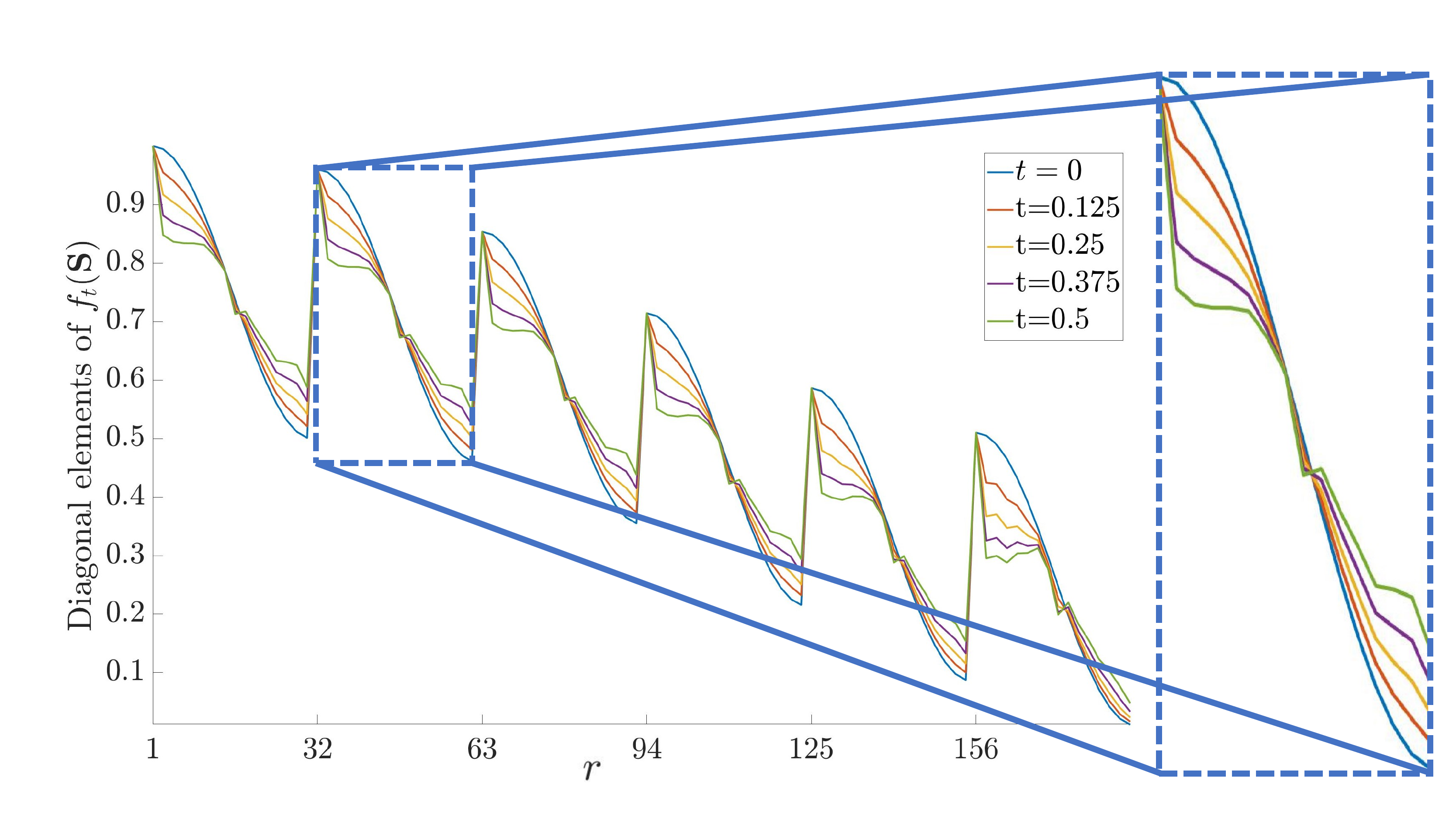}
		\caption{The diagonal elements of $f_t(\mathbf{S})$ for several values of $t$.}
		\label{fig:ProofIllustration_FilterOutputs}
\end{figure}

\color{black} 
We observe that at ${t}=0$, the obtained spectrum coincides with the spectrum of $\mathbf{K}_{xy}$. In the inset, we also observe that eigenvalues at indices $r=m+1,2m+1,\ldots$, corresponding to common eigenvectors, remain unchanged when $t$ increases (in the inset we see it at index $33=1+m$ since $m=32$). Conversely, all other eigenvalues, corresponding to measurement-specific eigenvectors, vary, where the dominant ones (with indices $r$ such that $y(r)$ is small) are attenuated and the others are enhanced.
\color{black}

In Fig. \ref{fig:ProofIllustration_Top10Comparison}, we plot the eigenvalues of $\gamma(t)$ without using the approximation in \eqref{eq:gamma_approx}. In Fig. \ref{fig:ProofIllustration_Top10Comparison} (a) we plot the eigenvalues for $t=0$ and in Fig. \ref{fig:ProofIllustration_Top10Comparison} (b) for $t=0.5$. In each plot, we mark the five top eigenvalues by red asterisks.

\begin{figure}[t]\centering
	\begin{tabular}{cc}
			\hspace{-0in} \includegraphics[scale=0.12]{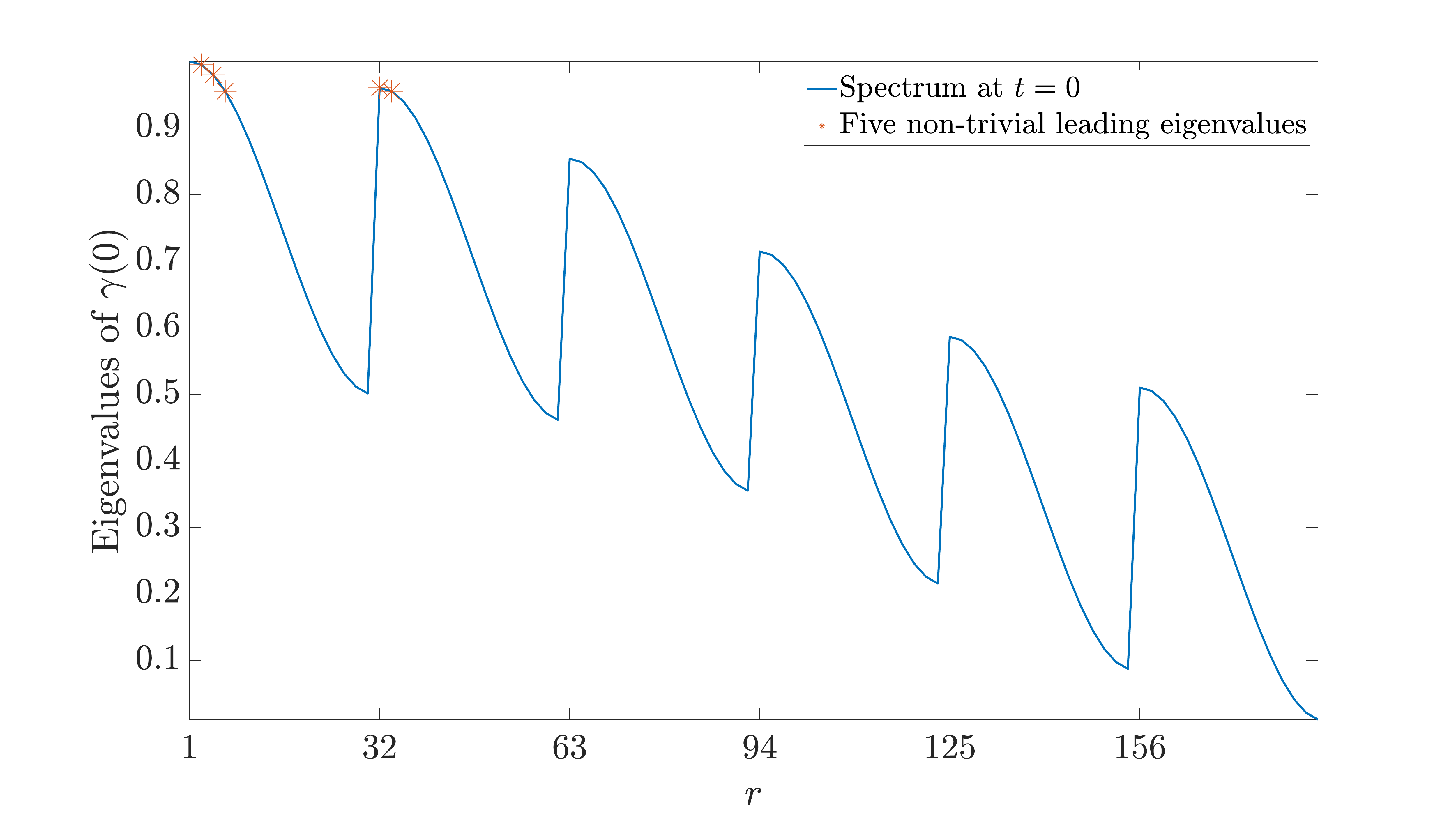} &
			\includegraphics[scale=0.12]{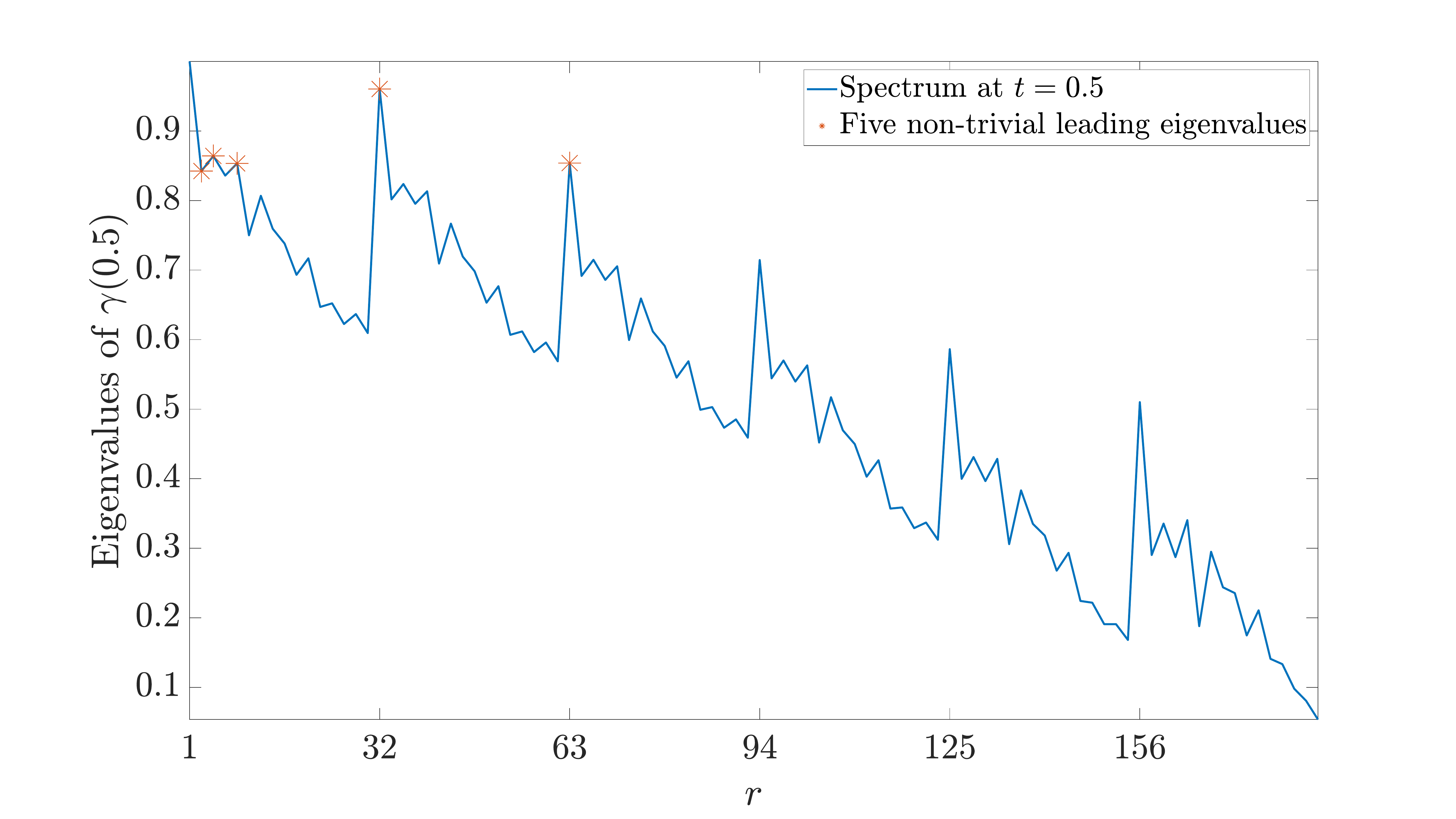}
			\\
			(a) & (b)
	\end{tabular}
	\caption{(a) The eigenvalues of $\gamma(0)=\mathbf{K}_{xy}$. (b) The eigenvalues of $\gamma(0.5)$. The $5$ top eigenvalues are marked by red asterisks.}
	\label{fig:ProofIllustration_Top10Comparison}
\end{figure}

We observe that for $t=0$, besides the trivial eigenvalue, only one eigenvalue among the top five eigenvalues corresponds to a common eigenvector. In contrast, for $t=0.5$, two eigenvalues among the top five correspond to common eigenvectors. 
This demonstrates how $\gamma(t)$ at $t=0.5$ attenuates the eigenvalues associated with non-common eigenvectors, and as a consequence, the common eigenvectors become more dominant.

Next, we examine the diffusion distance\footnotemark induced by $\gamma(t)$ in light of the above spectral analysis. 
\footnotetext{We consider the variant of the diffusion distance defined in \eqref{eq:unnorm_d_dist} with respect to $\gamma(t)$.}
The diffusion distance induced by the kernel $\gamma(t)$ between two graph nodes: $p_1=r(x_1,y_1)$ and $p_2=r(x_2,y_2)$ is given by 
\begin{equation}
d_{\gamma^s(t)}(p_1,p_2)=\| \boldsymbol{\delta}_{p_1}^T\gamma(t)^s - \boldsymbol{\delta}_{p_2}^T\gamma(t)^s\|_2,
\label{eq:DiffusionDistance}
\end{equation}
where $s>0$ represent the number of diffusion steps, and $\boldsymbol{\delta}_{p_1}$ and $\boldsymbol{\delta}_{p_2}$ are vectors, whose elements at $p_1$ and $p_2$ equal $1$, respectively, and all other elements equal $0$.

We pick three points:
\begin{align*}
	p_1 &= r(6,16) \\
	p_2 &= r(6,19) \\
	p_3 &= r(9,16).
\end{align*}
We note that $p_1$ and $p_2$ share the same vertices with respect to $G_x$, whereas $p_1$ and $p_3$ share the same vertices with respect to $G_y$. We compute $\gamma^s(0.25) \boldsymbol{\delta}_{p_i}$ for $i=1,2,3$ and for $s=50$, and display the resulting diffusion patterns in Fig. \ref{fig:FunctionsPropagation}.

\begin{figure}[t]\centering
	\begin{tabular}{ccc}
		\hspace{-0in} \includegraphics[scale=0.08]{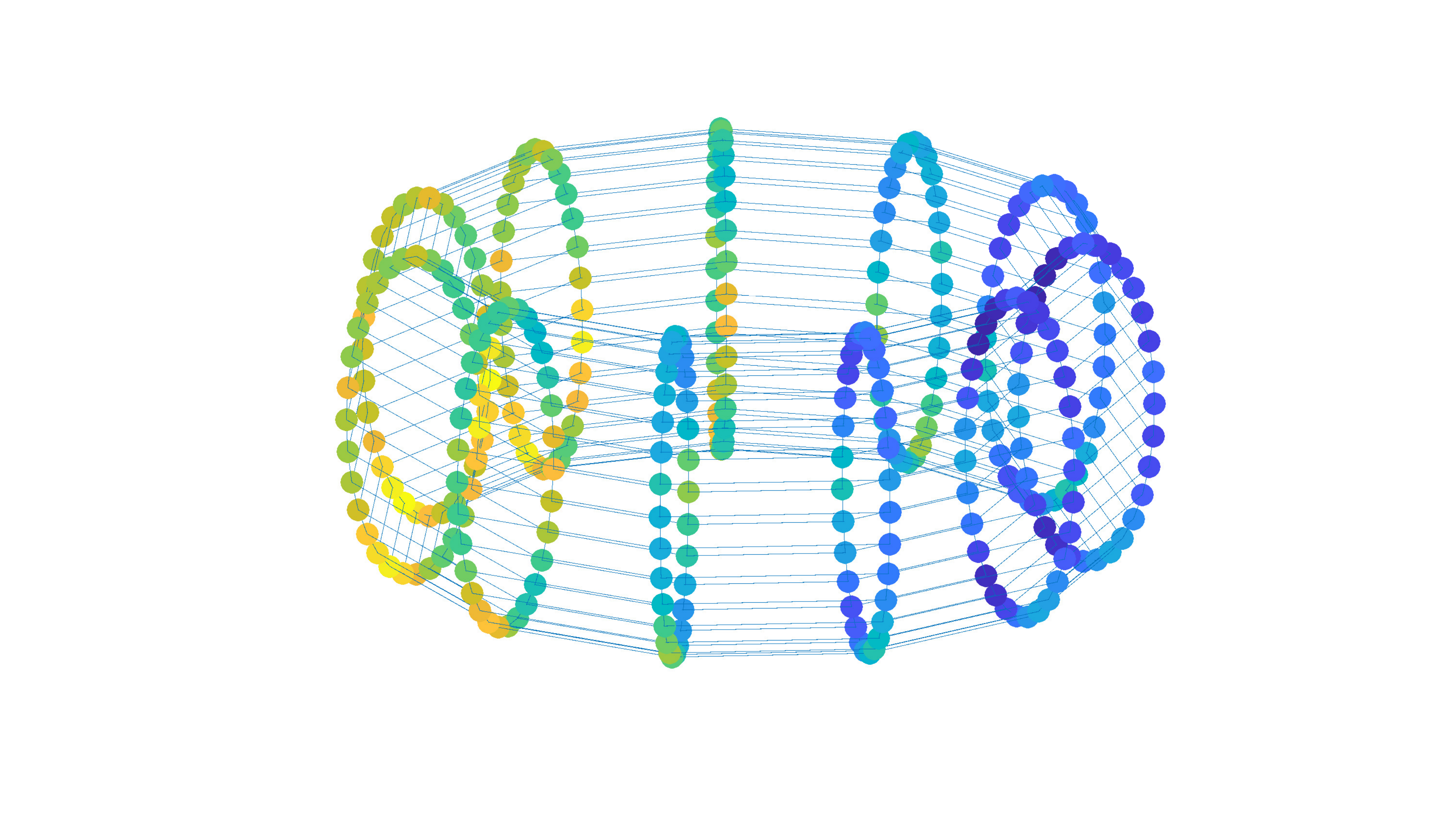} &
		\includegraphics[scale=0.08]{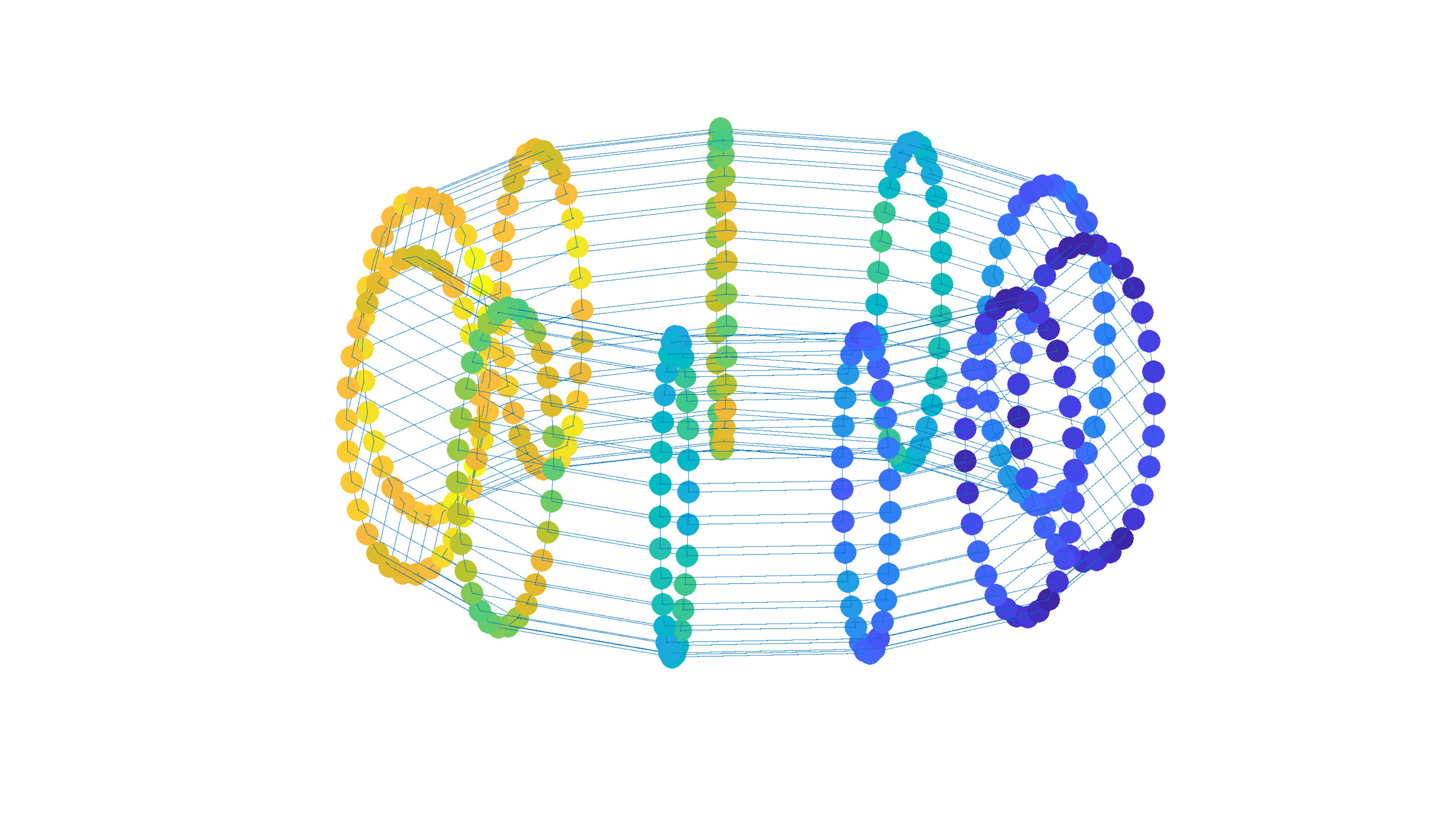} &
		\includegraphics[scale=0.08]{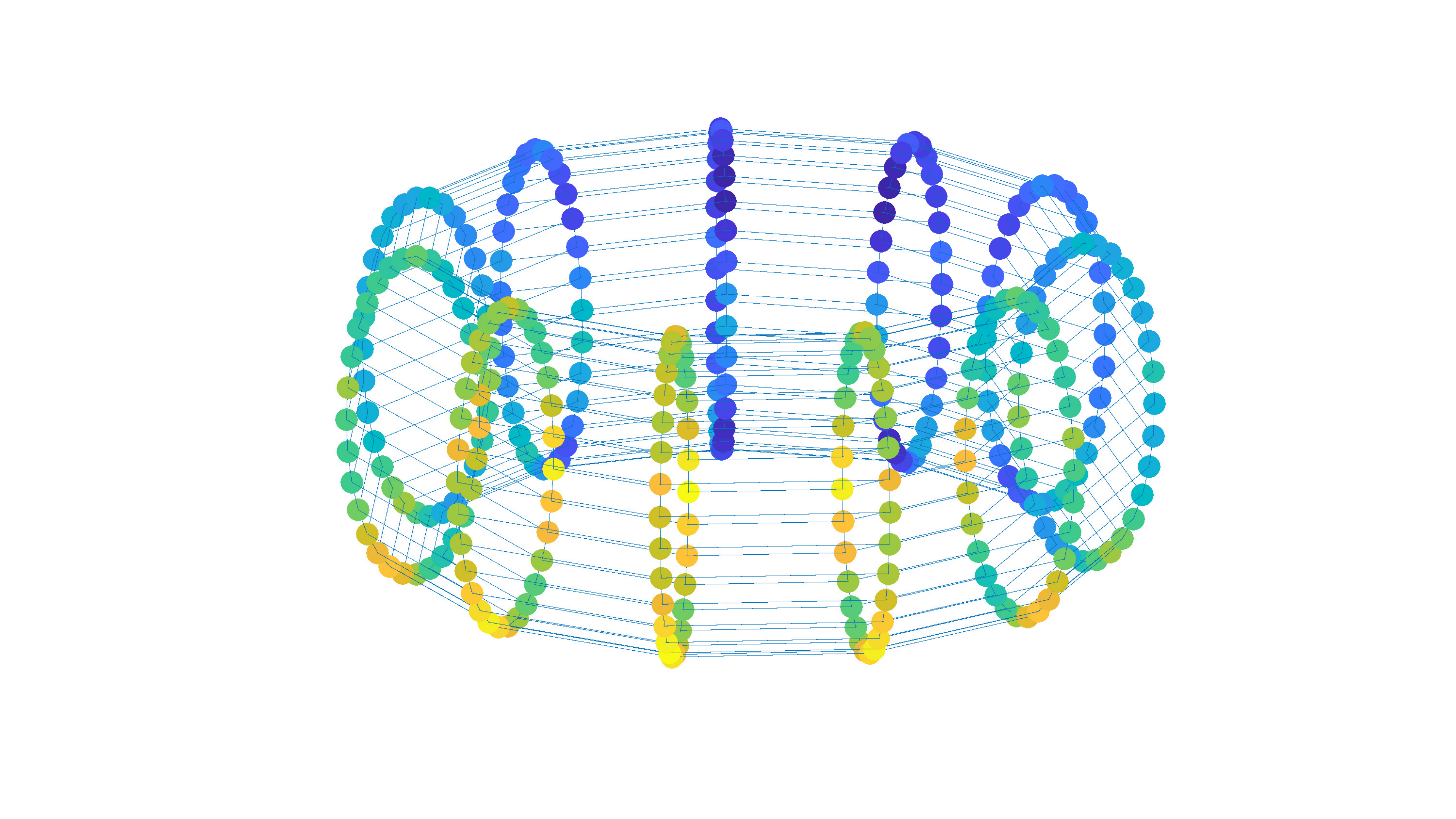}\\
		(a) & (b) & (c)
	\end{tabular}
	\caption{The diffusion patterns displayed on the nodes of $G_{xy}$, namely $V_X \times V_Y$. Since $G_x$ and $G_y$ are cycle graphs, we plot $G_{xy}$ as a 2-torus in $\mathbb{R}^3$. The $mn$ nodes are colored according to the $mn$-dimensional vectors $\gamma^s(0.25) \boldsymbol{\delta}_{p_1}$, $\gamma^s(0.25) \boldsymbol{\delta}_{p_2}$, and $\gamma^s(0.25) \boldsymbol{\delta}_{p_3}$ in (a), (b), and (c), respectively.}
	\label{fig:FunctionsPropagation}
\end{figure}

We observe that the diffusion patterns initialized at $p_1$ and $p_2$, whose common coordinate  $x(p_1)=x(p_2)=6$ is the same, are similar, leading to small diffusion distance $d_{\gamma^s(0.25)}(p_1,p_2)$. Conversely, the diffusion patterns initialized at $p_1$ and $p_3$, whose common \color{black} coordinate \color{black} $6=x(p_1) \ne x(p_3) = 9$ is not equal, are different, leading to large diffusion distance $d_{\gamma^s(0.25)}(p_1,p_3)$. This demonstrates the fact that the diffusion distance is more sensitive to distances with respect to $G_x$ than to distances with respect to $G_y$, thereby facilitating the discovery of the common eigenvectors.

\begin{figure}[t]\centering
	\begin{tabular}{c}
		\hspace{-0.1in} \includegraphics[scale=0.25]{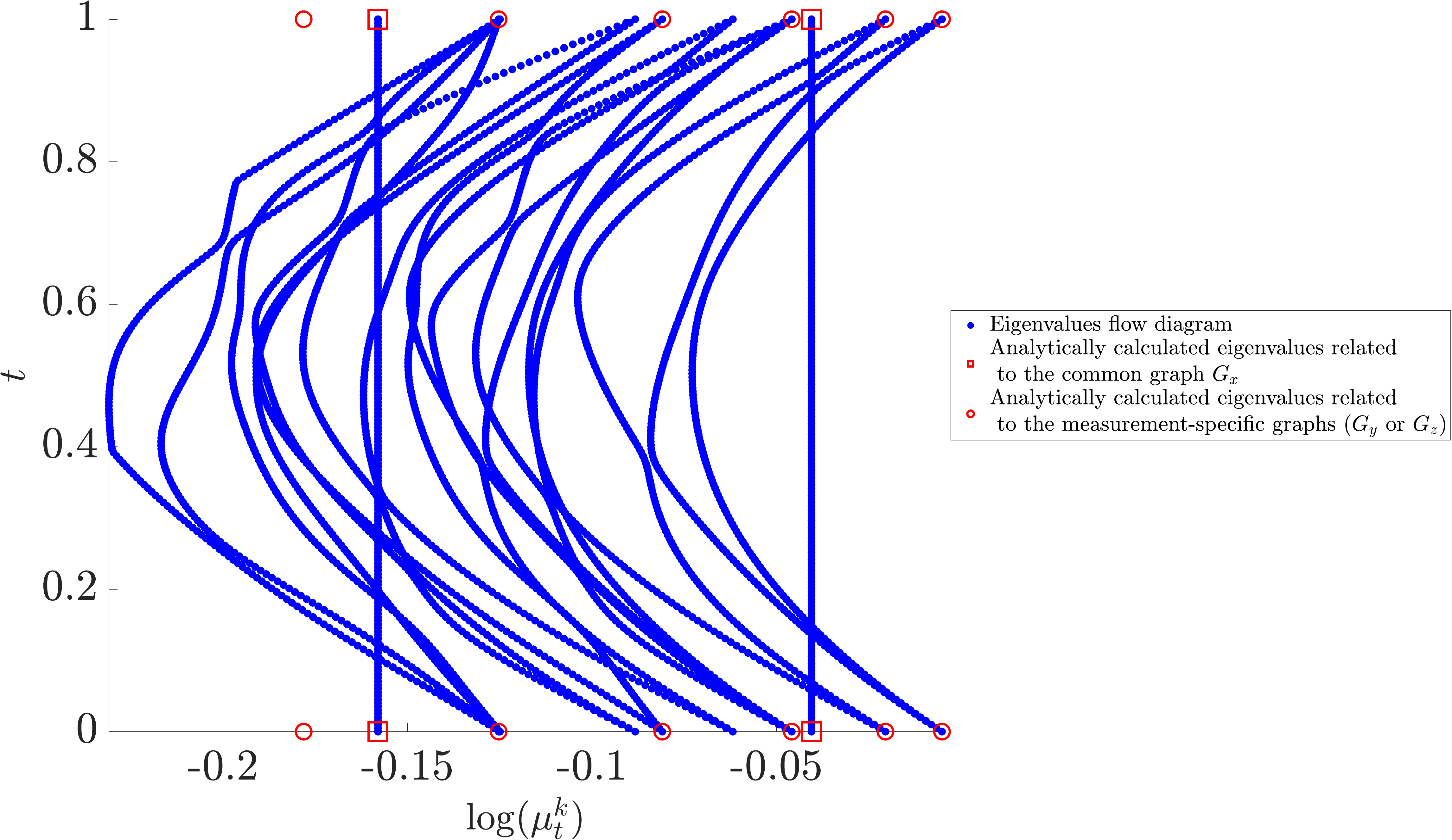}
	\end{tabular}
	\caption{The \acrshort*{EVFD} of the considered use case. The analytic eigenvalues corresponding to the common graph $G_x$ are marked by squares and the analytic eigenvalues corresponding to the measurement-specific graphs ($G_y$ or $G_z$) are marked by circles (the mixed eigenvalues are not marked).}
	\label{fig:DiscreteGraph_EVDiagram}
\end{figure} 
In Fig. \ref{fig:DiscreteGraph_EVDiagram}, we present the \acrshort*{EVFD} obtained by \Cref{alg:EvfdCalculation} with $N_t=200$ and $K=20$. 
We observe that the eigenvalues, which are associated with the common eigenvectors, are indeed constant with respect to $t$, while the rest are attenuated. Therefore, we use the EVD of $\gamma(t)=\mathbf{V}_{\gamma(t)} \mathbf{S}_{\gamma(t)}\mathbf{V}_{\gamma(t)}^T$, and recast  the diffusion distance as:
\begin{equation*}
    d^2_{\gamma^s(t)}(p_1,p_2)=  \sum_{r=1}^{mn} \mathbf{S}^{2s}_{\gamma(t)}(r,r) \left(  \mathbf{V}_{\gamma^(t)}(p_1,r)  - \mathbf{V}_{\gamma(t)}(p_2,r)  \right) ^2.
\end{equation*} 

Let $\mathcal{S}_c$ denote the set of indexes $r$ corresponding to the common eigenvectors of $\gamma(t)$, namely, $r=1,m+1,2m+1,\ldots$.
Accordingly, $d^2_{\gamma^s(t)}(p_1,p_2)$ can be decomposed as follows:
\begin{equation*}
	d^2_{\gamma^s(t)}(p_1,p_2) = d^{\text{c}}_{\text{s}}(p_1,p_2,t) + d^{\text{nc}}_{\text{s}}(p_1,p_2,t),
\end{equation*}
where
\begin{equation*}
	 {d^{\text{c}}_{\text{s}}}(p_1,p_2,t) \triangleq  \sum_{r \in \mathcal{S}_c } \mathbf{S}^{2s}_{\gamma(t)}(r,r) \left(  \mathbf{V}_{\gamma(t)}(p_1,r)  - \mathbf{V}_{\gamma(t)}(p_2,r)  \right) ^2,
\end{equation*}
and
\begin{equation*}
	{d^{\text{nc}}_{\text{s}}}(p_1,p_2,t) \triangleq  \sum_{r \notin \mathcal{S}_c } \mathbf{S}^{2s}_{\gamma(t)}(r,r) \left(  \mathbf{V}_{\gamma(t)}(p_1,r)  - \mathbf{V}_{\gamma(t)}(p_2,r)  \right) ^2.
\end{equation*}

This decomposition enables us to examine the contribution of the common and non-common eigenvectors to the diffusion distance, separately.
In Fig. \ref{fig:DiffusionDistance}, we plot the resulting $d^{\text{c}}_{\text{s}}(p_1,p_2,t)$ and $d^{\text{nc}}_{\text{s}}(p_1,p_2,t)$. We observe that $d^{\text{c}}_{\text{s}}(p_1,p_2,t)$ is constant with respect to $t$, which coincides with the fact that the eigenvalues corresponding to the common eigenvectors are fixed along the geodesic path. Conversely, we observe that $d^{\text{nc}}_{\text{s}}(p_1,p_2,t)$ is depends on $t$ and attains high values near the boundaries $t=0$ and $t=1$ and low values in the middle of the geodesic path, implying on some degree of invariance to the measurement-specific eigenvectors.
\begin{figure}[t]\centering
	\includegraphics[scale=0.2]{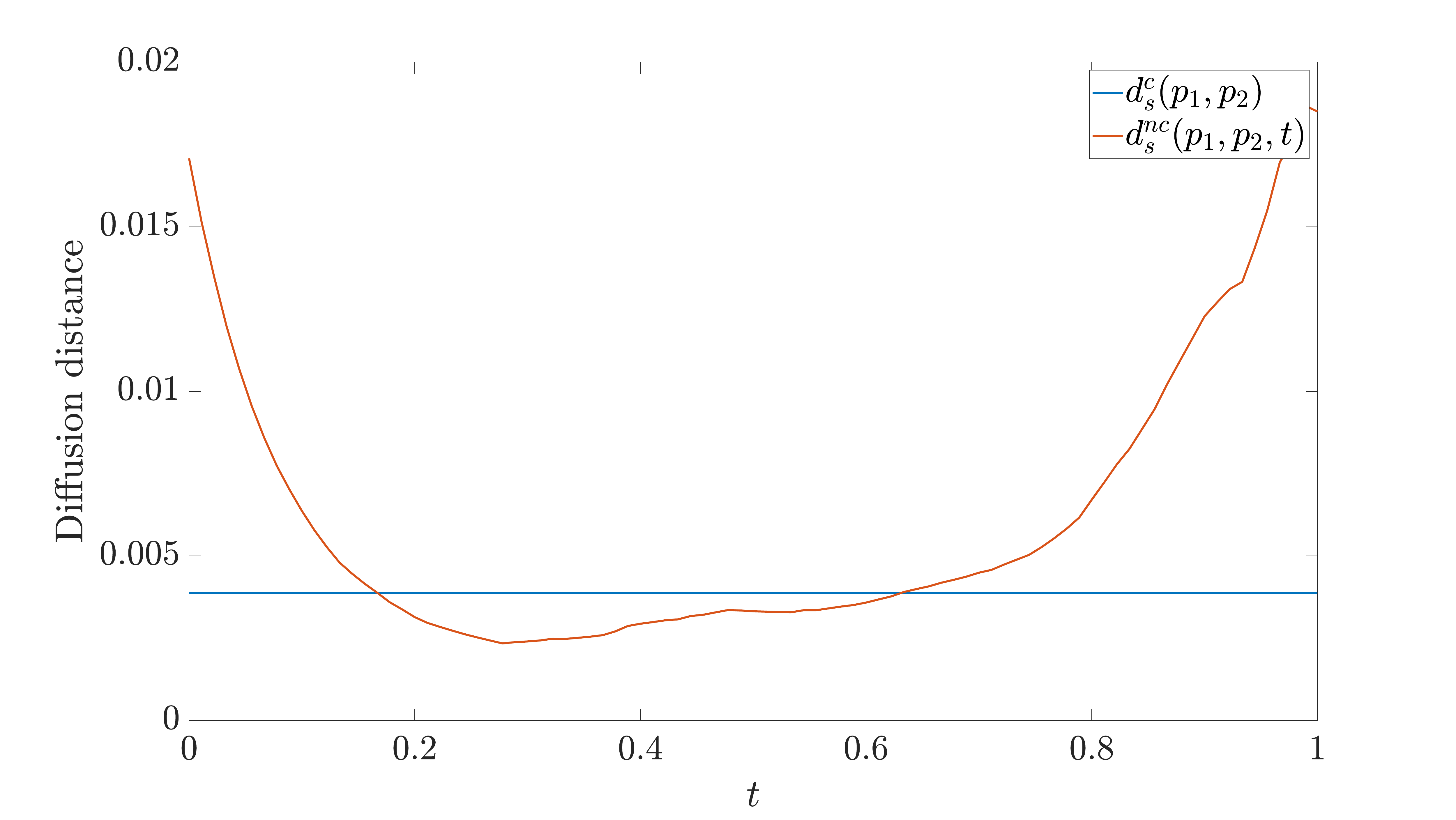}	
	\caption{${d^{\text{c}}_{s}}(p_1,p_2,t)$ and ${d^{\text{nc}}_{s}}(p_1,p_2,t)$ along the geodesic path.}
	\label{fig:DiffusionDistance}
\end{figure}

The consequence of this empirical result is that the diffusion distance computed based on $\gamma(t=0)$ and $\gamma(t=1)$, namely, $\mathbf{K}_{xy}$ and $\mathbf{K}_{xz}$, is dominated by the non-common graphs.
In contrast, the diffusion distance based on $\gamma(t)$ from the middle of the geodesic path is dominated by the common graph.

The illustrations in the figures throughout the derivation so far are based on a specific permutation. We remark that similar trends are obtained for different permutation realizations.
In the remainder of this appendix, we derive an expression for the average spectrum that is obtained by taking into account all possible permutations. 

Each realization of $\pi$ induces a kernel $\mathbf{K}_{xz}$ and a corresponding geodesic path $\gamma_\pi(t)$. Each $\gamma_\pi (t)$ starts at a fixed matrix $\mathbf{K}_{xy}$ but ends at different matrix $\mathbf{K}_{xz}$ depending on the permutation, where the notation of the dependency of $\mathbf{K}_{xz}$ on $\pi$ is omitted for simplicity.
Suppose now that $\pi$ is drawn uniformly at random, and consider the ``mean path'' $\bar{\gamma}(t)$ that starts at the fixed $\mathbf{K}_{xy}$ and ends at $\overline{\mathbf{K}}_{xz}$,  defined by
\begin{equation*}
\bar{\gamma}(t) = \frac{1}{m!} \sum_{\pi} \gamma_{\pi}(t),
\end{equation*}
where 
\begin{equation*}
	\overline{\mathbf{K}}_{xz} = \frac{1}{m!}\sum _\pi \mathbf{K}_{xz}.
\end{equation*}

\begin{proposition}
	\label{prop:ExpectationC}
The mean matrix $\overline{\mathbf{C}}$, defined by
    \begin{equation*}
    	\overline{\mathbf{C}} = \frac{1}{m!}\sum _\pi \mathbf{C}_\pi,
    \end{equation*}
where $\mathbf{C}_\pi$ is defined as in steps 1 and 2 with the permutation $\pi$, is a diagonal matrix, whose diagonal elements are given by:
	\begin{equation*}
	\bar{c}(r) = \frac{1}{2} + \frac{1}{4} \cos\left(\frac{2\pi\left(x(r)-1\right)}{n}\right) +\frac{1}{4} \delta\left(y(r),1\right) 
	\end{equation*}
	where $1\leq r \leq mn$.
\end{proposition}
We note that $\bar{c}(r)$ is the $r$th eigenvalue of the diffusion kernel $\widetilde{\mathbf{K}}_x$ corresponding to the product graph $G_x \times Q_m$, where $Q_m$ is the complete graph with $m$ nodes.

In Fig. \ref{fig:_DiagWithExpectation} we overlay $\bar{c}(r)$ on Fig. \ref{fig:ProofIllustration_C}(b).
As can be seen, the preservation of the common eigenvectors and the attenuation of the non-common eigenvectors are evident when considering the mean as well as when considering a particular permutation.

\begin{figure}[t]\centering
	\includegraphics[scale=0.25]{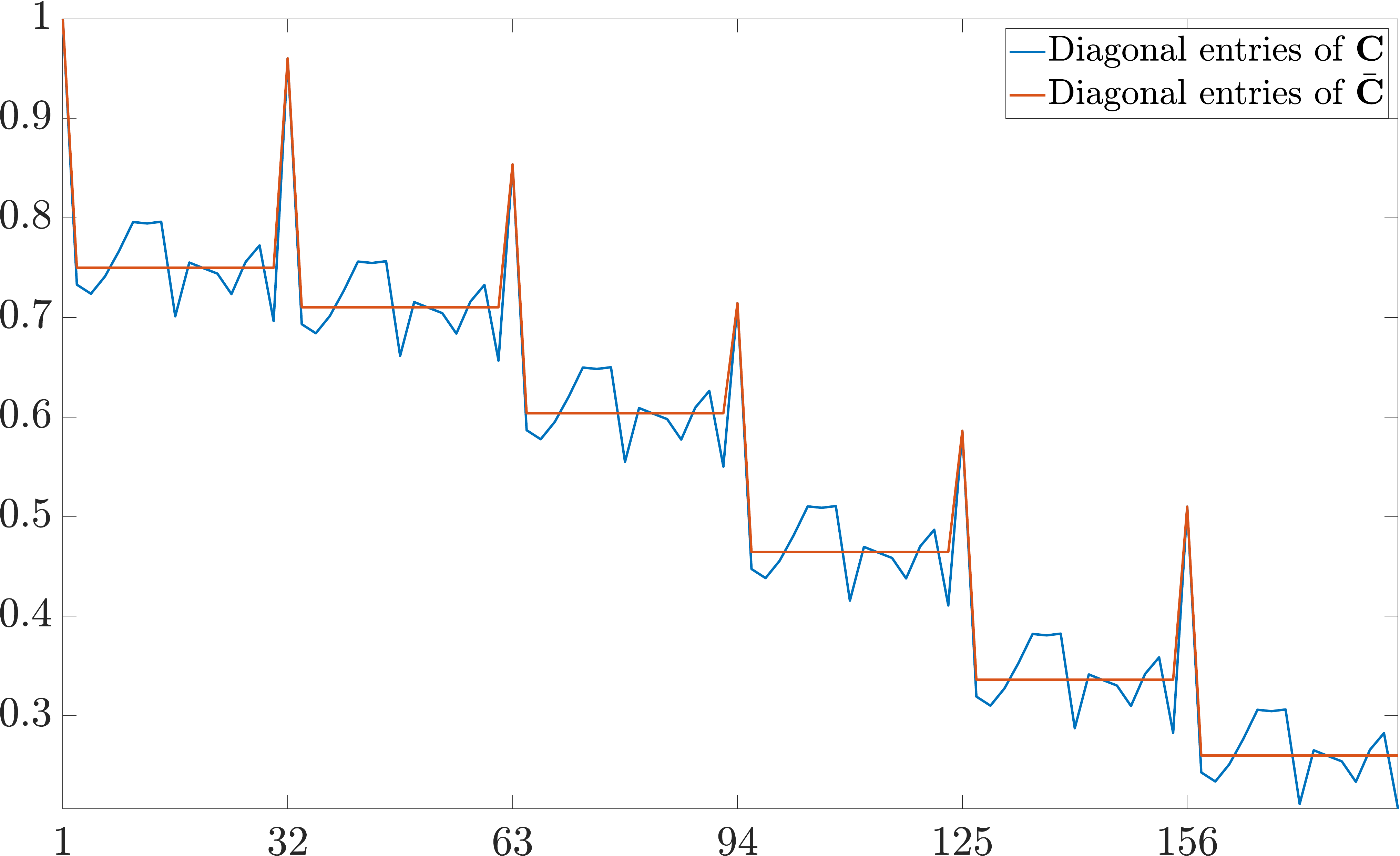}	
	\caption{The diagonal elements of $\mathbf{C}_\pi$ compared with the analytic expression for the diagonal elements of $\overline{\mathbf{C}}$.}
	\label{fig:_DiagWithExpectation}
\end{figure}

As we show in the next proposition, the fact that $\mathbf{C}$ is approximately diagonal can be made more precise.
\begin{proposition}
	\label{prop:ExpressionForVarC}
		Let $\pi$ be a permutation drawn uniformly at random. The off-diagonal entries of $\mathbf{C}$ satisfy:
		\begin{equation*}
		P( |\mathbf{C}_{i,j}| \geq \alpha ) \leq \frac{1}{\alpha}\sqrt{\frac{1}{32(m-1)}}
		\end{equation*}
		for every $1\leq i,j \leq nm, i\neq j$ and $\alpha>0$. Accordingly, the off-diagonal entries of $\mathbf{C}$ approach zero as $m \rightarrow \infty$.
\end{proposition}

Given the closed-form expression for the matrix $\overline{\mathbf{C}}$, deriving a closed-form expression for $\overline{\mathbf{C}}_\mathbf{s}$ is straight-forward since:  $\overline{\mathbf{C}}_{\mathbf{s}}\triangleq \mathbf{S}^{-\frac{1}{2}} \overline{\mathbf{C}}  \mathbf{S}^{-\frac{1}{2}}$. Specifically, the diagonal entries of $\overline{\mathbf{C}}_{\mathbf{s}}$ are given by:
\begin{gather}
\label{eq:ExpectationTildeC}
\bar{c}_{\mathbf{s}}(r)= \frac{\bar{c}(r)}{\mu_{xy}^{(x(r),y(r))} } = 
	\begin{cases}
	1, & y(r)=1 \\
	1- \frac{ \cos\left(\frac{2\pi\left(y(r)-1\right)}{m}\right) }{2 + \cos\left(\frac{2\pi\left(x(r)-1\right)}{n}\right) + \cos\left(\frac{2\pi\left(y(r)-1\right)}{m}\right) }, & y(r)>1 
	\end{cases}
\end{gather}
\color{black}
Recall that the common eigenvectors with indices $r=1,m+1,2m+1,\ldots$ satisfy $y(r)=1$. By \eqref{eq:ExpectationTildeC}, these eigenvectors are preserved, whereas the non-common eigenvectors with indices $r$ such that $y(r) \neq 1$ vary. 
Particularly, non-common eigenvectors with indices $r$ such that $y(r) < \lfloor \frac{m}{4} \rfloor$ and $y(r) > \lfloor \frac{3m}{4} \rfloor$ are suppressed, whereas the rest $\lfloor \frac{m}{4} \rfloor < y(r) < \lfloor \frac{3m}{4} \rfloor$ are enhanced.
\color{black}

In order to derive a closed-form expression for $\bar{\gamma}(t)$ based on the above derivation, we assume that
\begin{equation}
\overline{\mathbf{C}}_{\mathbf{s}}^t \approx \frac{1}{m!}
\sum _{\pi} \mathbf{C}_{\mathbf{s}}^t.
\label{eq:c_t_approx}
\end{equation}
The accuracy of the approximation is known as the Jensen gap \cite{gao2017bounds,abramovich2016some}. 
We define the normalized Jensen gap as follows:
\begin{equation*}
{J}_{gap}(t)=\frac{ || \overline{\mathbf{C}}_{\mathbf{s}}^t - \frac{1}{m!}
\sum _{\pi} \mathbf{C}_{\mathbf{s}}^t  ||^2_F}{ || \overline{\mathbf{C}}_{\mathbf{s}}^t||^2_F}.
\end{equation*}
At the boundaries $t=0$ and $t=1$, we have that ${J}_{gap}(0)={J}_{gap}(1)=0$.
Empirically, we report that the values for $0<t<1$ are negligible. For example, in Fig. \ref{fig:NormalizedJensenGap}, we plot ${J}_{gap}(t)$ obtained for the realization of the permutation considered above, with parameters $n=11,m=31$.

\begin{figure}[t]\centering
	\includegraphics[scale=0.2]{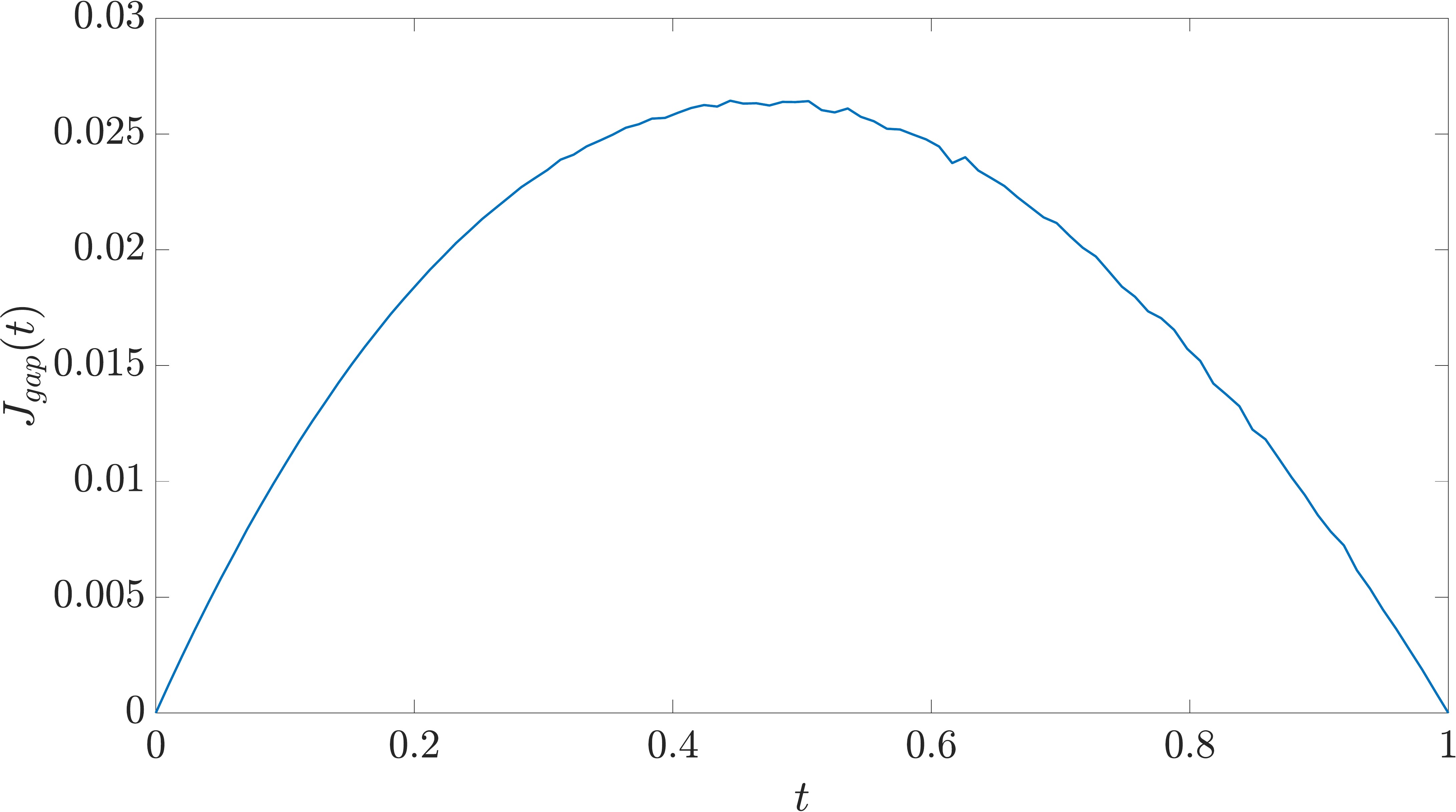}	
	\caption{Empirically computed normalized Jensen gap for $n=11,m=31$.}
	\label{fig:NormalizedJensenGap}
\end{figure}

Using the approximation \ref{eq:c_t_approx} in \Cref{prop:GammaAlternativeForm}, we obtain 
\begin{equation*}
	\bar{\gamma}(t) \approx  \mathbf{V} \mathbf{S}^{\frac{1}{2}} 	\overline{\mathbf{C}}_{\mathbf{s}}^t \mathbf{S}^{\frac{1}{2}} \mathbf{V}^T.
\end{equation*}
Since the matrix $\overline{\mathbf{C}}_{\mathbf{s}}$ is a diagonal matrix whose entries are given by \eqref{eq:ExpectationTildeC}, the eigenvalues of $\bar{\gamma}(t)$ are the diagonal elements of $\overline{\mathbf{F}}_t \triangleq \mathbf{S}^{\frac{1}{2}} \overline{\mathbf{C}}_{\mathbf{s}}^t \mathbf{S}^{\frac{1}{2}}$,
explicitly given by:
\begin{equation*}
	\overline{\mathbf{F}}_{t} (r,r) =  \bigg(\bar{c}(r)\bigg)^{t}\bigg(\mu_{xy}^{(x(r),y(r))} \bigg)^{1-t}.
\end{equation*}
By \eqref{eq:ExpectationTildeC}, we obtain that if $y(r)=1$, namely, the $r$th eigenvalue of $\bar{\gamma}(t)$ corresponds to a common eigenvector, it equals $\mu_{xy}^{(x(r),1)}$. By \ref{eq:ProductEvecsEvalsDiscrete}, we have that $\mu_{xy}^{(x(r),1)} = \frac{1}{2} \mu_x^{x(r)} + \frac{1}{2}$, namely, the $r$th eigenvalue of $\bar{\gamma}(t)$ equals an eigenvalue of the common kernel $\mathbf{K}_x$ up to an affine constant transformation, independent of $t$.

In contrast, for $y(r)>1$, the $r$th eigenvalue of $\bar{\gamma}(t)$ is given by a weighted geometric mean between the $r$th eigenvalue of the kernel $\mathbf{K}_{xy}$ and $\bar{c}(r)$.
As $t$ increases and approaches $1$, the $r$th eigenvalue of $\bar{\gamma}(t)$ approaches $\bar{c}(r)$, and thereby, alleviates the dependency on the non-common eigenvectors.  	

We conclude this appendix with another result involving the eigenvectors of $\mathbf{K}_{xy}$ and the eigenvectors of $\mathbf{K}_{xz}$.
\begin{proposition}
	Let $v_{xy}^{(k,l)}$ be an eigenvector of $\mathbf{K}_{xy}$, and let $v_{xz}^{(k',l')}$ be an eigenvector of $\mathbf{K}_{xz}$. Assume that the permutation $\pi$ is drawn uniformly at random from the set of all possible permutations of $m$ elements. 
	If $k \neq k'$, then $\langle v_{xy}^{(k,l)}, v_{xz}^{(k',l')} \rangle = 0$.
	If $k = k'$ and $l=l'=1$, then $\langle v_{xy}^{(k,l)}, v_{xz}^{(k',l')} \rangle = 1$. 
	If $k = k'$ and $l,l' \neq 1$, then
	\begin{equation*}
		P(|(\langle v_{xy}^{(k,l)}, v_{xz}^{(k',l')} \rangle)| \geq \alpha ) \leq \frac{1}{\alpha}\sqrt{\frac{2}{m}}
	\end{equation*}
\end{proposition}
In other words, the probability that the absolute value of the inner product $\langle v_{xy}^{(k,l)}, v_{xz}^{(k',l')} \rangle$ between non-trivial eigenvectors that correspond to the same common eigenvector ($k=k'$) is greater than $\alpha$ becomes smaller as the size of the non-common graph increases.
In addition, in the proof of this proposition, we observe that the assumption that $\pi$ is fixed and does not depend on $x$ actually makes the problem at hand more challenging. If $\pi$ is not fixed and randomly drawn for any $x$, then there exists a larger number of possible permutations, making the Chebyshev inequality used in the proof tighter. 

\subsection{Proofs}
\label{subsec:proofs}
In this subsection, we present the proofs of the propositions above.
We start with several lemmas that will be used in the subsequent proofs of the propositions.

\begin{lemma}
	\label{lemma:mean_perm}
	Let $v \in \mathbb{R}^m$ and let $\pi$ be a permutation drawn uniformly at random from all possible permutations of $M$ elements. Define the vector 
    $$
    	\bar{v}(y) = \frac{1}{m!} \sum_{\pi} v(\pi(y)).
    $$
for $y=1,\ldots,m$.	We have that $\bar{v}$ is a constant vector, whose entries are given by
	\begin{equation*}
	\bar{v}(y)=\frac{1}{m} \sum_{y'=1}^m v(y')
	\end{equation*}
	for any $1\le y \le m$, independently of $y$.
\end{lemma}
\begin{proof}
	Fix $y \in \{1, \ldots, m\}$. 
	For each $y' \in \{1,\ldots,m\}$ there exist $(m-1)!$ possible permutations for which $\pi(y)=y'$. Therefore, we have
	\begin{equation*}
	\bar{v}(y)= \frac{1}{m!}  \sum_{\pi} v\left(\pi\left(y\right)\right) = \frac{1}{m} \sum_{y'=1}^m v(y') 
	\end{equation*}	
\end{proof}

\begin{lemma}
	\label{lemma:double_perm}
	Let $v \in \mathbb{R}^m$ be a non-trivial eigenvector of a diffusion kernel $\mathbf{K}$ and let $\pi$ be a permutation drawn uniformly at random from all possible permutations of $m$ elements. 
	Define $$
    c(y,y')=\frac{1}{m!} \sum_{\pi}v\left(\pi\left(y\right)\right) v\left(\pi\left(y'\right)\right) 
    $$ for any $1\leq y,y' \leq m$.

	Then $c(y,y')$ is given by:
	\begin{equation*}
	c(y,y') = -\frac{1}{m(m-1)} + \frac{1}{m-1} \delta(y,y')
	\end{equation*}
	where $\delta(y,y')$ is the Kronecker delta function defined by
	\begin{equation*}
	\delta(y,y') = \begin{cases}
	1  & y = y' \\
	0  & y \neq y'
	\end{cases}
	\end{equation*}
\end{lemma}
\begin{proof}
	Fix $y \neq y'$. There exist $(m-2)!$ possible permutations for which $\pi(y)=s$ and $\pi(y')=s'$. Therefore, we have
	\begin{align*}
	c(y,y') &= \frac{1}{m!} \sum_{\pi}  v\left(\pi\left(y\right)\right) v\left(\pi\left(y'\right)\right)\\ &= \frac{1}{m!} (m-2)! \sum_{s=1}^m \sum_{s'=1, s\neq s}^m v(s)v(s')\\ 
	&=  \frac{1}{m(m-1)} \sum_{s=1}^m \sum_{s'=1, s'\neq s}^m v(s)v(s') \\
	&= - \frac{1}{m(m-1)} \sum_{s=1}^m  v^2(s) =   - \frac{1}{m(m-1)}
	\end{align*}
	Fix $y=y'$. There exist $(m-1)!$ possible permutations for which $\pi(y)=\pi(y')=s$. Therefore, we have
	\begin{align*}
	c(y,y') &= E_{\pi}\left[v^2\left(\pi\left(y\right)\right)\right] \\ 
	&= \frac{1}{m} \sum_{s=1}^m {v^2(s)} =\frac{1}{m}
	\end{align*}
	
	Merging these two expressions yields the desired statement:
	\begin{align*}
	c(y,y') &=  \frac{1}{m}\delta(y,y') -\frac{1}{m(m-1)} \left(1-\delta(y,y')\right) \\
	& = -\frac{1}{m(m-1)} + \frac{1}{m-1} \delta(y,y')
	\end{align*}
\end{proof}

\begin{lemma}
	\label{lemma:double_invperm}
	Let $v \in \mathbb{R}^m$ be a non-trivial eigenvector of a diffusion kernel $\mathbf{K}$, and let $\pi$ be a permutation drawn uniformly at random from all possible permutations of $M$ elements. 
	Define $$
    d(y,y')=\frac{1}{m!} \sum_{\pi}\left[v\left(\pi^{-1}\left(\pi(y)\oplus 1\right)\right) v\left(\pi^{-1}\left(\pi(y')\ominus1\right)\right) \right]
    $$ 
    for any $1\leq y,y' \leq m$, where $x\oplus y \triangleq mod(x+y,m)$ and $x\ominus y \triangleq mod(x-y,m)$. 
	Then $d(y,y')$ is given by:
	\begin{equation*}
	d(y,y') = \frac{1}{(m-1)(m-2)}\bigg(v^2(y)+v^2(y')+mv(y)v(y')+\delta(y,y') \left(1+mv^2(y)\right)\bigg) 
	\end{equation*}
	where $\delta(y,y')$ is the Kronecker delta function defined in \cref{lemma:double_perm}.
\end{lemma}
\begin{proof}
	Similarly as we did in the proof of \cref{lemma:double_perm}, also here we recast the expression: 
	\begin{align}
	\label{eq:perms_recast}
	d(y,y')=& \frac{1}{m!} \sum_{\pi}v\left(\pi^{-1}\left(\pi(y)\oplus 1\right)\right) v\left(\pi^{-1}\left(\pi(y')\ominus1\right)\right)  \nonumber \\
	=&  \frac{1}{m!} \sum_{s=1}^m  \sum_{s'=1}^m N(s,s') v(s)v(s')
	\end{align}
	where $N(s,s')$ denotes the number of permutations for which $s=\pi^{-1}\left(\pi(y)\oplus 1\right)$ and $s'=\pi^{-1}\left(\pi(y')\ominus 1\right)$.
	We notice that $\pi(s)=\pi(y) \oplus 1$, hence $N(y,s')=0$ for any $s'=1,\ldots,m$. Similarly: $\pi(s')=\pi(y')\ominus 1$, hence $N(s,y')=0$ for any $s=1,\ldots,m$.
	
	Fix $y=y'$. Since that $\pi(s) \neq \pi(s')$ we get that $s\neq s'$. As a consequence, the available number of permutations mapping $(y,y')$ to $(s,s')$ is: $m^2-m-2(m-1)= (m-1)(m-2)$, hence:
	\begin{equation*}
	N(s,s')=
	\begin{cases}
	0, & s=y, 1\leq s' \leq m \\
	0, & 1\leq s \leq m, s'=y'\\
	0, & s=s' \\
	C, & otherwise
	\end{cases}
	\end{equation*}
	where the constant $C$ is given by dividing the total number of permutations $m!$ by the total number of non-zeros entries $(m-1)(m-2)$: 
	\begin{equation*}
	C=	\frac{m!}{(m-1)(m-2)}= (m-3)!m
	\end{equation*}
	Plugging $N(s,s')$ into \cref{eq:perms_recast}, we get the following:
	\begin{align*}
	d(y,y')&= \frac{1}{m!} \sum_{s=1}^m  \sum_{s'=1}^m N(s,s') v(s)v(s') \\
	&= \frac{1}{m!} m(m-3)! \left(\sum_{s=1}^m  \sum_{s'=1}^m v(s)v(s') -2\sum_{s=1}^m v^2(s)- 2\sum_{s=1}^m v(s) +2v(y)  \right) \\
	&= \frac{1}{m}\left(2v^2(y))-1\right)
	\end{align*}
	Fix $y\neq y'$. First we show that there exists no permutation for which $s=y'$ and $s'\neq y$. We prove by contradiction. Assume that $s=y'$, recall that $s=\pi^{-1}\left(\pi(y)\oplus 1\right)$, hence: $\pi(y')= \pi(y)\oplus 1 $. In additions, recall that: $s'=\pi^{-1}\left(\pi(y')\ominus 1\right)$, hence $\pi(s')=\left(\pi(y')\ominus 1\right)$. Substituting $\pi(y')$ in the previous expression yields:
	\begin{equation*}
	\pi(s')=\pi(y)\oplus 1 \ominus 1=\pi(y)
	\end{equation*}
contradicting that $s'\neq y$.
	
Consider now a permutation that satisfy $s=y'$. Such a permutation also satisfies: $\pi(y)\oplus1=\pi(y')$, hence there exist $(m-2)m!$ such permutations.
	
	In this case, we get:
	\begin{equation*}
	N(s,s')=
	\begin{cases}
	0, & s=y, 1\leq s' \leq m \\
	0, & 1\leq s \leq m, s'=y'\\
	0, & s=y', 1\leq s' \leq m, s'\neq y \\
	0, & s'=y, 1\leq s \leq m, s\neq y' \\
	(m-2)m!, & s=y',s'=y \\
	C', & otherwise
	\end{cases}
	\end{equation*}
	where the constant $C'$ is given by dividing the remaining number of permutations $m!-(m-2)m!$ by the number of non-zeros entries $(m-2)^2$:
	\begin{equation*}
	C'=\frac{m!-(m-2)m!}{(m-2)^2}=(m-3)m!
	\end{equation*}
	Plugging $N(s,s')$ into \cref{eq:perms_recast} yields:
	\begin{align*}
	d(y,y')&= \frac{1}{m!} \sum_{s=1}^m  \sum_{s'=1}^m N(s,s') v(s)v(s') \\
	&= \frac{m(m-3)!}{m!}  \left(\sum_{s=1}^m  \sum_{s'=1}^m v(s)v(s') -2v(y)\sum_{s=1}^m v(s) -2v(y')\sum_{s=1}^m v(s)  +4v(y)v(y')\right) \\
	&+ \frac{m(m-2)!-m(m-3)!}{m!}v(y)v(y')\\
	&=\frac{m+2}{(m-1)(m-2)}v(y)v(y')
	\end{align*}
	
To summarize, for any $y$ and $y'$, we have:
	\begin{equation*}
	d(y,y') = \frac{1}{(m-1)(m-2)}\bigg(v^2(y)+v^2(y')+mv(y)v(y')+\delta(y,y') \left(1+mv^2(y)\right)\bigg) 
	\end{equation*}
\end{proof}

\begin{lemma}\label{lemma:zero_sum}
Let $v \in \mathbb{R}^m$ be a non-trivial eigenvector of a diffusion kernel $\mathbf{K}$. Then:
\begin{equation*}
\sum_{m=1}^M  v(m) = 0
\end{equation*}
\end{lemma}
\begin{proof}
	The diffusion kernel $\mathbf{K}$ is doubly-stochastic, hence - the  trivial (constant) vector, denoted by $v^1$ is an eigenvector of $\mathbf{K}$:
	\begin{equation*}
	 	v^1 (m)=\frac{1}{\sqrt{m}}
	\end{equation*}
	The elements sum of a vector can be recast as the inner-product with the trivial vector:
	\begin{equation*}
		\sum_{m=1}^M  v(m) = \sqrt{m} \sum_{m=1}^M  v(m)  v^{1}(m)  =  \sqrt{m} \langle v, v^1 \rangle
	\end{equation*}
	Due to orthogonality of the eigenvectors $ \langle v, v^1 \rangle =0$ for every $v \in \mathbb{R}^m$ which is a non-trivial eigenvector of a diffusion kernel $\mathbf{K}$.
\end{proof}

\begin{lemma}
	\label{lemma:inners}
	Let $v \in \mathbb{R}^m$ be a non-trivial eigenvector of a diffusion kernel $\mathbf{K}$, and let $\pi$ be a permutation drawn uniformly at random from all possible permutations of $m$ elements. Let $u=\mathbf{P}_\pi v$ be a permutation of the eigenvector of $\mathbf{K}$, where $\mathbf{P}_\pi$ is the permutation matrix of $\pi$. We have
	\begin{equation*}
	\Pr\left( |\langle u, v \rangle| \geq \alpha \right) \leq \frac{1}{\alpha}\sqrt{\frac{2}{m}}
	\end{equation*}
	for any $0<\alpha<1$.
	In particular, this inner product approaches zero as $m$ increases.
\end{lemma}
\begin{proof}
	Compute the mean of the inner product, given by
	\begin{equation*}
	\frac{1}{m!} \sum _\pi \langle u, v \rangle = \frac{1}{m!} \sum _\pi \sum_{y=1}^m v(y)v(\pi(y)) =  \sum_{y=1}^m v(y) \left[ \frac{1}{m!} \sum _\pi v(\pi(y)) \right] 
	\end{equation*}
	By applying \cref{lemma:mean_perm}, we have
	\begin{equation*}
	\frac{1}{m!} \sum _\pi \langle u, v \rangle = \sum_{y=1}^m v(y) \bar{v}(y) = 0
	\end{equation*}
	since $\bar{v}$ is a constant vector and sum of elements of $v$ is zero by \cref{lemma:zero_sum}.
	
	Similarly, compute variance of the inner product:
	\begin{align*}
	\frac{1}{m!} \sum _\pi  (\langle u, v \rangle)^2  &= \frac{1}{m!} \sum _\pi  \sum_{y,y'=1}^m v(y)v(y')v(\pi(y))v(\pi(y')) \\
	&= \sum_{y,y'=1}^m v(y)v(y') \left[ \frac{1}{m!} \sum _\pi v(\pi(y))v(\pi(y')) \right] \\
	&= \sum_{y,y'=1}^m v(y)v(y') c(y,y') 
	\end{align*}
	From \cref{lemma:double_perm}, we have
	\begin{align*}
	\frac{1}{m!} \sum _\pi (\langle u, v \rangle)^2  &= -\frac{1}{m(m-1)} \sum_{y,y'=1}^m v(y)v(y')  -\frac{1}{m-1}\sum_{y=1}^m v^2(y) \\
	\end{align*}
	By \cref{lemma:zero_sum}, $\sum_{y=1}^mv(y) = 0$, and since the norm of the eigenvector is $1$, we have
	\begin{align*}
	\frac{1}{m!} \sum _\pi (\langle u, v \rangle)^2 &=  \frac{1}{(m-1)} 
	\end{align*}
	Based on the mean and variance, applying Chebyshev's inequality yields that:
	\begin{equation*}
	P( |\langle u, v \rangle| \geq \alpha ) \leq \frac{1}{\alpha}\sqrt{\frac{1}{m-1}}
	\end{equation*}
\end{proof}

\begin{proof}[Proof of \cref{lemma:ProductEvecsEvalsDiscrete}]
\label{lemma:proof-ProductEvecsEvalsDiscrete}	
	The Laplacian matrix of a graph $G$ with affinity matrix $\mathbf{A}$ is defined by: $\mathbf{L}=\mathbf{D}-\mathbf{A}$, where $\mathbf{D}$ is the degree matrix. 
	We denote the Laplacian matrices of the product graphs by $\mathbf{L}_{xy}$ and $\mathbf{L}_{xz}$. Theorem 4.4.5 in \cite{horn2012matrix} links between the spectral decomposition of $\mathbf{L}_{xy}$ and the spectral decomposition of $\mathbf{L}_{x}$ and $\mathbf{L}_{y}$. Concretely, it entails that the eigenvectors of $\mathbf{L}_{xy}$ are given by the Cartesian product of the eigenvectors of $\mathbf{L}_{x}$ and $\mathbf{L}_{y}$, i.e.:
	\begin{gather}
	\label{eq:ProductLapEvecsDiscrete}
	\phi^{(k,l)}_{xy}(t)=\phi^k_x\left(x\left(t\right)\right)\phi^l_y\left(y\left(t\right)\right)
	\end{gather}
	where $\phi^{(k,l)}_{xy} \in \mathbb{R}^{mn}$ is an eigenvector of $\mathbf{L}_{xy}$ , $\phi^{k}_{x} \in \mathbb{R}^{n}$ is the $k$th eigenvector of $\mathbf{L}_{x}$ and $\phi^{l}_{y} \in \mathbb{R}^{m}$ is the $l$th eigenvector of $\mathbf{L}_{y}$.
	In addition, the eigenvalues of $\mathbf{L}_{xy}$ are given by the sum of the eigenvalues of $\mathbf{L}_{x}$ and $\mathbf{L}_{y}$, i.e.:
	\begin{align}
	\label{eq:ProductLapEvalsDiscrete}
	\xi^{(k,l)}_{xy}=\xi_x^k+\xi_y^l
	\end{align}
	where $\xi^{(k,l)}_{xy}$ is the eigenvalue corresponding to $\phi^{(k,l)}_{xy}$ , $\xi^{k}_{x}$ is the eigenvalue corresponding to $\phi^{k}_{x}$, and $\xi^{l}_{y}$ is the eigenvalue corresponding to $\phi^{l}_{y}$.
	We now link the spectral decomposition of a diffusion kernel $\mathbf{K}_{x}$ and the spectral decomposition of the Laplacian $\mathbf{L}_x$ (and similarly $\mathbf{K}_y$ and $\mathbf{L}_y$).

	For d-regular graphs: $\mathbf{L}_x=d \left( \mathbf{I}-\mathbf{K}_x \right)$ , specifically for the cycle graphs we consider $d=2$.
	Hence each eigenvector of $\mathbf{L}_x$ is also an eigenvectors of $\mathbf{K}_x$. More concretely, the eigenvalues of $\mathbf{K}_x$ are given by the following mapping:
	\begin{equation*}
	\mu_x=1-\frac{1}{d}\xi_x
	\end{equation*}
	where $\mu_x$ is the eigenvalue of $\mathbf{K}$ corresponding to the eigenvalue of $\mathbf{L}$, $\xi _x$. Specifically, in our case:
	\begin{align}
	\label{eq:single_graph_mappings}
	\mu_x^k=1-\frac{1}{d}\xi^k_x && v_x^k =\phi^k_x \\ 
	\mu_y^l=1-\frac{1}{d}\xi^l_y && v_y^l =\phi^l_y
	\end{align}
	where $\{(\xi^{k}_x,\phi^{k}_x)\}_{k=1}^n$ is the set of eigenvalues and eigenvectors of $\mathbf{L}_{x}$ and  $\{(\xi^{l}_y,\phi^{l}_y)\}_{l=1}^m$ is the set of eigenvalues and eigenvectors of $\mathbf{L}_{y}$.
	According to theorem 4.4.5 in \cite{horn2012matrix}, the spectral decomposition of $\mathbf{L}_{xy}$ are given by:
	\begin{align*}
	\phi^{(k,l)}_{xy}(t)=\phi^k_x\left(x\left(t\right)\right)\phi^l_y\left(y\left(t\right)\right)&& \xi^{(k,l)}_{xy}=\xi_x^k+\xi_y^l
	\end{align*}
	The Cartesian product of two d-regular graphs is a 2d-regular graph. Therefore $G_{xy}$ is  2d-regular graph and the spectral decomposition of $\mathbf{K}_{xy}$ are given by:
	\begin{align}
	\label{eq:product_graph_mappings}
	\mu_{xy}^{(k,l)}=1-\frac{1}{2d}\xi^{(k,l)}_{xy} && 
	v_{x,y}^{(k,l)} =\phi^{(k,l)}_{x,y}
	\end{align}
    
    Substituting \cref{eq:ProductLapEvecsDiscrete} and \cref{eq:ProductLapEvalsDiscrete} in \cref{eq:product_graph_mappings}  we get:
    \begin{align}
	    \label{eq:product_graph_mappings2}
	    \mu_{xy}^{(k,l)}=1-\frac{1}{2d}(\xi_x^k+\xi_y^l) && 
	    v_{x,y}^{(k,l)}=\phi^k_x\left(x\left(t\right)\right)\phi^l_y\left(y\left(t\right)\right)
    \end{align}
    Then, substituting \cref{eq:single_graph_mappings} in \cref{eq:product_graph_mappings2}, we get:
    \begin{align*}
    \mu_{xy}^{(k,l)}=\frac{1}{2}(\mu_x^k+\mu_y^l) && 
    v_{x,y}^{(k,l)} =v^k_x\left(x\left(t\right)\right)v^l_y\left(y\left(t\right)\right)
    \end{align*}
\end{proof}

\begin{lemma}
	\label{lemma:proof-ExpressionForC}
	The matrix $\mathbf{C}$ is block-diagonal and can be recast as:
	\begin{equation*}
	\mathbf{C} = \textbf{V}\textbf{R}\textbf{V}^T
	\end{equation*}
	where $\textbf{R}$ is a block diagonal matrix of $n$ blocks, whose $k$th block is given by:
	\begin{equation*}
	\mathbf{r}^k(s,s')= \frac{1}{4}\cos\left(\frac{2\pi\left(k-1\right)}{n}\right) \delta(s,s')+\frac{1}{4}\sqrt{\frac{m}{2}}\mathbf{r}_{\pi}(s,s')
	\end{equation*}
	the matrix $\mathbf{r}_{\pi}(s,s')$ depends on the permutation $\pi$ and is given by:
	\begin{equation*}
	\mathbf{r}_{\pi}(s,s')  =\frac{1}{\sqrt{2m}}\bigg(\beta^m(\pi(s)) \delta(\pi(s) \oplus 1,\pi(s')) +\beta^2(\pi(s))\delta(\pi(s) \ominus 1,\pi(s'))\bigg) 
	\end{equation*}
	and $\beta^x(s)$ is defined as follows:
	\begin{gather*}
		\beta^x(s)\triangleq \left\{\begin{array}{lr}
		\frac{1}{\sqrt{2}}, & s=1 \\
		\sqrt{2}, & s=x \\
		1, & otherwise \\
		\end{array} \right.
	\end{gather*}
\end{lemma}
\begin{proof}
	We denote the blocks of $\mathbf{B}$ by $\mathbf{b}$, such that: $\mathbf{B} = \mathbf{I}_n \otimes \mathbf{b}$  where $\otimes$ is the matrix Kronecker product and $\mathbf{I}_n$ is an $n\times n$ identity matrix. The matrix $\mathbf{b}$ is explicitly given by:
	
	\begin{gather}
	\mathbf{b}_{i,j}= \left\{\begin{array}{lr}
	1, & i=j=1 \\
	0, & \left(i=1 \land j>1\right) \lor \left(j=1 \land i>1 \right) \\
	\sum_{s=1}^{m} v_y^i(s) v_y^j(\pi(s)), & 1<i,j\leq m \\
	\end{array} \right.
	\label{eq:bij}
	\end{gather}
	
	Since $\mathbf{B}$ is block diagonal, $\mathbf{C}$ is block diagonal as well. We denote the $k$th block of $\mathbf{C}$ by $\mathbf{c}^k \in \mathbb{R}^{m \times m}$. The entries of $\mathbf{c}^k$ are given by:
	\begin{gather}
	\mathbf{c}^k_{i,j} = \sum_{u=1}^m \mu_{xy}^{(k,u)} \mathbf{b}_{i,u} \mathbf{b}_{j,u}
	\label{eq:cij}
	\end{gather}
	where $\mu_{xy}^{(k,u)}$ are give by \cref{lemma:ProductEvecsEvalsDiscrete}:
	\begin{equation*}
	\mu_{xy}^{(k,u)} =  \frac{1}{2} \left(\mu_x^k+\mu_y^u \right)
	\end{equation*}	
	For any $1< i \leq m$, we have
	\begin{align*}
	\mathbf{c}^k_{1,i}=0 \\ 
	\mathbf{c}^k_{i,1}=0
	\end{align*}
	For $i,j \neq 1$, substituting \cref{eq:bij} into \cref{eq:cij} yields:
	\begin{align*}
	\mathbf{c}^k_{i,j} &= \sum_{u=1}^m \mu_{xy}^{(k,u)} \mathbf{b}_{i,u} \mathbf{b}_{j,u} \\ 
	&=  \sum_{u=1}^m \mu_{xy}^{(k,u)} \sum_{s=1}^{m} v_y^i(s) v_y^u(\pi(s)) \sum_{s'=1}^{m} v_y^j(s') v_y^u(\pi(s')) \\ 
	&= \sum_{s=1}^{m} v_y^i(s) \sum_{s'=1}^{m} v_y^j(s')\sum_{u=1}^m \mu_{xy}^{(k,u)} v_y^u(\pi(s)) v_y^u(\pi(s')) \\ 
	& =  \sum_{s=1}^{m} v_y^i(s) \sum_{s'=1}^{m} v_y^j(s') \mathbf{r}^k(s,s')
	\end{align*}
	where:
	\begin{equation*}
	\mathbf{r}^k(s,s')\triangleq\sum_{u=1}^m s_{xy}^{(k,u)} v_y^u(\pi(s)) v_y^u(\pi(s'))
	\end{equation*}
	
	Note that defining $\mathbf{R}$ to be a block diagonal matrix whose $k$th block is $\mathbf{r}^k$ yields:
	\begin{equation*}
	\mathbf{C}= \textbf{V}\textbf{R}\textbf{V}^T
	\end{equation*}
	
	Next we derive the explicit expression of $\mathbf{r}^k$. The eigenvalue $\mu_{xy}^{(k,u)}$ is given by:
	\begin{equation*}
	\mu_{xy}^{(k,u)} =\frac{1}{2} + \frac{1}{4}\cos\left(\frac{2\pi\left(k-1\right)}{n}\right) + \frac{1}{4}\sqrt{\frac{m}{2}}v^2_y(u)
	\end{equation*} 
	Then, we substitute this expression into the definition of $\mathbf{r}^k(s,s')$ and get:
	\begin{align}
	\label{eq:rk}
	\mathbf{r}^k(s,s')&= \left(\frac{1}{2} + \frac{1}{4}\cos\left(\frac{2\pi\left(k-1\right)}{n}\right)\right)\sum_{u=1}^m  v_y^u(\pi(s)) v_y^u(\pi(s'))+\frac{1}{4}\sqrt{\frac{m}{2}}\sum_{u=1}^m v^2_y(u) v_y^u(\pi(s)) v_y^u(\pi(s'))\\
	&= \left(\frac{1}{2} + \frac{1}{4}\cos\left(\frac{2\pi\left(k-1\right)}{n}\right)\right) \mathbf{r}_{d}(s,s')+\frac{1}{4}\sqrt{\frac{m}{2}}\mathbf{r}_{\pi}(s,s')
	\end{align}
	where
	\begin{align}
	\mathbf{r}_{d}(s,s') &\triangleq \sum_{u=1}^m  v_y^u(\pi(s)) v_y^u(\pi(s')) 
	\end{align}
	and 
	\begin{align}
	\label{eq:rpi}
	\mathbf{r}_{\pi}(s,s')  &\triangleq \sum_{u=1}^m  v^2_y(u) v_y^u(\pi(s)) v_y^u(\pi(s')) 
	\end{align}
	
	We note that the matrix $\mathbf{r}_{d}(s,s')$ actually equals the eye matrix:
	\begin{align}
		\label{eq:rd}
		\mathbf{r}_{d}(s,s') &=\sum_{u=1}^m  v_y^u(\pi(s)) v_y^u(\pi(s'))=  \sum_{u=1}^m  v_y^{\pi(s)}(u) v_y^{\pi(s')}(u) = \delta(\pi(s),\pi(s)')=\delta(s,s')
	\end{align}
	
	In order to derive a simplified expression for $\mathbf{r}_{\pi}(s,s')$ we first derive a simplified expression for the elements consisting the sum in $\mathbf{r}_{\pi}(s,s')$:
	\begin{align*}
		v^2_y(u) v_y^u(\pi(s)) v_y^u(\pi(s'))&= \sqrt{\frac{2}{m}} \cos\left(\frac{2\pi\left(u-1\right)}{m}\right) \frac{\alpha(\pi(s))}{\sqrt{m}}\cos\left(\frac{2\pi\left(\pi(s)-1\right)}{m}\right)\\
		&=  \frac{\sqrt{2}\alpha(\pi(s))}{2m}\left( \cos\left(\frac{2\pi\left(\pi(s)\right)}{m} \left(u-1\right) \right) +\cos\left(\frac{2\pi\left(\pi(s)-2\right)}{m} \left(u-1\right) \right) \right) \\ 
		&=  \frac{\alpha(\pi(s))}{\sqrt{2m}} \left(\frac{1}{\alpha(\pi(s) \oplus 1)}v^{\pi(s) \oplus 1}_y(u)+\frac{1}{\alpha(\pi(s)\ominus 1)}v^{\pi(s)\ominus 1}_y(u) \right)
	\end{align*}
	where $\alpha(r)$ is defined in \cref{eq:alpha}, and:
	\begin{gather*}
		x\oplus y \triangleq mod(x+y,m)\\ 
		x\ominus y \triangleq mod(x-y,m)
	\end{gather*}
	Then, we substitute these expression in \cref{eq:rpi}, and compute a simplified expression for $\mathbf{r}_{\pi}$:
	\begin{align*}
		\mathbf{r}_{\pi}(s,s') &=\sum_{u=1}^m  v^2_y(u) v_y^{\pi(s)}(u)  v_y^{\pi(s')}(u)\\
		=&\frac{\alpha(\pi(s))}{\sqrt{2m}} \sum_{u=1}^m v_y^{\pi(s')}(u) \left(\frac{1}{\alpha(\pi(s)\oplus 1)}v^{\pi(s)\oplus 1}_y(u)+\frac{1}{\alpha(\pi(s)\ominus 1)}v^{\pi(s) \ominus 1}_y(u) \right)   \\ 
		=&\frac{\alpha(\pi(s))}{\sqrt{2m}} \left(\frac{\delta(\pi(s) \oplus 1,\pi(s'))}{\alpha(\pi(s)\oplus 1)} + \frac{\delta(\pi(s)\ominus 1,\pi(s'))}{\alpha(\pi(s) \ominus 1)}  \right) \\
		=& \frac{1}{\sqrt{2m}}\bigg(\beta^m(\pi(s)) \delta(\pi(s) \oplus 1,\pi(s')) +\beta^2(\pi(s))\delta(\pi(s) \ominus 1,\pi(s'))\bigg) 
	\end{align*}
	where
	\begin{gather*}
	\beta^x(s)= \left\{\begin{array}{lr}
	\frac{1}{\sqrt{2}}, & s=1 \\
	\sqrt{2}, & s=x \\
	1, & otherwise \\
	\end{array} \right.
	\end{gather*}
	
	Substituting this explicit form of $\mathbf{r}_{\pi}(s,s')$ into $\mathbf{r}^k(s,s')$ gives:
	\begin{align*}
	\mathbf{r}^k(s,s')=&
	\left(\frac{1}{2} + \frac{1}{4}\cos\left(\frac{2\pi\left(k-1\right)}{n}\right)\right) \delta(s,s')+ \\
	&\frac{1}{8}\bigg(\beta^m(\pi(s)) \delta(\pi(s)\oplus 1,\pi(s')) +\beta^2(\pi(s))\delta(\pi(s)\ominus 1,\pi(s'))\bigg) 
	\end{align*}
\end{proof}

\begin{proof}[Proof for \cref{prop:ExpressionForB}]
	\label{prop:proof-ExpressionForB}
	Recalling that the matrix $\mathbf{B}$ is defined as the pairwise inner products between the eigenvectors of $\mathbf{K}_{xy}$ and the eigenvectors of $\mathbf{K}_{xz}$, we can write each entry as follows:
	\begin{equation*}
	\mathbf{B}_{i,j} = \langle v^{(x(i),y(i))}_{xy} , v^{(x(j),y(j))}_{xz} \rangle 
	\end{equation*}
	By \cref{eq:ProductEvecsEvalsDiscrete}, for $1\le r \le mn$, we have:
	\begin{align*}
	v^{(x(i),y(i))}_{xy}(r)	&=	v^{x(i)}_x(x(r))v^{y(i)}_y(y(r)) \\
	v^{(x(j),z(j))}_{xz}(r)	&=	v^{x(j)}_x(x(r))v^{z(j)}_z(z(r)) = v^{x(j)}_x(x(r))v^{y(j)}_y(\pi(y(r)))
	\end{align*}
	Then, we compute:
	\begin{align*}
	\langle v^{(x(i),y(i))}_{xy},v^{(x(j),y(j))}_{xz} \rangle &= \sum_{r=1}^{mn} v^{(x(i),y(i))}_{xy}(r)v^{(x(j),y(j))}_{xz}(r)=\\
	&= \sum_{r=1}^{mn} v^{x(i)}_x(x(r))v^{y(i)}_y(y(r)) v^{x(j)}_x(x(r))v^{y(j)}_y(\pi(y(r))) =\\
	&= \sum_{y=1}^{m}\sum_{x=1}^{n}  v^{x(i)}_x(x)v^{y(i)}_y(y) v^{x(j)}_x(x)v^{x(j)}_y(\pi(y))=\\
	&= \sum_{x=1}^n v^{x(i)}_x(x)v^{x(j)}_x(x) \sum_{y=1}^{m} v^{y(i)}_y(y) v^{y(j)}_y(\pi(y)) =\\
	&= \delta(x(i),x(j)) \sum_{y=1}^{m} v^{y(i)}_y(y) v^{y(j)}_y(\pi(y)) 
	\end{align*}
	where the last transition is due to the orthogonality of the eigenvectors of $\mathbf{K}_x$.
	Recall that $x(r)=1+\left\lfloor \frac{r}{m} \right\rfloor$, hence, $\mathbf{B}$ can be recast as:
	\begin{equation*}
	\mathbf{B}= \mathbf{1}_n \otimes \mathbf{b}
	\end{equation*}
	where $\otimes$ is the matrix Kronecker product, $\mathbf{I}_n$ is the $n\times n$ identity matrix, and $\mathbf{b} \in \mathbf{R}^{m \times m}$ is given by:
	\begin{equation*}
	 \mathbf{b}_{l,l'}= \sum_{y=1}^{m} v^{l}_y(y) v^{l'}_y(\pi(y))
	\end{equation*}
	Since  $v^{1}$ is the trivial (constant) eigenvector with norm $1$, we get that $\mathbf{b}_{1,1'}=1$.
	If $l=1$ and $l' \neq 1$, we have $\mathbf{b}_{1,l'}=0$ due to \cref{lemma:mean_perm}. If $l'=1$ and $l \neq 1$, we have $\mathbf{b}_{l,1}=0$ due to the orthogonality of the eigenvectors. If $l,l' \neq 1$, by \cref{lemma:inners}, we have:
	\begin{equation*}
	P(| \mathbf{b}_{l,l'}|\geq \alpha) = P( \sum_{y=1}^{m} v^{l}_y(y) v^{l'}_y(\pi(y)) \geq \alpha )   \leq \frac{1}{\alpha}\sqrt{\frac{1}{m-1}}
	\end{equation*}
\end{proof}

\begin{proof}[Proof for \cref{prop:GammaAlternativeForm}]
	\label{prop:proofGammaAlternativeForm}
	The matrix $\gamma(t)$ on the geodesic path is given by:
		\begin{equation*}
		\gamma(t) = \mathbf{K}^{\frac{1}{2}} \mathbf{M}^{t} \mathbf{K}^{\frac{1}{2}} 
		\end{equation*}
	where  $\mathbf{M}\triangleq \mathbf{V} \mathbf{C}_{\mathbf{S}} \mathbf{V}^T $. 
	Since $\mathbf{V}\mathbf{V}^T = \mathbf{I}$, the matrix power $\mathbf{M}^{t}$ is given by $\mathbf{M}^{t} = \mathbf{V} \mathbf{C}^t_{\mathbf{S}} \mathbf{V}^T$. Using the EVD of $\mathbf{K}$ yields that $\gamma(t)$ can be written as:
		\begin{equation*}
		\gamma(t) = \mathbf{K}^{\frac{1}{2}} \mathbf{M}^{t} \mathbf{K}^{\frac{1}{2}} = \mathbf{V}\mathbf{S}^{\frac{1}{2}}\mathbf{V}^T \mathbf{M}^{t} \mathbf{V}\mathbf{S}^{\frac{1}{2}}\mathbf{V}^T =  \mathbf{V}\mathbf{S}^{\frac{1}{2}}\mathbf{V}^T \mathbf{V} \mathbf{C}^t_{\mathbf{S}}  \mathbf{V}^T \mathbf{V}\mathbf{S}^{\frac{1}{2}}\mathbf{V}^T = \mathbf{V}\mathbf{S}^{\frac{1}{2}} \mathbf{C}^t_{\mathbf{S}} \mathbf{S}^{\frac{1}{2}}\mathbf{V}^T
		\end{equation*} 	
\end{proof}

\begin{proof}[Proof for \cref{prop:ExpectationC}]
\label{prop:proof-ExpectationC}
	Since $\mathbf{B}$ is block diagonal also $\mathbf{C}$ is block diagonal. Similarly as we did in the proof for \cref{prop:ExpressionForVarC}, we denote the $k$th block of $\mathbf{C}$ by $\mathbf{c}^k \in \mathbb{R}^{m \times m}$. The entries of $\mathbf{c}^k$ are given by:
	\begin{gather*}
	\mathbf{c}^k_{i,j} = \sum_{u=1}^m \mu_{xy}^{(k,u)} \mathbf{b}_{i,u} \mathbf{b}_{j,u}
	\end{gather*}
	where $\mu_{xy}^{(k,u)}$ can be computed according to \cref{lemma:ProductEvecsEvalsDiscrete} and $\mathbf{b}_{i,j}$ is given by \cref{eq:bij}.
	We compute:
	\begin{align}
		\mathbf{c}^k_{1,1} =& \sum_{u=1}^m \mu_{xy}^{(k,u)}  \mathbf{b}_{1,u} \mathbf{b}_{1,u} 
		= \mu_{xy}^{(k,1)} = \frac{3}{4} + \frac{1}{4}\cos\left(\frac{2\pi\left(k-1\right)}{n}\right)
	\end{align}
	
	In order to compute $\mathbf{c}^k_{i,j}$ for $i,j>1$ we utilize the results from \cref{lemma:proof-ExpressionForC}. According to \cref{lemma:proof-ExpressionForC}, $\mathbf{c}^k_{i,j}$ for $i,j>1$ is given by: 
	\begin{equation*}
		\mathbf{c}^k_{i,j}= \sum_{s=1}^{m} v_y^i(s) \sum_{s'=1}^{m} v_y^j(s')\mathbf{r}^k(s,s')
	\end{equation*}	
	Where $\mathbf{r}^k$ is given by \cref{eq:rk}. We substitute \cref{eq:rk} and \cref{eq:rd} in the expression for $\mathbf{c}^k_{i,j}$ and get the following expression:
	\begin{align*}
		\mathbf{c}^k_{i,j}& =\left(\frac{1}{2} + \frac{1}{4}\cos\left(\frac{2\pi\left(k-1\right)}{n}\right)\right) \delta(i,j) + \sum_{s=1}^{m} v_y^i(s) \sum_{s'=1}^{m} v_y^j(s')\mathbf{r}_{\pi}(s,s') 
	\end{align*}
	
	We substitute the expression for $\mathbf{r}_{\pi}$ from \cref{eq:rpi} and compute the mean matrix $\frac{1}{m!}\sum_{\pi}\mathbf{r}_{\pi}$:
	\begin{align*}	
		\frac{1}{m!}\sum_{\pi}\mathbf{r}_{\pi} &= \frac{1}{m!}\sum_{\pi}  \sum_{u=1}^m  v^2_y(u) v_y^u(\pi(s)) v_y^u(\pi(s')) \\
		&= \sum_{u=1}^m  v^2_y(u)  \frac{1}{m!}\sum_{\pi} v_y^u(\pi(s)) v_y^u(\pi(s')) \\
		&= \sum_{u=1}^m  v^2_y(u) \left( -\frac{1}{m(m-1)} + \frac{1}{m-1} \delta(s,s') \right) = 0
	\end{align*}
	Then, we utilize this result in order to calculate the mean expression for $\mathbf{c}^k$:
	\begin{align*}	
		\frac{1}{m!}\sum_{\pi}\mathbf{c}^k_{i,j} &=  \left(\frac{1}{2} + \frac{1}{4}\cos\left(\frac{2\pi\left(k-1\right)}{n}\right)\right) \delta(i,j) + \frac{1}{m!}\sum_{\pi} \sum_{s=1}^{m} v_y^i(s) \sum_{s'=1}^{m} v_y^j(s')\mathbf{r}_{\pi}(s,s') \\
		&=  \left(\frac{1}{2} + \frac{1}{4}\cos\left(\frac{2\pi\left(k-1\right)}{n}\right)\right) \delta(i,j)+\frac{1}{m!} \sum_{s=1}^{m} v_y^i(s) \sum_{s'=1}^{m} v_y^j(s') \sum_{\pi} \mathbf{r}_{\pi}(s,s') \\
		&= \left(\frac{1}{2} + \frac{1}{4}\cos\left(\frac{2\pi\left(k-1\right)}{n}\right)\right) \delta(i,j)
	\end{align*} 
	
	Therefore, the mean matrix $\frac{1}{m!}\sum_{\pi}\mathbf{c}^k$ is a diagonal matrix, whose diagonal elements are given by:
	\begin{equation*}
	\bar{c}^k(s) = \frac{1}{2} + \frac{1}{4} \cos\left(\frac{2\pi\left(k-1\right)}{n}\right) +\frac{1}{4} \delta(s,1) 
	\end{equation*}
	where $1\leq s \leq m$.
	As consequence, the mean matrix $\frac{1}{m!}\sum_{\pi}\mathbf{C}$ is a diagonal matrix, whose diagonal elements are given by:
	\begin{equation*}
	\bar{c}(r) = \frac{1}{2} + \frac{1}{4} \cos\left(\frac{2\pi\left(x(r)-1\right)}{n}\right) +\frac{1}{4} \delta\left(y(r),1\right) 
	\end{equation*}
	where $1\leq r \leq mn$.
\end{proof}	

\begin{proof}[Proof for \cref{prop:ExpressionForVarC}]
	\label{prop:proof-ExpressionForVarC}
	By \cref{lemma:proof-ExpressionForC}, for $i,j>1$, we compute:
	\begin{align*}
	\mathbf{c}^k_{i,j}&= \sum_{s=1}^{m} v_y^i(s) \sum_{s'=1}^{m} v_y^j(s')\mathbf{r}^k(s,s')\\
	&=\left(\frac{1}{2} + \frac{1}{4}\cos\left(\frac{2\pi\left(k-1\right)}{n}\right)\right)\delta(i,j)\\
	&+\frac{1}{8} \left(  \sum_{s=1}^{m}\beta^m(\pi(s)) v_y^i(s) v_y^j(\pi^{-1}(\pi(s)\oplus1))+ 
	\sum_{s=1}^{m}\beta^2(\pi(s)) v_y^i(s) v_y^j(\pi^{-1}(\pi(s)\ominus1)) \right) \\	
	&= 	\left(\frac{1}{2} + \frac{1}{4}\cos\left(\frac{2\pi\left(k-1\right)}{n}\right)\right)\delta(i,j)\\
	&+\frac{1}{8} \left(  \sum_{s=1}^{m}\beta^m(\pi(s)) v_y^i(s) v_y^j(\pi^{e+}(s))+ 
	\sum_{s=1}^{m} \beta^2(\pi(s)) v_y^i(s) v_y^j(\pi^{e-}(s)) \right)  \\
	\end{align*}
	where:
	\begin{gather*}
	\pi^{e+}(s)=\pi^{-1}(\pi(s)\oplus 1) \\
	\pi^{e-}(s)=\pi^{-1}(\pi(s)\ominus 1)
	\end{gather*} 
	
	Since $\beta^x(s)$ can be recast as$\beta^x(s)=1+(\sqrt{2}-1)\delta(s,x)+(\frac{1}{\sqrt{2}}-1)\delta(s,1)$, we have:
	\begin{align*}
	\mathbf{c}^k_{i,j}&=	\left(\frac{1}{2} + \frac{1}{4}\cos\left(\frac{2\pi\left(k-1\right)}{n}\right)\right)\delta(i,j)\\
	&+\frac{1}{8} \left(  \sum_{s=1}^{m} v_y^i(s) v_y^j(\pi^{e+}(s))+ 
	\sum_{s=1}^{m}  v_y^i(s) v_y^j(\pi^{e-}(s)) \right)\\
	& +\frac{(\sqrt{2}-1)}{8} \left(   v_y^i(m) v_y^j(\pi^{e+}(m))+ v_y^i(m) v_y^j(\pi^{e-}(m)) \right)\\ 
	& +\frac{(\frac{1}{\sqrt{2}}-1)}{8} \left(   v_y^i(1) v_y^j(\pi^{e+}(1))+ v_y^i(1) v_y^j(\pi^{e-}(1)) \right)  \\
	&= 	\left(\frac{1}{2} + \frac{1}{4}\cos\left(\frac{2\pi\left(k-1\right)}{n}\right)\right)\delta(i,j)\\
	&+\frac{1}{8} \left(  \sum_{s=1}^{m} v_y^i(s) v_y^j(\pi^{e+}(s))+ 
	\sum_{s=1}^{m}  v_y^i(s) v_y^j(\pi^{e-}(s)) \right)  \\
	\end{align*}
	
The off-diagonal entries $i \neq j$ are given by:
	\begin{align*}
	\mathbf{c}^k_{i,j}&=\frac{1}{8} \left(  \sum_{s=1}^{m} v_y^i(s) v_y^j(\pi^{e+}(s))+ 
	\sum_{s=1}^{m}  v_y^i(s) v_y^j(\pi^{e-}(s)) \right).
	\end{align*}

By \cref{lemma:mean_perm}, the mean of the off-diagonal entries is
	\begin{align*}
	\frac{1}{m!} \sum _\pi \mathbf{c}^k_{i,j} =0
	\end{align*}
The variance of the off-diagonal entries is:
	\begin{align*}
	\frac{1}{m!} \sum _\pi \left(\mathbf{c}^k_{i,j}\right)^2 &= \frac{1}{64}\sum_{s=1}^m\sum_{s'=1}^m v_y^i(s) v_y^i(s') \left[ \frac{1}{m!} \sum _\pi  v_y^j(\pi^{e+}(s)) v_y^j(\pi^{e+}(s')) \right] \\ 
	&+ \frac{1}{64}\sum_{s=1}^m\sum_{s'=1}^m v_y^i(s) v_y^i(s') \left[ \frac{1}{m!} \sum _\pi  v_y^j(\pi^{e-}(s)) v_y^j(\pi^{e-}(s')) \right] \\ 
	&+ \frac{1}{64}\sum_{s=1}^m\sum_{s'=1}^m v_y^i(s) v_y^i(s') \left[ \frac{1}{m!} \sum _\pi  v_y^j(\pi^{e+}(s)) v_y^j(\pi^{e-}(s')) \right] \\ 
	&= \frac{2}{64}\sum_{s=1}^m\sum_{s'=1}^m v_y^i(s) v_y^i(s') \left[ \frac{1}{m!} \sum _\pi  v_y^j(\pi^{e+}(s)) v_y^j(\pi^{e+}(s')) \right] \\ 
	&+ \frac{1}{64}\sum_{s=1}^m\sum_{s'=1}^m v_y^i(s) v_y^i(s') \left[ \frac{1}{m!} \sum _\pi  v_y^j(\pi^{e+}(s)) v_y^j(\pi^{e-}(s')) \right]	
	\end{align*}
	By \cref{lemma:double_perm}, the first term becomes:
	\begin{align*}
	\sum_{s=1}^m\sum_{s'=1}^m v_y^i(s) v_y^i(s') \left[ \frac{1}{m!} \sum _\pi  v_y^j(\pi^{e+}(s))v_y^j(\pi^{e-}(s')) \right] = \frac{1}{m-1} \sum_{s=1}^m \left(v_y^i(s)\right)^2 = \frac{1}{m-1}
	\end{align*}
	
	In the second term, we substitute the expression from \cref{lemma:double_invperm} and get:

	\begin{equation*}
	\begin{split}
	&\sum_{s=1}^m\sum_{s'=1}^m v_y^i(s) v_y^i(s') \left[ \frac{1}{m!} \sum _\pi  v_y^j(\pi^{e+}(s)) v_y^j(\pi^{e-}(s')) \right]	   \\
	=&\frac{1}{(m-1)(m-2)} \bigg[ \sum_{s=1}^m\sum_{s'=1}^m v_y^i(s) v_y^i(s') \left( \left(v_y^j(s)\right)^2 +\left(v_y^j(s')\right)^2+ m v_y^j(s) v_y^j(s') \right) \\
	&	+ \sum_{s=1}^m \left(v_y^i(s)\right)^2 \left(1+m\left(v_y^j(s)\right)^2\right) \bigg]  \\
	= &\frac{1}{(m-1)(m-2)} + 	\frac{m}{(m-1)(m-2)}\sum_{s=1}^m \left(v_y^i(s)\right)^2 \left(v_y^j(s)\right)^2\\
	\leq& \frac{2}{(m-1)(m-2)} \\
	\end{split}
	\end{equation*}
Plugging the simplified two terms back into the variance yields:
	\begin{align*}
	\frac{1}{m!} \sum _\pi  \left(\mathbf{c}^k_{i,j}\right)^2 \leq \frac{2}{64(m-1)} + \frac{2}{64(m-2)(m-1)}  = \frac{1}{32(m-1)}
	\end{align*}
	
Finally, based on the mean and variance, applying Chebyshev's inequality yields that the probability to draw a permutation in which the absolute value of the off-diagonal entries is greater or equal to $\alpha$ is bounded from above by:
	\begin{equation*}
	P( |\mathbf{c}^k_{i,j}| \geq \alpha ) \leq \frac{1}{\alpha}\sqrt{\frac{1}{32(m-1)}}
	\end{equation*}
\end{proof}